\documentclass[a4paper,12pt,twoside,authoryear,print,index]{PhDThesisPSnPDF}

\ifsetCustomMargin
  \RequirePackage[left=37mm,right=30mm,top=35mm,bottom=30mm]{geometry}
  \setFancyHdr %
\fi

\RequirePackage[labelsep=space,tableposition=bottom]{caption}

\usepackage{subcaption}

\usepackage{booktabs} %
\usepackage{multirow}

\usepackage{siunitx} %

\usepackage[todonotes,table,dvipsnames]{xcolor}
\usepackage[colorinlistoftodos]{todonotes}\reversemarginpar

\usepackage{enumitem}

\usepackage[outline]{contour}
\newcommand{\fancy}[1]{\color{white}\contour{black}{#1}}

\newcommand{\robVar}[1]{\bm{\mathrm{#1}}}

\newcommand{\robnu}{\bm \nu}
\newcommand{\robtau}{\bm \tau}

\newcommand{\ageVar}[1]{\mathbbm{#1}}

\newcommand{\agenu}{\fancy{$\nu$}}
\newcommand{\agetau}{\fancy{$\tau$}}

\newcommand{\comVar}[1]{\bm{#1}}

\usepackage{varwidth}
\newcommand\PapersFontSize{\fontsize{8}{10}\selectfont}

\usepackage{lmodern}
\newlength\longest

\renewcommand*\contentsname{Contents}

\usepackage{fancyhdr}
\fancypagestyle{plain}{\fancyhf{}} %

\usepackage{scrextend}

\deffootnote[1em]{1em}{1em}{%
	\textsuperscript{\makebox[1em][l]{\thefootnotemark}}}

\newcommand{\chapreface}[1]
{
	\begin{center}
		\vspace{-1cm}
		\begin{tcolorbox}[rounded corners, colback=white!30,
			colframe=white!20!black!0,width=1.05\textwidth, center]
				\vspace{0.5cm}
				\begin{tcolorbox}[rounded corners, colback=white!30,
					colframe=white!20!black!0,width=1\textwidth, center]
					#1
				\end{tcolorbox}
		\end{tcolorbox}
		
		\vspace{0.1cm}
		\titlerule
		\vspace{0.1cm}
		
	\end{center}
	
}

\usepackage{titlesec, blindtext, color}
\definecolor{gray85}{gray}{0.85}

\newcommand\Yuge{\fontsize{80}{20}\selectfont}

\titleformat{\part}[display]{\Huge\mdseries\scshape\filright}{\partname~\thepart:}{20pt}{\Huge}

\titleformat{\chapter}[display]{\raggedleft\Yuge\mdseries\slshape}{\textcolor{gray85}\thechapter}{0pt}{\Huge}[\vspace*{.1\baselineskip}\titlerule]
\titlespacing{\chapter}{0pt}{0pt}{15pt}

\titleformat{\section}[hang]{\Large\mdseries}{\thesection}{10pt}{\Large}

\titleformat{\subsection}[hang]{\large\mdseries}{\thesubsection}{10pt}{\large}

\titleformat{\subsubsection}[hang]{\normalsize\bfseries}{\thesubsubsection}{10pt}{\normalsize}

\usepackage{booktabs} %
\usepackage{multirow}
\usepackage{longtable}
\definecolor{Gray}{gray}{0.9} %

\DeclareMathOperator*{\skewOp}{S}

\DeclareMathOperator*{\argmin}{\arg\!\min}

\usepackage{tcolorbox}

\usepackage{float} %

\usepackage{fontawesome}

\def \hrsup_size{0.48}
\def \framesu_size{0.48}

\usepackage{amsmath,amssymb}
\usepackage{mathtools}
\usepackage{commath}
\usepackage{fixmath}
\usepackage{bm,bbm}

\DeclareMathAlphabet{\mathcal}{OMS}{cmsy}{m}{n}

\usepackage{amsthm}  
\newtheorem*{proposition}{Proposition}

\definecolor{CustomGray}{gray}{0.85}

\usepackage{array}
\newcolumntype{P}[1]{>{\centering\arraybackslash}p{#1}}

\renewcommand{\cite}{\citep}
\usepackage{multicol}

\usepackage[mathscr]{euscript}
\newtheorem{thm}{Theorem}
\newtheorem{lemma}[thm]{Lemma}

\usepackage{IEEEtrantools}

\usepackage{indentfirst}

\usepackage{rotating}

\makeatletter
\newsavebox\myboxA
\newsavebox\myboxB
\newlength\mylenA
\newcommand*\xoverline[2][0.75]{%
   \sbox{\myboxA}{$\m@th#2$}%
   \setbox\myboxB\null%
   \ht\myboxB=\ht\myboxA%
   \dp\myboxB=\dp\myboxA%
   \wd\myboxB=#1\wd\myboxA%
   \sbox\myboxB{$\m@th\overline{\copy\myboxB}$}%
   \setlength\mylenA{\the\wd\myboxA}%
   \addtolength\mylenA{-\the\wd\myboxB}%
   \ifdim\wd\myboxB<\wd\myboxA%
      \rlap{\hskip 0.5\mylenA\usebox\myboxB}{\usebox\myboxA}%
   \else
       \hskip -0.5\mylenA\rlap{\usebox\myboxA}{\hskip 0.5\mylenA\usebox\myboxB}
   \fi}
\makeatother

\newcommand{\mcrot}[4]{\multicolumn{#1}{#2}{\rlap{\rotatebox{#3}{#4}~}}}

\usepackage{lscape} 

\begin{document}

\frontmatter

\thispagestyle{empty}
\begin{figure}[h!]
 \centering
 \includegraphics[scale=0.20]{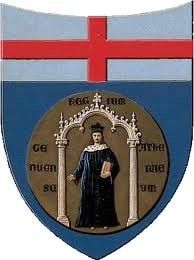} %
	\begin{center} 
		\Large
		{\textsc{University of Genova}}\\
		  \vspace{0.15em}
		  \large
	         \textsc{PhD Program in Bioengineering and Robotics}\\

	\end{center}
\end{figure}

\vspace{-0.65cm}

\begin{center} 

		\Large
		\textbf{Enabling Human-Robot Collaboration via \\ 
			Holistic Human Perception and \\
			 Partner-Aware Control} \\
\end{center}

 	\begin{center} 
		\textbf{Venkata Sai Yeshasvi Tirupachuri}\\
		\vspace{1em}
		\normalsize
		Thesis submitted for the degree of \textit{Doctor of Philosophy} ($31^\circ$ cycle) \\
	    \vspace{1em}	
		\normalsize
		April \\ 
		2020\\ 
	\end{center}

    \begin{center}
    	\textbf{Supervisors} \\
    	Daniele Pucci \\ 
    	Giulio Sandini
    \end{center}
    
    \begin{center}
     \vspace{-0.35cm}	
    	\begin{tabular}{ l@{\hspace{0.25cm}}l }
    		\textbf{Jury \& External Reviewers*} & \\
    		Rachid Alami* & Senior Scientist, LAAS-CNRS, France \\ 
    		Antonio Franchi & Assoc. Professor, University of Twente, Netherlands \\
    		Robert Griffin & Research Scientist, IHMC, USA \\
    		Dana Kuli\'{c}* & Professor, Monash University, Australia \\
    		Ludovic Righetti & Assoc. Professor, New York University, USA \\
    		Oliver Stasse & Senior Researcher, LAAS-CNRS, France
    	\end{tabular}
        
    \end{center}

\vspace{1em}

\begin{figure}[h!]
	\centering
	\begin{subfigure}{.4\textwidth}
		\centering
		\includegraphics[trim=0.25cm 0.25cm 0.25cm 0.25cm, clip, scale=0.315]{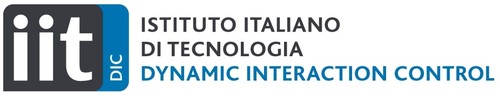}
		\begin{center} 
			\small{Dynamic Interaction Control Lab \\ Italian Institute of Technology}
		\end{center}
	\end{subfigure}%
	\begin{subfigure}{.5\textwidth}
		\centering
		\includegraphics[scale=0.425]{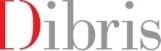}
		\begin{center} 
			\small{Department of Informatics, Bioengineering, Robotics and Systems Engineering\\}
		\end{center}
	\end{subfigure}
\end{figure}

\begin{dedication} 
	
	\begin{center}
		\thispagestyle{empty}
		To my beloved grandmother and parents \dots
		\vspace*{\fill}
	\end{center}

\end{dedication}

\begin{acknowledgements}

This research work would not have materialized without the trust and constant support from Dr. Daniele Pucci. Thank you for believing in me and giving me an incredible opportunity to be a part of our research team. To my dearest collaborator, Gabriele Nava, who has been a trusted confidant over the years, thank you for constantly inspiring and supporting me in carrying out this research. My comrades in the trenches of AnDy: Claudia, Silvio, Diego, Lorenzo, Kourosh, Ines, Francisco, thank you for always being there in exploring new ideas and materializing them to fruition. The last four years have been an enriching and rewarding experience which was made possible by many of the current and former colleagues and friends. Thanks to everyone, especially Stefano, Giulio, Nuno, Prashanth, Luca Fiorio, Enrico, Gianluca, Yue, Ugo, Vadim, Prof. Gabriel Baud-Bovy, Lisa, Mariacarmela, Chiara, Marco, Prof. Fulvio Mastrogiovanni, Marie, Naveen, Joan, Ali, Dimitris, Alberto, Michele, Stefano De Simone, Matteo, Julien, Marta, Lucia, Luca Recla. I would also like to thank Prof. Rachid Alami and Prof. Dana Kuli\'{c} for providing their invaluable reviews of this doctoral thesis. 

Finally, this journey would not have been easier and possible without the incredible love and support from my dearest Lizeth, my parents, and my dear friends Stefano Bracchi, Dara and Tagore.

Above all, I am forever grateful for all your kindness and compassion. Thank you!

\end{acknowledgements}

\begin{dedication} 
	
	\begin{figure}[hbt!]
		\centering
		\begin{subfigure}{0.5\textwidth}
			\centering
			\includegraphics[scale=0.37]{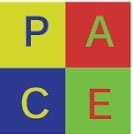}
			\includegraphics[scale=0.3]{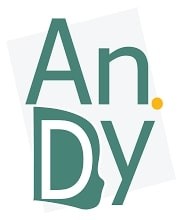}
			\includegraphics[scale=0.3]{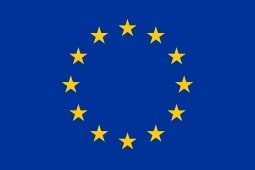}
		\end{subfigure}%
	\end{figure}
	
	\begin{center}
		\thispagestyle{empty}
		This work is supported by PACE project, Marie Skłodowska-Curie grant agreement No. 642961 and An.Dy project which has received funding from the European Union's Horizon 2020 Research and Innovation Programme under grant agreement No. 731540. 
		\vspace*{\fill}
	\end{center}

\end{dedication}

\begin{abstract}
	
    As robotic technology advances, the barriers to the coexistence of humans and robots are slowly coming down. Application domains like elderly care, collaborative manufacturing, collaborative manipulation, etc., are considered the need of the hour, and progress in robotics holds the potential to address many societal challenges. The future socio-technical systems constitute of blended workforce with a symbiotic relationship between human and robot partners working collaboratively. This thesis attempts to address some of the research challenges in enabling human-robot collaboration. In particular, the challenge of a holistic perception of a human partner to continuously communicate his intentions and needs in real-time to a robot partner is crucial for the successful realization of a collaborative task. Towards that end, we present a holistic human perception framework for real-time monitoring of whole-body human motion and dynamics. On the other hand, the challenge of leveraging assistance from a human partner will lead to improved human-robot collaboration. In this direction, we attempt at methodically defining what constitutes assistance from a human partner and propose partner-aware robot control strategies to endow robots with the capacity to meaningfully engage in a collaborative task.
	
\end{abstract}

\newpage
\thispagestyle{plain} %
\mbox{}
\thispagestyle{empty}
\null\vfill

\settowidth\longest{\large Start where you are. Use what you have. Do what you can.;}

\hspace{6cm}
\parbox{\longest}{%
	\raggedright{\large%
		Start where you are. \\
		Use what you have. \\
	    Do what you can. \par\bigskip
	}   
	\hspace{2cm}\large{--- Arthur R. Ashe Jr.}\par%
}

\vfill\vfill

\newpage
\thispagestyle{plain} %
\mbox{}

\subsection*{Nomenclature}
\vspace{-0.5cm}
\rule{\textwidth}{0.3pt}
\begin{longtable}{l p{10cm}}
    \label{nomenclature_table}
$p$                          & Scalar (non-bold small letter)\\
\rowcolor{Gray}
$\comVar{p}$                 & Vector (bold small letter)\\
$\dot {\comVar{p}}$          & First-order time derivative\\
\rowcolor{Gray}
$\ddot {\comVar{p}}$         & Second-order time derivative\\
$\comVar{A}$                 & Matrix (bold capital letter)\\
\rowcolor{Gray}
$(\cdot)^{\top}$                & Transpose matrix\\
$|\cdot|$                    & Matrix determinant\\
\rowcolor{Gray}
$\bm \skewOp(\cdot)$         & Skew-symmetric matrix\\
$\mathcal{A}$                & Coordinate reference frame (calligraphic letter)\\
\rowcolor{Gray}
$O_\mathcal{A}$              & Origin of a coordinate frame $\mathcal{A}$\\
${}^\mathcal{A}\comVar{p}$       & Vector expressed in frame $\mathcal{A}$\\
\rowcolor{Gray}
$\mathcal{I}$                & Inertial frame of reference\\
$\sum$                       & Summation operator\\
\rowcolor{Gray}
$(\cdot)^\dagger$            & Pseudoinverse of a matrix\\
$(\cdot)^{-1}$               & Inverse matrix\\
\rowcolor{Gray}
$rank(\cdot)$                & Rank of a matrix\\
$\arg \max(\cdot)$           & Maximizing argument\\
\rowcolor{Gray}
$\arg \min(\cdot)$           & Minimizing argument\\
$diag(\cdot)$                & Diagonal matrix\\
\rowcolor{Gray}
$\|\cdot\|$                  & Norm\\
$\comVar{J}$                 & Jacobian matrix\\
\rowcolor{Gray}
$\comVar{M}$                 & Mass matrix\\
$\comVar{C}$                 & Coriolis matrix\\
\end{longtable}

\clearpage

\subsection*{Abbreviations and Acronyms}
\rule{\textwidth}{0.3pt}
\begin{longtable}{l p{10.6cm}}
    \label{abbrevationsAndAcronims_table}
HRI          & Human-Robot Interaction\\
\rowcolor{Gray}
HRC          & Human-Robot Collaboration\\
DoF          & Degrees of Freedom\\
\rowcolor{Gray}
$1$D         & One-dimensional\\
$2$D         & Two-dimensional\\
\rowcolor{Gray}
$3$D         & Three-dimensional\\
$6$D         & Six-dimensional\\
\rowcolor{Gray}
F/T          & Force-Torque\\
IMU          & Inertial Measurement Unit\\
\rowcolor{Gray}
IK           & Inverse Kinematics\\
ID           & Inverse Dynamics\\
\rowcolor{Gray}
RNEA         & Recursive Newton-Euler Algorithm\\
URDF         & Unified Robot Description Format\\
\rowcolor{Gray}
API          & Application Programming Interface\\
YARP 		 & Yet Another Robot Platform\\
\rowcolor{Gray}
ROS 		 & Robot Operating System\\
RGB          & Red-Green-Blue\\
\rowcolor{Gray}
CoM          & Centre of Mass\\
PDF          & Probability Density Function\\
\rowcolor{Gray}
MAP          & Maximum-A-Posteriori\\
MSE          & Mean Square Error\\
\rowcolor{Gray}
RMSE         & Root Mean Square Error\\
\end{longtable}

\tableofcontents

\listoffigures

\listoftables

\mainmatter

\chapter*{Prologue}
\addcontentsline{toc}{chapter}{Prologue}

Ever since the creation of the first industrial robot Unimate from Unimation, the field of robotics made tremendous progress in developing a variety of robots with new venues for their use ranging from disaster response to elderly care~\cite{gasparetto2019unimate, zamalloa2017dissecting}. Industrial Robots (\textsc{ir}) have been massively adopted in manufacturing industries. Historically, industrial robots have been employed in caged environments and are programmed to do heavy duty tasks repetitively with minimal supervisory control from humans. Though industrial robots typically offer higher operation speeds and payload capacities, they generally lack the flexibility to change the production lines quickly. This puts them at a disadvantage under emerging economic pressures of competitive markets that deem increased productivity and flexibility of production lines to meet the demand of rapid product changes.

Over the last two decades, the field of Human-Robot Interaction (\textsc{hri}) emerged as an established independent research area that is dedicated to understanding, designing and evaluating robotic systems intended to be used by or with humans~\cite{goodrich2008human}. In particular, the research area of Human-Robot Collaboration (\textsc{hrc}) focuses on the aspects of bringing together humans and robots as collaborative partners with a shared sense of purpose in achieving a common goal. Recent technological advances in hardware design, sensory and actuation capabilities paved way to a new class of robots called collaborative robots or cobots that are intended to be used safely along-side human partners either to assist them or augment their physical capabilities~\cite{ogenyi2019physical}. 

The human partner is an integral part of human-robot collaboration scenarios~\cite{pfeiffer2016robots}. In general, the human partner establishes a common goal for collaboration and the robot partner needs to be communicated about the human intentions and needs during the entire duration of the collaborative task. Several communication interfaces are investigated in literature for various collaborative applications~\cite{ajoudani2018progress}. Interest in multi-modal interfaces is clearly on the rise owing to the new possibilities they enable for active collaboration. However, achieving a holistic human perception not limiting to just the kinematics but also the dynamics of the human partner helps in achieving sophisticated human as an actor models in collaborative scenarios. Furthermore, it enables to quantitatively measure several human factors for complex tasks across different sectors.

Current generation of robotic platforms ranging from cobots to humanoids facilitate various control modalities for improved interactions with humans partners. The critical aspect of safety during physical human-robot interaction has been extensively investigated in literature and several reactive robot behaviors have been successfully tested~\cite{haddadin2017robot, Haddadin2015}. In contrast to unexpected collisions or disturbances from a human partner, collaboration scenarios involve sustained intentional physical interactions, that are often helpful for the robot partner. So, instead of exhibiting simple reactive behaviors, a robot partner can exploit the help from a human partner to perform the collaborative task efficiently. However, a clear definition of what constitutes a helpful interaction from a human partner is still an open question that needs to be methodically addressed.

In view of the above observations, we believe there is a strong potential benefit in establishing a holistic human perception framework towards partner-aware reactive robot behaviors on collaborative tasks. The research presented in this thesis is a minuscule contribution towards realizing the future socio-technical systems where human and robot partners constitute the blended workforce across different sectors. This research work has been carried out during my PhD within Dynamic Interaction Control (\textsc{dic}) - a robotics research group at the Italian Institute of Technology (\textsc{iit}). The doctoral program has been carried in accordance with the requirements of University of Genoa, Italy in order to obtain a PhD title. My research work is funded by the European projects: Perception and Action in Complex Environments\footnote{Marie Skłodowska-Curie grant agreement No. 642961} (\textsc{pace}), and Advancing
Anticipatory  Behaviors in Dyadic Human-Robot Collaboration\footnote{Horizon 2020 Research and Innovation Programme under grant agreement No. 731540} (An.Dy). The present document is organized into three parts:

\vspace{0.25cm}
\noindent \textbf{Part I: Background and Thesis Context}

\begin{itemize}
	\item \textbf{Chapter~\ref{cha:human-and-robot-partners} Human and Robot Partners} presents a brief introduction of the current robotic technology trends in relation to humans. A brief literature review related to the challenges in human perception and robot control for collaboration scenarios is presented with a the motivation behind the research work in this doctoral thesis within the scope of An.Dy project.
	\item \textbf{Chapter~\ref{cha:RigidMultiBodySystem} Rigid Multi-Body System} introduces the basic notation followed throughout this thesis, followed by rigid body kinematics and dynamics representation. Furthermore, the modeling and dynamics of a rigid multi-body system is presented in this chapter.
	\item \textbf{Chapter~\ref{cha:human-modeling} Recall on Human Modeling} described the importance of digital human modeling in different fields and presents modeling details human as a system of articulated rigid bodies that is relevant to the context of robotics research.
	\item  \textbf{Chapter~\ref{cha:technologies} Enabling Technologies} describes the technologies used for human motion measurements and environmental interaction force and moment measurements. Furthermore, details of the humanoid robotic platform used for experiments in this research are presented in this chapter.
\end{itemize}

\vspace{0.25cm}
\noindent \textbf{Part II: Holistic Human Perception}

\begin{itemize}
	\item \textbf{Chapter~\ref{cha:human-kinematics-estimation} Human Kinematics Estimation} introduces the problem of inverse kinematics in the context of real-time human motion tracking. Different approaches to solving an inverse kinematics problem for a highly articulated human model are presented.
	\item \textbf{Chapter~\ref{cha:human-dynamics-estimation} Human Dynamics Estimation} recalls the problem of inverse dynamics formulated as a stochastic estimation problem. A novel sensorless external force estimation approach is presented along with experimental validation of human dynamic variables estimation for a floating base human model.
	\item \textbf{Chapter~\ref{cha:software-architecture} Software Architecture} explains our modular and extensible software infrastructure with wearable technology towards realizing a sophisticated holistic human perception framework. 
\end{itemize}

\vspace{0.25cm}
\noindent \textbf{Part III: Reactive Robot Control}

\begin{itemize}
	\item \textbf{Chapter~\ref{cha:partner-aware-control} Partner-Aware Control} recalls classical feedback linearization control technique and describes the typical interaction characterization for human-robot collaboration scenarios. Further, this chapter presents a coupled dynamics formalism of an external agent and a robotic agent engaged in physical collaboration. More importantly, interaction characterization in terms of external agent joint torques is presented and a partner-aware control law through Lyapunov stability analysis is proposed. Experimental validation with a whole-body humanoid robot controller for performing the task of sit-to-stand transition is carried using two complex humanoid robots.
	\item \textbf{Chapter~\ref{cha:trajectory-advancement} Trajectory Advancement} explores the idea of an intuitive reactive robot behavior through the trajectory advancement that endows a robot with the capacity to accomplish a task quicker by leveraging assistance from the human partner. The details of trajectory advancement proposition through the Lyapunov analysis is presented with experimental validation using a simple trajectory tracking controller for an end-effector of the robot and a more complex whole-body controller for performing the task of sit-to-stand transition.
	\item \textbf{Chapter~\ref{cha:whole-body-retargeting} Whole-Body Retargeting \& Teleoperation} presents our research towards telexistence with the main focus on whole-body human motion retargeting to robotic systems and teleoperation. Thanks to the flexibility of our software architecture presented in Chapter~\ref{cha:software-architecture} that enables us to perform experiments with minimal changes. Our approach is validated through whole-body retargeting experiments with multiple human subjects and multiple robots. Furthermore, teleoperation experiments with two state-of-the-art whole-body controllers are presented in this chapter.
\end{itemize}

\vspace{0.25cm}
\noindent \textbf{Research Contributions}

\begin{itemize}
	\item Research investigation into the coupled dynamics formalism and partner-aware control during a  physical collaboration task between an external partner and a robot partner presented in Chapter~\ref{cha:partner-aware-control} is published as a part of the conference proceedings and it secured a "Best Student Paper" award.
	
	\begin{tcolorbox}[rounded corners, colback=white!30,
		colframe=white!20!black!30,width=0.9\textwidth, center]
		\PapersFontSize  \href{https://link.springer.com/chapter/10.1007/978-3-030-29513-4_78}{\underline{Tirupachuri, Y.}, Nava, G., Latella, C., Ferigo, D., Rapetti, L., Tagliapietra, L., Nori, F., \& Pucci, D. (2019, September). Towards partner-aware humanoid robot control under physical interactions. In Proceedings of SAI Intelligent Systems Conference (pp. 1073-1092). Springer, Cham.}
		
		\vspace{0.15cm}
		\PapersFontSize{\textbf{Video:} \href{https://youtu.be/auHfyuTvkuY}{https://youtu.be/auHfyuTvkuY}}
	\end{tcolorbox}

	\item The concept of trajectory advancement and the experimental validation presented in Chapter~\ref{cha:trajectory-advancement} is published as a part of the conference proceedings and the extended experimental validation is accepted as a workshop dissemination.
	
	\begin{tcolorbox}[rounded corners, colback=white!30,
		colframe=white!20!black!30,width=0.9\textwidth, center]
		\PapersFontSize \href{https://arxiv.org/pdf/1907.13445.pdf}{\underline{Tirupachuri, Y.}, Nava, G., Rapetti, L., Latella, C., \& Pucci, D. (2019). Trajectory Advancement during Human-Robot Collaboration, in press IEEE Ro-MAN 2019. arXiv preprint arXiv:1907.13445.}
		
		\vspace{0.15cm}
		\PapersFontSize{\textbf{Video:} \href{https://youtu.be/mGHU5QTk9BI}{https://youtu.be/mGHU5QTk9BI}}
		
		\vspace{0.2cm}
		
		\PapersFontSize \href{https://arxiv.org/pdf/1910.06786.pdf}{\underline{Tirupachuri, Y.}, Nava, G., Rapetti, L., Latella, C., \& Pucci, D. (2019). Trajectory Advancement for Robot Stand-up with Human Assistance, Italian Robotics and Intelligent Machines Conference (I-RIM 2019) 2019. arXiv preprint arXiv:1907.13445.}
		
		\vspace{0.15cm}
		\PapersFontSize{\textbf{Video:} \href{https://youtu.be/OZ-cgzTm_pM}{https://youtu.be/OZ-cgzTm\_pM}}
		
	\end{tcolorbox}

	\item The research on whole-body human motion retargeting to a robot platform and the experimental validation with state-of-the-art whole-body controllers for humanoid robots detailed in Chapter~\ref{cha:whole-body-retargeting} is presented as a part of the conference proceedings.
	
	\begin{tcolorbox}[rounded corners, colback=white!30,
		colframe=white!20!black!30,width=0.9\textwidth, center]
		\PapersFontSize \href{https://arxiv.org/abs/1909.10080}{\underline{Tirupachuri, Y.}, Darvish, K., Romualdi, G., Rapetti, L., Ferigo, D., Chavez, F. J. A., \& Pucci, D. (2019). Whole-Body Geometric Retargeting for Humanoid Robots, in press IEEE Humanoids 2019. arXiv preprint arXiv:1909.10080.}
		
		\vspace{0.15cm}
		\PapersFontSize{\textbf{Video:} \href{https://youtu.be/hUj83DMWxCo}{https://youtu.be/hUj83DMWxCo}}
				
	\end{tcolorbox}

	\item Investigations into extending our stochastic human dynamics estimation framework for a floating base human model laid out in Chapter~\ref{cha:human-dynamics-estimation} are presented as a journal contribution.
	
	\begin{tcolorbox}[rounded corners, colback=white!30,
		colframe=white!20!black!30,width=0.9\textwidth, center]
		\PapersFontSize \href{https://www.mdpi.com/1424-8220/19/12/2794}{Latella, C., Traversaro, S., Ferigo, D., \underline{Tirupachuri, Y.}, Rapetti, L., Andrade Chavez, F. J., Nori, F., \& Pucci, D. (2019). Simultaneous Floating-Base Estimation of Human Kinematics and Joint Torques. Sensors, 19(12), 2794.}
		
		\vspace{0.15cm}
		\PapersFontSize{\textbf{Video:} \href{https://youtu.be/kLF4GS7tDxY}{https://youtu.be/kLF4GS7tDxY}, \href{https://youtu.be/VuvVmXXiYEA}{https://youtu.be/VuvVmXXiYEA}}
		
	\end{tcolorbox}

	\item Our work on real-time human motion tracking through various inverse kinematics algorithms and their benchmarking presented in Chapter~\ref{cha:human-kinematics-estimation} is under review for a conference.
	
	\begin{tcolorbox}[rounded corners, colback=white!30,
		colframe=white!20!black!30,width=0.9\textwidth, center]
		\PapersFontSize \href{https://arxiv.org/abs/1909.07669}{Rapetti, L., \underline{Tirupachuri, Y.}, Darvish, K., Latella, C., \& Pucci, D. (2020). Model-Based Real-Time Motion Tracking using Dynamical Inverse Kinematics. arXiv preprint arXiv:1909.07669.}
		
		\vspace{0.15cm}
		\PapersFontSize{\textbf{Video:} \href{https://youtu.be/_-oe8F8UP7g}{https://youtu.be/\_-oe8F8UP7g}}
		
	\end{tcolorbox}

    \item The research effort on sensorless external force estimation through an updated formulation of our stochastic human dynamics estimation presented in Chapter~\ref{cha:human-dynamics-estimation} will soon be submitted to an upcoming conference. \PapersFontSize{\textbf{Video:} \href{https://youtu.be/4hLnP_6-rCs}{https://youtu.be/4hLnP\_6-rCs}}
	
\end{itemize}

\part{Background and Thesis Context}

\chapter{Human and Robot Partners}  %
\label{cha:human-and-robot-partners}

\section{Introduction}
\label{sec:introduction}

The current decade (2010-2020) will be etched in history as the decade of many key successes in the field of robotics. The success stories are marked by some of the most practical and interesting research work witnessed through government-backed competitions like DARPA Robotics Challenge (\textsc{drc})~\cite{Krotkov:2017:DRC:3074644.3074647, doi:10.1177/0278364916688254}. Also, the total amount of investments in robotics has been steadily increasing throughout the decade as highlighted in Fig.~\ref{fig:robotics-funding}.

\begin{figure}[H]
	\centering
	\begin{subfigure}{0.925\textwidth}
		\centering
		\includegraphics[width=\textwidth]{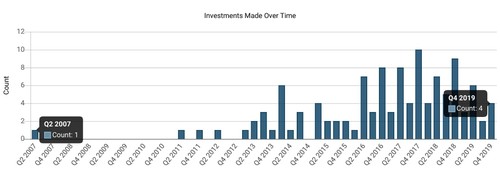}
		\label{fig:robotics-funding-count}
	\end{subfigure}
	\begin{subfigure}{0.925\textwidth}
		\centering
		\vspace{-.75cm} \includegraphics[width=\textwidth]{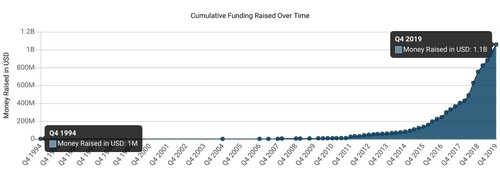}
		\label{fig:robotics-funding-cumulative}
	\end{subfigure}
    \vspace{-0.55cm}
	\caption{Investment trends in robotics - Source: \href{https://www.crunchbase.com/hub/robotics-companies-seed-funding}{Crunchbase}}
	\label{fig:robotics-funding}
\end{figure}

A myriad of sophisticated products like Jibo, Mayfield Robotics Kuri, Anki Cozmo, Rethink Robotics Baxter, Franka Emika Panda, Universal Robotics U5, DJI Phantom from some of the most notable robotic startups are introduced in the market. However, as the decade draws in, streams of failures dawn on may robotic startups leading to eventual bankruptcy. On careful analysis, one can conclude that most of the startups that failed attempted to provide social robots or promised a higher emphasis on the social component in their products oriented to consumers. While on the other hand, successful companies focused more on delivering products that are function-oriented e.g., Franka Emika and Universal Robotics products are for collaborative manufacturing, DJI drones are for creative fields and entertainment.

An interesting thought exercise is to compare and contrast the smart devices market and consumer robotics market. At the hardware level, the key enabling technologies behind the smart devices and the current robots are similar e.g., interactive touch screens, high capacity graphics processing units, high bandwidth, and low latency communication devices, etc. However, the momentum sustained by the smart devices and the rapid cultural adaptation of them in the previous decade is markedly more prominent than that of robots in the current decade. Although built as general-purpose computing units, the promise, and the strength of smart devices lie in providing connectivity in the digital space, access to information and a myriad of tools to enhance creativity and productivity. In contrast, robots are expected to share the same physical space alongside humans and are expected to have many anthropomorphic traits ranging from physical structure to emotional intelligence that can assist or augment the quality of human life. So, to have more cultural acceptance and adoption, robots need to have relational skills~\cite{damiano2015towards, parviainen2017vulnerable} that ensure a social component in their existence. On the other hand, the need for standardized measurement tools and methodological principles to quantify the effect of a robot for interacting with humans reliably is backed in the research community~\cite{belhassein2019towards, bartneck2009measurement}.

Businesses are capitalistic entities by nature and they always try to lower the costs and increase the profit margins. The global trend of income inequality in western society is a direct result of decades of operating costs optimization by relocating the business operations from labor expensive regions to inexpensive regions. The latest advances in robotics and automation are perceived as a great threat that completely replaces human labor across different sectors \cite{frey2017future}. Current sentiments assume a general rule of thumb that any routine work is on the brink of replacement by robots and automation. An example of routine work is the assembly task on a shop floor. However, a recent qualitative study on assembly work highlights the intricacies of the "\emph{perceived}" routine work of assembly and the importance of experiential knowledge \cite{pfeiffer2016robots}.

\begin{tcolorbox}[sharp corners, colback=white!30,
	colframe=blue!20!black!30, 
	title=Excerpt from \cite{pfeiffer2016robots}]
	\small \emph{Automation always has aimed and always will aim to substitute for human labor, and with new robotic concepts ahead we will see many attempts to do so in the world of assembly. However, the limits of automation are to be seen after decades of robotic use in industry: coping with imponderables, the flexibility and ability to adapt, the unlimited variability of behavior and tacit skills and body intelligence — all these dimensions of genuinely human work cannot be replaced by robots.}
\end{tcolorbox}

Humans will remain an integral part of the organic workforce, even as the complexity and composition of the work evolves to meet the demands of rapid market changes. Recently collaborative robots gained a lot of popularity and are increasingly adopted across several industries \cite{Braganca2019} as the technology is becoming more flexible, versatile and safe \cite{saenz2018survey}. However, current robotic technology is limited in terms of active collaboration through physical interactions with human partners. This limitation is directly linked to the deficiencies of human perception involving both whole-body motion and articular stress which leads to a blind spot for robots working alongside humans. Furthermore, a collaborative human partner aims at helping the robot partner through physical interactions. Under such circumstances instead of exhibiting simple reactive behaviors, a robot partner can actively use the help to accomplish a shared goal.

\section{Brief Literature Review}
\label{sec:brief-literature-review}

Robots existed as separate entities till now but the horizons of a symbiotic human-robot partnership is impending. In particular, application domains like elderly care, collaborative manufacturing, collaborative manipulation, etc., are considered the need of the hour. Across all these domains, it is crucial for robots to physically interact with humans to either assist them or to augment their capabilities. Such human in the loop physical human-robot interaction (\textsc{phri}) or physical human-robot collaboration (\textsc{phrc}) scenarios demand careful consideration of both the human and the robot systems while designing controllers to facilitate robust interaction strategies for successful task completion. More importantly, a generalized agent-robot interaction formalism is needed to study the physical interaction adaptation and exploitation, where the \emph{agent} can be a human or another robot. 

The three main components of any physical agent-robot collaboration (\textsc{parc}) scenario are 1) a \textit{robotic agent}, 2) an \textit{external agent} be it a human or another robot, and 3) the \textit{environment} surrounding them. Over time, physical interactions are present between any of the two components. An intuitive conceptual representation of the interactions occurring during collaborative scenarios is presented in \citep{losey2018review}. 

\subsection{Challenges in Human Perception}
\label{sec:challenges-in-human-perception}

The knowledge of human intent is a key element for the successful realization of \textsc{phri} tasks. The process of human intent detection is broadly divided into \textit{intent information}, \textit{intent measurement} and \textit{intent interpretation} \cite{losey2018review}. The choice of a communication channel is directly related to intent measurement and affects the robot's ability to understand human intent. Accordingly, a myriad of technologies has been used as interfaces for different applications of \textsc{phri}. In the context of rehabilitation robotics, electroencephalography (\textsc{eeg}) \cite{mattar2018biomimetic,sarac2013brain,mcmullen2014demonstration} and electromyography (\textsc{emg}) \cite{radmand2014characterization,au2008powered,song2008assistive,zhou2018multi} proved to be invaluable. Force myography (\textsc{fmg}) \cite{cho2016force,yap2016design,rasouli2016towards} is a relatively new technology that has been successfully used in rehabilitation. \textsc{emg} has also been successfully used by \cite{peternel2018robot} to realize a collaborative application to continuously monitor human fatigue. Force-Torque (\textsc{f/t}) sensors mounted on the robots are often the most relied technology in \textsc{phrc} scenarios for reactive robot control as they facilitate direct monitoring and regulation of the interaction wrenches\footnote{Wrench refers to forces and moments and is sometimes used in this thesis in place of forces and moments} between the human and the robot partners \cite{bussy2012human,bussy2012proactive,ikemoto2012physical,peternel2013learning,donner2016cooperative}. 

Vision-based techniques like human skeletal tracking \cite{reily2018skeleton}, human motion estimation \cite{kyrkjebo2018inertial}, and hand gesture recognition \cite{rautaray2015vision} are also used as interfaces. In general, the designer decides an interface, to communicate the human intent, depending on the application and often using a single interface mode is limiting. Hence, a combination of vision and haptic interfaces are used in literature to realize successful applications of human-robot collaboration \cite{agravante2014collaborative,de2007human}.

We believe there is an impending change in the paradigm of human perception and the future technologies for human-robot collaboration will leverage on getting as much holistic information as possible from human partners, especially for domains like collaborative manufacturing. A step in this direction is Digital Human Models (\textsc{dhm}) simulation tools that are embraced currently in manufacturing systems~\cite{endo2014dhaiba, demirel2007applications, fritzsche2010ergonomics, SHAHROKHI200955}. Ergonomics or human factors are a critical component in modern hybrid assembly lines that consists of a shared workspace between human and robot partners~\cite{galin2019review, Alexopoulos2013ErgoToolkitAE, rajesh2016review, irshad2019coupling}. Several investigations into ergonomic analysis using digital human modeling tools are available in the literature. However, these tools are often used as offline standalone tools during the preliminary design phases of a product assembly line and reiterating or redesigning is a resource exhaustive process. Furthermore, lack of holistic human perception poses a big challenge to quantitatively evaluate and validate the performance metrics of passive or active assistive devices, which encouraged researchers to find interesting, albeit complex applications for humanoid robots~\cite{ito2018evaluation, imamura2018evaluation, yoshida2017towards}. The development of low-cost easy-to-use tools for real-time monitoring of human motion and dynamics is still an open challenge~\cite{doi:10.1080/00207543.2019.1636321, moes2010digital}. Having both the kinematic quantities like joint positions, velocities and accelerations; and dynamic quantities like joint torques of the human will enable real-time monitoring to build robust controllers for successful task completion taking into account the physical interactions.
 
\subsection{Challenges in Robot Control}
\label{sec:challenges-in-robot-control}

Collisions pose a significant challenge to robots in dynamic environments of human habitats and the capacity to detect, isolate, identify and react is fundamental for their coexistence alongside humans \cite{haddadin2017robot}. Given the task of tracking a reference trajectory, impedance control \citep{hogan1984impedance,magrini2015control} is one of the most exploited approaches to achieve stable robot behavior while maintaining contacts. It facilitates a compliant and safe physical interaction between the agents in \textsc{hrc} scenarios. Any interaction wrench from the human is handled safely by changing the robot's actual trajectory i.e. the forces and moments are controlled by acting on position and orientation changes. Novel adaptive control schemes are successfully implemented expanding the applicability of impedance control \citep{gopinathan2017user,li2013impedance,gribovskaya2011motion}. The quality of interaction is further augmented through online adaptive admittance controls schemes which consider human intent inside the control loop \citep{ranatunga2016adaptive,lecours2012variable}. Adaptive control schemes endow robots with compliant characteristics that are safe for physically interacting with them. An obvious outcome of such compliance is the momentary deviation from the reference trajectory to accommodate external interactions but the original trajectory is restored when the interaction stops \citep{geravand2013human}.

Interaction between an external agent and a humanoid robot is particularly challenging because of the complexity of the robotic system \citep{goodrich2008human}. Unlike traditional industrial robots which are fixed base by design, humanoid robots are designed as floating base systems to facilitate anthropomorphic navigational capabilities. The aspect of balancing has received a lot of attention in the humanoid robotics community and several prior efforts  \cite{caux1998balance,hirai1998development,hyon2007full} went into building controllers that ensure stable robot behavior. More recently momentum-based control proved to be a robust approach and several successful applications have been realized \cite{stephens2010dynamic,herzog2014balancing,koolen2016design,hofmann2009exploiting} ensuring contact stability \cite{nori2015} with the environment by monitoring contact wrenches through quadratic programming \cite{ott2011posture,wensing2013generation,nava2016stability}. In general, these controllers are built to ensure robustness to any external perturbations and hence they are often blind to any helpful interaction an external agent is trying to have with the robot to help achieve its task.

Collaboration scenarios involve sustained intentional physical interactions, that are often helpful for the robot partner. So, instead of exhibiting simple compliant reactive behaviors, a robot can exhibit partner-aware behavior by exploiting the help from the human partner to perform the collaborative task efficiently. However, a clear definition of what constitutes a helpful interaction of a human partner is still an open question that needs to be methodically addressed.

\section{Motivation}
\label{sec:motivation}

The Horizon-$2020$ European project named An.Dy - Advancing
Anticipatory  Behaviors in Dyadic Human-Robot Collaboration
(H$2020$-ICT-$2016$-$2017$, No.$731540$) \citep{andy} aims at advancing the current state-of-the-art in human-robot collaboration scenarios that require intentional physical interactions between the human partner and the robot partner. Successful physical collaboration needs sustained dynamic wrench exchange between the interacting agents. Some example scenarios of physical human-robot collaboration are highlighted in Fig.~\ref{fig:andy-phri-scenarios} where an industrial manipulator helping an assembly line worker in transporting a heavy load (\ref{fig:andy-phri-scenario-1}); an actuated exoskeleton assisting overhead tasks (\ref{fig:andy-phri-scenario-2}); and a whole-body collaboration from a humanoid robot for assisted object transportation (\ref{fig:andy-phri-scenario-3}).

\begin{figure}[hbt!]
	\centering
	\begin{subfigure}{0.275\textwidth}
		\centering
		\includegraphics[width=\textwidth]{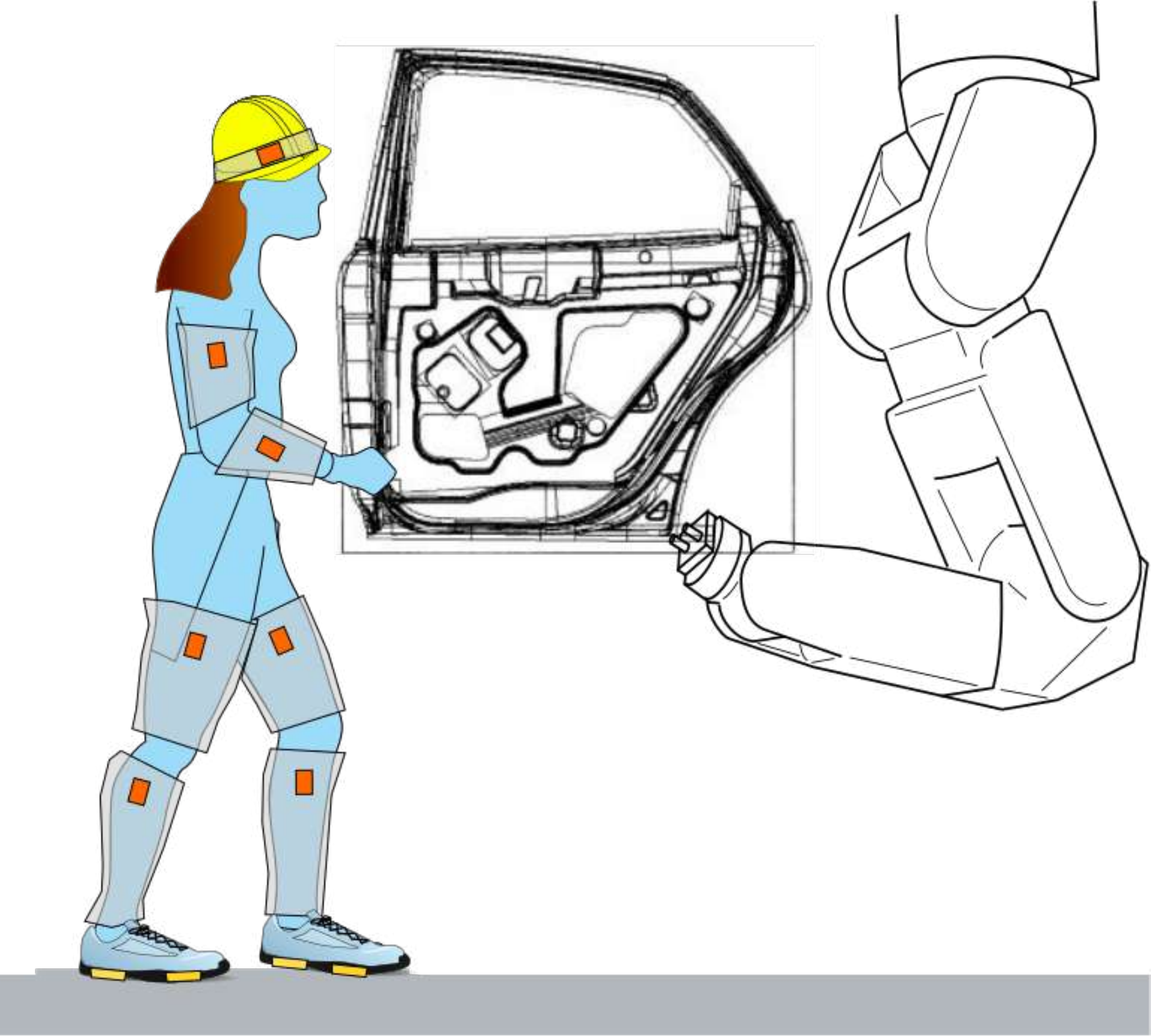}
		\caption{}
		\label{fig:andy-phri-scenario-1}
	\end{subfigure}%
	\begin{subfigure}{0.275\textwidth}
		\centering
		\includegraphics[width=\textwidth]{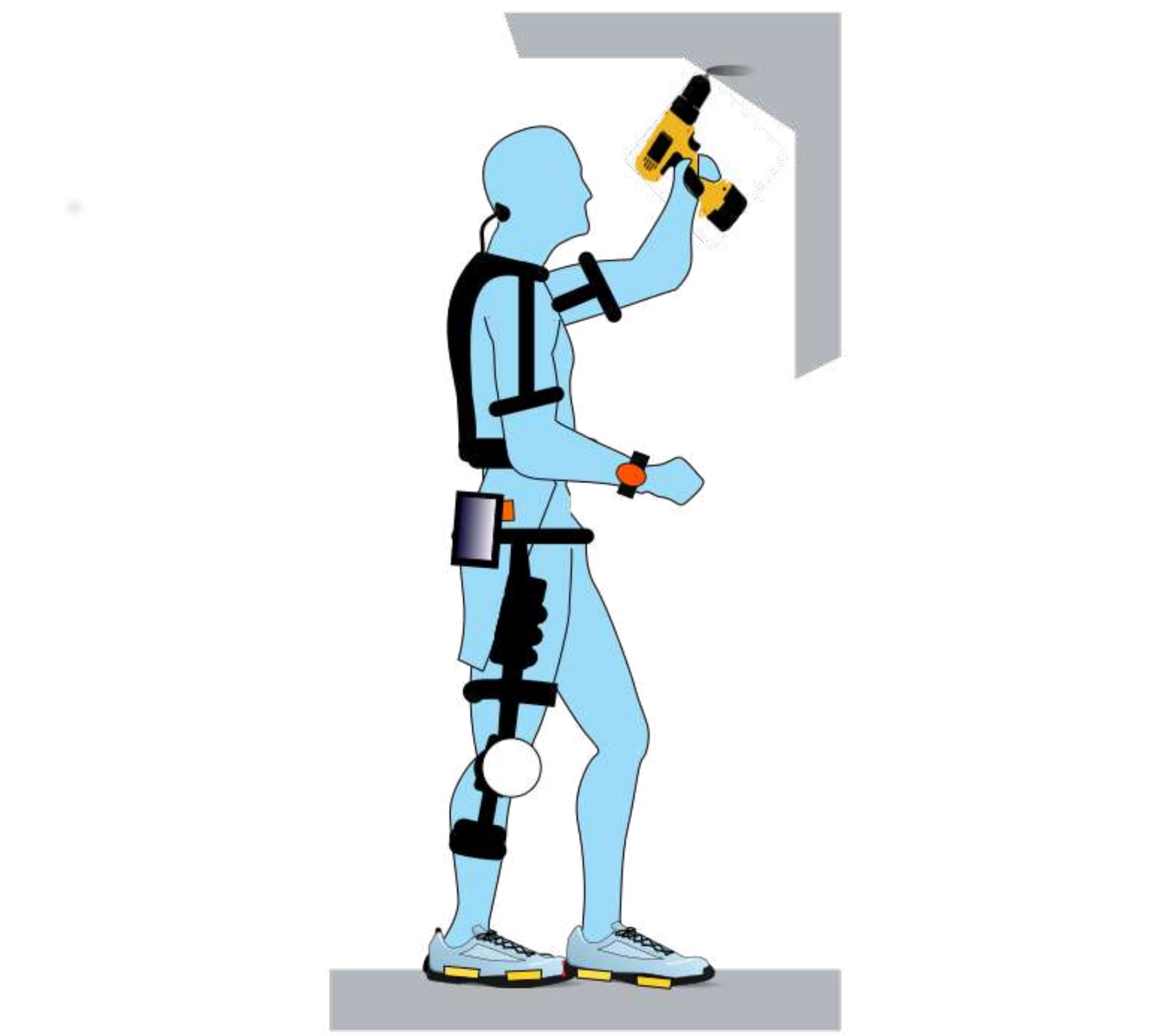}
		\caption{}
		\label{fig:andy-phri-scenario-2}
	\end{subfigure}%
	\begin{subfigure}{0.275\textwidth}
		\centering
		\includegraphics[width=\textwidth]{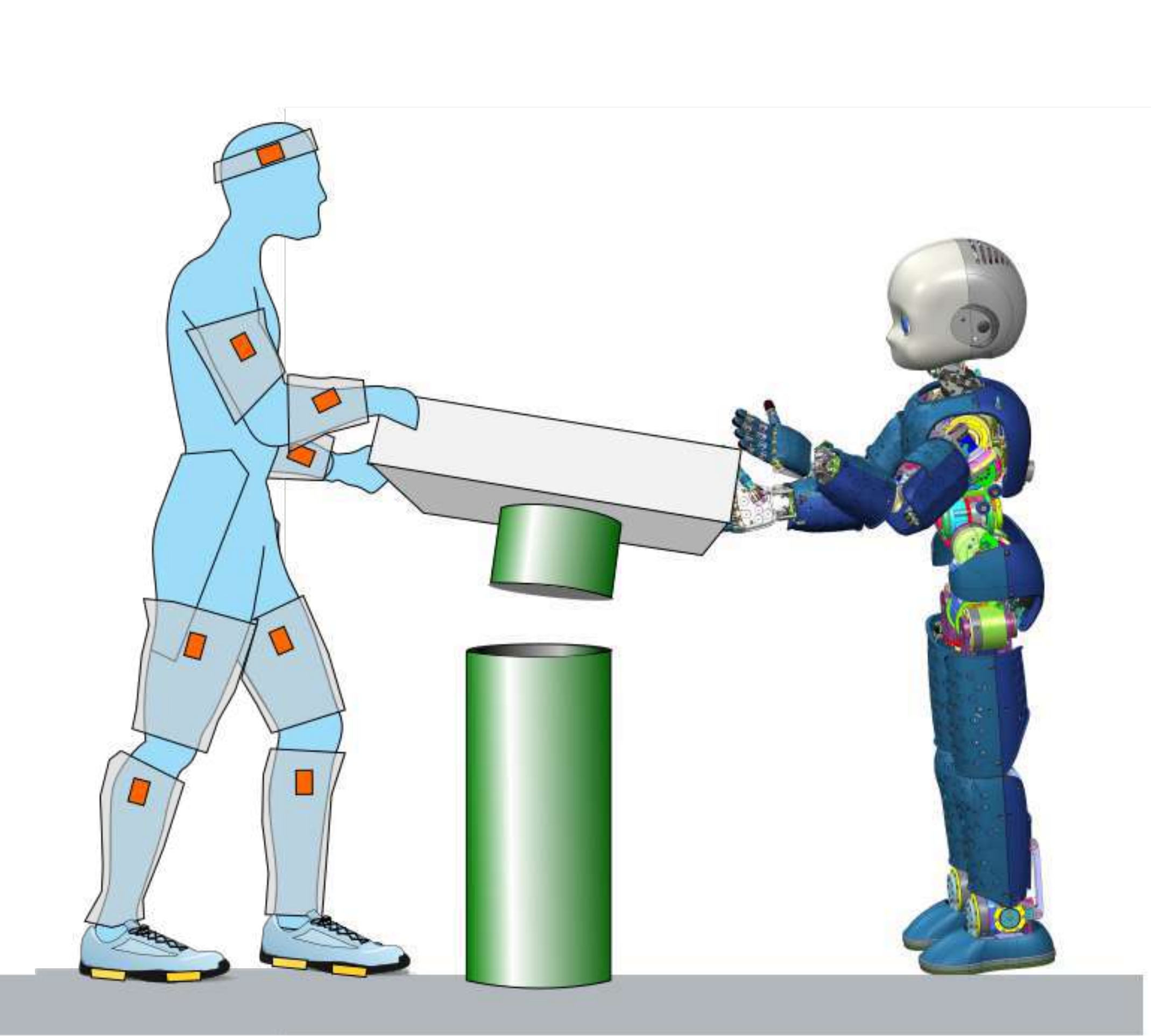}
		\caption{}
		\label{fig:andy-phri-scenario-3}
	\end{subfigure}
	\caption{Example scenarios of physical human-robot collaboration}
	\label{fig:andy-phri-scenarios}
\end{figure}

\noindent Two of the main objectives of An.Dy project are:

\begin{itemize}
	\item AnDy Suit: A novel wearable technology for real-time monitoring of human whole-body kinematics (position, velocity, acceleration), and dynamics (contact forces, joint torques, and muscle activations).
	\item  AnDy Control: Development of partner-aware reactive and predictive robot control strategies for efficient physical human-robot collaboration leveraging AnDy suit technology.
\end{itemize}

A growing technological trend referred to as \emph{Human 2.0} is about augmenting humans covering a wide range of cognitive and physical improvements. The two main categories of physical augmentation (source - \href{https://www.gigabitmagazine.com/ai/gartner-tech-trends-2020-augmented-humans}{Gartner tech trends 2020}) are: Sensory augmentation (hearing, vision, perception); Appendage and biological function augmentation (exoskeletons, prosthetics). We envision a future where technologies like AnDy suit become an integral part of human augmentation and the advanced control strategies enhance the physical interaction experience with collaborative robot partners. A potential outcome of this is a future blended workforce where the robots are seamlessly incorporated in the human umwelt.

Towards the goal of enabling human-robot collaboration via holistic human perception and partner aware control, we consider both the robot and the human are a system of rigid bodies. So, the next chapter presents the basic notation and representational details relevant to rigid body systems.

\chapter{Rigid Multi-Body System}
\label{cha:RigidMultiBodySystem}

\section{Notation}
\label{sec:background-notation}

The mathematical notation introduced in this section is followed throughout this thesis. Additional notation as required for a specific chapter will be introduced at the beginning of the chapter.

\begin{itemize}

    \item $\mathcal{I}$ denoted the inertial frame of reference, with $z$-axis pointing against gravity.
    
    \item The constant $g$ denotes the norm of the gravitational acceleration.
    
    \item  $\mathbb{R}$ denotes the set of real numbers and $\mathbb{N}$ denotes the set of natural numbers.
    
    \item $\comVar{p} \in \mathbb{R}^n$ denotes a $n$-dimensional column vector of real numbers and $p$ denotes a scalar quantity.
    
    \item Given two $n$-dimensional column vectors of real numbers, i.e. $\comVar{u},\comVar{v} \in \mathbb{R}^n$, $\comVar{u}^\top \comVar{v}$ denotes their inner product, with the transpose operator $\top$ .
    
    \item {Given two $n$-dimensional column vectors of real numbers, i.e. $\comVar{u},\comVar{v} \in \mathbb{R}^n$, $\comVar{u}{\times}~ \comVar{v}$ denotes their cross product, where}

\begin{equation}\label{crossProcuct_inR3}
    \comVar{u}{\times} :=
\begin{bmatrix}
   0    &  -u_z  &   u_y    \\
  u_z   &   0    &  -u_x     \\
 -u_y   &   u_x  &   0
\end{bmatrix} \in \mathbb{R}^{3\times 3}
\end{equation}

     \item The operator $\left\lVert . \right\rVert$ indicates the squared norm of a vector, such that given a vector $\comVar{u} \in \mathbb{R}^{3}$, it is defined as,
     
     \begin{equation}\label{squared_norm}
        \left\lVert \comVar{u} \right\rVert = \sqrt{\comVar{u}_1^2 + \comVar{u}_2^2 + \comVar{u}_3^2}
    \end{equation}

    \item $\comVar{I}_n \in \mathbb{R}^{n \times n}$ denotes an identity matrix of dimension~$n$; $\comVar{0}_n \in \mathbb{R}^n$ denotes the zero column vector of dimension~$n$; and  $\comVar{0}_{n \times m} \in \mathbb{R}^{n \times m}$ denotes the zero matrix of dimension~$n \times m$.
    
    \item Given a vector $\comVar{p}$ and a reference frame $\mathcal{A}$, the notation ${}^\mathcal{A}\comVar{p}$ denotes the vector $\comVar{p}$ expressed in $\mathcal{A}$.
    
    \item Let $SO(3)$ be the set of $\mathbb{R}^{3 \times 3}$ orthogonal matrices with determinant equal to one, such that
    
    \begin{equation}
        SO(3) := \{\comVar{R} \in \mathbb{R}^{3 \times 3} \ |~~ \comVar{R}^T \comVar{R} = \comVar{I}_3
        ~,~~|\comVar{R}| = 1 \}
    \end{equation}

    \item Let $so(3)$ be the set of the skew-symmetric matrices $\in \mathbb{R}^{3 \times 3}$, such that

    \begin{equation}
        so(3) :=  \{ \bm\skewOp \in \mathbb{R}^{3 \times 3}  \ |~~ {\bm \skewOp}^T = -\bm \skewOp \}
    \end{equation}

    \item Let the set $SE(3)$ be defined as

    \begin{equation}
        SE(3) :=  \Big\{
        \begin{bmatrix}
            \comVar{R}            & \comVar{p} \\
            \comVar{0}_{1\times3} & 1
        \end{bmatrix} \in \mathbb{R}^{4 \times 4} \ |~~
        \comVar{R} \in SO(3)~,~~\comVar{p} \in \mathbb{R}^3 \Big\}
    \end{equation}
    
    \item The operator $ \bm \skewOp(.) :\mathbb{R}^{3} \to so(3)$ denotes \textit{skew-symmetric} vector operation, such that given two vectors $\comVar{v}, \comVar{u} \in \mathbb{R}^{3}$, it is defined as $\comVar{v} \times \comVar{u} = S(\comVar{v}) \comVar{u}$.

    \item The \textit{vee} operator $.^{\vee} : so(3) \to \mathbb{R}^{3}$ denotes the inverse of the \textit{skew-symmetric} vector operator, such that given a matrix $\comVar{A} \in so(3)$ and a vector $\comVar{u} \in \mathbb{R}^{3}$, it is defined as $\comVar{A} \comVar{u} = \comVar{A}^{\vee} \times \comVar{u}$.
    
\end{itemize}

\section{Rigid Body Representation}

\begin{tcolorbox}[sharp corners, colback=white!30,
     colframe=white!20!black!30, 
     title=Definition]
\emph{A rigid body is defined as an object that is non-deformable under the application of external forces and moments. Considering any physical object to be a collection of point masses, the distance between any two point masses of a rigid body does not change when an external wrench is applied.}
\end{tcolorbox}

Although a perfectly non-deformable object seems unrealistic, for all practical purposes a rigid body is considered to have \emph{negligible} deformation.

The motion of a rigid body is well represented by associating a \emph{coordinate frame} that is attached to its body, typically with the origin of the frame at the Center of Mass (CoM) of the rigid body. This frame is usually referred to as the \emph{body frame} denoted by $\mathcal{B}$ \citep{handbook_robotics,Siciliano2009}. Consider an inertial frame of reference $\mathcal{I}$, with the origin at a 3D point $O_\mathcal{I}$. Let $O_\mathcal{B}$ be a 3D point that is the origin of the body frame $\mathcal{B}$. The motion of the rigid body (translation, rotation or a combination of the two) can be expressed with respect to (w.r.t) the inertial frame $\mathcal{I}$.

\subsection{Rigid Body Pose}

The coordinates of the origin $O_\mathcal{B}$ w.r.t. $\mathcal{I}$ are elements of the position vector,

\begin{equation}\label{coordinate_vector}
    {}^\mathcal{I}\comVar{o}_\mathcal{B} = {\begin{bmatrix}
         {}^\mathcal{I} {o_x}\\
         {}^\mathcal{I} {o_y}\\
         {}^\mathcal{I} {o_z}
        \end{bmatrix}}_\mathcal{B} \in \mathbb{R}^3
    \end{equation}

The orientation of $\mathcal{B}$ w.r.t. $\mathcal{I}$ is described by a rotation matrix $^{\mathcal{I}}\comVar{R}_{\mathcal{B}} \in SO(3)$, regardless of the positions of the origins $O_\mathcal{I}$ and $O_\mathcal{B}$.

\subsubsection{Homogeneous Transformation}

A compact representation of the pose combining both the position and the orientation of a rigid body is given by the homogeneous transformation matrix $^\mathcal{I}\comVar{H}_\mathcal{B} \in SE(3)$.

\begin{equation}
    \label{homogeneous_matrix}
    {}^\mathcal{I} \comVar{H}_\mathcal{B} =
    \begin{bmatrix}
        {}^\mathcal{I} \comVar{R}_\mathcal{B} & {}^\mathcal{I}\comVar{o}_\mathcal{B}\\
        \comVar{0}_{1\times3}                   & 1
    \end{bmatrix}
\end{equation}

Consider a 3D point $P$ present in a rigid body as shown in Figure~\ref{fig:Pose_representation}. The vector $^\mathcal{B}\comVar{p} \in \mathbb{R}^3$ denotes the point $P$ w.r.t frame $\mathcal{B}$. The vector $^\mathcal{I}\comVar{p} \in \mathbb{R}^3$ denotes the point $P$ w.r.t frame $\mathcal{I}$ can be computed using the homogeneous matrix ${}^\mathcal{I} \comVar{H}_\mathcal{B}$ as following,

\begin{equation}\label{homogenous_transformation_matrix_form}
    {}^\mathcal{I} \comVar{p} = 
    \begin{bmatrix}
        {}^\mathcal{I} \comVar{R}_\mathcal{B} & {}^\mathcal{I}\comVar{o}_\mathcal{B}\\
        \comVar{0}_{1\times3}                   & 1
    \end{bmatrix}
    \begin{bmatrix}
    {}^\mathcal{B} \comVar{p} \\
    1
    \end{bmatrix}
\end{equation}

\begin{equation}\label{homogenous_transformation}
    {}^\mathcal{I} \comVar{p} =  {}^\mathcal{I}\comVar{o}_\mathcal{B} + {}^\mathcal{I} \comVar{R}_\mathcal{B} \ {}^\mathcal{B} \comVar{p}
\end{equation}

If $O_\mathcal{B}
  \equiv O_\mathcal{I}$ i.e., null position vector
   ${}^\mathcal{I}\comVar{o}_\mathcal{B}$, Eq. \eqref{homogenous_transformation} falls into the
    a pure rotational case and the transformation becomes, 
    
\begin{equation}\label{pure_rotational_transformation}
    {}^\mathcal{I} \comVar{p} = {}^\mathcal{I} \comVar{R}_\mathcal{B} \ {}^\mathcal{B} \comVar{p}
\end{equation}
    
If the frame $\mathcal{B}$
     has the same orientation as that of the inertial frame of reference $\mathcal{I}$ i.e., the rotation matrix
      ${}^\mathcal{I} \comVar{R}_\mathcal{B} = \comVar{1}_3$, then
       \eqref{homogenous_transformation} falls into a pure translation case and the transformation becomes,

\begin{equation}\label{pure_translational_transformation}
    {}^\mathcal{I} {\comVar{p}} =
    {}^\mathcal{B} {\comVar{p}} + {}^\mathcal{I}\comVar{o}_\mathcal{B}
\end{equation}
       
\begin{figure}[H]
    \centering
    \includegraphics[width=.65\textwidth]{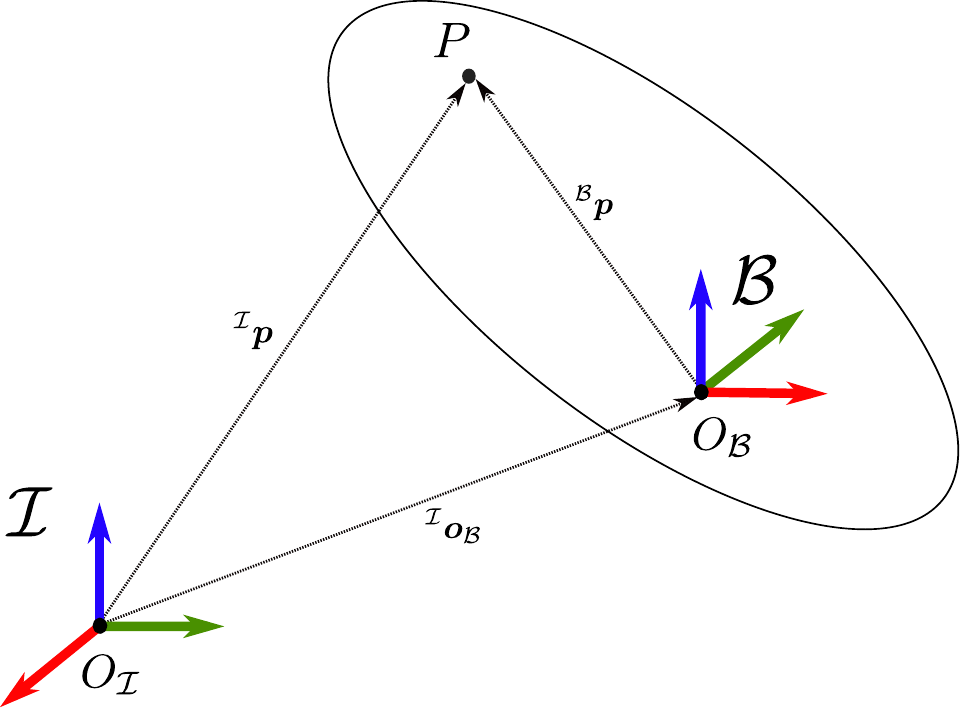}
    \caption{Representation of a rigid-body with body frame $\mathcal{B}$ and a 3D point $P$ in two different coordinate frames.  The figure introduces RGB (Red-Green-Blue) convention for ${x}$-${y}$-${z}$ axes.}
   \label{fig:Pose_representation}
\end{figure}

\begin{tcolorbox}[sharp corners, colback=white!30,
     colframe=white!20!black!30, 
     title=Frame Convention]
\emph{Throughout the thesis, the reference frames are represented using a RGB (Red-Green-Blue) convention for ${x}$-${y}$-${z}$ axes, respectively.}
\end{tcolorbox}
 
\subsection{The Derivative of a Rotation Matrix}

Given a rotational motion of a rigid body, the rotational matrix ${}^\mathcal{I} \comVar{R}_\mathcal{B}$ is a time-varying quantity. The angular velocity ${}^{\mathcal{I}}\comVar{\omega}_\mathcal{B}$ of the body w.r.t. the inertial frame of reference $\mathcal{I}$, the relation between the time derivative of the rotation matrix and the transformation of the rotation matrix can be established as following (refer to \citep{Siciliano2009}, Section 3.1.1, for a detailed description),

\begin{equation}\label{def_skewOp}
    \bm \skewOp({}^{\mathcal{I}}\comVar{\omega}_\mathcal{B}) = {}^\mathcal{I}\dot{\comVar{R}}_\mathcal{B} \  {}^\mathcal{I}\comVar{R}^T_\mathcal{B}
\end{equation}

where, 

\begin{equation}\label{skew_omega}
\bm \skewOp({}^{\mathcal{I}}\comVar{\omega}_\mathcal{B}) := {}^{\mathcal{I}}\comVar{\omega}_\mathcal{B}{\times}~ = 
\begin{bmatrix}
   0        &  -\omega_z  &   \omega_y    \\
\omega_z    &      0      &  -\omega_x    \\
-\omega_y   &   \omega_x  &      0 
\end{bmatrix}_\mathcal{B} \in so(3)
\end{equation}

\vspace{-0.25cm}

\subsection{Velocity}

\vspace{-0.15cm}

The velocity of the 3D point $P$ of a rigid body as shown in Figure~\ref{fig:Pose_representation} w.r.t. $\mathcal{I}$ can be obtained from the first-order time derivative of Eq. \eqref{homogenous_transformation}, such that

\begin{equation} \label{vel_pdot_step1}
    {}^\mathcal{I} \dot{\comVar{p}} = \frac{d}{dt} \Big({}^\mathcal{I} {\comVar{p}}\Big) = {}^\mathcal{I}\dot{\comVar{o}}_\mathcal{B} +
     {}^\mathcal{I}\dot{\comVar{R}}_\mathcal{B} \ {}^\mathcal{B}{\comVar{p}}
\end{equation}

Using the relation from Eq.~\eqref{skew_omega}, the above relation becomes,

\begin{equation} \label{vel_pdot_step2}
    {}^\mathcal{I} \dot{\comVar{p}} ={}^\mathcal{I}\dot{\comVar{o}}_\mathcal{B} +
      \bm \skewOp({}^{\mathcal{I}}\comVar{\omega}_{\mathcal{B}}) \ {}^\mathcal{I}{\comVar{R}}_\mathcal{B} \ {}^\mathcal{B}{\comVar{p}}
\end{equation}

\begin{equation} \label{vel_pdot_step3}
    {}^\mathcal{I} \dot{\comVar{p}} ={}^\mathcal{I}\dot{\comVar{o}}_\mathcal{B} + {}^{\mathcal{I}}\comVar{\omega}_\mathcal{B}{\times} \ {}^\mathcal{I}{\comVar{R}}_\mathcal{B} \ {}^\mathcal{B}{\comVar{p}}
\end{equation}

\vspace{-0.25cm}

\subsection{Acceleration}

\vspace{-0.15cm}

The velocity of the 3D point $P$ of a rigid body as shown in Figure~\ref{fig:Pose_representation} w.r.t. $\mathcal{I}$ can be obtained from the second-order time derivative of Eq. \eqref{homogenous_transformation}, such that

\begin{eqnarray} \label{acc_ddp} \notag
    {}^\mathcal{I} \ddot{\comVar{p}} &=&
    \frac{d^{2}}{d^{2}t} \Big({}^\mathcal{I} {{\comVar{p}}}\Big) 
     = {}^\mathcal{I}\ddot{\comVar{o}}_\mathcal{B} +
      {}^\mathcal{I}\dot{\comVar{\omega}}_\mathcal{B}{\times}~{}^\mathcal{I}{\comVar{R}}_\mathcal{B} \ 
       {}^\mathcal{B} {\comVar{p}} + {}^\mathcal{I}\comVar{\omega}_\mathcal{B}{\times}~{}^\mathcal{I}{\dot {\comVar{R}}}_\mathcal{B} \ 
        {}^\mathcal{B} {\comVar{p}} \\
      &=& {}^\mathcal{I}\ddot{\comVar{o}}_\mathcal{B} +
      {}^\mathcal{I}\dot{\comVar{\omega}}_\mathcal{B}{\times}~{}^\mathcal{I}{\comVar{R}}_\mathcal{B} \
       {}^\mathcal{B} {\comVar{p}} + {}^\mathcal{I}\comVar{\omega}_\mathcal{B}{\times}~\Big({}^\mathcal{I}\comVar{\omega}_\mathcal{B}{\times}~{}^\mathcal{I}{{\comVar{R}}}_\mathcal{B} \
         {}^\mathcal{B} {\comVar{p}} \Big)
\end{eqnarray}

where, $\dot {\comVar{\omega}}_\mathcal{B}$ is the angular acceleration of the rigid body.

\section{6D Vectors}

A proper mathematical representation of the rigid body kinematics and dynamics becomes essential when dealing with a multi-body system to express it succinctly. 6D vectors are a handy tool for expressing rigid body dynamics in a compact form. A thorough material on the most commonly used 6D vector representations for dealing with rigid body dynamics is presented in~\cite{traversaro2017thesis}. 

\subsection{6D Motion Vectors}

The 6D velocity vector of a rigid body expressed w.r.t. the inertial frame of reference $\mathcal{I}$ is written as,

\begin{equation} \label{spatial_vel}
    {}^\mathcal{I} {\comVar{v}}_\mathcal{B} = 
    \begin{bmatrix}
    {}^\mathcal{I} \dot{\comVar{p}} \\
    {}^\mathcal{I}\comVar{\omega}_\mathcal{B}
    \end{bmatrix} \in \mathbb{R}^6
\end{equation}

The 6D acceleration vector of a rigid body expressed w.r.t. the inertial frame of reference $\mathcal{I}$ is written as,

\begin{equation} \label{spatial_acc}
    {}^\mathcal{I} {\comVar{a}}_\mathcal{B} = 
    \begin{bmatrix}
    {}^\mathcal{I} \ddot{\comVar{p}} \\
    {}^\mathcal{I}\dot{\comVar{\omega}}_\mathcal{B}
    \end{bmatrix} \in \mathbb{R}^6
\end{equation}

\subsubsection{Adjoint Transformation for Motion Vectors}

Homogeneous transformation introduced in Eq.~\eqref{homogeneous_matrix} cannot be used directly with 6D motion vectors. The change of frame of reference for 6D motion vectors is achieved through a new transformation matrix called \emph{adjoint matrix} denoted by $\comVar{X} \in \mathbb{R}^{6 \times 6}$. Given two generic frames of reference $\mathcal{A}$ and $\mathcal{B}$, and ${}^\mathcal{B}{\comVar{o}}_\mathcal{A}$ the position vector of the origin of $\mathcal{A}$ w.r.t. $\mathcal{B}$, the adjoint transformation matrix is defined as,

\begin{equation} \label{adjoint_motion_matrix}
{}^\mathcal{B} \comVar{X}_\mathcal{A} =
\begin{bmatrix}
    {}^\mathcal{B} \comVar{R}_\mathcal{A}   & \comVar{0}_{3}\\
     - {}^\mathcal{B} \comVar{R}_\mathcal{A}
          \bm \skewOp\big({}^\mathcal{B}{\comVar{o}}_\mathcal{A}\big) &
          {}^\mathcal{B} \comVar{R}_\mathcal{A} 
    \end{bmatrix}
\end{equation}

and the transformation of 6D motion vectors is denoted as,

\begin{equation} \label{adjoint_motion_changeOfFrame}
{}^\mathcal{B} {\comVar{v}_\mathcal{A}} = {}^\mathcal{B} \comVar{X}_\mathcal{A}
\ {}^\mathcal{A}{\comVar{v}_\mathcal{A}}
\end{equation}

\subsubsection{Cross Product for Motion Vectors}

The cross product operation as defined in Eq.~\eqref{crossProcuct_inR3} is not applicable to 6D motion vectors. Consider a rigid body with a body frame $\mathcal{B}$ that is moving with a velocity ${^{\mathcal{I}}\comVar{v}}_\mathcal{B}$, as defined in Eq.~\eqref{spatial_vel}, the cross product operator is defined as,

\begin{equation}\label{crossOperator_inR6}
    {{}^\mathcal{I} \comVar{v}}_\mathcal{B}{\times}~ :=
    \begin{bmatrix}
        {}^\mathcal{I}\comVar{\omega}_\mathcal{B}{\times}~  &   {}^\mathcal{I}\dot{\comVar{p}}  \\
        \comVar{0}_{3}                        &  {}^\mathcal{I}\comVar{\omega}_\mathcal{B}{\times}~
        \end{bmatrix} \in \mathbb{R}^{6\times 6}
\end{equation}

\subsection{6D Force Vectors}

6D force vectors represent the forces and moments, with forces being the first three elements followed by the moments as the next three elements. A 6D force vector with respect to the inertial frame $\mathcal{I}$ is denoted as,

\begin{equation} \label{spatial_force}
    {}^\mathcal{I} {\comVar{f}} = 
    \begin{bmatrix}
    {}^\mathcal{I}{\comVar{\textbf{f}}} \\
    {}^\mathcal{I}{\comVar{m}}
    \end{bmatrix} \in \mathbb{R}^6
\end{equation}

Note that the 6D force vector denoted by ${}^\mathcal{I} {\comVar{f}}$ only expresses the forces and moments coordinates with respect to the inertial frame of reference $\mathcal{I}$ but they do not specify the frame (associated to any rigid body) on which they are acting up. This distinction is important to observe because the 6D motion vectors represent the motion associated with a rigid body and are associated with frame that is attached to the body. While in the case of 6D force vectors, the forces and moments are quantities that are external to the rigid body.

\subsubsection{Adjoint Transformation Dual for Force Vectors}

Given two generic frames of reference $\mathcal{A}$ and $\mathcal{B}$, and ${}^\mathcal{B}{\comVar{o}}_\mathcal{A}$ the position vector of the origin of $\mathcal{A}$ w.r.t. $\mathcal{B}$, the adjoint transformation matrix for force vectors is defined as,

\begin{equation} \label{adjoint_force_matrix}
{}^\mathcal{B} \comVar{X}_\mathcal{A}^* =
\begin{bmatrix}
    {}^\mathcal{B} \comVar{R}_\mathcal{A}   &  - {}^\mathcal{B} \comVar{R}_\mathcal{A}
          \bm \skewOp\big({}^\mathcal{B}{\comVar{o}}_\mathcal{A}\big)\\
     \comVar{0}_{3} &  {}^\mathcal{B} \comVar{R}_\mathcal{A} 
    \end{bmatrix}
\end{equation}

and the transformation of 6D force vectors is denoted as,

\begin{equation} \label{adjoint_force_changeOfFrame}
{}^\mathcal{B} {\comVar{f}} = {}^\mathcal{B} \comVar{X}_\mathcal{A}^* \
 {{}^\mathcal{A}\comVar{f}}
\end{equation}

\subsubsection{Cross Product Dual for Force Vectors}

Similar to the cross product for the 6D motion vectors, we have a cross product for the 6D force vectors that is the dual version of Eq.~\eqref{crossOperator_inR6} that is defined as,

\begin{equation}\label{dualCrossOperator_inR6}
    {}^\mathcal{I}{\comVar{v}}_\mathcal{B}{\times}^*~ :=
    \begin{bmatrix}
        {}^\mathcal{I}\comVar{\omega}_\mathcal{B}{\times}~  &   \comVar{0}_{3}            \\
        {}^\mathcal{I}\dot{\comVar{p}}        &  {}^\mathcal{I}\comVar{\omega}_\mathcal{B}{\times}~
        \end{bmatrix} \in \mathbb{R}^{6\times 6}
\end{equation}

and the cross product dual operator to be used with the 6D motion vectors and 6D force vectors is denoted as,

\begin{equation}\label{dualCrossProduct_inR6}
    {{}^\mathcal{I}\comVar{v}}_\mathcal{B}{\times}^*~ {}^{\mathcal{B}}{\comVar{f}} \in \mathbb{R}^6
\end{equation}

\section{Rigid Body Dynamics}
\label{sec:rigid-body-dynamics}

The dynamics of a rigid body deals with understanding the motion induced by the application of an external force or torque on the rigid body. The equations of motion describe the mathematical model for studying the dynamics of a rigid body. Newton-Euler representation of equations of motion expressed in a compact form using 6D vectors is denoted as,

\begin{equation} \label{equation_of_motion_1body}
    {}^{\mathcal{I}}{\comVar{f}}^{\mathcal{B}} = \frac{d}{dt}{\left({\comVar{I}} \  {}^{\mathcal{I}}{\comVar{v}}_{\mathcal{B}}\right)} = {\comVar{I}} \  {}^{\mathcal{I}}{\comVar{a}}_{\mathcal{B}} + {}^{\mathcal{I}}{\comVar
      {v}}_{\mathcal{B}}{\times^*}~ {\comVar{I}} \ {}^{\mathcal{I}}{\comVar{v}}_{\mathcal{B}}
\end{equation}

where, 
\begin{itemize}
    \item{${}^{\mathcal{I}}{\comVar{f}}^{\mathcal{B}} \in \mathbb{R}^6$ in the net 6D force i.e., sum of all the forces and moments acting on the rigid body expressed in the inertial frame of reference}
    \item{${\comVar{I}}$ is the inertia matrix $\in \mathbb{R}^{6 \times 6}$, such that}
    \begin{equation}\label{spatial_inertia_tensor}
        {\comVar{I}} =
        \begin{bmatrix} {\comVar{I}}_c + \textrm{m}~\comVar{c}{\times}~\comVar{c}^\top &
        \textrm{m}~\comVar{c}{\times}~\\ \textrm{m}~\comVar{c}{\times}^\top~& \textrm{m}~\comVar{1}_{3}
    \end{bmatrix}
\end{equation}

where $\textrm{m}$ is the body mass, ${\comVar{I}}_c$ is the rotational inertia w.r.t. the body center of mass (CoM), and $\comVar{c}$ is the position vector from the CoM of the body to the origin of the body frame $\mathcal{B}$

\item{The term ${}^{\mathcal{I}}{\comVar
		{v}}_{\mathcal{B}}{\times^*}$ is the operator that maps ${\comVar{I}}$ to its derivative $\dot{{\comVar{  I}}}$}
  
Eq.~\ref{equation_of_motion_1body} describes the relation between the net forces and moments acting on a rigid body and the rate of change of its momentum  ${\comVar{ I}} \ {}^{\mathcal{I}}{\comVar
	{v}}_{\mathcal{B}}$.

\end{itemize}

\section{Rigid Multi-Body System}
\label{sec:background-rigid-multi-body-system}

A rigid multi-body system is composed of two or more interconnected rigid bodies. The connection between the rigid bodies is made through \emph{joints}. The relative motion between two interconnected rigid bodies is determined by the motion subspace of the joint present between them. So, a combination of rigid bodies through joints compose an articulated multi-body system. The resulting motion of the entire articulated rigid multi-body system can be determined by accounting the elementary motions of each of the rigid bodies present, taking into account the constraints imposed by various joints. 

\subsection{Modeling}
\label{sec:rigid-multi-body-system-modeling}

A robotic system is often considered as an articulated rigid multi-body system that can be represented as a kinematic tree \citep{Featherstone2007,Siciliano2009}. According to the technical jargon in robotics community, a rigid body of a robotic system is often referred to as a \emph{link}. The links and the joints of a robotic system forms the edges and nodes of the kinematic tree topological representation of an articulated rigid multi-body system as shown in Figure.~\ref{figs:kinematicsTree}.

 \begin{figure}[hbt!]
     \centering
         \includegraphics[width=0.9\textwidth]{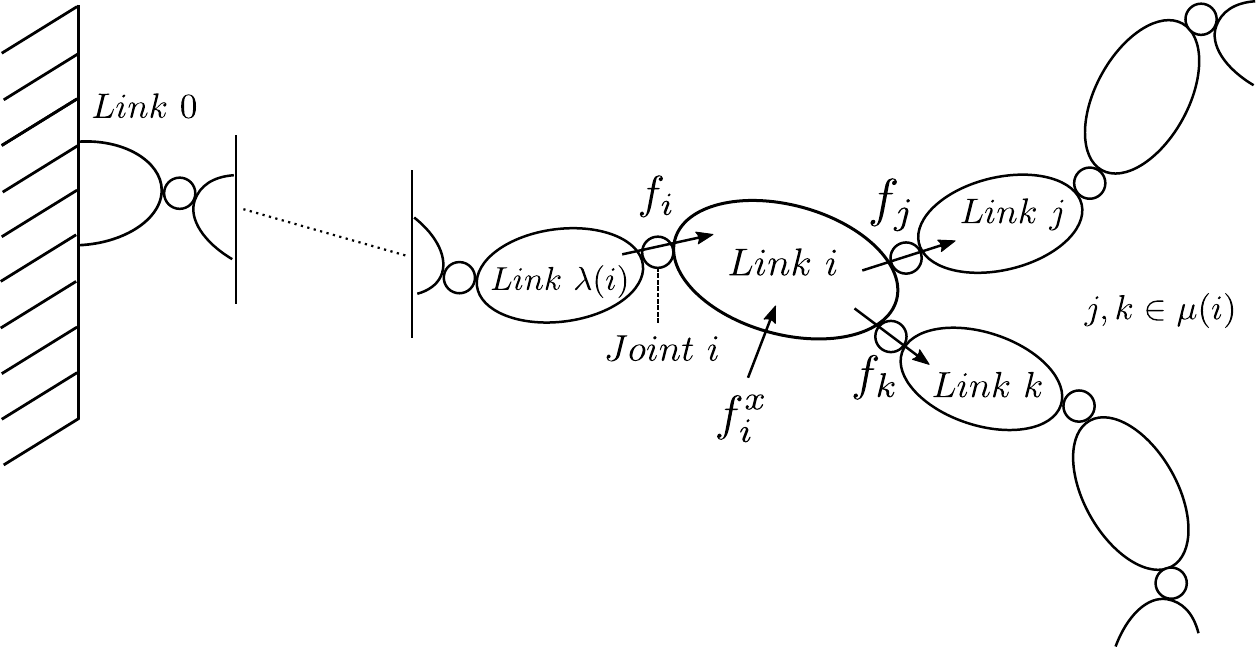}
         \caption{ Kinematic tree representation of an articulated rigid multi-body system.}
     \label{figs:kinematicsTree}
 \end{figure}

Consider a robotic system that consists of $N_B \in \mathbb{N}$ $(\ge 2)$ links. One of the links is consider to be a \emph{base link} that is typically associated with the topological numbering of $0$. Let $i$ be the index of a generic link in the tree such that $1<i<N_B$. The node numbers in the kinematic tree are always selected in a topological order, such that the $i$-th link has a higher number than its unique parent link $\lambda(i)$ and a smaller number than all the link in the set of its children $\mu(i)$.

The $i$-th link and its parent $\lambda(i)$ are coupled through joint $i$ \citep{Denavit1955}. The degrees of freedom (DoF) of the $i$-th joint is denoted as $n_i \in \mathbb{N}$ $(0<n_i<6)$ and its motion subspace is denoted as ${\bar {\comVar{S}}}_i \in \mathbb R^{6 \times n_i}$. The total number of degrees of freedom present in the articulated systems is denoted by $n$ = $n_1$ + $...$ + $n_{N_B}$.

The articulation of a robotic system is achieved through the control (actuation) of the various joints present in the system. Depending on the nature of the joint, actuation is achieved by applying either a torque (for rotary motion) or a force (for translatory motion). The configuration of the $i$-th joint is typically denoted by $\comVar{s}_i \in \mathbb{R}^{n_i}$ where $n_i$ is the number of DoF of the $i$-th joint. The actuation of the $i$-th joint is typically denoted by ${\tau}_i \in \mathbb{R}^{n_i}$. The variable $\comVar{q}_i$ is referred to as the \emph{generalized coordinate} of the mechanical system and the topology of the entire system is represented by a  vector $\comVar{s}$ as,

\begin{equation}\label{topology_s}
    \comVar{s} = \begin{bmatrix}
        \comVar{s}_1^\top && \comVar{s}_2^\top, && ..... && \comVar{s}_n^{\top}
    \end{bmatrix}^\top \in \mathbb{R}^n
\end{equation}

Traditional robotic manipulators as shown in Figure.~\ref{fig:fixed-base-system} are considered to be \emph{fixed base systems} \citep{Siciliano2009}. It means that the base link of the robot is rigidly fixed to the ground and it does not move with respect to the inertial frame of reference. So, it is a general practice to consider the coordinate frame attached to the base link as the inertial frame of reference. On the contrary, robotic systems like \emph{Humanoids} are designed on the principles of anthropomorphism to facilitate human like navigational capabilities. On such robotic systems, none of the links are assumed to have an \emph{a-priori} pose with respect to the inertial frame of reference including the base link as shown in Figure.~\ref{fig:floating-base-system}. Such systems are referred to as \emph{floating base systems} or \emph{free-floating systems} \citep{Siciliano2009} and the coordinate frame attached to the base link is referred to as \emph{floating base frame} denoted by $\mathcal{F}$. So, to refer to a floating base robotic systems, the pose and velocity of the base link are also an integral part at the modeling level and are considered as generalized coordinates. The mechanical systems considered in this thesis are assumed to be floating base systems.   

\begin{figure}[hbt!]
    \centering
    \begin{subfigure}{.5\textwidth}
         \centering
         \includegraphics[width=0.65\textwidth]{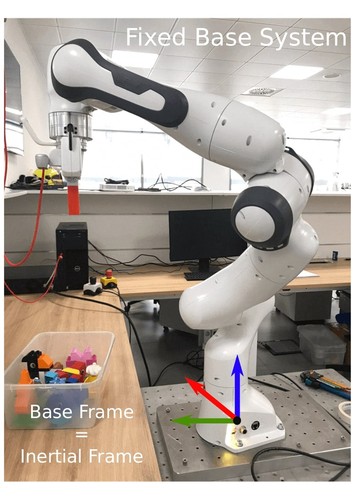}
         \caption{Fixed Base System}
         \label{fig:fixed-base-system}
    \end{subfigure}%
    \begin{subfigure}{.5\textwidth}
        \centering
        \includegraphics[width=0.65\textwidth]{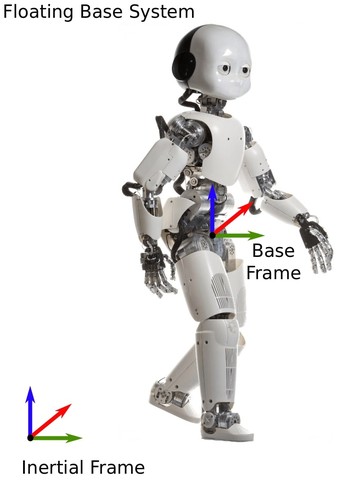}
        \caption{Floating Base System}
        \label{fig:floating-base-system}
    \end{subfigure}
    \caption{Examples of fixed base and floating base mechanical system}
    \label{fig:fixed-floating}
\end{figure}

\subsubsection{System Configuration}
\label{sec:system-configuration}

The configuration space of a \textit{free-floating} mechanical system is characterized by the \textit{floating base frame} $\mathcal{F}$ and the \textit{joint positions}. It is defined as a set of elements with $6$ dimensions representing the \textit{floating base} pose and the total number of joints $n$. Hence, it lies on the Lie group $\mathbb{Q} = \mathbb{R}^3 \times SO(3) \times \mathbb{R}^n$. 
 
An element in the configuration space is denoted by $\comVar{q} = (\comVar{q}_{\mathcal{B}},\comVar{s}) \in \mathbb{Q}$, which consists of the pose of the floating base frame $\comVar{q}_{\mathcal{B}} = (^{\mathcal{I}}{\comVar{p}}_{\mathcal{F}}, ^{\mathcal{I}}{\comVar{R}}_{\mathcal{F}}) \in \mathbb{R}^3 \times SO(3)$ where  $^{\mathcal{I}}{\comVar{p}}_{\mathcal{F}} \in \mathbb{R}^3$ denotes the position of the floating base frame $\mathcal{F}$ with respect to the inertial frame $\mathcal{I}$; $^{\mathcal{I}}{\comVar{R}}_{\mathcal{F}} \in SO(3)$ denotes the rotation matrix representing the orientation of the floating base frame $\mathcal{F}$ with respect to the inertial frame $\mathcal{I}$; and the joint positions vector $\comVar{s} \in \mathbb{R}^n$ captures the topology of the mechanical system.

\subsubsection{System Velocity}
\label{sec:system-velocity}

The system velocity is characterized by the \emph{linear and angular velocity} of the floating base frame along with the \textit{joint velocities}. The configuration velocity space lies on the Lie group $\mathbb{V} = \mathbb{R}^3 \times \mathbb{R}^3 \times \mathbb{R}^n$ and an element $\nu \in \mathbb{V}$ is defined as $\nu = (^{\mathcal{I}}{}{\comVar{v}}_{\mathcal{F}}, \dot{\comVar{s}})$ where ${}^{\mathcal{I}}{\comVar{v}}_{\mathcal{F}}=(^{\mathcal{I}}{\dot{\comVar{p}}}_{\mathcal{F}}, {}^{\mathcal{I}}{\comVar{\omega}}_{\mathcal{F}}) \in \mathbb{R}^6$ denotes the linear and angular velocity of the floating base frame $ \mathcal{F}$ expressed with respect to the inertial frame $\mathcal{I}$, and $\dot{\comVar{s}} \in \mathbb{R}^n$ denotes the joint velocities.

\subsection{Equations of Motion}

A multi-body mechanical system with links connected through joints is a dynamical system that evolves in time through articulation. Joint actuation is one of the most relevant and prominent ways to produce a dynamic motion in a mechanical system. The dynamics of a mechanical system is typically described using a set of nonlinear, second-order, ordinary differential equations. The \emph{state} of the system is composed of system configuration $\comVar{q}$ and the system velocity $\nu$ as described in section~\ref{sec:rigid-multi-body-system-modeling}. Two prominent representations of the dynamics of a rigid multi-body system are 1) Newton-Euler representation and 2) Euler-Poincar\'{e} Representation.
  
\subsubsection{Newton-Euler Representation}
\label{sec:newton-euler-dynamics-representation}

Similar to the application of Newton-Euler representation for the equations of motion for a rigid body explained in section~\ref{sec:rigid-body-dynamics}, we can express the dynamics of a rigid multi-body system based on the balance of forces acting on each of the rigid bodies present in the system. This approach is referred in robotics literature as the Recursive Newton-Euler Algorithm (RNEA) (refer Chapter 3 of \cite{Featherstone2007}). Under the Newton-Euler representation, all the quantities, except the net external forces, are expressed with respect to the body frame. The external forces are expressed with respect to the inertial frame of reference $\mathcal{I}$. 
 
Considering the modeling of a rigid multi body system presented in section~\ref{sec:background-rigid-multi-body-system}, the velocity of the $i$-th link and the velocity of the $i$-th joint are defined recursively as,
 
\begin{eqnarray}
 \label{eq:vJi}
 \comVar{v}_{Ji} &=& \comVar{\bar S}_i \dot {\comVar{q}_i} \\
 \label{eq:vi}
 \comVar{v}_i &=& \prescript{i}{} {\comVar{X}_{\lambda(i)}} \comVar{{
 		v}}_{\lambda(i)} + \comVar{ v}_{Ji}
\end{eqnarray}
 
The acceleration of the $i$-th link is defined as,
 
\begin{eqnarray}
 \label{eq:ai}
 \comVar{ a}_i &=& \prescript{i}{}{\comVar{X}_{\lambda(i)}} \comVar{ a}_{\lambda(i)} + \comVar{\bar S}_i \ddot { \comVar{q}}_i + \comVar{ v}_i{\times}~ \comVar{ v}_{Ji} 
\end{eqnarray}
 
 Equations \eqref{eq:vJi}, \eqref{eq:vi} and \eqref{eq:ai} are propagated
 throughout the kinematic tree with the initial boundary
 conditions $\bm { v}_0 =\bm 0$ and $\bm { a}_0 = -\bm
 { g}$, which corresponds to the gravitational
 acceleration vector expressed in the body frame $0$, such that
 $\bm { g} = 
 \begin{bmatrix}
 0 & 0 & -9.81 & 0 & 0 & 0
 \end{bmatrix}^\top$.
 
 The net 6D force $\comVar{ f}^{\mathcal{B}}_i$ acting on the $i$-th link is related to the rate of change of momentum as described in Eq.~\eqref{equation_of_motion_1body}
 
 \begin{eqnarray}
 \label{eq:fBi}
 \comVar{f}^{\mathcal{B}}_i & =&  {\comVar{I}}_i \comVar{ a}_i + \comVar{ v}_i{\times^*}~ {\comVar{ I}}_i \comVar{ v}_i
 \end{eqnarray}
 
 The internal forces and moments transmitted through the $i$-th joint are denoted by the 6D force vector $\comVar{f}_{Ji}$ and all the external forces and moments acting on the $i$-th link are denoted by the 6D force vector $ \comVar{f}_i^x$. Now, the balance of forces is give by the following relation,
 
 \begin{eqnarray}
 \label{eq:fi}
 \comVar{f}_{Ji} & =&  \comVar{f}^{\mathcal{B}}_i - \prescript{i}{} {\comVar{X}_{0}^*} \ \comVar{f}_i^x + \sum_{j \in \mu(i)} \prescript{i}{}{\comVar{X}_{j}^*} \ \comVar{f}_{Jj}
 \end{eqnarray}
 
 On computing the forces and moments across each joint, the joint torques are computed through the joint motion subspace using the following relation,
 
 \begin{eqnarray}
 \label{eq:taui}
 \comVar{\tau}_i & =&  \comVar{\bar S}^T_i \ \comVar{f}_{Ji}
 \end{eqnarray}

\subsubsection{Euler-Poincar\'{e} Representation}
\label{sec:lagrandian-dynamics-represenatation}

Another most commonly used approach to derive the dynamic equations of motion of a mechanical systems is  \emph{Lagrangian Analysis}, that relies on the energy properties of the mechanical system under consideration. The detailed derivation is available in several robotics text books like \citep{Siciliano2009}, Chapter 7 or \citep{khalil2004modeling}, Chapter 9. Most of the textbooks present the Euler-Lagrangian equations of motion for a fixed base mechanical systems. However, in this thesis, the mechanical systems considered are assumed to be a floating base systems and hence the equations of motions for such systems are slightly different than a fixed base system. Specifically, Euler-Poincar\'{e} equations~\cite{Marsden2010} describe the equations of motion for a floating base system. The reader is advised to refer to Chapter 3 of \cite{traversaro2017thesis} for a detailed derivation of the equations of motion for a floating base system.

The equations of motion of a \textit{floating base} mechanical system are described by,

\begin{equation}
\comVar{M}(\comVar{q}) \dot{\robnu} + \comVar{C}(\comVar{q},{\robnu}) {\robnu} + \comVar{G}(\comVar{q}) = \comVar{B} {\robtau} + \sum_{i = 1}^{n_c} \comVar{J}_{c_i}^\top \ {\comVar{f}}_i
\label{eq:equations-of-motion}
\end{equation}

where, $\comVar{M} \in \mathbb{R}^{ (n+6) \times (n+6)}$ is the mass matrix, $\comVar{C} \in \mathbb{R}^{ (n+6) \times (n+6) }$ is the Coriolis matrix, $\comVar{G} \in \mathbb{R}^{n+6}$ is the gravity term, $\comVar{B} = (\comVar{0}_{n \times 6},\comVar{I}_n)^T$ is a selector matrix, ${\robtau}  \in \mathbb{R}^{n}$ is a vector representing the joint torques, ${\comVar{f}}_i \in \mathbb{R}^{6n_c}$ represents the 6D forces acting on $i$-th contact link expressed with respect to the body frame, and $\comVar{J}_{c_i} \in \mathbb{R}^{ (n+6) \times 6n_c}$ is the $i$-th contact jacobian. 

\chapter{Recall on Human Modeling}
\label{cha:human-modeling}

\chapreface{ Digital Human Models (\textsc{dhm}) are becoming increasingly popular as enabling technologies in several domains. This chapter presents the significance of a faithfully articulated model for a human partner across different domains and lays the considerations we made in our human modeling aimed towards the goal of human-robot collaboration. At the modeling level, we consider the human partner to be a mechanical system of rigid bodies. More specifically, we consider the human to be a floating base mechanical systems. }

\section{Digital Human Modeling}

A faithfully articulated full body human model is very useful across various domains~\cite{alami2006safe, xiang2010predictive, GUO2020106544, duffy2016handbook}. Human bodies are anatomically very complex that consists of bones, joints, muscles, tendons and skin. Capturing all of the details for digital representation is an enormous endeavor. On one hand, graphics community uses \emph{virtual humans} to do character animation for entertainment or interactive virtual reality applications. On the other hand,  biomechanical community uses human models to understanding the biomechanics of complex human movement. Furthermore, there is a growing interest in the human-robot interaction community towards using digital human model abstractions~\cite{alami2013human, Saikia2017}.

An established standard for virtual humans in graphics industry is called \emph{H-Anim} (Humanoid Animation) \cite{HAnim} which is a hierarchical representation of the human skeletal structure. A simplest skeletal structure is composed of joints and the bones. Naturally, all the joints are considered to be rotational type and the change in their configuration modifies the pose of various bones in the skeletal structure. So, the joints are the basis for the articulation of a skeleton. \emph{Levels of Articulation} (\textsc{loa}) refers to the number of joints considered for the skeleton. Figure.~\ref{fig:H-Anim-LOA1} highlights the joints that are considered for Level of Articulation 1 (\textsc{loa1}) in H-Anim standard.

\begin{figure}[!hbt]
    \centering
    \includegraphics[scale=0.6]{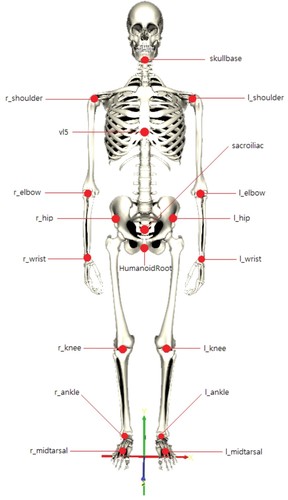}
    \caption{Joints considered for Level of Articulation 1 (\textsc{loa1}) in Humanoid Animation (H-Anim) standard \cite{HAnim}. Skeletal model in anatomically neutral N-pose.}
    \label{fig:H-Anim-LOA1}
\end{figure}

Coming to the field of biomechanics, the human model is more holistic with the inclusion of muscles and tendons. The articulation of the skeleton is achieved through the muscle activation rather than using the joints directly. Motions capture systems are an invaluable asset that captures motion data that can be used along with a complex human model for detailed simulation and analysis of anatomically accurate models.

Most commercial motion capture systems often define their models that facilitate motion data in a clinically meaningful manner without the full complexity of the human model. The models they employ are similar to the graphics community as a way to provide a real-time visualization of the human motion projected to a virtual avatar. However, they often adhere to International Society of Biomechanics (\textsc{isb}) standards for their modeling. One such motion capturing system we consider as reference for our human model is Xsens MVN motion capture system \citep{Roetenberg2009}. The kinematic model consists of a total of 23 body segments whose names are as indicated in Figure.~\ref{fig:xsens-mvn-model-names} and the frames associated with each segment are as shown in Figure.~\ref{fig:xsens-mvn-model-frames}.
  
\begin{figure}[hbt!]
    \centering
    \begin{subfigure}{.35\textwidth}
         \centering
         \includegraphics[scale=0.425]{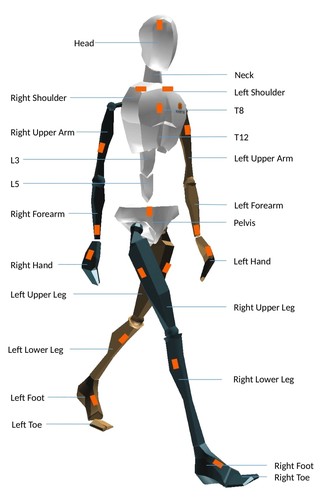}
         \caption{Segment names}
         \label{fig:xsens-mvn-model-names}
    \end{subfigure}%
    \begin{subfigure}{.65\textwidth}
        \centering
        \includegraphics[scale=0.525]{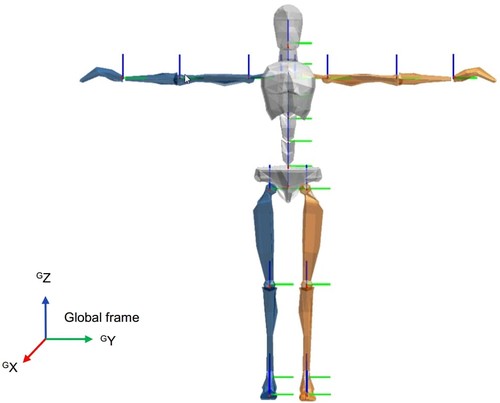}
        \caption{Segment frame definitions when the model is in T-pose}
        \label{fig:xsens-mvn-model-frames}
    \end{subfigure}
    \caption{The definition of the 23 segments in the kinematic model of Xsens MVN motion capturing system \cite{karatsidis2017estimation}}
    \label{fig:xsens-model}
\end{figure}
  
In this thesis, the human body is assumed to be composed of a set of rigid bodies for various body segments. Following the details presented in Section~\ref{sec:background-rigid-multi-body-system}, we model whole human body as a rigid multi-body system. A total of $23$ rigid bodies or links constitute our human model i.e., $N_B=23$. This choice is motivated by the use of Xsens MVN motion capture system at the initial stages of our research.

\section{Link Modeling}

The names associated with the links in our model are consistent with that of the Xsens model. Unlike graphical virtual humans that are in general very detailed from the level of facial expressions to the fine articulation at the hands, we are primarily interested in whole-body motion and dynamics of humans. So, the \emph{geometry} of the links considered are very simple when compared to the sophisticated virtual humans used in graphics community.

\subsection{Geometry}

The main shapes we considered for the links are \emph{parallelepiped boxes},
 \emph{cylinders} and \emph{spheres}. This choice is motivated by our interest in representing the human model close to an animated avatar rather than a simple multi rigid-body structure. The dimensions of each of the body segment are obtained through a calibration procedure that is capable of scaling based on the dimensions of the actual human being we want to model. We rely on an anthropomorphic reference \citep{Winter1990} that attributes a certain percentage of the total human mass to different body segments. Table~\ref{table_linkModelling} highlights the link names (similar to the Xsens model), link geometry considered and the link mass in terms of percentage of the total mass of the human. 

\begin{table}[!ht]
\vspace{0.15cm}
\centering
\small
\caption{Link names, link geometry and the link mass in terms of percentage of the total mass of the human subject (extracted from \citep{Winter1990})}
\label{table_linkModelling}
\begin{tabular}{ccc}
\hline\hline
\\
\textbf{Label} & \textbf{Shape} & \textbf{$\%$ Mass of human}\\
\\
\hline
\\
Pelvis          &  parallelepiped  & $0.08$\\
\rowcolor{Gray}
L5              &  parallelepiped  & $0.102$\\
L3              &  parallelepiped  & $0.102$\\
\rowcolor{Gray}

T12             &  parallelepiped  & $0.102$\\
T8              &  parallelepiped  & $0.04$\\
\rowcolor{Gray}
Neck            &  cylinder        & $0.012$\\
Head            &  sphere          & $0.036$\\
\rowcolor{Gray}
RightShoulder   &  cylinder        & $0.031$\\
RightUpperArm   &  cylinder        & $0.030$\\
\rowcolor{Gray}
RightForeArm    &  cylinder        & $0.020$\\
RightHand       &  parallelepiped  & $0.006$\\
\rowcolor{Gray}
LeftShoulder    &  cylinder        & $0.031$\\
LeftUpperArm    &  cylinder        & $0.030$\\
\rowcolor{Gray}
LeftForeArm     &  cylinder        & $0.020$\\
LeftHand        &  parallelepiped  & $0.006$\\
\rowcolor{Gray}                            
RightUpperLeg   &  cylinder        & $0.125$\\
RightLowerLeg   &  cylinder        & $0.0365$\\
\rowcolor{Gray}
RightFoot       &  parallelepiped  & $0.013$ \\
RightToe        &  parallelepiped  & $0.015$ \\
\rowcolor{Gray}
LeftUpperLeg   &  cylinder        & $0.125$  \\
LeftLowerLeg   &  cylinder        & $0.0365$ \\
\rowcolor{Gray}
LeftFoot       &  parallelepiped  & $0.013$  \\
LeftToe        &  parallelepiped  & $0.0015$ \\
\\
\hline\hline
\end{tabular}
\end{table}

\subsection{Inertial Properties}

Inertial properties of the body segments i.e., mass, center of mass and moments of inertial are important to understand and capture the dynamics of the human. A popular approach for human body segments inertial properties estimation is based on \emph{geometric approximation} where a collection of measurements from the human are mapped to a geometric model to derive the segments inertial properties based on the geometry and density assumption. Unlike muscles of a human body that generally vary in density, we assumed density isotropy for all the body segments \citep{Hanavan1964}. Under these assumptions, the inertial tensor $\comVar{I}$ is computed as,

\begin{equation}\label{inertiaTensor}
\comVar{I}  =  \begin{bmatrix}
        \mathrm{I}_{xx} & 0               & 0               \\
         0              & \mathrm{I}_{yy} & 0               \\
         0              & 0               & \mathrm{I}_{zz} \\
    \end{bmatrix}
\end{equation}

where $\mathrm{I}_{xx}$, $\mathrm{I}_{yy}$ and $\mathrm{I}_{zz}$ are the
 principal moments of inertia.  Table \ref{PrincipalMomentsInertia_shapes}
  lists analytical formulas for the principal moments of inertia computation.

\begin{table}[!ht]
\vspace{0.15cm}
\centering
\caption{Principal moments of inertia of three different geometric shapes considered for the links with mass $\textrm{m}$:
 (\emph{on left column}) a rectangular parallelepiped of width $\alpha$, height
  $\beta$ and depth $\gamma$; (\emph{on middle column}) a circular cylindrical
   of radius $r$ and height $h$; (\emph{on right column}) a sphere of radius
    $r$}
\label{PrincipalMomentsInertia_shapes}
\centering
\scriptsize
\begin{tabular}{c|cccc}
\hline\hline
\\
\textbf{Inertia} & \textbf{Parallelepiped} & \textbf{Cylinder}
                 & \textbf{Sphere}\\
\\
\hline
\\
$\mathrm{I}_{xx}$      & $\frac{1}{12}\textrm{m} ~
 \big(\alpha^{2}+\beta^{2}\big)$
                       & $\frac{1}{12}\textrm{m} ~ \big(3r^{2}+h^{2}\big)$
                       & $\frac{2}{5}\textrm{m}r^{2}$ \\
                       \\
\rowcolor{Gray}
$\mathrm{I}_{yy}$      & $\frac{1}{12}\textrm{m} ~
 \big(\beta^{2}+\gamma^{2}\big)$
                       & $\frac{1}{2}\textrm{m} r^2$
                       & $\frac{2}{5}\textrm{m}r^{2}$ \\
                       \\
$\mathrm{I}_{zz}$      & $\frac{1}{12}\textrm{m} ~
 \big(\gamma^{2}+\alpha^{2}\big)$
                       & $\frac{1}{12}\textrm{m} ~ \big(3r^{2}+h^{2}\big)$
                       & $\frac{2}{5}\textrm{m}r^{2}$ \\
\\
\hline\hline
\end{tabular}
\end{table}

\section{Joint Modeling}

 The links are coupled through joints and as with the link names, we use the same joint names as that of the Xsens model. Typically, a rotational joint like the shoulder is usually composed of three degrees of freedom (DoF) rotation, although some anatomical joints like elbow, knee and finger joints are less than three degrees of freedom rotation. For the sake of generality, we assumed all the joints present in our model to be a three degrees of freedom joints. Given a total of $23$ links considered in our model, a total of $22$ three degree of freedom rotational joints are present in our model. In other words, the total number of internal degrees of freedom are $n = 22 \times 3 = 66$. A scaled down version of our model consists of only $n = 48$ degrees of freedom, by setting some of the joints to be of two degrees of freedom or one degree of freedom. Table~\ref{table_jointModelling} highlights the joint names, total degrees of freedom of a joint with the reduced DoF in parenthesis and the links connected through a joint.
 
 \begin{table}[!ht]
    \vspace{0.15cm}
\centering
\small
\caption{Joint names, DoFs per each joint (reduced DoFs) and the links connected through a joint}
\label{table_jointModelling}
\begin{tabular}{ccccc}
\\
\hline\hline
\\
\multicolumn{1}{c}{ \textbf{Label}} & \multicolumn{1}{c}{\textbf{DoF}} &
 \multicolumn{3}{c}{\textbf{Connected links}}\\
\\
\hline
\\
\multicolumn{1}{c}{jL5S1} & \multicolumn{1}{c}{$3~(2)$} &
 \multicolumn{1}{r}{Pelvis} & \multicolumn{1}{c}{$\longleftrightarrow$} &
 \multicolumn{1}{l}{L5} \\
\rowcolor{Gray}
\multicolumn{1}{c}{jL4L3} & \multicolumn{1}{c}{$3~(2)$} & \multicolumn{1}{r}{L5} &
 \multicolumn{1}{c}{$\longleftrightarrow$} & \multicolumn{1}{l}{L3} \\
\multicolumn{1}{c}{jL1T12} & \multicolumn{1}{c}{$3~(2)$} & \multicolumn{1}{r}{L3} &
 \multicolumn{1}{c}{$\longleftrightarrow$} & \multicolumn{1}{l}{T12} \\
\rowcolor{Gray}
\multicolumn{1}{c}{jT9T8} & \multicolumn{1}{c}{$3$} & \multicolumn{1}{r}{T12} &
 \multicolumn{1}{c}{$\longleftrightarrow$} & \multicolumn{1}{l}{T8} \\
\multicolumn{1}{c}{jT1C7} & \multicolumn{1}{c}{$3$} & \multicolumn{1}{r}{T8} &
 \multicolumn{1}{c}{$\longleftrightarrow$} & \multicolumn{1}{l}{Neck} \\
\rowcolor{Gray}
\rowcolor{Gray}
\multicolumn{1}{c}{jC1Head} & \multicolumn{1}{c}{$3~(2)$} &
 \multicolumn{1}{r}{Neck} & \multicolumn{1}{c}{$\longleftrightarrow$} &
  \multicolumn{1}{l}{Head} \\
\multicolumn{1}{c}{jRightHip} & \multicolumn{1}{c}{$3$} &
 \multicolumn{1}{r}{Pelvis} & \multicolumn{1}{c}{$\longleftrightarrow$} &
 \multicolumn{1}{l}{RightUpperLeg} \\
\rowcolor{Gray}
\multicolumn{1}{c}{jRightKnee} & \multicolumn{1}{c}{$3~(2)$} &
 \multicolumn{1}{r}{RightUpperLeg} & \multicolumn{1}{c}{$\longleftrightarrow$}
  & \multicolumn{1}{l}{RightLowerLeg} \\
\multicolumn{1}{c}{jRightAnkle} & \multicolumn{1}{c}{$3$} &
 \multicolumn{1}{r}{RightLowerLeg} & \multicolumn{1}{c}{$\longleftrightarrow$}
  & \multicolumn{1}{l}{RightFoot} \\
\rowcolor{Gray}
\multicolumn{1}{c}{jRightBallFoot} & \multicolumn{1}{c}{$3~(1)$} &
 \multicolumn{1}{r}{RightFoot} & \multicolumn{1}{c}{$\longleftrightarrow$} &
  \multicolumn{1}{l}{RightToe} \\
\multicolumn{1}{c}{jLeftHip} & \multicolumn{1}{c}{$3$} &
 \multicolumn{1}{r}{Pelvis} & \multicolumn{1}{c}{$\longleftrightarrow$} &
  \multicolumn{1}{l}{LeftUpperLeg} \\
\rowcolor{Gray}
\multicolumn{1}{c}{jLeftKnee} & \multicolumn{1}{c}{$3~(2)$} &
 \multicolumn{1}{r}{LeftUpperLeg} & \multicolumn{1}{c}{$\longleftrightarrow$} &
  \multicolumn{1}{l}{LeftLowerLeg} \\
\multicolumn{1}{c}{jLeftAnkle} & \multicolumn{1}{c}{$3$} &
 \multicolumn{1}{r}{LeftLowerLeg} & \multicolumn{1}{c}{$\longleftrightarrow$} &
  \multicolumn{1}{l}{LeftFoot} \\
\rowcolor{Gray}
\multicolumn{1}{c}{jLeftBallFoot} & \multicolumn{1}{c}{$3~(1)$} &
 \multicolumn{1}{r}{LeftFoot} & \multicolumn{1}{c}{$\longleftrightarrow$} &
  \multicolumn{1}{l}{LeftToe} \\
\multicolumn{1}{c}{jRightC7Shoulder} & \multicolumn{1}{c}{$3~(1)$} &
 \multicolumn{1}{r}{T8} & \multicolumn{1}{c}{$\longleftrightarrow$} &
  \multicolumn{1}{l}{RightShoulder} \\
\rowcolor{Gray}
\multicolumn{1}{c}{jRightShoulder} & \multicolumn{1}{c}{$3$} &
 \multicolumn{1}{r}{RightShoulder} & \multicolumn{1}{c}{$\longleftrightarrow$}
  & \multicolumn{1}{l}{RightUpperArm} \\
\multicolumn{1}{c}{jRightElbow} & \multicolumn{1}{c}{$3~(2)$} &
 \multicolumn{1}{r}{RightUpperArm} & \multicolumn{1}{c}{$\longleftrightarrow$}
  & \multicolumn{1}{l}{RightForeArm} \\
\rowcolor{Gray}
\multicolumn{1}{c}{jRightWrist} & \multicolumn{1}{c}{$3~(2)$} &
 \multicolumn{1}{r}{RightForeArm} & \multicolumn{1}{c}{$\longleftrightarrow$} &
  \multicolumn{1}{l}{RightHand} \\
\multicolumn{1}{c}{jLeftC7Shoulder} & \multicolumn{1}{c}{$3~(1)$} &
 \multicolumn{1}{r}{T8} & \multicolumn{1}{c}{$\longleftrightarrow$} &
  \multicolumn{1}{l}{LeftShoulder} \\
\rowcolor{Gray}
\multicolumn{1}{c}{jLeftShoulder} & \multicolumn{1}{c}{$3$} &
 \multicolumn{1}{r}{LeftShoulder} & \multicolumn{1}{c}{$\longleftrightarrow$} &
  \multicolumn{1}{l}{LeftUpperArm} \\
\multicolumn{1}{c}{jLeftElbow} & \multicolumn{1}{c}{$3~(2)$} &
 \multicolumn{1}{r}{LeftUpperArm} & \multicolumn{1}{c}{$\longleftrightarrow$} &
  \multicolumn{1}{l}{LeftForeArm} \\
\rowcolor{Gray}
\multicolumn{1}{c}{jLeftWrist} & \multicolumn{1}{c}{$3~(2)$} &
 \multicolumn{1}{r}{LeftForeArm} & \multicolumn{1}{c}{$\longleftrightarrow$} &
  \multicolumn{1}{l}{LeftHand} \\
\\
\hline\hline
\\
\end{tabular}
\end{table}

\section{Model Representation}

Both the graphics and biomechanics communities have an active developer base whose contributions lead to tools with many advanced features. Blender~\citep{Blender} is a free and open-source 3D computer graphics software used for character animation. \textsc{collada} (with a .dae extension) is one of the most widely used file formats for sharing graphics related models. Typically, along with the geometry of the model, \textsc{collada} facilitates storing appearance, scene and animation information that are essential for character animation. OpenSim~\citep{delp2007opensim} is an open source software for biomechanical modeling, simulation and analysis. Opensim Model (with .osim extension) is the file format used for sharing an anatomically accurate human model. In our observation we noticed that \textsc{collada} and OpenSim Model formats are not compatible with each other as the purposes they serve are not overlapping.

Our main motivation behind developing a human model is towards facilitating human-robot interaction and human-robot collaboration research. Lately, \emph{Robot Operating System} (\textsc{ros})~\citep{ros2018} emerged as the most popular open-source middle-ware for robotics research. Unified Robot Description Format (\textsc{urdf}) (with .urdf extension) is an Extensible Markup Language (\textsc{xml}) schema based file format that is widely adopted to represent robot models. \textsc{urdf} models are not as anatomical accurate as OpenSim models or as detailed as \textsc{collada} models, but they capture all the details needed to perform the simulation and analysis of rigid multi-body systems. So, a natural choice for us to represent our human model is the \textsc{urdf} format. Currently, the process of generating a human \textsc{urdf} model relies on the Xsens motion capture data during the calibration procedure~\cite{latella2018thesis}. Human models of subjects with varying body dimensions are available as an open source repository \texttt{human-gazebo}\footnote{https://github.com/robotology/human-gazebo}.

\begin{figure}[ht!]
  \centering
  \includegraphics[width=0.9\columnwidth]{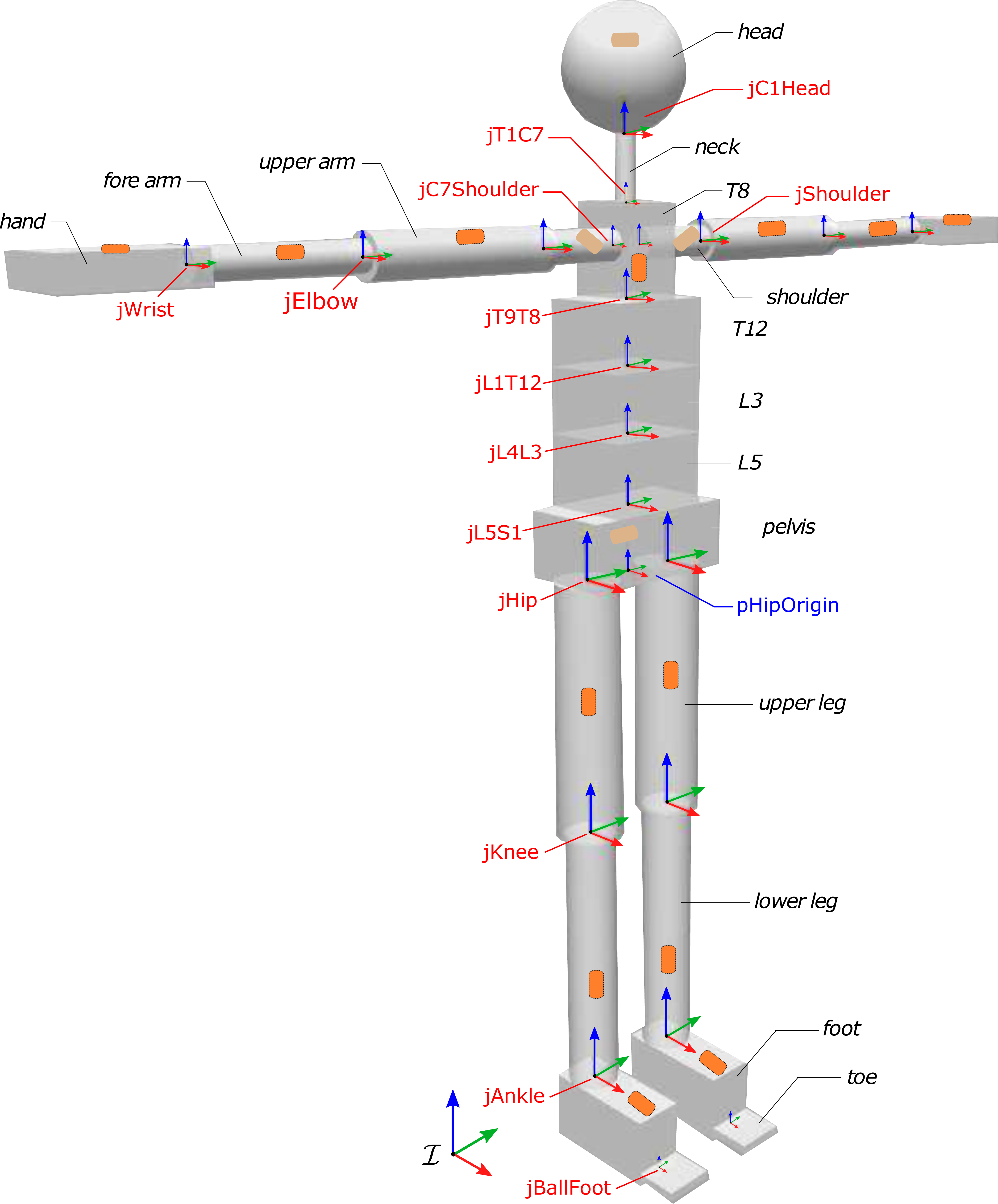}
  \caption{\textsc{urdf} model of human body}
  \label{human_urdfModel_sens}
\end{figure}
\chapter{Enabling Technologies}
\label{cha:technologies}

\chapreface{The outcome of any scientific endeavor depends significantly on the use of proper tools to address the research questions and accomplish the set objectives. This chapter presents the enabling technologies that serve as the main tools for building a holistic human perception framework and evaluating our partner-aware robot control techniques towards realizing our research goals of enabling human-robot collaboration.}

\section{Human Motion Perception}
\label{sec:technologies-human-motion-perception}

Human motion tracking is used in a variety of applications like gaming, animation, virtual reality etc. Different representations of human body with different level of complexity from contours~\cite{shio1991,Leung1995}, stick figure~\cite{Niyogi1994,Bharatkumar1994}, and volumes~\cite{wachter1999tracking,gall2009motion} are investigated in literature. Motion tracking through optical tracking technology is one of the most mature fields with it roots starting from over two decades \cite{aggarwal1999human, Wagh2019}. Several motion capture systems based on reflective markers also exists as commercial products such as Vicon and Qualisys. However, they impose space constraints and are not robust to occlusions. 

Inertial tracking technologies quickly rose to popularity as they provide a convenient alternative to human motion tracking without any space constraints. Furthermore, they ensure low latency that is crucial for real-time applications \cite{zhu2004real}. So, in this research, we chose to employ inertial tracking technology over optical tracking technology. At the time of this writing, we relied on Xsens MVN inertial tracking system \citep{Roetenberg2009}, a commercial platform of multiple inertial measurement units (IMUs) distrusted over the entire human body as shown in Figure.~\ref{fig:xsens-human}.

\begin{figure}[hbt!]
	\centering
	\includegraphics[scale=0.585]{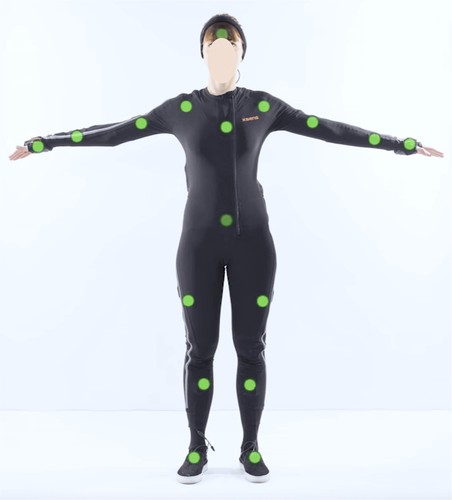}
	\caption{Xsens whole-body motion tracking suit with distributed IMUs present at the locations of the green dots}
	\label{fig:xsens-human}
\end{figure}

As discussed in the human modeling chapter \ref{cha:human-modeling}, the kinematic model considered by Xsens system is composed of 23 body segments. So, the motion capturing system provides, the pose of a body segment with respect to the inertial frame of reference. Additionally, one can also access the linear and angular velocity of a body segment, and the \emph{sensor acceleration} (Refer to Chapter 2 of \cite{traversaro2017thesis}). Although the joint configuration values are accessible, they are pertinent to the kinematic model considered by Xsens system. To reduce our reliance on the kinematic model of a commercial platform and have more control on our human model, we do not make use of the joint configuration data. Instead, we employ the motion measurement data in terms of link pose and velocity, and developed our inverse kinematic algorithms.

\section{Force-Torque Sensing}
\label{sec:force-torque-sensing}

Six axis force-torque (F/T) sensing is an important capability that has been exploited in robotics for regulating contact forces and torques~\cite{siciliano2012robot}, joint torque estimation and whole-body control of complex humanoid robots~\cite{nori2015}. Also, force-torque sensing is an invaluable tool in understanding human dynamics~\cite{riemer2008improving, skals2017prediction}. Strain gauge-based sensing is one of the most popular approaches. In the case industrial robots force-torque sensors are embedded at the end-effector and provides reliable measurements. However, in the case of complex humanoid robots, the reliability of the measurements deteriorate and procedures for quick and reliable calibration are highly important~\cite{chavez2016model}. Coming to the case human dynamics analysis, ground fixed force plates provide very reliable and accurate information of ground reaction forces. However, as they are ground fixed, they limit the mobility of the human subject. This limits the use for force plates for tasks that consider the human model to be a floating base system.

We rely on the six axis force-torque sensor technology\footnote{\href{http://wiki.icub.org/wiki/FT_sensor}{http://wiki.icub.org/wiki/FT\_sensor}} developed at the Istituto Italiano di Tecnologia (\textsc{iit}) shown in Fig.~\ref{fig:iCubFT}.

\begin{figure}[ht!]
	\centering
	\begin{subfigure}{0.19\textwidth}
		\includegraphics[width=\textwidth]{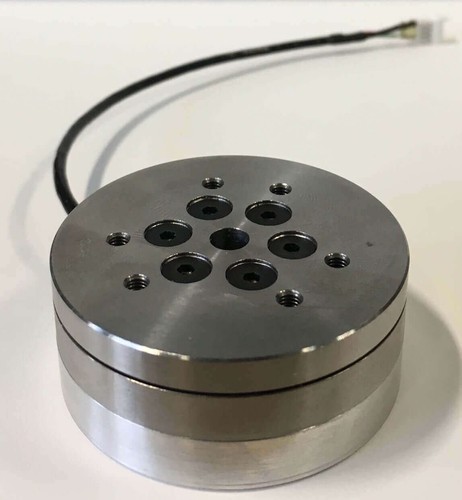}
		\caption{}
		\label{fig:iCubFT}
	\end{subfigure}
    \begin{subfigure}{0.19\textwidth}
    	\includegraphics[width=\textwidth]{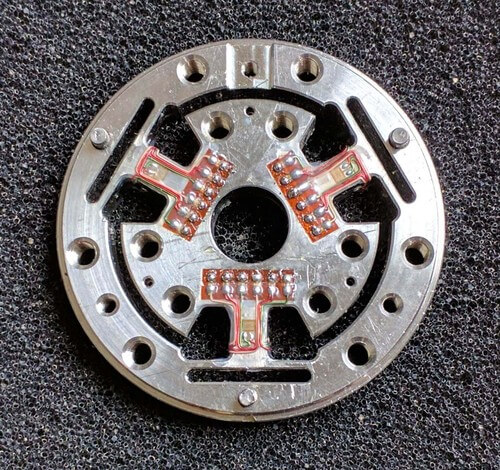}
    	\caption{}
    	\label{fig:ft_strain2_guage}
    \end{subfigure}
	\begin{subfigure}{0.19\textwidth}
		\includegraphics[width=\textwidth]{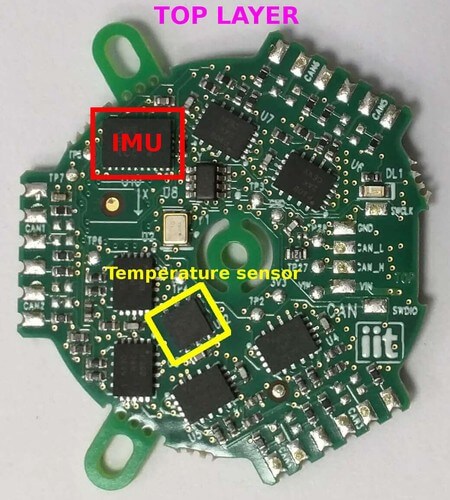}
		\caption{}
		\label{fig:ft_strain2_electronics_top}
	\end{subfigure}
 	\begin{subfigure}{0.19\textwidth}
 		\includegraphics[width=\textwidth]{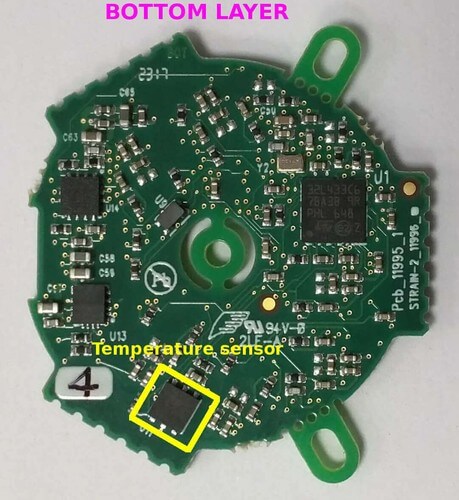}
 		\caption{}
 		\label{fig:ft_strain2_electronics_bottom}
 	\end{subfigure}
	\begin{subfigure}{0.19\textwidth}
		\includegraphics[width=\textwidth]{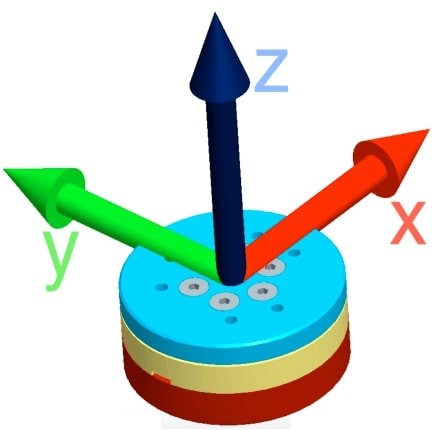}
		\caption{}
		\label{fig:FTref}
	\end{subfigure}
	\caption{(\subref{fig:iCubFT}) A six axis force-torque sensor developed at the Istituto Italiano di Tecnologia (\textsc{iit}); (\subref{fig:ft_strain2_guage}) Strain gauge; (\subref{fig:ft_strain2_electronics_top}) Top layer of internal electronics showing an inertial measurement unit (\textsc{imu}) and a temperature sensor;  (\subref{fig:ft_strain2_electronics_bottom}) Bottom layer of internal electronics showing a temperature sensor; (\subref{fig:FTref}) Sensor reference frame;}
\end{figure}

\section{Ground Reaction Force Monitoring}
\label{sec:HDE-ftShoes}

The six axis force-torque sensor technology presented in section~\ref{sec:force-torque-sensing} are small and provide reliable sensor measurements with proper calibration procedures~\cite{chavez2016model}. Leveraging this sensor technology, we developed sensorized wearable shoes as shown in Fig.~\ref{fig:Figs_footInShoe}. They will be referred to as \emph{ftShoes} and they do not pose the limitations of the force plates on human mobility. Each shoe is equipped with two force-torque sensors, one at the front and the other at the rear, as shown in Figure.~\ref{fig:Figs_footInShoe}. The sensor measurements from the front and the read sensors are combined and expressed with respect to the human foot frame that is shown in Figure.~\ref{fig:Figs_footInShoe}. Unlike the ground fixed force platform, the \emph{ftShoes} enable real-time ground reaction forces monitoring of human subjects while doing dynamic tasks like walking.

\begin{figure}[hbt!]	
	\centering
	\includegraphics[width=0.85\textwidth]{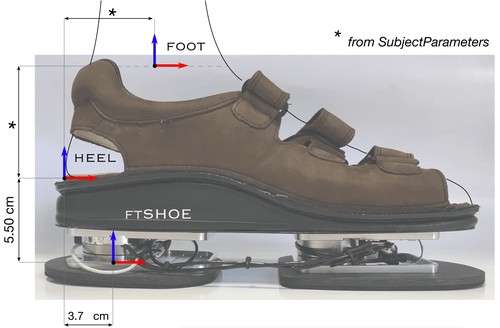}
	\caption{Sensorized wearable shoe with two force-torque sensors placed at the front and the rear, with frame references for measurements transformation from the sensors to the human foot}
	\label{fig:Figs_footInShoe}
\end{figure}

\section{iCub Humanoid Robot}
\label{sec:background-icub}

The iCub\footnote{http://www.icub.org/} humanoid robot is the main robotic platform used for experiments presented in this thesis. It is developed by the iCub Facility at the Italian Institute of Technology (\textsc{iit}). It is a child-sized humanoid robot originally developed by the RobotCub\footnote{http://www.robotcub.org/} European Project for research in embodied cognition \citep{sandini2014}. Since its initial release in 2006, the platform has been continuously updated with latest improvements and features both in terms of hardware and software. This section introduces the key details concerning the latest ``standard'' version of the iCub, informally referred hereafter as iCub 2.5, as of the beginning of 2017. 

The iCub platform consists of 53 degrees-of-freedom (\textsc{dof}) that are distributed as following: 3 for the head and 3 for the eyes; 7 for each arm and 9 for each hand; 3 for the torso and 6 for each leg. The actuation is achieved using Brushless DC electric motor (\textsc{bldc}) with an Harmonic Drive transmission, making them suitable for joint torque control. More details on the actuation and mechanics of the iCub 2.5 can be found in \citep{parmiggiani2012}. 

\begin{figure}[!ht]
	\centering
	\includegraphics[width=0.475\textwidth]{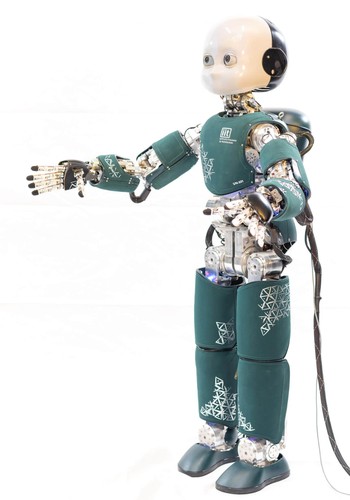}
	\caption{iCub humanoid robot (version 2.5)}
\end{figure}

Traditional robotic manipulators are fixed-base systems meaning that one of the known links of the robot is always in rigid contact with the environment. Alternatively, humanoid robots are floating-base systems \cite{Featherstone2007} meaning that the contact information is typically not known a-priori. It is a direct reflection of their design to facilitate anthropomorphic navigational capabilities. One of the research goals of the iCub project is 
to endow humanoids with advanced physical interaction capabilities. This is motivated by the idea that future robots will be required to physically interact with the environment and, in the long run, with humans or other robotic agents. A key requirement for this interaction control is the capability to measure and control the forces that a humanoid robot is exchanging at its contacts, i.e. contact forces control. While in traditional industrial applications this is achieved by placing a six axis force-torque sensor between the robot and the environment, this is not feasible in humanoid robots, in which the contact location is typically not known a-priori. To overcome this limitations, a unique set of dynamics-related sensors have been added during the years to the iCub. 

\begin{figure}[!ht]
	\centering
	\begin{subfigure}{0.33\textwidth}
		\centering
		\includegraphics[width=.75\linewidth]{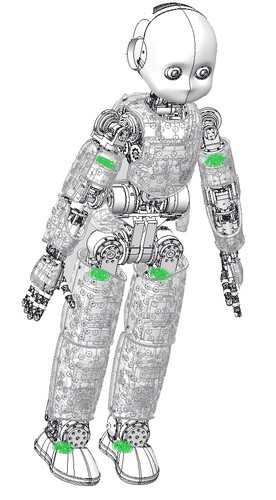}
		\caption{}
		\label{fig:icub-ftsens}
	\end{subfigure}%
	\begin{subfigure}{0.66\textwidth}
		\centering
		\includegraphics[width=.365\linewidth]{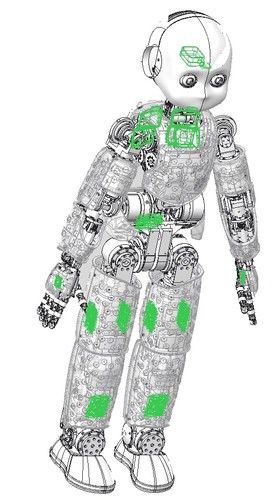}%
		\includegraphics[width=.365\linewidth]{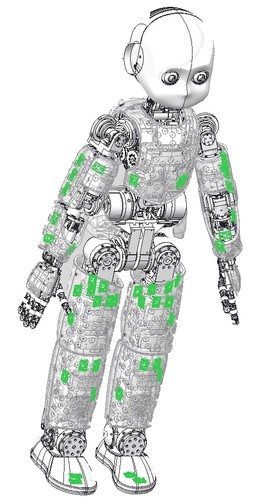}
		\caption{}
		\label{fig:distributedInertialSensing}
	\end{subfigure}
	\caption{Distribution of a) six embedded six axis force-torque sensors b) inertial sensors i.e. gyroscopes (left) and accelerometers (right) in iCub 2.5.}
	\label{fig:iCub-sensors}
\end{figure}

The main force sensors available on the iCub are six internal six axis force-torque sensors. Four of them are mounted at the base of each limb while  two of them are mounted in feet right below the ankles. The locations of these sensors are highlighted in Figure~\ref{fig:icub-ftsens}. iCub has a full-fledged Inertial Measurement Unit, equipped with a 3 DOFs magnetometer, accelerometer and gyroscope, that is mounted on the head of the robot. Additionally, several motor control boards are distributed in the robot structure, and each motor control board is equipped with both a 3 DOFs accelerometers and a 3 DOFs gyroscopes. Furthermore, some dedicated electronic boards to read distributed tactile sensors are equipped with a 3 Degree-of-Freedom (DOFs) accelerometers. These distributed inertial sensing is highlighted in Figure~\ref{fig:distributedInertialSensing}.

The availability of these distributed sensors, coupled with a novel modeling formalism and estimation algorithms \cite{traversaro2017thesis, traversaro2015situ, nori2015simultaneous, andrade2019six} enabled the development of whole-body controllers for the iCub humanoid robot capable of performing highly dynamic movements while balancing~\cite{pucci2016video,nori2015, dafarra2016torque}.

\part{Holistic Human Perception}
\chapter{Human Kinematics Estimation}
\label{cha:human-kinematics-estimation}

\chapreface{The problem of inverse kinematics is well addressed in robotics literature. This chapter presents a brief introduction to the problem of inverse kinematics in robotics and highlights the role of inverse kinematics in the context of real-time human motion tracking through the kinematics estimation of a human wearing a whole-body inertial motion tracking system. A description of two instantaneous optimization approaches and a dynamical optimization approach to real-time human motion tracking is presented with a benchmarking in terms of the computation time.}

Kinematics deals with the description of the motion of a multi rigid-body system without considering the influence causing the motion or the inertial parameters of rigid bodies. A multi rigid-body system with a set of rigid bodies connected without any loops is referred as a \emph{serial kinematic chain}. Robotic manipulators as shown in Figure.~\ref{fig:panda-ee-example} are an ideal example of a serial kinematic chain. One of the links of the system is referred as an \emph{End-Effector} that is present at the end of the kinematic chain. 

\vspace{0.1cm}
\begin{figure}[H]
    \centering
    \includegraphics[width=0.5\textwidth]{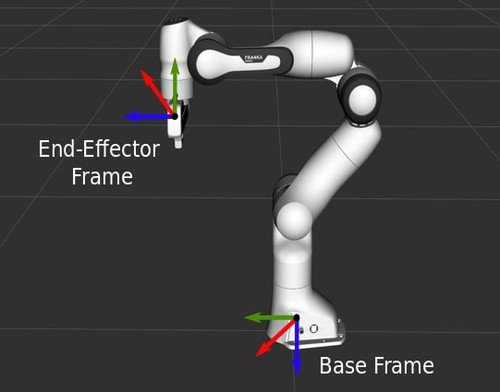}
    \caption{ Example of a robotic manipulator}
    \label{fig:panda-ee-example}
\end{figure}

End-effectors are designed to interact with the environment. So, the knowledge of the end-effector pose with respect to the base frame is crucial for realizing any task with the manipulator. Given the values of the joint configuration i.e., the topology $\comVar{s}$ as indicated in Eq.~\ref{topology_s}, and the knowledge of the kinematic chain model, one can compute the end-effector pose with respect to the base. This is known as the \emph{Forward Kinematics} (\textsc{fk}) problem. However, it is more intuitive to specify an interested pose for the end-effector in the Cartesian space and then compute the joint configuration that enables the end-effector placement correctly. This is referred as \emph{Inverse Kinematics} (\textsc{ik}) problem. As an example, the desired pose of the end-effector along with the desired pose of the rest of the robot links is indicated in \emph{orange} color in Figure.~\ref{fig:panda-ik-example}. An example inverse kinematics solution of a series of joint configurations to realize the desired end-effector pose is highlighted in Figure.~\ref{fig:panda-ik-solution-example}.

\begin{figure}[H]
    \centering
    \begin{subfigure}{0.49\textwidth}
         \centering
         \includegraphics[width=\textwidth]{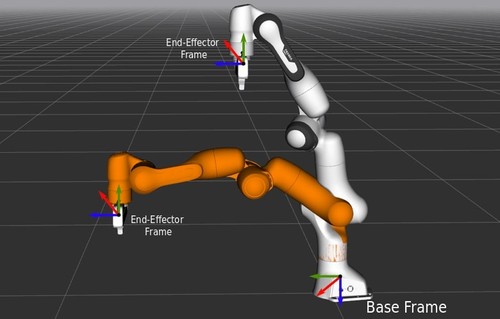}
         \caption{}
         \label{fig:panda-ik-example}
    \end{subfigure}
    \begin{subfigure}{0.49\textwidth}
        \centering
        \includegraphics[width=0.825\textwidth]{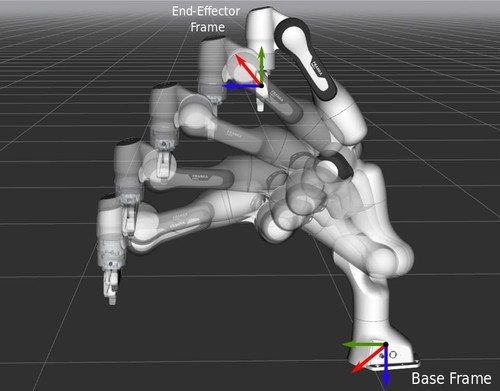}
        \caption{}
        \label{fig:panda-ik-solution-example}
    \end{subfigure}
    \caption{Inverse kinematics example of a robotic manipulator}
    \label{fig:panda-ik}
\end{figure}

Investigations into inverse kinematics problem first originated in robotic community and significant research effort was put towards finding efficient algorithms that provide real-time performance. Apart from robotics community, \textsc{ik} is also a critical component in graphics community for animating articulated models \citep{aristidou2018inverse} or for real-time human motion tracking \citep{pons2011outdoor}. Analytical methods provide an exact solution and are usually faster to compute closed-form solution for mechanisms of lower degrees of freedom like robotic manipulators \citep{craig2009introduction}. However, when the considered model is highly redundant, as our human model with 66 DoFs, application of analytical methods is not feasible. In the subsequent sections we refer to inverse kinematics only in the context of human motion tracking.

\section{Inverse Kinematics Problem Definition}

This section presents the inverse kinematic problem in the context of real-time human motion tracking. A \emph{posture} is defined as a particular configuration of human body parts that results from a particular set of human joints configuration. A human subject wearing a whole-body motion tracking suit moves various joints to attain a series of postures depending on the task. For example, when the human subject moves from an anatomically neutral \emph{N-pose} to \emph{T-pose}, there is a change in the configuration of the shoulder joints as shown in Figure.~\ref{fig:human-posture-change-t-pose}. Similarly, when moving from N-pose to a squatting position, there is a change in the configuration of the knee and ankle joints as shown in Figure.~\ref{fig:human-posture-change-squat-front}.

\begin{figure}[H]
    \centering
    \begin{subfigure}{0.49\textwidth}
        \centering
        \includegraphics[width=\textwidth]{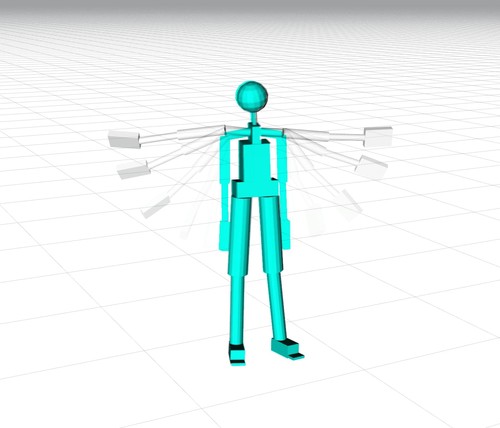}
        \caption{Transition from N-pose to T-pose}
        \label{fig:human-posture-change-t-pose}
    \end{subfigure}
    \begin{subfigure}{0.49\textwidth}
        \centering
        \includegraphics[width=\textwidth]{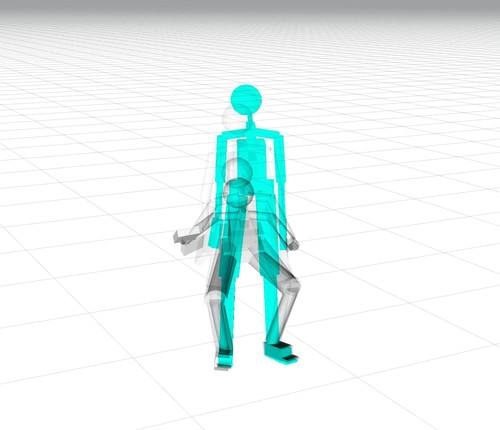}
        \caption{Transition from N-pose to Squat}
        \label{fig:human-posture-change-squat-front}
    \end{subfigure}
    \caption{Example of a series of human postures during different tasks}
    \label{fig:human-posture-changes}
\end{figure}

The motion tracking system provides position and orientation measurements of various body segments in real-time. These measurements are expressed with respect to the inertial frame of reference. Now, the inverse kinematics problem is defined as, given the position and orientation measurements of various body segments, we need to compute the joint configuration of the human.

The set of frames $\mathbb{L} = \{ \mathcal{L}_0, \mathcal{L}_1, .... \mathcal{L}_{N_B-1}\}$ correspond to the number of body segments (links) considered in the human model with $N_B = 23$. The human model is considered as a floating base system i.e. $\mathcal{L}_1 = \mathcal{F}$. Given a system configuration at a time instant $t$, $\comVar{q}(t) = (\comVar{q}_{\mathcal{B}}(t),\comVar{s}(t))$, the Cartesian pose of a link with respect to the inertial frame of reference can be computed as a combination of a series of transformations obtained through the kinematic equations of the model. This geometric mapping through a series of transformations constitute the forward kinematics that is denoted as a function,

\begin{equation} \label{eq:human-forward-kinematics}
    ^{\mathcal{I}}\comVar{H}_{\mathcal{L}_i}(t) = f(\comVar{q}(t)) \ \ \forall \ \ i = 1....,N_B-1
\end{equation}

The positions of all the links is put together as a vector denoted by,

\begin{equation}
    \comVar{p}_{\mathbb{L}}  = 
    \begin{bmatrix}
    ^{\mathcal{I}}{\comVar{p}}_{\mathcal{F}} \\
    ^{\mathcal{I}}{\comVar{p}}_{\mathcal{L}_1} \\
    ^{\mathcal{I}}{\comVar{p}}_{\mathcal{L}_2} \\
    ... \\
    ... \\
    ^{\mathcal{I}}{\comVar{p}}_{\mathcal{L}_{N_B-1}}
    \end{bmatrix} \in \mathbb{R}^{3 N_B}
\end{equation}

The rotations of all the links is put together as a matrix denoted by,

\begin{equation}
    \comVar{R}_{\mathbb{L}}  = 
    \begin{bmatrix}
    ^{\mathcal{I}}{\comVar{R}}_{\mathcal{F}} \\
    ^{\mathcal{I}}{\comVar{R}}_{\mathcal{L}_1} \\
    ^{\mathcal{I}}{\comVar{R}}_{\mathcal{L}_2} \\
    ... \\
    ... \\
    ^{\mathcal{I}}{\comVar{R}}_{\mathcal{L}_{{N_B}-1}}
    \end{bmatrix} \in \mathbb{R}^{3 N_B \times 3}
\end{equation}

The posture change of the human results in a new set of measurements for the position and orientation of the links, say at a time instant $t+1$. The new measurements are referred as \emph{targets}. The vector of position targets is denoted as,

\begin{equation}
    \comVar{p}_{\mathbb{L}}^d  = 
    \begin{bmatrix}
    ^{\mathcal{I}}{\comVar{p}}_{\mathcal{F}}^d \\
    ^{\mathcal{I}}{\comVar{p}}_{\mathcal{L}_1}^d \\
    ^{\mathcal{I}}{\comVar{p}}_{\mathcal{L}_2}^d \\
    ... \\
    ... \\
    ^{\mathcal{I}}{\comVar{p}}_{\mathcal{L}_{N_B - 1}}^d
    \end{bmatrix} \in \mathbb{R}^{3 N_B}
\end{equation}

Similarly, the vector of rotational targets is denoted as,

\begin{equation}
    \comVar{R}_{\mathbb{L}}^d  = 
    \begin{bmatrix}
    ^{\mathcal{I}}{\comVar{R}}_{\mathcal{F}}^d \\
    ^{\mathcal{I}}{\comVar{R}}_{\mathcal{L}_1}^d \\
    ^{\mathcal{I}}{\comVar{R}}_{\mathcal{L}_2}^d \\
    ... \\
    ... \\
    ^{\mathcal{I}}{\comVar{R}}_{\mathcal{L}_{N_B-1}}^d
    \end{bmatrix} \in \mathbb{R}^{3 N_B \times 3}
\end{equation}

A distance measure between the current link positions and target position is defined in terms of \emph{Euclidean} distance as,

\begin{equation}
    d_{pos} (\comVar{p}_{\mathbb{L}}^d, \comVar{p}_{\mathbb{L}} ) = 
    \begin{bmatrix}
        ^{\mathcal{I}}\comVar{p}_{\mathcal{F}}^d - ^{\mathcal{I}}\comVar{p}_{\mathcal{F}} \\
        ^{\mathcal{I}}\comVar{p}_{\mathcal{L}_1}^d - ^{\mathcal{I}}\comVar{p}_{\mathcal{L}_1} \\
        ^{\mathcal{I}}\comVar{p}_{\mathcal{L}_2}^d - ^{\mathcal{I}}\comVar{p}_{\mathcal{L}_2} \\
        ... \\
        ... \\
        ^{\mathcal{I}}\comVar{p}_{\mathcal{L}_{N_B-1}}^d - ^{\mathcal{I}}\comVar{p}_{\mathcal{L}_{N_B-1}}
    \end{bmatrix}
\end{equation}

\begin{tcolorbox}[sharp corners, colback=white!30,
     colframe=white!20!black!30, 
     title=Position Targets Remark]
\emph{The human model considered in this work consists of only rotational joints. So, the only relevant position target corresponds to the base frame $\mathcal{F}$.}
\end{tcolorbox}

As orientation is parametrized in terms of rotation matrix, the distance measure is defined using the $ \bm \skewOp(.)^{\vee}$ operator denoted as,

\begin{equation}
    d_{ori} (\comVar{R}_{\mathbb{L}}^d, \comVar{R}_{\mathbb{L}} ) =
    \begin{bmatrix}
         \bm \skewOp((^{\mathcal{I}}\comVar{R}_{\mathcal{F}}^d)^T \ ^{\mathcal{I}}\comVar{R}_{\mathcal{F}})^{\vee} \\
         \bm \skewOp((^{\mathcal{I}}\comVar{R}_{\mathcal{L}_1}^d)^T \ ^{\mathcal{I}}\comVar{R}_{\mathcal{L}_1})^{\vee} \\
         \bm \skewOp((^{\mathcal{I}}\comVar{R}_{\mathcal{L}_2}^d)^T \ ^{\mathcal{I}}\comVar{R}_{\mathcal{L}_2})^{\vee} \\
         ... \\
         ... \\
         \bm \skewOp((^{\mathcal{I}}\comVar{R}_{\mathcal{L}_{N_B - 1}}^d)^T \ ^{\mathcal{I}}\comVar{R}_{\mathcal{L}_{N_B - 1}})^{\vee}
    \end{bmatrix}
\end{equation}

Now, the solution of the inverse kinematics problem is the joint configuration change that resulted in the posture change of human during a task. The following sections present different ways of formulating the inverse kinematics problem for real-time human motion tracking.

\section{Instantaneous Optimization}

Unlike analytical methods, \emph{iterative methods} are ideal for systems with high degrees of freedom. One of the most common approaches is to formulate the inverse kinematics problem as a non-linear optimization problem and solve via iterative algorithmss~\cite{goldenberg1985,buss2004}.  This class of algorithms, referred to as \textit{instantaneous optimization}, aims to converge to a stable solution for each time-step $t_k$. A general formulation of the optimization problem is defined as following:

\begin{subequations}\label{nonlinear_optimization_inverse_kinematics}
\begin{align}
& \underset{\comVar{q}(t_k)}{\text{minimize}} & \left\lVert \comVar{K}_{pos} \ d_{pos} (\comVar{p}_{\mathbb{L}}^d, \comVar{p}_{\mathbb{L}} ) +  \comVar{K}_{ori} \ d_{ori} (\comVar{R}_{\mathbb{L}}^d, \comVar{R}_{\mathbb{L}} ) \right\lVert \\
& \text{subject to} 
& \comVar{A}  \comVar{q}(t_k) \leq \comVar{b}
\end{align}
\end{subequations}

where $\comVar{K}_{pos}$ and $\comVar{K}_{ori}$ are diagonal weight matrices for position and orientation targets respectively, $\comVar{A}$ and $\comVar{b}$ are two parameters that represent the limits for the joint configuration of the model. A multibody kinematics library iDynTree~\cite{nori2015}, and an open-source non-linear optimization software library called \textsc{ipopt}~\cite{Wachter2006} are used to implement two different inverse kinematics problems. The first implementation, referred as \emph{whole-body instantaneous optimization}, considers the inverse kinematics problem of the full human model as a single non-linear optimization problem i.e., with $23$ body segments and $66$ DoFs. The second implementation, referred as \emph{pair-wise instantaneous optimization}, decomposes the human model into multiple sub models with each model containing a pair of links. So, inverse kinematics of each pair is solved in parallel and combined to get the final joint configuration of the full human model.

\section{Dynamical Optimization}
\label{sec:dynamical-optimization}

Real-time human motion tracking is a time critical application and in our experience we observed that instantaneous optimization algorithms perform poorly for highly dynamic movements like running. So, we resort to an alternative approach of formulating the inverse kinematics as a control problem~\cite{sciavicco1988}. This work is originally presented in~\citep{Rapetti2019}\footnote{Rapetti, L., \underline{ Tirupachuri, Y.}, Darvish, K., Latella, C., \& Pucci, D. (2020). Model-Based Real-Time Motion Tracking using Dynamical Inverse Kinematics. IEEE IROS 2020 under review arXiv:1909.07669.}. In this approach we consider the state of the system to be composed of both the system configuration $\comVar{q}$ and the system velocity $\nu$. The pose targets are grouped into a vector denoted as,

\begin{equation}
	\comVar{x} = 
	\begin{bmatrix}
		^{\mathcal{I}}{\comVar{p}}_{\mathcal{F}}^d \\
		^{\mathcal{I}}{\comVar{p}}_{\mathcal{L}_1}^d \\
		... \\
		^{\mathcal{I}}{\comVar{p}}_{\mathcal{L}_{N_B - 1}}^d \\
		^{\mathcal{I}}{\comVar{R}}_{\mathcal{F}}^d \\
		^{\mathcal{I}}{\comVar{R}}_{\mathcal{L}_1}^d \\
		... \\
		^{\mathcal{I}}{\comVar{R}}_{\mathcal{L}_{N_B - 1}}^d
	\end{bmatrix}
\end{equation}

The pose distance vectors are grouped as a pose residual vector denoted as,

\begin{equation}
	\comVar{r} = 
	\begin{bmatrix}
		d_{pos} \\
		d_{ori}
	\end{bmatrix}
\end{equation}

The velocity of the $i$-th link $^{\mathcal{I}}\comVar{v}_{\mathcal{L}_i} = (^{\mathcal{I}}\dot{\comVar{p}}_{\mathcal{L}_i}, ^{\mathcal{I}}\comVar{\omega}_{\mathcal{L}_i})$ is related to the system velocity $\nu$ through the Jacobian matrix $\comVar{J}_{\mathcal{L}_i}(\comVar{q}) \in \mathbb{R}^{6 \times (n + 6)}$ which is dependent on the system configuration $\comVar{q}$ through the following relation,

\begin{equation}
	^{\mathcal{I}}\comVar{v}_{\mathcal{L}_i} = \comVar{J}_{\mathcal{L}_i}(\comVar{q}) \ \nu
\end{equation}

Similar to the pose targets, the velocity targets are grouped as a vector denoted by,

\begin{equation}
	\comVar{v} = 
	\begin{bmatrix}
		^{\mathcal{I}}{\dot{\comVar{p}}}_{\mathcal{F}}^d \\
		^{\mathcal{I}}{\dot{\comVar{p}}}_{\mathcal{L}_1}^d \\
		... \\
		^{\mathcal{I}}{\dot{\comVar{p}}}_{\mathcal{L}_{N_B - 1}}^d \\
		^{\mathcal{I}}\comVar{\omega}_{\mathcal{F}} \\
		^{\mathcal{I}}\comVar{\omega}_{\mathcal{L}_1} \\
		... \\
		^{\mathcal{I}}\comVar{\omega}_{\mathcal{L}_{N_B - 1}}
	\end{bmatrix}
\end{equation}

The current link velocities are denoted as $\comVar{J}(\comVar{q}) \ \nu$, where

\begin{equation}
	\comVar{J}(\comVar{q}) =
	\begin{bmatrix}
		\comVar{J}_{\mathcal{F}}(\comVar{q}) \\
		\comVar{J}_{\mathcal{L}_1}(\comVar{q}) \\
		\comVar{J}_{\mathcal{L}_2}(\comVar{q}) \\
		... \\
		... \\
		\comVar{J}_{\mathcal{L}_{N_B - 1}}(\comVar{q})
	\end{bmatrix}
\end{equation}

Now, the velocity residual vector is denoted as,

\begin{equation}
	\comVar{u} = \comVar{v} - \comVar{J}(\comVar{q}) \ \nu
\end{equation}

Given the residual vectors $\comVar{u}, \comVar{r}$, we consider the dynamic system,

\begin{equation}
	\comVar{u} + \comVar{K} \ \comVar{r} = 0
\end{equation}

where, $\comVar{K} \in \mathbb{R}^{3 N_B \times 3 N_B}$ is a positive definite diagonal matrix. The system velocity $\nu$ is considered as a control input to drive the residual vectors towards zero. This approach is referred as \emph{dynamical optimization} as the system configuration and velocity is controlled to dynamically converge to the targets. Fig.~\ref{fig:dynamical-optimization} presents the dynamical optimization algorithm. In contrast to the instantaneous optimization approach, here the solution is computed in a single iteration. So, the absence of iterations make this approach faster for solving the whole-body inverse kinematics for the human model.

\begin{figure}[!ht]
	\centering
	\includegraphics[width=\textwidth]{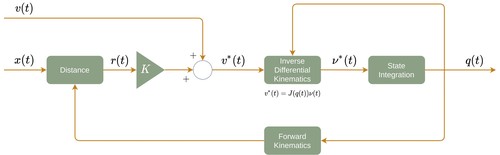}
	\caption{Dynamical optimization for real-time inverse kinematics}
	\label{fig:dynamical-optimization}
\end{figure}

\section{Benchmarking}

The proposed approaches of solving inverse kinematics for real-time human motion tracking are benchmarked in terms of \emph{computational time} using three different dynamic tasks:

\begin{itemize}
    \item \textbf{T-pose} : The human subject standing on two feet moves the arms parallel to the ground
    \item \textbf{Walking} : The human subject walks on a treadmill at a constant speed of $\SI{4}{\kilo\meter\per\hour}$
    \item \textbf{Running} : The human subject runs on a treadmill at a constant speed of $\SI{10}{\kilo\meter\per\hour}$
\end{itemize}

The motion data is acquired with the Xsens Awinda wearable suit~\cite{Roetenberg2009} with wireless \textsc{imu} units distributed through out the human body. The motion capturing system provides the pose and velocity of the 23 body segments and are associated with the links of our human models at a sampling rate of $\SI{60}{\hertz}$. Furthermore, we considered the full human model with 66 DoFs referred as \emph{Human66} and also the reduced human model with 48 DoFs referred as \emph{Human48}. The motion data is streamed through \textsc{yarp} middleware~\cite{yarp2006} that facilitates recording and real-time playback of data. Box plots of computational time for the proposed inverse kinematics approaches with two different human models is highlighted in Figure.~\ref{fig:ik-computational-time}. It is evident that the dynamical optimization approach is the best way that solves the inverse kinematics of different human models across different tasks when compared with the instantaneous optimization approaches. Furthermore, the computational time of the dynamic approaches increase as the dynamicity of the task increases. So, the task of T-pose has the lowest computational time followed by walking and then running.

Between the two instantaneous approaches, pair-wise instantaneous optimization seems to perform poorly for the T-pose task. This is explained by the additional time taken for parallel computation of small \textsc{ik} problems and combining their output. However, for the tasks of high dynamicity like walking and running, the computational time of pair-wise instantaneous optimization is lower than the whole-body instantaneous optimization.

\begin{figure}[H]
	\includegraphics[width=\textwidth]{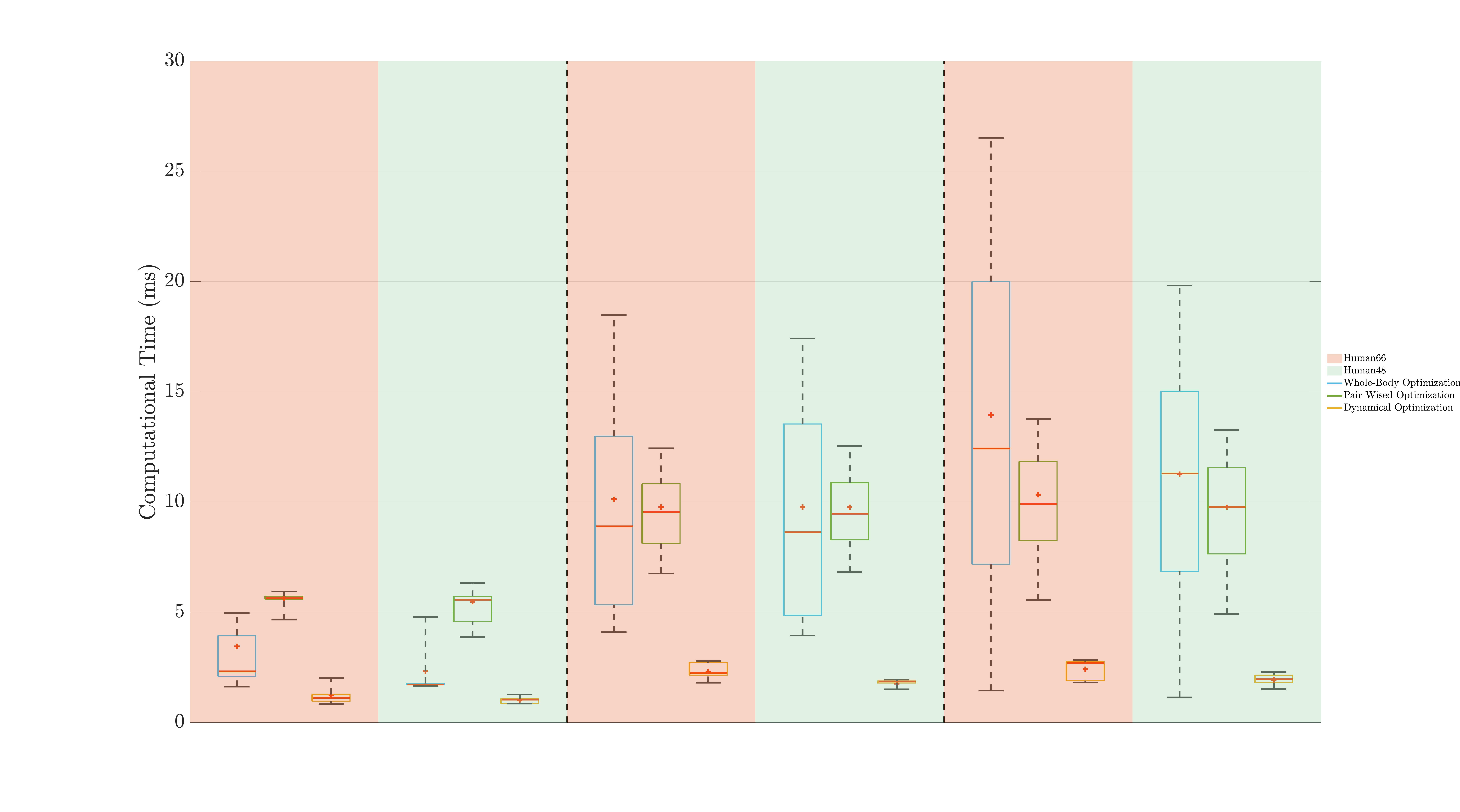}
	\caption{Box plots of computational time for the proposed inverse kinematics approaches with two different human models}
	\label{fig:ik-computational-time}
\end{figure}

\chapter{Human Dynamics Estimation}
\label{cha:human-dynamics-estimation}

\chapreface{The knowledge of joint torques of a human body serves useful in various domains like ergonomic analysis, biomechanical studies, collaborative robot control such as exoskeletons. This chapter presents our stochastic framework for human dynamics estimation to perform a sensorless external force estimation on different links of the human. Furthermore, an experimental validation of our framework for a floating base fully articulated human model performing different dynamic tasks is presented.}

The dynamics of a mechanical system describes the motion of the rigid bodies under external influence with the inertial parameters considered. In robotics literature, the two main problems concerning the dynamics are:

\begin{itemize}
    \item \emph{Forward Dynamics (FD)}: Given the actuation forces for joints of the robot, forward dynamics finds the resulting accelerations and motion.
    \item \emph{Inverse Dynamics (ID)}: Give the desired motion of the robot, inverse dynamics finds the actuation forces for the joints that will induced the desired motion.
\end{itemize}

Unlike in robotics, human motion is generated through the joint articulation facilitated by sophisticated musculoskeletal structure. Internal body forces and joint torques are crucial in understanding the human dynamics. Although our model lacks the sophistication of biomechanical models, it enables us to model the joint torques i.e., articular stress. In the context of human dynamics, the inverse dynamics problem deals with estimation of the joint torques of human.

\section{Recourse on Stochastic Dynamics Estimation}
\label{sec:recourse-on-stochastic-dynamics-estimation}

A stochastic formulation of the inverse dynamics estimation problem for a whole-body fixed base human model is first investigated in \cite{Latella2019, latella2018thesis} and later extended to a floating base human model in our recent work~\cite{latella2019simultaneous}. The central idea is to represent the equations of motion as a set of linear equations by defining a new human dynamic variables vector $\comVar{d}$ and the measurement vector $\comVar{y}$. The relation between the dynamic variables $\comVar{d}$ and the model is encapsulated in a model constrains equation. The measurements equations represents the relation between the dynamic variables $\comVar{d}$ and the measurements vector $\comVar{y}$.

\subsection{Model Constraints Equation}
\label{sec:model-constraints-equation}
 
Considering the human model described in Chapter~\ref{cha:human-modeling} with $N_B$ rigid bodies and $n$ internal DoFs, the dynamic variables vector is constructed as,
 
\begin{eqnarray} \label{eq:dynvec}
 \comVar{d} &=& 
 \begin{bmatrix}
 \comVar{d}_{link}^\top & \comVar{d}_{joint}^\top
 \end{bmatrix}^\top
 \in \mathbb R^{12N_B + 7n} \label{d}
\end{eqnarray}
 
\begin{subequations}
 \begin{eqnarray}
 	\comVar{d}_{link} &=& \begin{bmatrix}
 	\comVar{\alpha}_{0}^g & {\comVar{\underline f}_{0}^x} &
 	\comVar{\alpha}_{1}^g & {\comVar{\underline f}_{1}^x} &
 	\hdots &
 	\comVar{\alpha}_{N_B-1}^g & {\comVar{\underline f}_{N_B-1}^x}
 	\end{bmatrix}
 	\in \mathbb R^{12N_B} \label{d_L}\\
 	\comVar{d}_{joint} &=& \begin{bmatrix}
 	\comVar{\underline f}_{1} & \comVar{\underline f}_{2} & \hdots & \comVar{\underline f}_{n} &
 	\ddot {s}_{1} & \ddot {s}_{2} & \hdots &
 	\ddot {s}_{n}
 	\end{bmatrix}
 	\in \mathbb R^{7n} \label{d_J}
 \end{eqnarray}
\end{subequations}
 
where, $\comVar{\alpha}^g \in \mathbb R^{6}$ is the proper sensor acceleration from the IMUs placed on human links \cite{latella2019simultaneous}. $\comVar{\underline f}_{i}^x$ denotes the external 6D force on the $i$-th link, $\comVar{\underline f}_{i}$ denotes the 6D force exchanged through the $i$-th joint, and $\ddot{s}_i$ denotes the acceleration of the $i$-th joint. The joint torques $\comVar{\tau}$ are retrieved by projecting the joint force and moments on the joint motion subspace as defined in Eq.~\eqref{eq:taui}. Given the definition of $\comVar{d}$, the Newton-Euler equations of motion Eq.\eqref{eq:ai}-\eqref{eq:taui} are rearranged in a matrix form, presenting a set of model constraints of the mechanical system given as,

\begin{equation} \label{eq:matRNEA} 
    \comVar{D}(\comVar{q}, \nu) \comVar{d} + \comVar{b}_D
 (\comVar{q}, \nu)= \comVar{0}
\end{equation}

where $\comVar{D}$ is a block matrix $\in \mathbb R^{(18 N_B+n) \times (12N_B+7n)}$ and $\comVar{b}_D$ is a vector $\in \mathbb R^{18 N_B+n}$ (refer to Chapter 4 of \cite{latella2018thesis}).

\subsection{Measurements Equation}
\label{sec:measurements-equations}

Consider that the human is equipped with $N_s$ number of sensors. Typically senors that are used for human dynamics estimation include inertial measurement units that provide the link acceleration or force-torque sensors that provide external forces acting on certain links of the human. The measurement equation is formulated as,

\begin{equation} \label{eq:measEquation}
     \comVar{Y}(\comVar{q}, \nu) \comVar{d} + \comVar{b}_Y (\comVar{q}, \nu)= \comVar{y}
\end{equation}

where, $\comVar{Y} \in \mathbb{R}^{N_s \times (12N_B + 7n)}$ is a block matrix defined as,

\begin{eqnarray}
    \comVar{Y} =
    \begin{bmatrix}
        \comVar{Y}_1 & \comVar{Y}_2 \quad \hdots \quad \comVar{Y}_{N_S}
    \end{bmatrix}^\top \in \mathbb R^{N_S \times (12N_B + 7n)}
\end{eqnarray}

and the bias vector $\comVar{b}_Y$ is defined as,

\begin{eqnarray}
    \comVar{b}_Y =
    \begin{bmatrix}
        \comVar{b}_{Y_1} & \comVar{b}_{Y_2} \quad \hdots \quad \comVar{b}_{Y_{N_S}}
    \end{bmatrix}^\top \in \mathbb R^{N_S}
\end{eqnarray}

The dimensions depend on the type of the sensor e.g., in the case of an IMU the size of sensor reading is $3$ for the linear acceleration and in the case of a force-torque sensor, it is $6$ for the six axis forces and torques. An illustrative example of a simple three link 2 DoF mechanical system with and IMU and external 6D force measurements is presented in Chapter 4 of \cite{latella2018thesis}.

\subsection{Stochastic Estimation}
\label{sec:HDE-stochastic-estimation}

On putting together the model constraints equation Eq.~\eqref{eq:matRNEA} and the measurements equation Eq.~\ref{eq:measEquation}, we obtain

\begin{eqnarray} \label{eq:systemEq}
    \begin{bmatrix}
    	\comVar{D}(\comVar{q}, \nu) \\ \comVar{Y}(\comVar{q}, \nu) 
     \end{bmatrix} \comVar{d} +
     \begin{bmatrix} 
     	\comVar{b}_D(\comVar{q}, \nu) \\
     	\comVar{b}_Y(\comVar{q}, \nu)
    \end{bmatrix} =
    \begin{bmatrix}  
    	\comVar{0} \\
     	\comVar{y}
     \end{bmatrix} \qquad
\end{eqnarray}

The central assumption is that the the number of available measurements in the vector $\comVar{y}$ ensure enough constraints on estimating the dynamic variables vector $\comVar{d}$. This condition is represented through the rank assumption as,

\begin{eqnarray}\label{rankAssumption}
rank\left(\begin{bmatrix}
	\comVar{D}(\comVar{q}, \nu) \\
    \comVar{Y}(\comVar{q}, \nu)  
     \end{bmatrix}\right) = 12N_B + 7n
\end{eqnarray}

Give these set of equations, under the assumptions of all the model constraints having equal relevance and all the sensor measurements having equal accuracy, one way to solve the system represented by Eq.~\eqref{eq:systemEq} is the least squares method using a Moore-Penrose pseudoinverse. Alternatively, a more pragmatic approach is to discard the equal relevance and accuracy assumptions and employ a weighted pseudoinverse to find a weighted least-squares solution. However, for a system of high dimension like the human considered here, it is a tedious task to find the proper weights.

We took an alternative approach and consider the dynamic variables $\comVar{d}$ and the measurements $\comVar{y}$ as stochastic variables with Gaussian distributions. The Maximum A-Posteriori \textsc{(map)} estimator tool is employed to estimate the dynamics variables $\comVar{d}$. The problem definition can be summarized as,
\vspace{0.5cm}
\begin{tcolorbox}[sharp corners, colback=white!30,
     colframe=white!20!black!30, 
     title=Stochastic Estimation Problem]
\emph{Given the measurements $\comVar{y}$ and the prior knowledge of the model
 constraints, estimate the vector $\comVar{d}$ by maximizing the conditional
  probability distribution of $\comVar{d}$ given $\comVar{y}$, such that}
\begin{equation}\label{d_map}
   \bm d^{\mbox{\tiny{MAP}}}=\arg \max_{\bm d}~p(\bm d| \bm y) \propto \arg
\max_{\bm d}~p(\bm d, \bm y)
\end{equation}
\end{tcolorbox}

The \textsc{map} solution coincides with the mean of the probability density function \textsc{(pdf)} $p(\comVar{d}|\comVar{y})$ yielding to:
\begin{eqnarray} \label{eq:d_gaussian}
    \comVar{d}^{\mbox{\tiny{MAP}}} = \comVar{\mu}_{d|y}
\end{eqnarray} 

The reader is advised to refer Appendix~\ref{app:appendix-HumanDynamicsEstimation} for the complete mathematical details in arriving at the \textsc{map} solution. In terms of optimization, the estimation problem can be presented as following,

\begin{equation}
	\comVar{d}^{\mbox{\tiny{MAP}}} = \arg\min_d 
	( ||\comVar{D} \comVar{d} + \comVar{b}_D||^2_{\comVar{\Sigma}_D^{-1} } + 
	  ||\comVar{Y} \comVar{d} + \comVar{b}_Y - \comVar{y}||^2_{\comVar{\Sigma}_y^{-1} } + 
	  || \comVar{d} - \comVar{\mu}_d||^2_{\comVar{\Sigma}_d^{-1} })
	  \label{eq:map-solution-optimization-problem}
\end{equation}

\subsection{Covariance Tuning}
\label{sec:HDE-covariance-tuning}

The solution to the estimation problem from eq.~\eqref{eq:map-solution-optimization-problem} is influenced by three covariances: $i)$ Prior covariance ${\comVar{\Sigma}}_d$, $ii)$ Measurement covariance ${\comVar{\Sigma}}_y$, and $iii)$ Model covariance ${\comVar{\Sigma}}_D$. The effect of these covariances on the estimator can be summarized as following:

\begin{itemize}
	\item Prior covariance ${\comVar{\Sigma}}_d$ enables the estimator to change the dynamic variable estimates. The prior mean $\comVar{\mu}_d$ for the dynamic variables is set to zero as we do not assume any prior information
	\item Measurement covariance ${\comVar{\Sigma}}_y$ influences the estimator to push some of the dynamic variable estimates towards the sensor measurements in $\comVar{y}$
	\item Model covariance ${\comVar{\Sigma}}_D$ guides the estimator to ensure all the dynamic variable estimates to be consistent with the model constraints
\end{itemize}

The combined contribution of this set of covariances affects the final estimation of the dynamic variables.

\section{Sensorless External Force Estimation}
\label{sec:sensorless-external-force-estimation}

Currently, our stochastic formulation of estimating the dynamic variables considers the external 6D force measurements (inside the measurement vector $\comVar{y}$) obtained from the sensorized \emph{ftShoes} described in section~\ref{sec:HDE-ftShoes}. Apart from the feet links on which ground reaction forces are present, we assume all the other links do not have any external 6D force on them. So, the external 6D force measurement values on all the other links are set to $\comVar{0}_6$. The measurement covariance values ${\bm {\Sigma}}_y$ for all the external 6D forces are set to a low value ($10^{-6}$) suggesting that the measurements are highly reliable. This choice implies that the estimation algorithm is forced to keep the feet external 6D force estimates to be as close as possible to the measurements coming from the \emph{ftShoes}, while the estimation of the external 6D force estimates on all the other links are towards $\comVar{0}_6$. 

\begin{figure}[ht!]
	\centering
	\begin{subfigure}[b]{0.5\textwidth}
		\includegraphics[width=0.925\textwidth]{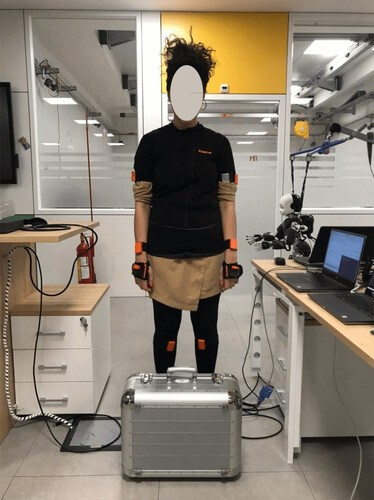}
		\caption{}
		\label{fig:subject_without_weight}
	\end{subfigure}%
	\begin{subfigure}[b]{0.5\textwidth}
		\includegraphics[width=0.925\textwidth]{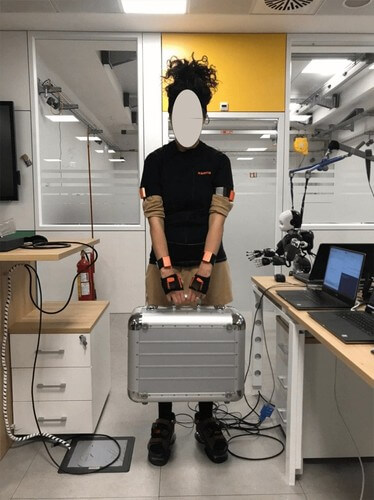}
		\caption{}
		\label{fig:subject_with_weight}
	\end{subfigure}
	\caption{Human subject lifting an external heavy object}
	\label{fig:subject_weight_lefting}
\end{figure}

Let us consider an example scenario where the human is lifting a heavy object as shown in Figure.~\ref{fig:subject_weight_lefting}. In particular, when the object is lifted, the additional weight of the object is measured by the \emph{ftShoes} as shown by the yellow arrows in Figure.~\ref{fig:rviz-weight-lifting-forces}. Now, rather than pushing the 6D force estimates at the hand links towards their set measurement of $\comVar{0}_6$, we consider the problem of reflecting the weight of the object as the 6D force estimates at the hand links. This problem is referred to as Non-collocated wrench estimation.  Considering the current formulation of the estimation problem from eq.~\eqref{eq:map-solution-optimization-problem}, the 6D force estimates at the hand links can be influenced by their measurement covariance value. However, the measurement covariance can only influence the estimator to push the estimates to be close to their set measurement of $\comVar{0}_6$ or away from that but we do not have any additional constraint that guides the estimates to reflect the object weight. Moreover, given that the dynamic variables vector $\comVar{d}$ from eq.~\eqref{d} is comprised of both the external link wrench ${\comVar{\underline f}^x}$ and the internal joint wrench ${\comVar{\underline f}}$, their estimation is coupled through the consideration of the model constraints through model covariance ${\comVar{\Sigma}}_D$. So, to achieve the goal of reflecting the object weight as the 6D force estimates at the hands, we need to decouple the estimation of external link wrench and the internal joint wrench quantities. This called for a new formulation of the estimation problem which will be discussed in detail in the following section.

\begin{figure}[H]
	\centering
	\begin{subfigure}[b]{0.5\textwidth}
		\centering
		\includegraphics[width=0.7\textwidth, trim=0 2cm 0 0, clip=true]{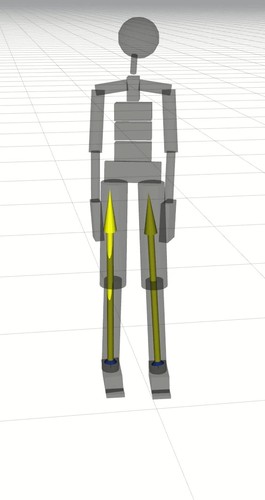}
		\caption{}
		\label{fig:rviz-forces-without-weight}
	\end{subfigure}%
	\begin{subfigure}[b]{0.5\textwidth}
		\centering
		\includegraphics[width=0.7\textwidth, trim=0 2cm 0 0, clip=true]{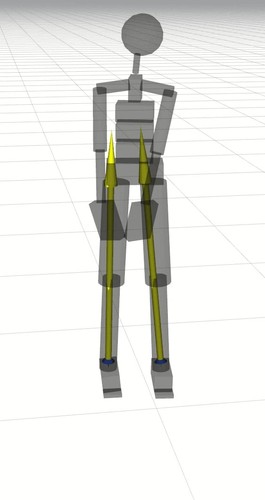}
		\caption{}
		\label{fig:rviz-forces-with-weight}
	\end{subfigure}
	\caption{Increase in ground reaction force measurements from \emph{ftShoes} while lifting a heavy object}
	\label{fig:rviz-weight-lifting-forces}
\end{figure}

\subsection{Stack of Tasks Stochastic Estimation}
\label{sec:stack-of-tasks}

The stochastic estimation problem for the floating base human model described in section~\ref{sec:HDE-stochastic-estimation} is reformulated as a two step stack of tasks estimation problem. We introduce a new model constraint in the first task that relates the rate of the change of momentum to the sum of all the external 6D force i.e. the principle of momentum conservation. To illustrate this a new frame, called \emph{centroidal} frame denoted by $\mathcal{G[I]}$. The centroidal frame is a unique coordinate frame with its origin at the center of mass of the system and its orientation similar to the orientation of the inertial frame of reference. Fig.~\ref{fig:sot_berdy_task1_frames} highlights the centroidal frame along with the other frame definitions of a mechanical system.

\begin{figure}[ht]
	\centering
    \includegraphics[width=0.65\textwidth]{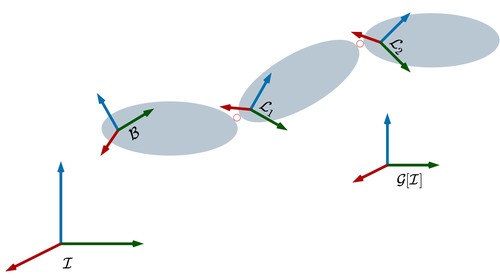}
	\caption{Centroidal frame definition}
	\label{fig:sot_berdy_task1_frames}
\end{figure}

Now the mathematical formulation of the new constraint when all the quantities are expressed in the centroidal frame~\cite{a9258435eeef4c5fb26a698d5db629f9, saccon2017centroidal} is defined as,

\begin{equation}
	^{\mathcal{G[I]}}\dot{\comVar{L}}:= \begin{bmatrix}
						m \ \ddot{\comVar{x}}_{com} \\
						^{\mathcal{G[I]}}\dot{\comVar{L}}_{ang}
					   \end{bmatrix} = 
					   \sum_{i=1}^{N_B} \ ^{\mathcal{G[I]}}{\comVar{\underline f}_{\mathcal{L}_{i}}^x} + 
					   \begin{bmatrix}
					       m \ \comVar{g} \\
					       \comVar{0}_{3 \times 1}
					   \end{bmatrix}
\label{eq:HDE-stack-of-tasks-momentum-constraint_centroidal}
\end{equation}

where, $m$ is the mass of the model, $\comVar{g} = \begin{bmatrix} 0 & 0 & -9.81\end{bmatrix}^\top$ is the gravity vector expressed in the inertial frame, $^{\mathcal{G[I]}}\dot{\comVar{L}}_{ang}$ is the rate of change of angular momentum. Now, the above constraint when expressed in the base frame $\mathcal{B}$ of the system, it becomes

\begin{equation}
	^{\mathcal{B}}\comVar{X}^{*}_{\mathcal{G},[\mathcal{I}]} \ ( \ ^{\mathcal{G},[\mathcal{I}]}\dot{\comVar{L}} - \comVar{w} \ ) = \sum_{i = 1}^{n} \ ^{\mathcal{B}}\comVar{X}^{*}_{\mathcal{G},[\mathcal{I}]} \ ^{\mathcal{G},[\mathcal{I}]}\comVar{f}_{\mathcal{L}_i}
	\label{eq:HDE-stack-of-tasks-momentum-constraint}
\end{equation}

The transformation matrix $^{\mathcal{B}}\comVar{X}_{\mathcal{G},[\mathcal{I}]}$ is computed as,

\begin{equation}
	^{\mathcal{B}}\comVar{X}_{\mathcal{G},[\mathcal{I}]} = \ ^{\mathcal{B}}\comVar{X}_{\mathcal{I}} \ ^{\mathcal{I}}\comVar{X}_{\mathcal{G},[\mathcal{I}]}
\end{equation}

where, 

\begin{equation}
	^{\mathcal{I}}\comVar{X}_{\mathcal{G},[\mathcal{I}]} = \begin{bmatrix} \comVar{I}_{3 \times 3} & \ \bm \skewOp(^{\mathcal{I}}\comVar{p}_{com}) \\ \comVar{0}_{3 \times 3} & \comVar{I}_{3 \times 3} \end{bmatrix}
\end{equation}

and, $^{\mathcal{I}}\comVar{p}_{com}$ is the position of the system center of mass expressed with respect to the inertial frame.

Eq.~\eqref{eq:HDE-stack-of-tasks-momentum-constraint} can be further represented in a compact form as,

\begin{equation}
	 ^{\mathcal{B}}\hat{\dot{\comVar{L}}} = \sum_{i = 1}^{n} \ ^{\mathcal{B}}\comVar{X}^{*}_{\mathcal{G},[\mathcal{I}]} \ ^{\mathcal{G},[\mathcal{I}]}\comVar{f}_{\mathcal{L}_i}
\end{equation}

The acceleration of center of mass is computed using multibody kinematics library iDynTree~\cite{nori2015}. However, the rate of change of angular momentum is set to $\comVar{0}_3$ as at the moment this quantity cannot be measured or estimated directly from a human subject. Furthermore, rather than considering the wrench on all the links, we consider only the feet and hand links as of interest and the new constraint simplifies to,

\begin{equation}
^{\mathcal{B}}\hat{\dot{\comVar{L}}} = \ ^{\mathcal{B}}\comVar{X}^{*}_{\mathcal{LF}} \ ^{\mathcal{LF}}\comVar{f}_{LF} \ + ^{\mathcal{B}}\comVar{X}^{*}_{\mathcal{RF}} \ ^{\mathcal{RF}}\comVar{f}_{RF} \ + ^{\mathcal{B}}\comVar{X}^{*}_{\mathcal{LH}} \ ^{\mathcal{LH}}\comVar{f}_{LH} \ + ^{\mathcal{B}}\comVar{X}^{*}_{\mathcal{RH}} \ ^{\mathcal{RH}}\comVar{f}_{RH}
\end{equation}

Where, $LF$ correspond to \emph{LeftFoot}, $RF$ corresponds to \emph{RightFoot}, $LH$ corresponds to \emph{LeftHand}, and $RH$ corresponds to \emph{RightHand}. Now, the estimation problem is separated as a two tasks problem.

\subsubsection{First Task}
\label{sec:HDE-stack-of-tasks-first-task}

The dynamic variables vector for the first task $\comVar{d}^{'}$ contains only the link external 6D force variables i.e.,

\begin{subequations}
	\begin{eqnarray}
	\comVar{d}^{'} &=& \begin{bmatrix}
	{\comVar{\underline f}_{0}^x} &
	{\comVar{\underline f}_{1}^x} &
	\hdots &
	{\comVar{\underline f}_{N_B-1}^x}
	\end{bmatrix}^\top
	\in \mathbb R^{6N_B} \label{task1_d}
	\end{eqnarray}
\end{subequations}

Similarly, the measurement vector of the first task $\comVar{y}^{'}$ is updated as,

\begin{subequations}
	\begin{eqnarray}
	\comVar{y}^{'} &=& \begin{bmatrix}
	{\comVar{\underline f}_{0}^{x (mes)}} &
	{\comVar{\underline f}_{1}^{x (mes)}} &
	\hdots &
	{\comVar{\underline f}_{N_B-1}^{x (mes)}} &
	^{\mathcal{B}}\hat{\dot{\comVar{L}}}
	\end{bmatrix}^\top
	\in \mathbb R^{6N_B + 6} \label{task1_y}
	\end{eqnarray}
\end{subequations}

Now, the new first set of measurements equation becomes,

\begin{eqnarray} \label{eq:task1-systemEq}
	\comVar{Y}^{'} \comVar{d}^{'} + \comVar{b}^{'}_Y = \comVar{y}^{'} \qquad
\end{eqnarray}

The maximum a-posteriori solution of the first task contains the estimates of the external 6D forces on all the links denoted as,

\begin{subequations}
	\begin{eqnarray}
	\comVar{d}^{ '(\mbox{\tiny{MAP}}) } &=& \begin{bmatrix}
	{\comVar{\underline f}_{0}^{x (est)}} &
	{\comVar{\underline f}_{1}^{x (est)}} &
	\hdots &
	{\comVar{\underline f}_{N_B-1}^{x (est)}}
	\end{bmatrix}^\top
	\in \mathbb R^{6N_B} \label{task1_d_map_solution}
	\end{eqnarray}
\end{subequations}

Coming to the choice of the covariances, the measurement covariance ${\bm {\Sigma}}_{y^{'}}$ for the rate of change of momentum~\eqref{eq:HDE-stack-of-tasks-momentum-constraint} is kept to a low value to place a high trust on these measurements. The measurement covariance ${\bm {\Sigma}}_{y^{'}}$ for the feet is kept low and for hands it is kept high. For covariance related to the dynamics variables ${\bm {\Sigma}}_{d^{'}}$, it is kept low for the feet and high for the hands. Such a choice guides the estimation algorithm to trust the forces and moments measured at the feet through the \emph{ftShoes} and ensure the estimates to be close to the measurements. On the other hand, the forces and moments measurement of $\comVar{0}_6$ at the hands is not given a high trust and the estimator is free to change the estimates guided by the new constraint of the net external forces in the measurement equation directed by Eq.~\eqref{eq:HDE-stack-of-tasks-momentum-constraint}. This task considers only the estimation of external link wrenches and it is decoupled from the internal wrench estimation. 

\subsubsection{Second Task}
\label{sec:HDE-stack-of-tasks-second-task}

The second task equations are exactly as described in Section~\ref{sec:recourse-on-stochastic-dynamics-estimation}, except for the measurements vector. The measurements vector for the second task $\comVar{y}^{''}$ is updated with the external 6D force \emph{estimates} Eq.~\eqref{task1_d_map_solution} from the first task instead of using the link external 6D force measurements directly.

A proper choice of covariances ${\bm{\Sigma}}_{y^{'}}$ and ${\bm{\Sigma}}_{d^{'}}$ in the first task results in the estimates that match the measurements of the forces and moments at the feet. On the other hand, the estimates of the forces and moments at the hands vary. Given these reasonable estimates from the first task will become the measurements for the second task, the measurement covariance ${\bm{\Sigma}}_{y}$ is kept low and the covariance ${\bm{\Sigma}}_{d^{'}}$ is also kept to a low value. This choice translates to placing high trust on the force and moment values passed as measurements for the second task and the estimator does not change them further during the second task.
 
\subsubsection{Optimization Problem}

The new optimization problem for the stack of tasks stochastic estimation is formulated as,

\begin{subequations}
	\begin{eqnarray}
	\comVar{d}^{\mbox{\tiny{MAP}}} & = & \arg\min_d 
	( ||\comVar{D} \comVar{d} + \comVar{b}_D||^2_{\comVar{\Sigma}_D^{-1} } + 
	||\comVar{Y} \comVar{d} + \comVar{b}_Y - \comVar{y}^{''}||^2_{\comVar{\Sigma}_y^{-1} } + 
	|| \comVar{d} - \comVar{\mu}_d||^2_{\comVar{\Sigma}_d^{-1} }) \notag     \\
	& s.t.& \nonumber \\
	&& \comVar{d}^{' \mbox{\tiny{MAP}}} = \arg\min_{d^{'}} 
	(|| \comVar{Y}^{'} \comVar{d}^{'} + \comVar{b}^{'}_Y - \comVar{y}^{'}||^2_{\comVar{\Sigma}_{y^{'}}^{-1} } + || \comVar{d}^{'} - \comVar{\mu}_{d^{'}}||^2_{\comVar{\Sigma}_{d^{'}}^{-1} })
	\end{eqnarray}
	\label{eq:sot-map-solution-optimization-problem}
\end{subequations}

\subsection{Human Joint Torques Visualization}
\label{sec:HDE-joint-torque-estimates}

Human joint torques are computed through the joint 6D force estimates that are a part of the dynamic variables solution vector $\comVar{d}^{\mbox{\tiny{MAP}}}$. At the current stage of our research we do not validate the joint torque estimates through any sensory measurements. Each human joint in the model is composed of up to three one degrees of freedom joint. The effort of each joint is computed as,

\begin{equation}
\label{eq:human-joint-effort}
{\tau}_{effort} = \sqrt{{\tau}_x^2 + {\tau}_y^2 + {\tau}_z^2}
\end{equation}

The human joint effort is visualized with the help of Rviz, a 3D visualizer popular in robotics community for displaying sensor data and state information using ROS~\cite{ros2018}. A sphere that ranges from \emph{green} color to \emph{red} color indicates the joint efforts from $[ minimum \ intensity, maximum \ intensity ]$ value. Depending on the task, the joint effort reflecting the articular stress changes for various joints of the human model. 

\subsection{Qualitative \& Quantitative Results}

We performed initial tests with human lifting a heavy object as shown in Fig.~\ref{fig:subject_weight_lefting} after modifying the estimation problem with the new constraint on the net external forces in the measurement equation directed by Eq.~\eqref{eq:HDE-stack-of-tasks-momentum-constraint}. The feet external force estimation (shown in blue arrows) and the joint torques estimation while the human is in neutral N-pose without lifting any heavy object is shown in Fig.~\ref{fig:rviz-full-HDE-without-weight}. On the other hand, under the influence of a heavy object lifted by the human, external force estimation at the hands is updated as shown in Fig.~\ref{fig:rviz-full-HDE-with-weight} and the associated joint torques are higher as expected. Also, note that in both the cases the feet force measurements (yellow arrows) and the estimates (blue arrows) are matching. This is a direct consequence of keeping a low value for the measurement covariance during the first task ${\bm {\Sigma}}_{y^{'}}$ and the second task ${\bm {\Sigma}}_{y}$ related to the feet 6D force which translates to having higher trust in the measurements from the \emph{ftShoes}. 

An offline data analysis is carried with a reduced human model (48 DoF~Chapter~\ref{cha:human-modeling}) is carried out and the Fig.~\ref{fig:HDE_force_estimation_fext_comparison} highlights the comparison between the measurements and the estimates of forces and moments at the feet and the hands. Under the influence of the weight of a heavy object, the ground reaction force measurements increases as seen on $f_z^x$ plot changes of the feet and accordingly, the estimates at the hands increase as seen on on $f_z^x$ plot changes of the hands. This change in the estimates at the hands is guided by choosing a high measurement covariance value ${\bm {\Sigma}}_{y^{'}}$ and a high covariance for the dynamic variables ${\bm {\Sigma}}_{d^{'}}$ during the first task. Such a choice of covariances translates to not trusting the measurement of $\comVar{0}_6$ for the forces and moments at the hands and allowing the estimates to vary through the new constraint on the net external forces in the measurement equation directed by Eq.~\eqref{eq:HDE-stack-of-tasks-momentum-constraint}. 

Although we do not validate the joint torque estimation through other sensory measurements for articular stress, the validation of the measurements with the estimates of the forces and moments at the feet presented in Fig.~\ref{fig:HDE_force_estimation_fext_comparison} is a guide to trust the joint torque estimates. Torque estimation changes at the joints of arms and legs as shown in Fig.~\ref{fig:HDE_force_estimation_arm_joint_torque_estimates} and Fig.~\ref{fig:HDE_force_estimation_leg_joint_torque_estimates} respectively.

\begin{figure}[H]
	\centering
	\begin{subfigure}[b]{0.5\textwidth}
		\centering
		\includegraphics[width=0.45\textwidth]{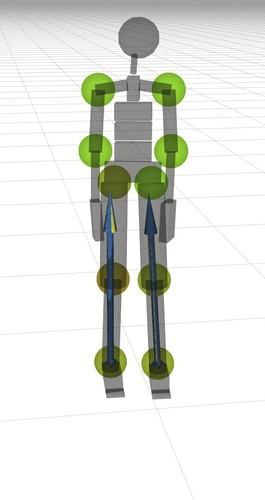}
		\caption{}
		\label{fig:rviz-full-HDE-without-weight}
	\end{subfigure}%
	\begin{subfigure}[b]{0.5\textwidth}
		\centering
		\includegraphics[width=0.45\textwidth]{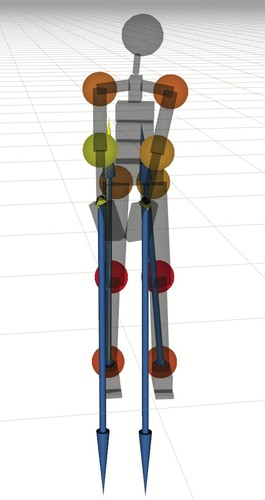}
		\caption{}
		\label{fig:rviz-full-HDE-with-weight}
	\end{subfigure}
	\caption{Rviz visualization of external force and joint torques estimates in neutral N-pose (\subref{fig:rviz-full-HDE-without-weight}) and while lifting a heavy object (\subref{fig:rviz-full-HDE-with-weight})}
	\label{rviz-full-HDE-weight-lifting-forces}
\end{figure}

\begin{figure}[H]
	\centering
	\begin{subfigure}[b]{\textwidth}
		\centering
		\includegraphics[width=.33\textwidth]{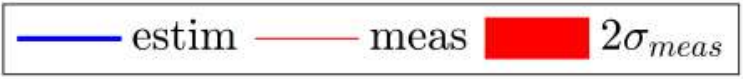}
	\end{subfigure}
	\begin{subfigure}[b]{\textwidth}
		\centering
		\includegraphics[width=0.875\textwidth]{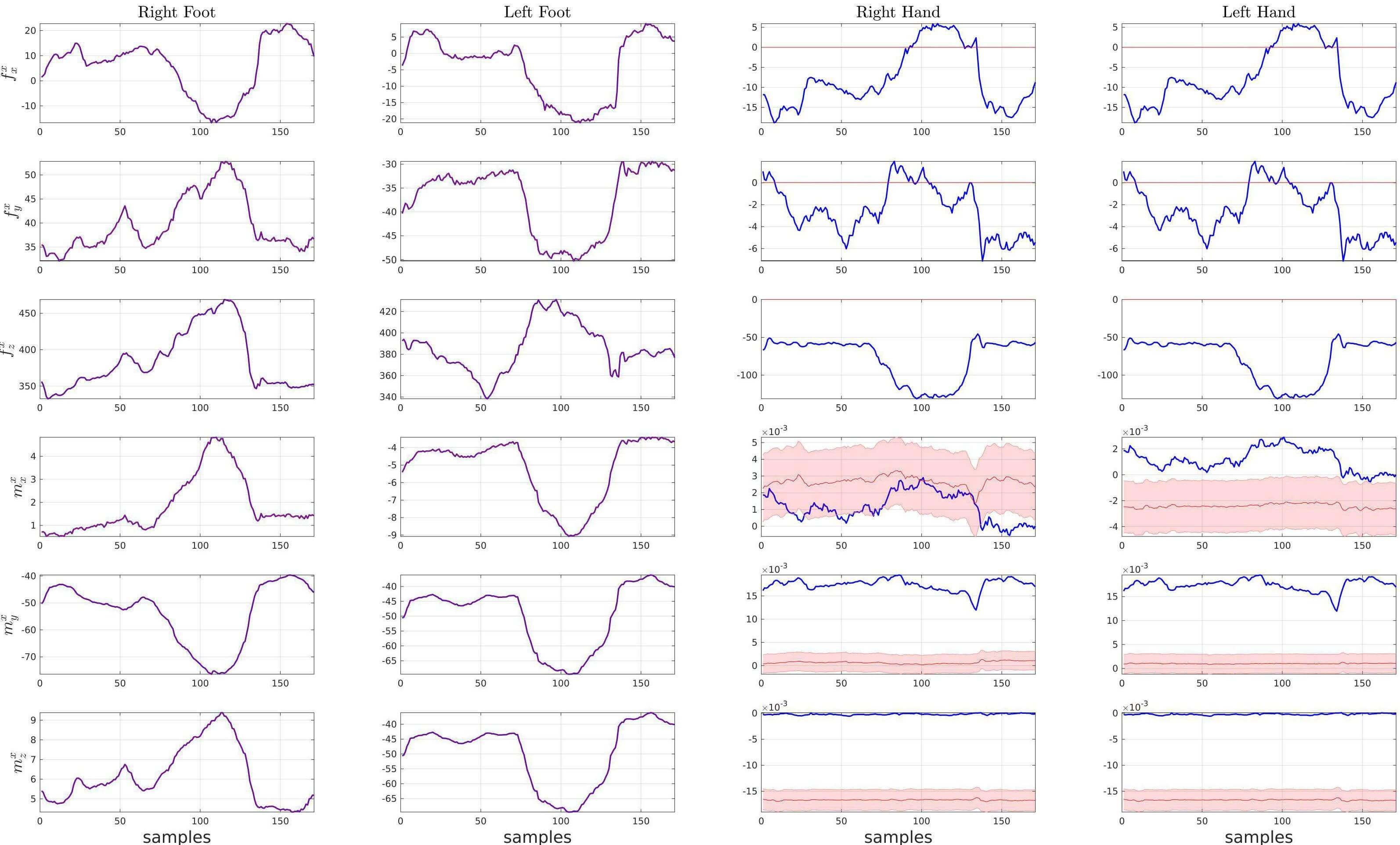}
	\end{subfigure}
	\caption{Comparison of the 6D force measurements (mean and standard deviation, in red) and estimates (mean, in blue) at the feet and the hands; The measurements at the hand links are zero while the estimates increase under the influence of the weight of a heavy object held by the human}
	\label{fig:HDE_force_estimation_fext_comparison}
\end{figure}

\begin{figure}[H]
	\centering
	\begin{subfigure}[b]{\textwidth}
		\centering
		\includegraphics[width=\textwidth]{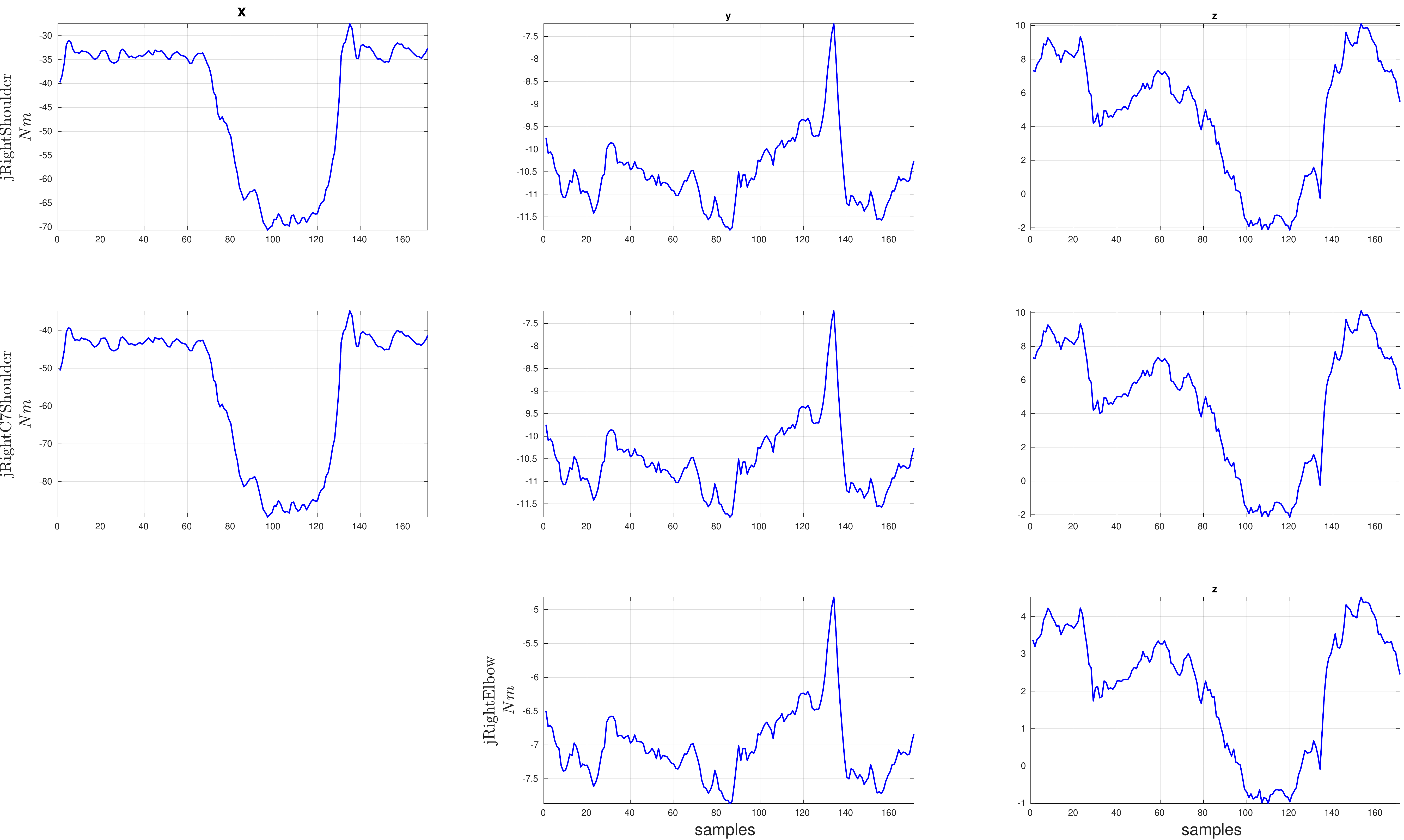}
		\label{fig:HDE_force_estimation_right_arm_joint_torque_estimates}
		\vspace{0.2cm}
	\end{subfigure}
	\begin{subfigure}[b]{\textwidth}
		\centering
		\includegraphics[width=\textwidth]{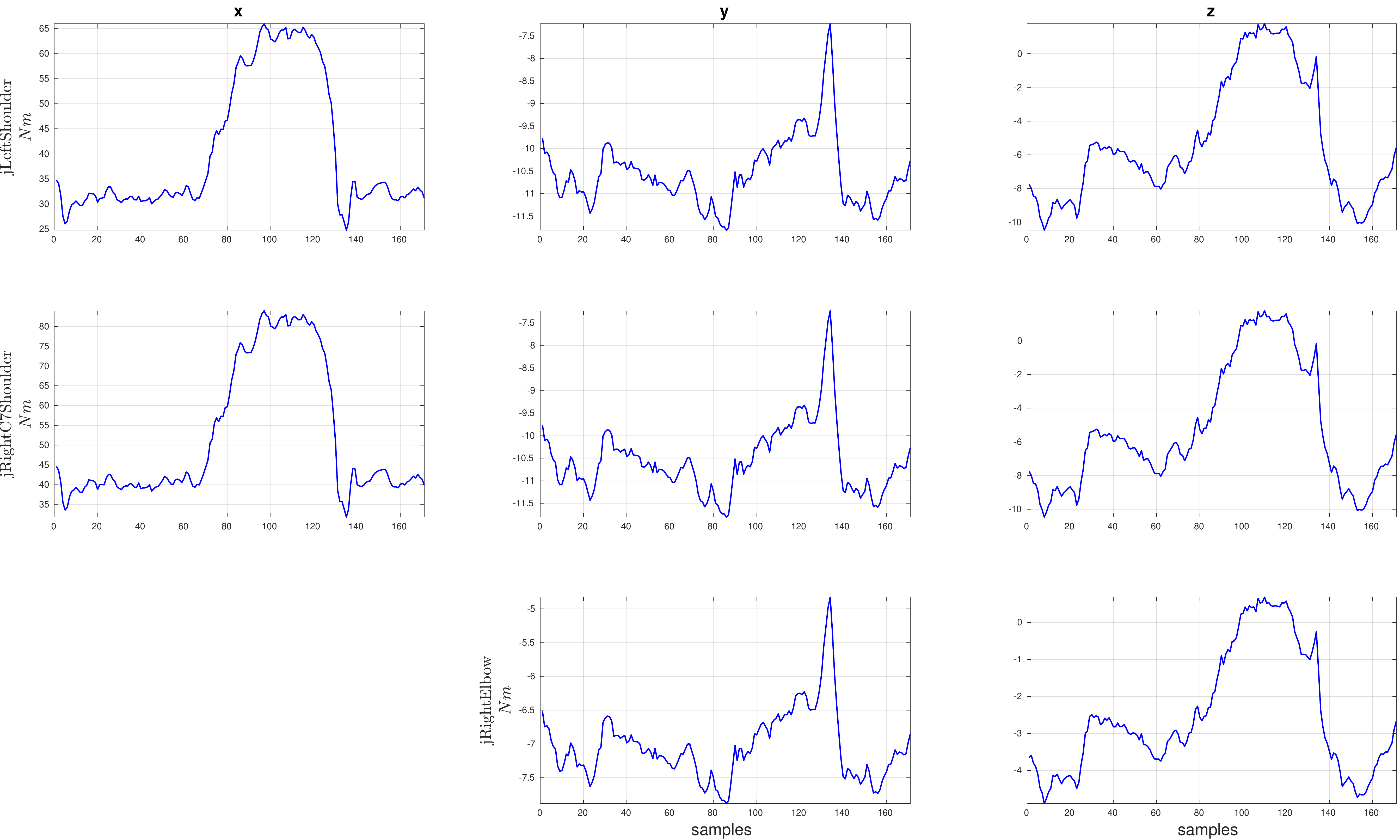}
		\label{fig:HDE_force_estimation_left_arm_joint_torque_estimates}
	\end{subfigure}
	\caption{Torques estimation changes of the arm joints under the influence of the weight of a heavy object}
	\label{fig:HDE_force_estimation_arm_joint_torque_estimates}
\end{figure}

\begin{figure}[H]
	\centering
	\begin{subfigure}[b]{\textwidth}
		\centering
		\includegraphics[width=\textwidth]{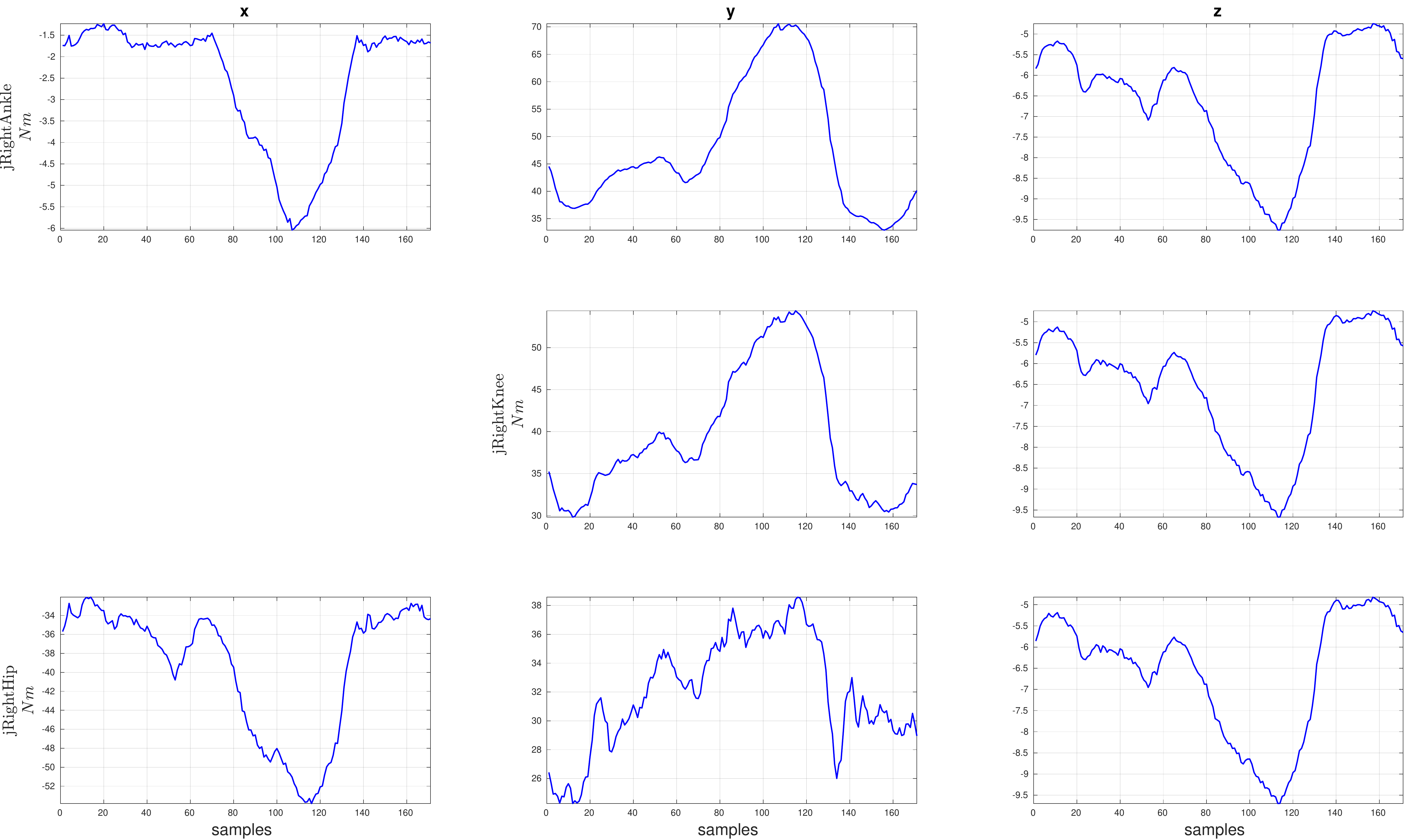}
		\label{fig:HDE_force_estimation_right_leg_joint_torque_estimates}
		\vspace{0.2cm}
	\end{subfigure}
	\begin{subfigure}[b]{\textwidth}
			\centering
		\includegraphics[width=\textwidth]{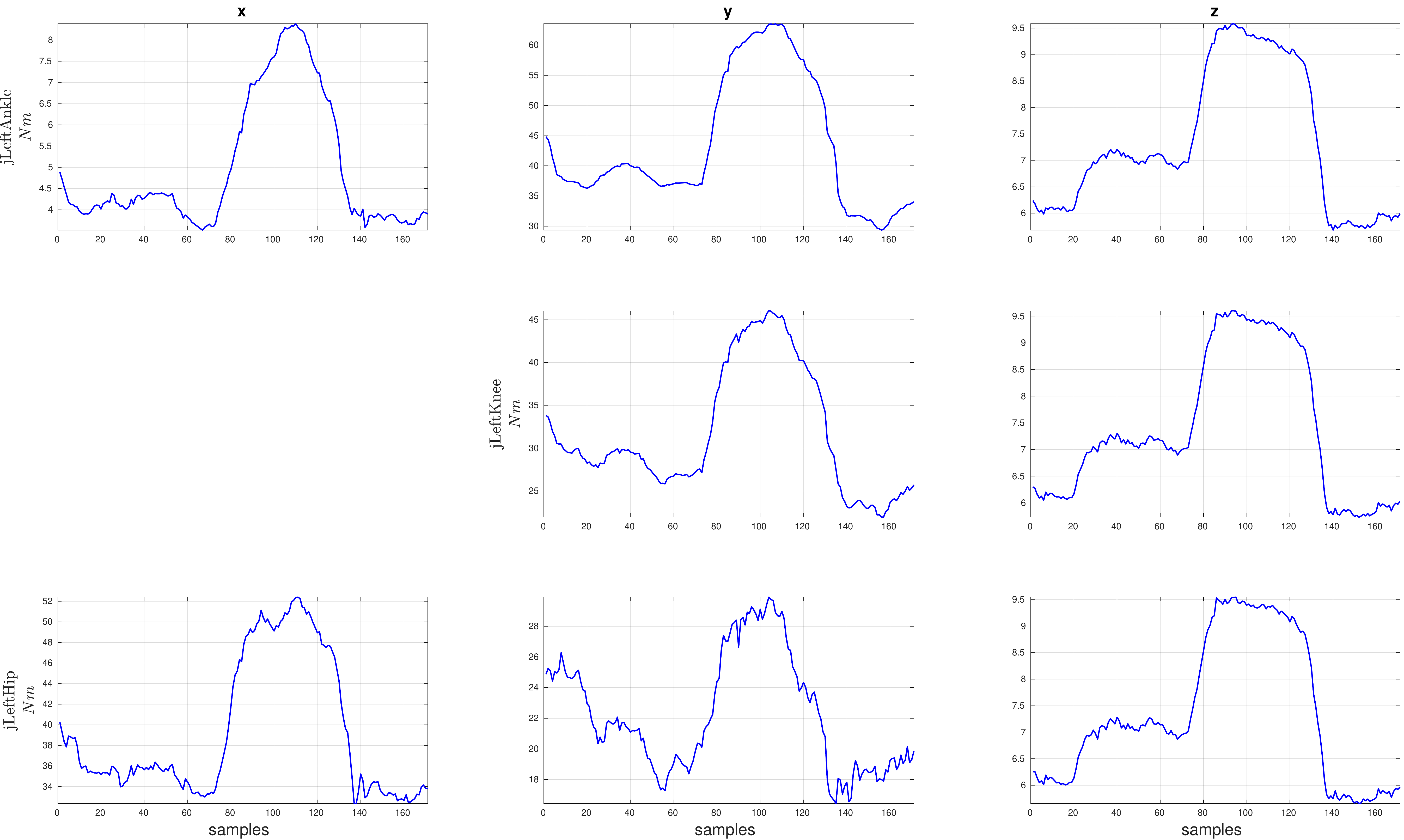}
		\label{fig:HDE_force_estimation_left_leg_joint_torque_estimates}
	\end{subfigure}
	\caption{Torques estimation changes of the leg joints under the influence of the weight of a heavy object}
	\label{fig:HDE_force_estimation_leg_joint_torque_estimates}
\end{figure}

\section{Floating Base Dynamics Estimation: \\ Experimental Validation}

The updated formulation for the floating base human dynamics estimation presented in section~\ref{sec:recourse-on-stochastic-dynamics-estimation} is validated through various experiments where a healthy male subject was asked to perform a set of tasks as listed in Table \ref{table:tasks_description}.

\begin{figure}[ht!]
  \centering
    \includegraphics[width=.6\textwidth]{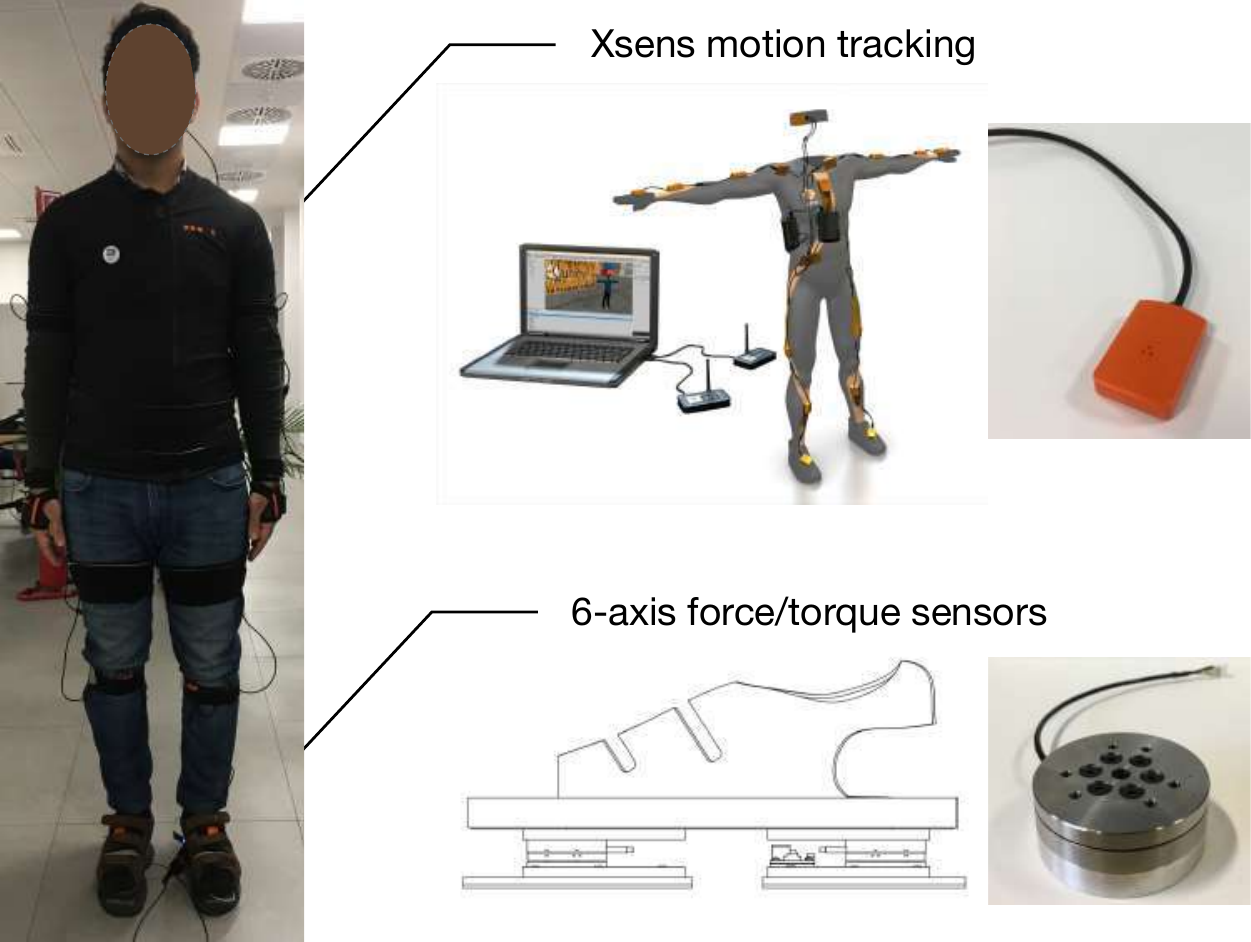}
  \caption{Subject equipped with the Xsens wearable motion tracking
   system and six axis force-torque shoes (\emph{ftShoes})}
  \label{fig:figs_suit_shoes_description}
\end{figure}

\begin{figure}[ht!]
  \centering
    \includegraphics[width=\textwidth]{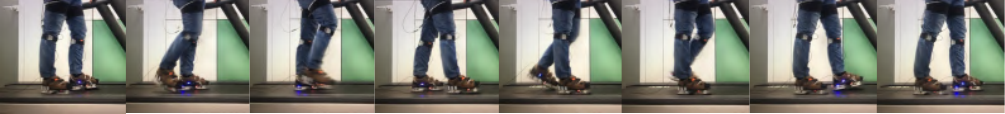}
  \caption{Task T4, Sequence $2$: walking on a treadmill}
  \label{fig:figs_walkingSequence}
\end{figure}

\begin{table*}[ht!]
\caption{Tasks performed for validation of floating base dynamics estimation}
\centering
\begin{tabular}{c l p{8cm}}
\toprule
\textbf{Task} & \textbf{Type}  & \textbf{Description}\\
\toprule
T1 & Static double support & Neutral pose, standing still\\
\midrule
T2 & Static right single support & Sequence $1$: static double support \\
   &                             & Sequence $2$: weight balancing on the right
    foot\\
\midrule
T3 & Static left single support & Sequence $1$: static double support \\
   &                            & Sequence $2$: weight balancing on the left
    foot\\
    \midrule
T4 & Static-walking-static &  Sequence $1$: static double support \\
   &                       &  Sequence $2$: walking on a treadmill (Figure
    \ref{fig:figs_walkingSequence})\\
   &                       &  Sequence $3$: static double support
\\
\bottomrule
\end{tabular}
\label{table:tasks_description}
\end{table*}

Data from the \emph{ftShoes} is analyzed to detect the feet contact switching based on a self-tuned threshold value for the force along the $z$-axis. The task representation with respect to the feet in contact with the ground is highlighted in Figure.~\ref{fig:figs_patternPerTask}.

\begin{figure}[ht!]
    \centering
    \includegraphics[width=\textwidth]{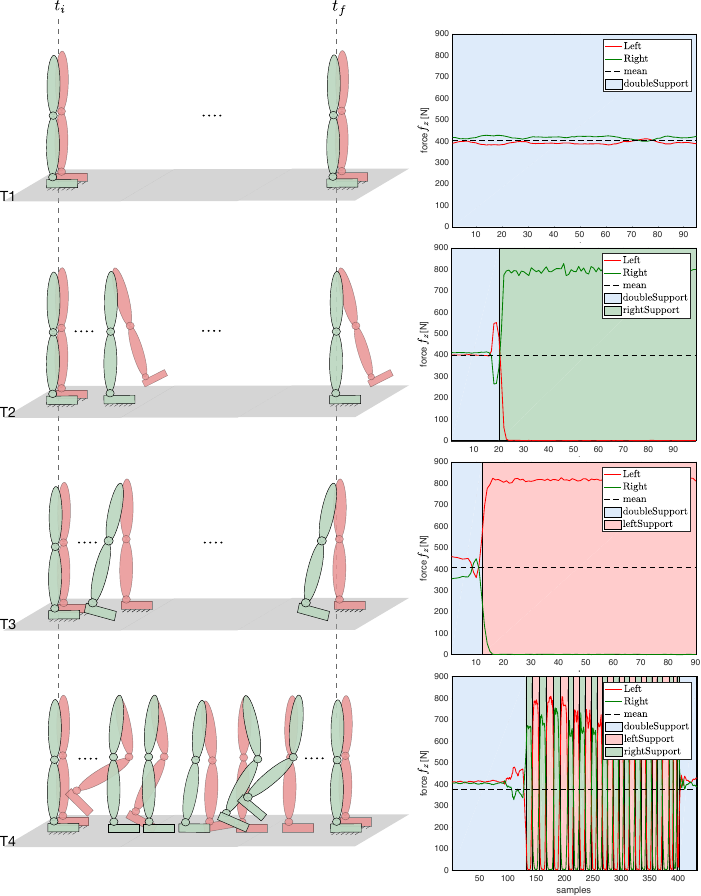}
    \caption{Tasks representation from initial time $t_i$ to final time $t_f$ (\emph{on the left}) and feet contact pattern classification (\emph {on the right})}
    \label{fig:figs_patternPerTask}
\end{figure}\textbf{}

The kinematic quantities are obtained through the human kinematics estimation presented in chapter~\ref{cha:human-kinematics-estimation}.

\subsection{Comparison between Measurements and Estimates}
\label{sec:HDE-Measurements-vs-Estimates}

The dynamic variables vector $\comVar{d}$ consists of link variables measurable via sensors i.e., linear proper sensor acceleration $\comVar{\alpha}_{lin}^g [\SI{}{\meter\per\second^2}]$, external force $\comVar{f}^x [\SI{}{\newton}]$ and external moment $\comVar{m}^x [\SI{}{\newton\meter}]$; and joint variables that are not measured via sensors i.e. joint 6D force $\comVar{f}$ and joint torques $\comVar{\tau}$. The proposed stochastic approach of solving for the human dynamics can be evaluated by comparing the variables that have sensor measurements with respect to their estimates. The validation has been performed along with a Root Mean Square Error (RMSE) metric for linear accelerations and external 6D force presented in Table~\ref{table:RMSE_comparisonMeasVSestim}. The covariance matrix associated with the sensor measurements is an important factor for the estimation problem. In this experimental analysis, covariances are chosen in a range $[10^{-6}$, $10^{-4}]$. A low covariance value reflects higher trust in the sensor measurements and the estimates are closer to the measurements leading to a lower RMSE for the associated variable.

\begin{table*}[ht!]
	\caption{RMSE analysis of the base linear proper sensor acceleration $\comVar{\alpha}_{lin}^g [\SI{}{\meter\per\second^2}]$, the external force $\comVar{f}^x [\SI{}{\newton}]$ and moment $\comVar{m}^x [\SI{}{\newton\meter}]$ floating-base algorithm estimations w.r.t. the measurements, for tasks T1, T2, T3 and T4.}
    \centering
    \resizebox{\linewidth}{!}{%
    \begin{tabular}{c l c c c c c c c c c}
        \toprule
        \textbf{Task} & \textbf{Link}  & \mcrot{1}{l}{15}{$\alpha_{lin,x}^g [\SI{}{\meter\per\second^2}]$} & \mcrot{1}{l}{15}{$\alpha_{lin,y}^g [\SI{}{\meter\per\second^2}]$} & \mcrot{1}{l}{15}{$\alpha_{lin,z}^g [\SI{}{\meter\per\second^2}]$} & \mcrot{1}{l}{15}{$f^x_x [\SI{}{\newton}]$} & \mcrot{1}{l}{15}{$f^x_y [\SI{}{\newton}]$} & \mcrot{1}{l}{15}{$f^x_z [\SI{}{\newton}]$} & \mcrot{1}{l}{15}{$m^x_x [\SI{}{\newton\meter}]$} & \mcrot{1}{l}{15}{$m^x_y [\SI{}{\newton\meter}]$} & \mcrot{1}{l}{15}{$ m^x_z [\SI{}{\newton\meter}]$}\\
        \toprule
        & Base (Pelvis) & $0.008$  & $0.014$ & $0.002$ & - & - & - & - & - & -\\
        T1 & Left foot & - & - & - & $0.050$ & $0.030$ & $2.514e^{-4}$ & $8.354e^{-4}$ & $0.002$  & $2.079e^{-5}$ \\
        & Right foot & - & - & - & $0.031$ & $0.048$ & $0.004$ & $0.0015$  & $0.001$  & $1.038e^{-5}$ \\
        \midrule
        & Base (Pelvis) & $0.003$  & $0.027$ & $0.018$ & - & - & - & - & - & -\\
        T2 & Left foot & - & - & - & $0.153$ & $0.071$ & $0.009$ & $0.002$ & $4.729e^{-4}$  & $1.683e^{-4}$ \\
        & Right foot & - & - & - & $0.013$ & $0.074$ & $0.005$ & $0.002$  & $4.247e^{-4}$  & $4.331e^{-5}$ \\
        \midrule
        & Base (Pelvis) & $0.012$  & $0.007$ & $0.007$ & - & - & - & - & - & -\\
        T3 & Left foot & - & - & - & $0.075$ & $0.019$ & $0.002$ & $5.968e^{-4}$ &
        $0.002$  & $5.218e^{-5}$ \\
        & Right foot & - & - & - & $0.065$ & $0.018$ & $0.003$ & $6.096e^{-4}$ & $0.002$  & $1.196e^{-4}$ \\
        \midrule
        & Base (Pelvis) & $0.011$  & $0.018$ & $0.033$ & - & - & - & - & - & -\\
        T4 & Left foot & - & - & - & $0.089$ & $0.056$ & $0.012$ & $0.002$ & $0.003$  & $1.322e^{-4}$ \\
        & Right foot & - & - & - & $0.084$ & $0.056$ & $0.019$ & $0.002$ & $0.003$  & $9.737e^{-5}$ \\
        \bottomrule
    \end{tabular}}
    \label{table:RMSE_comparisonMeasVSestim}
\end{table*}

The comparison between the base linear proper sensor acceleration $\comVar{\alpha}_{lin}^g [\SI{}{\meter\per\second^2}]$ measurement (mean and standard deviation, in red) and the estimation (mean, in blue) for tasks considered tasks is shown in Figure.~\ref{fig:figs_measVSestim_pelvisAccLin_total}. Similarly, the same comparison for the external force $\comVar{f}^x [\SI{}{\newton}]$ and external moment $\comVar{m}^x [\SI{}{\newton\meter}]$  is shown in Figures.~\ref{measVSestim_6Dfext_total}\subref{fig:figs_measVSestim_leftFoot6Dfext_total}
 and \ref{measVSestim_6Dfext_total}\subref{fig:figs_measVSestim_rightFoot6Dfext_total} for the left and right foot, respectively.
 
 \vspace{-0.5cm}

\begin{figure}[H]
	\centering
	\begin{subfigure}[b]{\textwidth}
		\centering
		\includegraphics[width=.33\textwidth]{images/legend}
	\end{subfigure}
	\begin{subfigure}[b]{\textwidth}
		\centering
		\includegraphics[width=\textwidth]{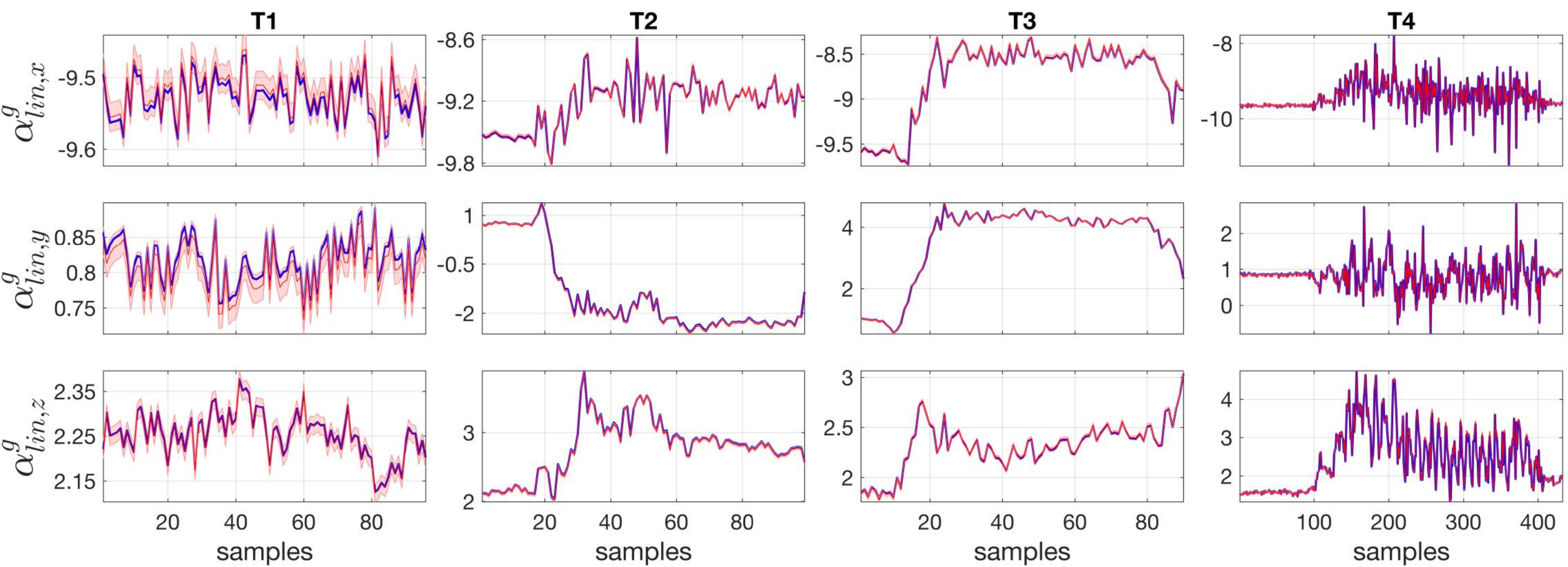}
	\end{subfigure}
	\caption{The base linear proper sensor acceleration $\comVar{\alpha}_{lin}^g [\SI{}{\meter\per\second^2}]$ comparison between measurement (mean and standard deviation, in red) and floating-base estimation (mean, in blue), for tasks T1, T2, T3 and T4}
	\label{fig:figs_measVSestim_pelvisAccLin_total}
\end{figure}

\newpage

\begin{figure}[H]
    \centering
    \begin{subfigure}[b]{\textwidth}
    	\centering
    	\includegraphics[width=.33\textwidth]{images/legend}
    \end{subfigure}
    \begin{subfigure}[b]{\textwidth}
        \centering
        \includegraphics[width=0.92\textwidth]{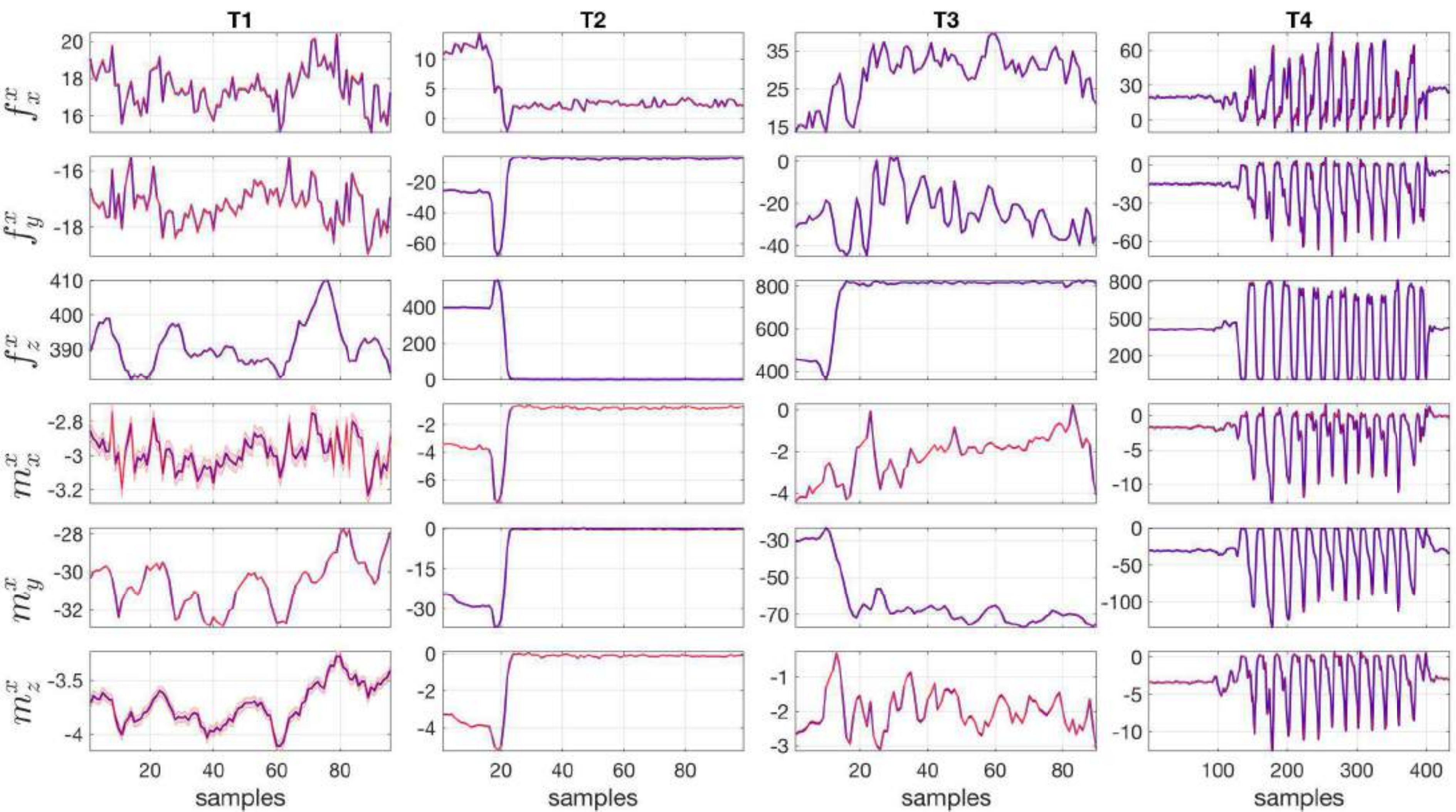}
        \caption{}
        \vspace{0.2cm}
        \label{fig:figs_measVSestim_leftFoot6Dfext_total}
         \vspace{0.05cm}
    \end{subfigure}
    \begin{subfigure}[b]{\textwidth}
        \centering
        \includegraphics[width=0.9\textwidth]{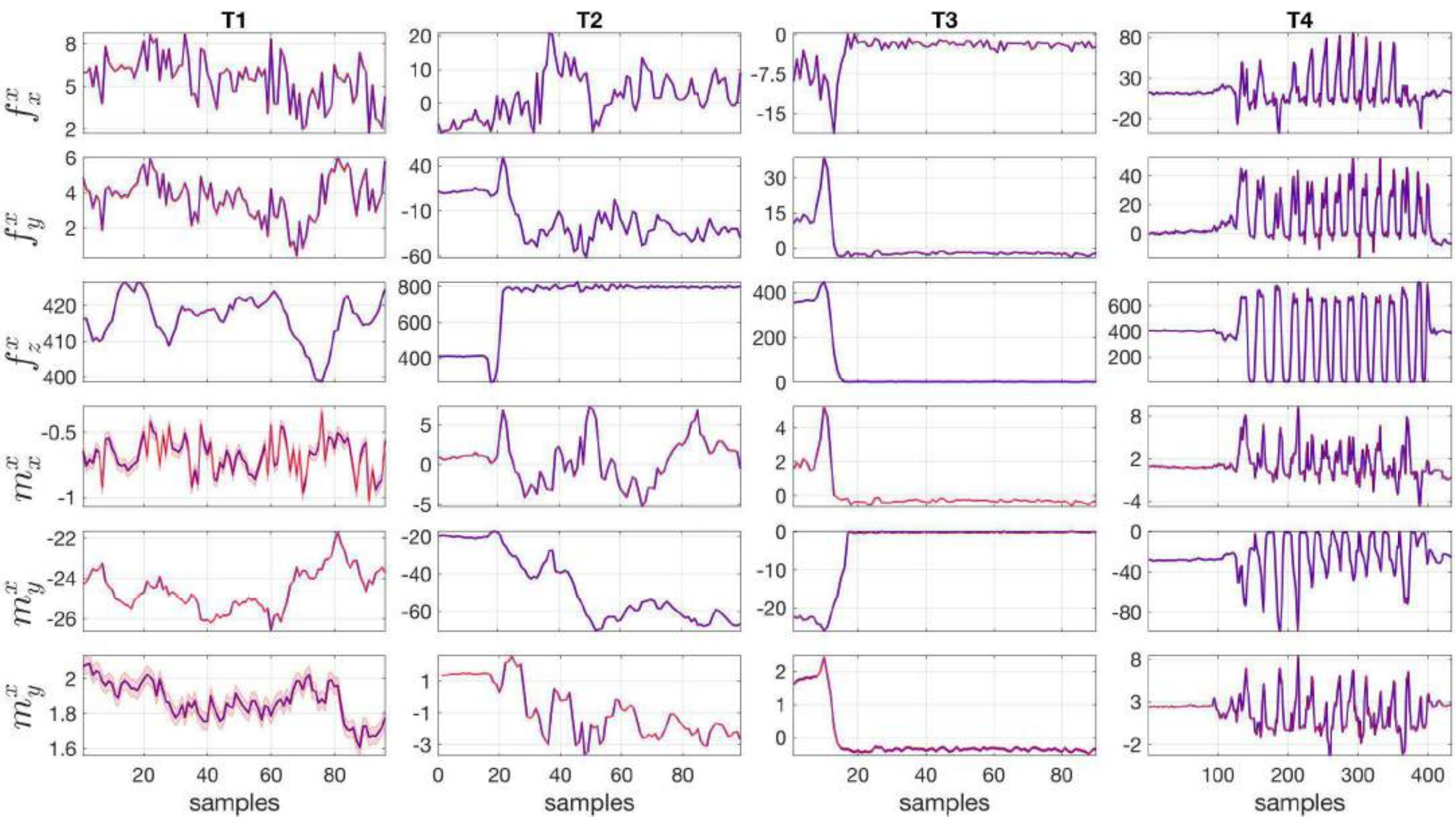}
        \caption{}
        \label{fig:figs_measVSestim_rightFoot6Dfext_total}
    \end{subfigure}
    \caption{The external force $\comVar{f}^x [\SI{}{\newton}]$  and external moment $\comVar{m}^x [\SI{}{\newton\meter}]$ comparison between measurement (mean and standard deviation, in red) and estimation via floating-base estimation (mean, in blue) for left (a) and right foot (b), for tasks T1, T2, T3 and T4}
\label{measVSestim_6Dfext_total}
\end{figure}

\subsection{Human Joint Torques Estimation during Walking}

Joint efforts during the balancing on single foot related to Tasks T2 (right foot \ref{fig:rviz_right_foot_balancing}) and T3 (left foot \ref{fig:rviz_left_foot_balancing}) are highlighted in Figure.~\ref{fig:rviz_balancing_efforts}. Similarly, joint effort changes during different phases of walking is shown in Figure.~\ref{fig:rviz_walking}. Furthermore, the ground reaction force measurements (yellow arrows) and the estimates (blue arrows) are shown. The joint torque estimates along with the joint angles for right leg (\subref{fig:rightLeg}) and left leg (\subref{fig:leftLeg}) during the walking Task T4 is highlighted in Figure.~\ref{fig:MAP_floatingTauVsAngles_T4}.

\begin{figure}[H]
	\centering
	\begin{subfigure}[b]{0.55\textwidth}
		\centering
		\includegraphics[width=\textwidth]{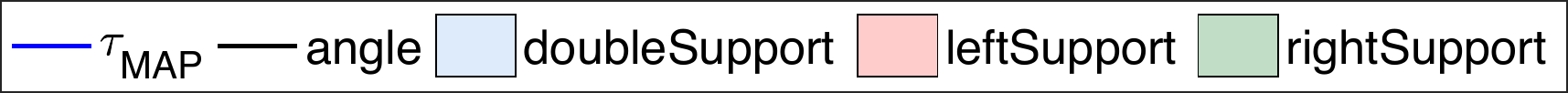}
	\end{subfigure}
	\begin{subfigure}[b]{\textwidth}
		\centering
		\includegraphics[width=0.825\textwidth]{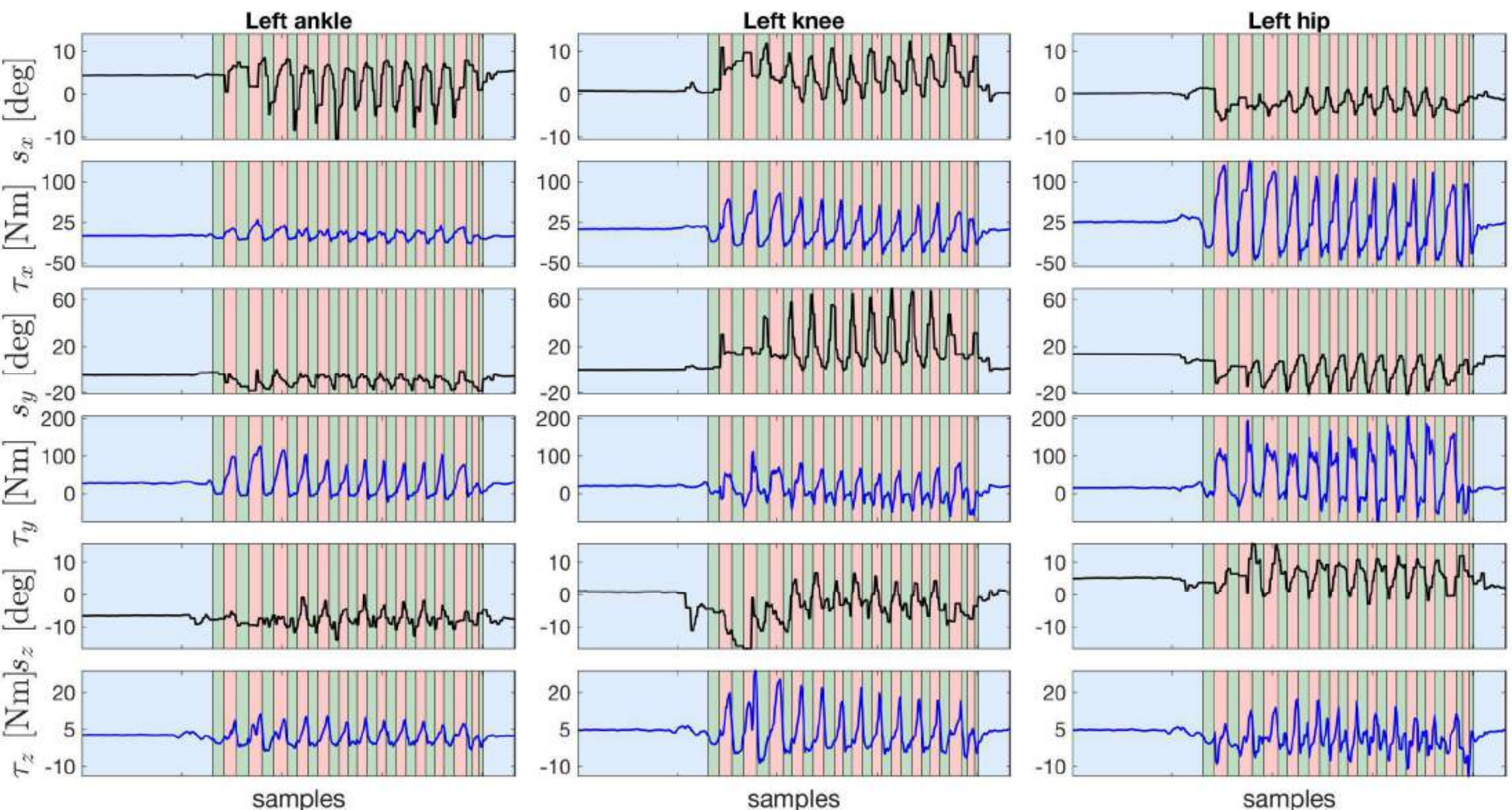}
		\caption{}
		\label{fig:leftLeg}
	\end{subfigure}
	\begin{subfigure}[b]{\textwidth}
		\centering
		\includegraphics[width=0.825\textwidth]{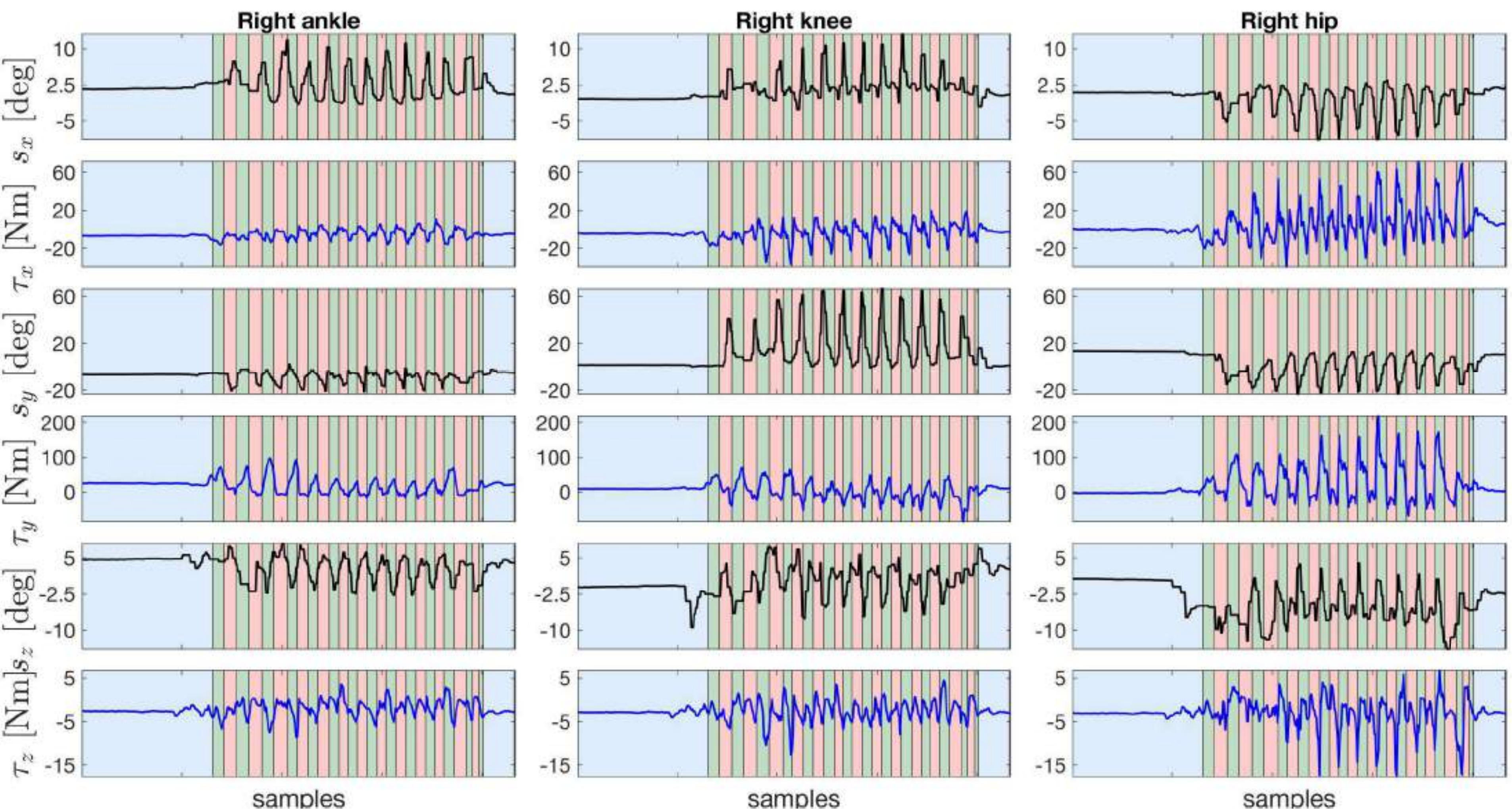}
		\caption{}
		\label{fig:rightLeg}
	\end{subfigure}
	\caption{Joint torques $\comVar{\tau} [\SI{}{\newton\meter}]$ changes (in blue) during Task T4 walking along with the joint angles (in black) of right leg (a) and left leg (b).}
	\label{fig:MAP_floatingTauVsAngles_T4}
\end{figure}

\begin{figure}[H]
  \centering
  \begin{subfigure}[b]{0.45\textwidth}
      \centering
      \includegraphics[scale=0.515]{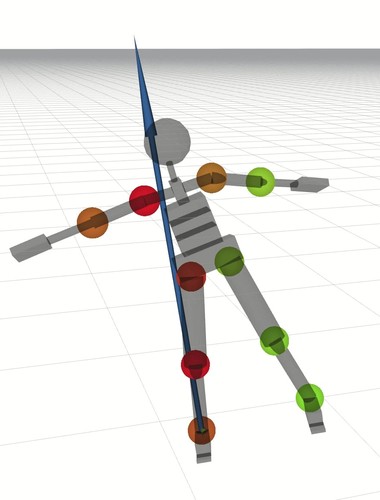}
      \caption{}
      \label{fig:rviz_right_foot_balancing}
  \end{subfigure}%
  \begin{subfigure}[b]{0.45\textwidth}
      \centering
      \includegraphics[scale=0.515]{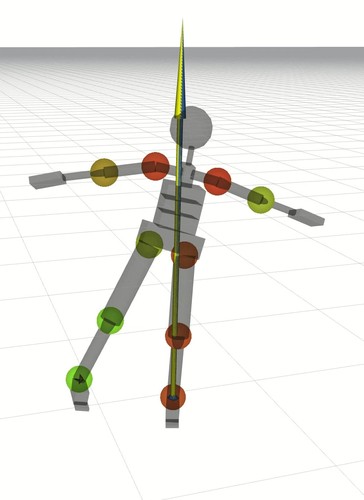}
      \caption{}
      \label{fig:rviz_left_foot_balancing}
  \end{subfigure}
  \begin{subfigure}[b]{0.35\textwidth}
      \centering
      \vspace{-1.5cm}
      \includegraphics[scale=0.25]{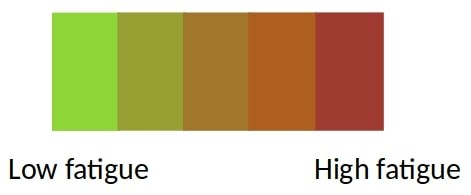}
      \label{fig:fatigue_legend_rotated}
  \end{subfigure}
  \caption{Rviz visualization of joint efforts for Task T2 balancing on the right foot (a) and Task T3 balancing on the left foot (b)}
  \label{fig:rviz_balancing_efforts}
\end{figure}

\begin{figure}[H]
    \centering
    \includegraphics[width=\textwidth]{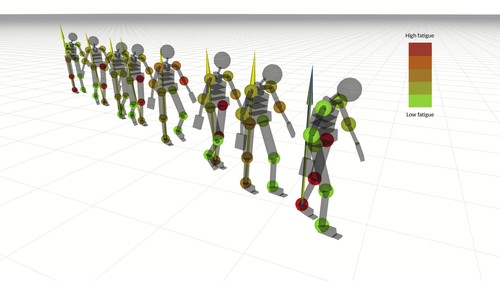}
    \caption{Rviz visualization of joint efforts during different phases of Task 4, walking}
    \label{fig:rviz_walking}
\end{figure}

\chapter{Software Architecture}
\label{cha:software-architecture}

\chapreface{Our framework of holistic human perception is aimed to be a platform for future extensions that can incorporate different sensory modalities such as hear rate monitoring, muscular fatigure, emotional fatigue for various aspects of human perception. Given the anticipated complexity, a sound software architecture is of utmost importance for our system to be robust. This chapter presents our software architecture and explains the main components that are currently functional.}

\section{Motivation}

Our ambition for holistic human perception is to develop a novel wearable whole body suit technology that integrates a lot of sensors and novel algorithms that processes the sensor inputs. Currently, at the hardware level we depend on commercial products like Xsens for human motion measurements and use our \emph{ftShoes} technology for force monitoring. In the future, we aim at using commercial platforms of \textsc{emg} sensors e.g., \textsc{bts} engineering for monitoring human muscles activation. So, a set of sensors can be considered as a system that provides a particular set of measurements, not particularly homogeneous\footnote{One example is an exoskeleton worn by the human that provides its joint states and torques}. Synchronizing the sensory information coming from different hardware is often daunting task and an ideal solution for achieving system integration of heterogeneous hardware components is to focus on the software level synchronization~\cite{samy2020fusion}. At the algorithmic level we currently perform real-time motion tracking and dynamics estimation. In the future, we aim at integrating algorithms for motion classification and prediction. Given the system complexity, it is crucial to have a modular software architecture that is extensible for future possibilities.  This chapter lays the details of our novel software architecture to support our vision of holistic human perception and we hope this can serve as a platform to develop new wearable technologies, both in research and industry, to enable effective human-robot collaboration.  

The software infrastructure is developed in \texttt{C++} using \textsc{yarp} middleware~\cite{paul2014middle}. We opted for a device based architecture based on \textsc{yarp} devices\footnote{\href{https://www.yarp.it/note_devices.html}{https://www.yarp.it/note\_devices.html}}. The three core components in this architecture are 1) Device Drivers or Devices, 2) Interfaces, and 3) Network Wrappers. A device in \textsc{yarp} is a \texttt{C++} class and interfaces are abstract base classes. A device can be thought of as a computation unit that processes some inputs and implement the methods defined in the interface. Two devices exchange information through the interfaces by an "attach" method which enables one device to view the interface implemented in the other device to which it is attached. Furthermore, multiple devices can be launched as a single process\footnote{All the devices need not be attached to one another}. Network wrappers are special cases of devices that provide network resources to send data to the network. A network wrapper attached to a device views the implemented interface and facilitates streaming of necessary data to the network.

\vspace{0.5cm}

At the highest level of abstraction we have two main layers: 1) Producer 2) Consumer

\vspace{-0.25cm}

\section{Producer Layer}

The producer layer is composed of all the wearable technology that provides various measurements. The technologies involved at this layer are heterogeneous such as Xsens and \emph{ftShoes}. An ideal design choice at the producer layer is to provide a unifying data structure down to the consumer layer. We developed an open source \texttt{C++} library called Wearables\footnote{\href{https://github.com/robotology/wearables}{https://github.com/robotology/wearables}} to combine data coming from the heterogeneous sources and provide a standard data output. The core components of the wearbles library are 1) Sensor and Wearable Interfaces, 2) Wearable Devices, and 3) Wearable Wrapper. Fig.~\ref{fig:software-architecture-wearables} provides the overview of the components in wearable library.

\begin{figure}[ht!]
	\centering
	\includegraphics[width=\textwidth]{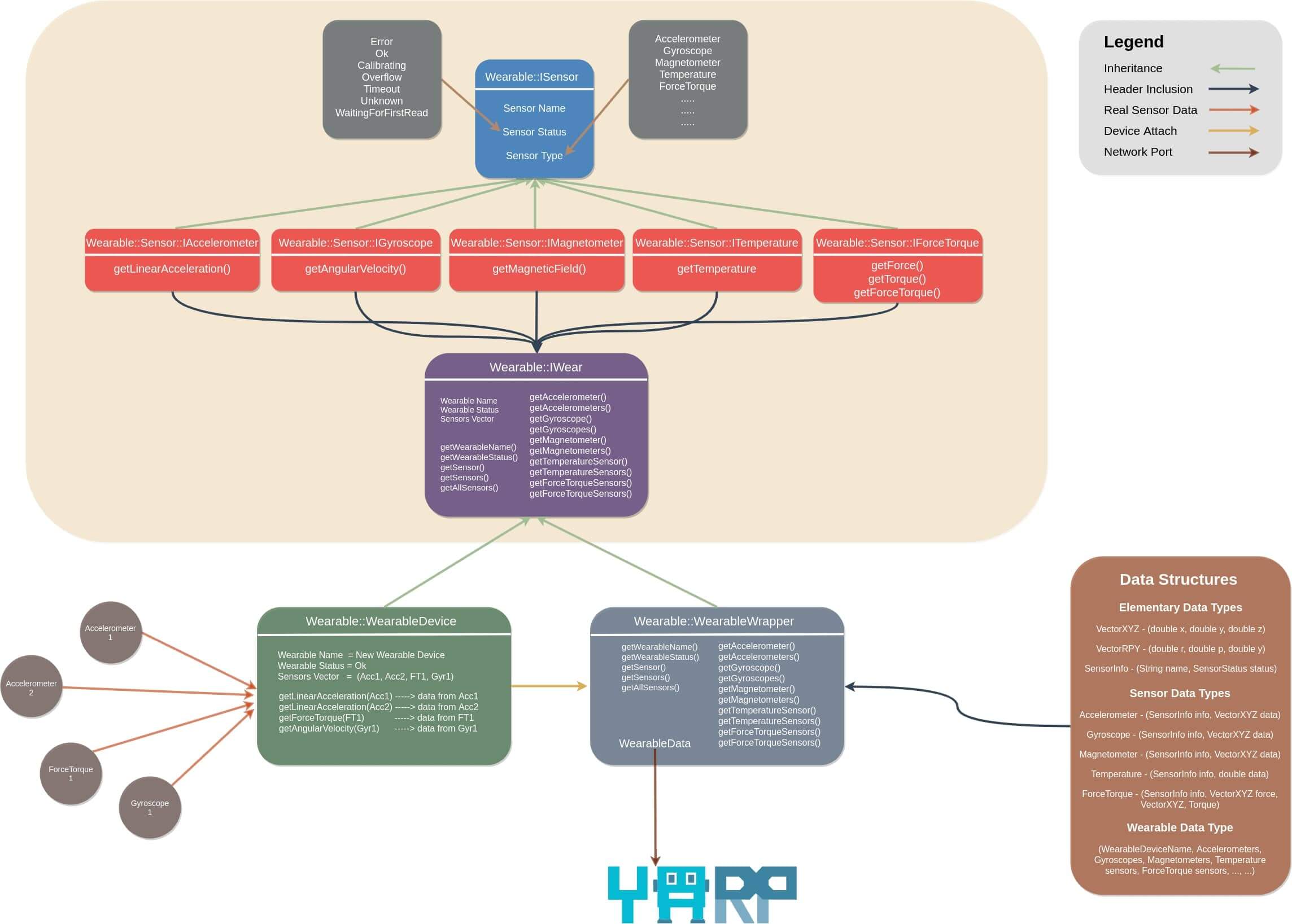}
	\caption{Software architecture of wearables library}
	\label{fig:software-architecture-wearables}
\end{figure}

\begin{sloppypar}

\subsection{Sensor \& Wearable Interfaces}

\subsubsection*{Sensor Interfaces}

The lowest level of abstract is sensor interface (\texttt{Wearable::ISensor}) containing the basic details that represent a sensor i.e., a) Sensor Name, b) Sensor Status, and c) Sensor Type. Sensor type represents the nature of the sensor and some of the examples are \texttt{accelerometers, gyroscopes, magnetometers, temperature sensors} and  \texttt{force torque sensors}. Sensor status indicates the state of the sensor and some of the possible states are \texttt{Error, Ok, Calibrating, Overflow, Timeout, Unknown} and \texttt{WaitingForFirstRead}.

\texttt{Wearable::ISensor} is the base abstract class from which particular type of sensor interfaces are inherited. Depending on the type of the sensor, relevant methods are present for accessing the data generated by the particular sensor. Some of the examples of the sensor type interfaces are:

\begin{itemize}
	\item \texttt{Wearable::Sensor::IAccelerometer} sensor interface that has \texttt{getLinearAcceleration()} method to access the linear acceleration data from accelerometers.
	\item \texttt{Wearable::Sensor::IGyroscope} sensor interface that has \texttt{getAngularVelocity()} method to access the angular velocity data from gyroscopes.
	\item \texttt{Wearable::Sensor::IMagnetometer} sensor interface that has \texttt{getMagneticField()} method to access the magnetic field data from magnetometers.
	\item \texttt{Wearable::Sensor::ITemperature} sensor interface that has \texttt{getTemperature()} method to access the temperature data from the sensor.
	\item \texttt{Wearable::Sensor::IForceTorque} sensor interface that has \texttt{getForce()}, \texttt{getTorque()}, and \texttt{getForceTorque()} methods to access the forces, moments, and forces and moments data from six axis force-torque sensors.	
\end{itemize}

\subsubsection*{Wearable Interface}

The next level of abstraction is the wearable interface that groups all the particular sensor type interfaces together. Similar to the \texttt{Wearable::ISensor} interface, some of the basic details are the name and the status along with the methods \texttt{getWearableName()} and \texttt{getWerableStatus()}. Additionally, there are methods like \texttt{getSensor()} to access a sensor based on the sensor name or \texttt{getSensors()} method to access all the sensors of a particular type. The \texttt{getAllSensors()} methods provides access to all the sensors. Furthermore, sensor specific methods like \texttt{getAccelerometer()} or \texttt{getAccelerometers()} provides access to sensors of a particular type. The \texttt{Wearable::IWear} interface is the base abstract class that forms the template for the rest of the components in the wearables library i.e., wearable device and wearable wrapper.

\subsection{Wearable Device}

A wearable device is a \texttt{C++} class that is inherited from the \texttt{Wearable::IWear} interface. The physical intuition of a wearable device is that it represents a system that constitute a set of sensors like Xsens or \emph{ftShoes}. A specific name is assigned to the wearable name parameter. The sensor measurements from real sensors of a particular type are given as inputs to a wearable device e.g., \texttt{Accelerometer 1, Accelerometer 2, Gyroscope 1, ForceTorque 1}. The status of a wearable device is determined by the status of all the sensors it encapsulates. Furthermore, a wearable device has an implementation of all the methods related to the sensors it encapsulates e.g.,  \texttt{getLinearAcceleration(Acc1), getLinearAcceleration(Acc2), getAngularVelocity(Gyr1), getForceTorque(FT1)}.

\subsection{Wearable Wrapper}

The wearable wrapper is  a \texttt{C++} class that is inherited from the \texttt{Wearable::IWear} interface. As stated before, it is also a device that can be attached to any wearable device. It implements methods to get sensors present in a particular wearable device to which the wearable wrapper is attached to e.g., \texttt{getAccelerometer(), getAccelerometers(), getGyroscope(), getGyroscopes(), getSensor(), getSensors(), getAllSensors()}. Furthermore, the data from different sensors present in a particular wearable device are serialized using a data structure template. The data structure template defines elementary data types that are composed to represent sensor specific data. Furthermore, the "Wearable Data" type is defined which has the fixed structure as (\texttt{WearableDeviceName::SensorName::SensorData}). As an example with respect to Fig.~\ref{fig:software-architecture-wearables} the wearable data output is composed as,

\begin{equation}
	\begin{matrix}
		\begin{array}{l}
			( \ (\texttt{WearableDeviceName::Accelerometer1::SensorData}), \\
			(\texttt{WearableDeviceName::Accelerometer2::SensorData}), \\
			(\texttt{WearableDeviceName::Gyroscope1::SensorData}), \\
			(), \text{---------> Mangetometer Data Placeholder} \\
			(), \text{---------> ForceTorque Data Placeholder} \\
			(\texttt{WearableDeviceName::ForceTorque1::SensorData}) \ ) \notag
		\end{array}
	\end{matrix}
\end{equation}

The wearable wrapper also implements a periodic thread\footnote{\href{http://www.icub.org/doc/icub-main/icub_periodic_thread.html}{http://www.icub.org/doc/icub-main/icub\_periodic\_thread.html}} that calls sensor methods like \texttt{getLinearAcceleration(Acc1)}, \texttt{getLinearAcceleration(Acc2)} of a wearable device to update the sensor readings from the associated sensors. Additionally. the wearable wrapper has the functionality of streaming the composed wearable data from a particular wearable device to the network through ports. Any other components can access the wearable data through the network ports.

\begin{figure}[ht!]
	\centering
	\includegraphics[width=\textwidth]{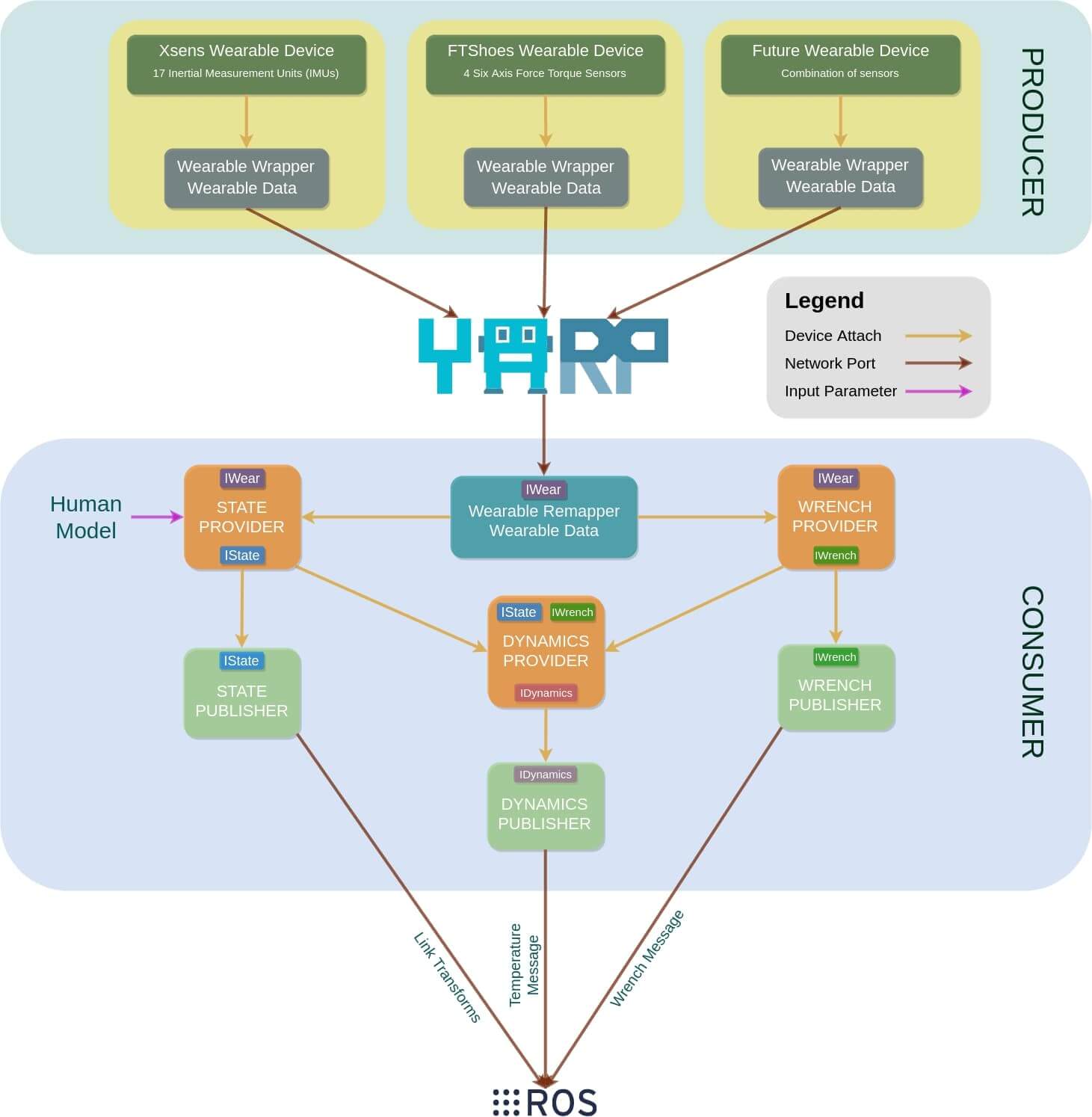}
	\caption{Software architecture for holistic human perception}
	\label{fig:software-architecture-human-dynamics-estimation}
\end{figure}

\subsection{Current Wearable Devices}

Our current architecture constitutes of two wearable devices that accesses data coming from the Xsens motion tracking system and \emph{ftShoes}. In the case of Xsens we rely on the software development kit \textsc{(sdk)} provided to access various human measurements such as sensor linear acceleration, link position and orientation. Similarly, in the case of \emph{ftShoes} we rely on interfaces developed to access data from six axis force-torque sensors developed at \textsc{iit} (Section~\ref{sec:force-torque-sensing}). For the sake of clarity, these details are not captured in Fig.~\ref{fig:software-architecture-human-dynamics-estimation}. Furthermore, a wearable device and an associated wearable wrapper are launched in a single process and they can be run on different machines connected to the same network.

\subsection{Future Wearable Devices}

Applications of human-robot collaboration may contain serial manipulators or humanoid robots or exoskeletons to augment humans. Under those scenarios, the robots or the exoskeleton can be considered as a new wearable device and implemented in the producer layer. So, various data like the joint quantities or interaction forces from these devices can be passed as wearable data for the consumer layer where new components can be implemented for different applications.

\section{Consumer Layer}

The consumer layer is composed of the components that utilizes the wearable data from the network. Fig.~\ref{fig:software-architecture-human-dynamics-estimation} presents a higher level of abstraction in the organization of producer and consumer layers. Similar to the components of the producer layer, the components of the consumer layer are also \textsc{yarp} device drivers. The main components present at the consumer layer are 1) Wearable Remapper, 2) Providers, and 3) Publishers. Another level of classification that can be done in the consumer layer is related to the aspect they deal with such as the state or the dynamics. Providers are devices that are attached to a wearable remapper device and access the wearable data through the wearable interface. Furthermore, providers are algorithmic units that implement a particular algorithm e.g. state estimation. On the other hand, publishers are devices that access output data from providers through relevant interfaces and publish the data to a \textsc{ros} network~\cite{ros2018}. This facilitates the use of \textsc{ros} tools like Rviz visualizer for real-time visualization of human quantities. All the components of the consumer layer that deals with the human quantities are implemented as an open source \texttt{C++} library called human-dynamics-estimation\footnote{\href{https://github.com/robotology/human-dynamics-estimation}{https://github.com/robotology/human-dynamics-estimation}}. The following section lays the details of the components at the consumer layer.

\subsection{Wearable Remapper}

A wearable remapper is a device that servers the dual purpose of a wearable wrapper. Analogous to the wearable wrapper, the remapper is also inherited from \texttt{Wearable::IWear} interface.  It connects to the network ports to receive the wearable data sent to the network by the wearable wrapper. So, the wearable remapper can be thought as a component that deserializes the wearable data into respective sensor data coming from several wearable devices.

\subsection{State Components}

The state provider is a device that is inherited from \texttt{Wearable::IWear} interface. It is attached to the wearable remapper and accesses the data from Xsens wearable device through \texttt{Wearable::IWear} interface. One of the input parameters is the human model described in Chapter~\ref{cha:human-modeling}. The whole-body inverse kinematics problem for real-time human motion tracking described in Chapter~\ref{cha:human-kinematics-estimation} is implemented inside the state provider. The state provider is inherited from \texttt{IState} interface and the state estimation quantities i.e. human base pose, joint positions and velocities data is handled through the implementation of \texttt{IState} interface methods.

The state publisher is a device that is attached to the state provider. It accesses the estimated human state quantities through the \texttt{IState} interface and computes the link tranforms information through forward kinematics with the given human model. This information is essential for Rviz visualizer to animate the human model as shown in the case of human walking Fig~\ref{fig:rviz_walking}. The computed link transforms are streamed to \textsc{ros} network through topics\footnote{\href{http://wiki.ros.org/Topics}{http://wiki.ros.org/Topics}}.   

\subsection{Wrench Components}

The wrench provider is a device that is inherited from \texttt{Wearable::IWear} interface. It is attached to the wearable remapper and accesses data from \emph{ftShoes} wearable device through \texttt{Wearable::IWear} interface. This device transforms the wrench data and expresses them with respect to the human links that are given as an input parameter. Currently, the wrench information is available only on the human feet. However, this can be expanded in the future for applications of \textsc{hrc}, where a humanoid robot or an exoskeleton can be considered as a wearable device and the wrench information from them is passed through the wearable data. The wrench provider is inherited from \texttt{IWrench} interface and the transformed wrench information coming various wearable devices is handled through the implementation of \texttt{IWrench} interface methods.

The wrench publisher is a device that is attached to the wrench provider. It accesses the transformed wrench through \texttt{IWrench} interface and expresses them as a wrench message data type to stream to \textsc{ros} network through topics. This information is essential for Rviz visualizer to display the arrows corresponding to the external wrench as shown in Fig.~\ref{fig:rviz_balancing_efforts}.

\subsection{Dynamics Components}

The dynamics provider is a device that is inherited from \texttt{IState} and \texttt{IWrench} interfaces. It is attached to the state provider and the wrench provider. The human state quantities are received through \texttt{IState} interface and the wrench quantities are received through \texttt{IWrench} interface. The human dynamics estimation algorithm described in Chapter~\ref{cha:human-dynamics-estimation} is implemented in the dynamics provider. It is also inherited from \texttt{IDynamics} interface and implements the methods to handle the human joint torque estimates.

The dynamics publisher is a device that is attached to the dynamics provider. It accesses the human joint torque estimates through \texttt{IDynamics} interface and computes the joint effort that is described in section~\ref{sec:HDE-joint-torque-estimates}. The joint effort information is sent as a temperature message to \textsc{ros} network through topics. This information is essential for Rviz visualizer to display the spheres at the human joints that show the joint effort as highlighted in Fig.~\ref{fig:rviz_balancing_efforts}. Furthermore, the dynamics publisher also expresses the external wrench estimates as a wrench message data type to stream to \textsc{ros} network through topics. This is essential to display the arrows (blue color) shown in Fig.~\ref{fig:rviz_balancing_efforts}.

\end{sloppypar}

\part{Reactive Robot Control} 
\chapter{Partner-Aware Control}
\label{cha:partner-aware-control}

\chapreface{Classical feedback linearization approach has been proved to be very robust in controlling complex robots. However, there are some clear limitations to it for collaborative tasks. This chapter presents those limitations first, followed by a coupled dynamics formalism for collaboration scenarios. Furthermore, the interaction characterization in terms of joint torques from an external agent assisting the robot are considered towards partner-aware control techniques validated through experiments with two humanoid  robots.}

\section{Recourse on Classical Robot Control}
\label{sec:recourse-on-classical-control}

Typically, one may be interested in controlling a robot related quantity, say $\comVar{x} \in \mathbb{R}^{p}$. In the case of joint space control, the interested task is to control a single robot joint or a group of robot joints along a pre-defined trajectory. Similarly, for the task space control, the interested task is to achieve a Cartesian trajectory tracking by a link of the robot where $\comVar{x}_d(t), \dot{\comVar{x}}_d(t), \ddot{\comVar{x}}_d(t) \in \mathbb{R}^6$ denote the desired pose, velocity and acceleration in the Cartesian space, parametrized in time $t$. The velocity tracking error to be minimized is denoted by $\dot{\widetilde{\comVar{x}}} = \dot{\comVar{x}}(t) - \dot{\comVar{x}}_d(t)$.

The robot link's actual velocity has a linear map to the robot's velocity through the Jacobian matrix $\comVar{J}(\comVar{q})_{\comVar{x}} \in \mathbb{R}^{6 \times (n+6) }$, i.e.

\begin{equation}
\dot{\comVar{x}}(t) = \comVar{J}_{\comVar{x}}(\comVar{q}) ~{\robnu}
\label{eq:linear-jacobian-mapping}
\end{equation}

The task of Cartesian trajectory tracking is achieved through a simple control objective as defined below, 

\begin{equation}
	\label{eq:cartesian-control-objective}
	\ddot{\comVar{x}} = \ddot{\comVar{x}}^* := \ddot{\comVar{x}}_d - \comVar{K}_D \ \dot{\widetilde{\comVar{x}}} - \comVar{K}_P \int_0^t \dot{\widetilde{\comVar{x}}} ~du, \quad \comVar{K}_D, \comVar{K}_P > 0
\end{equation}

where $\comVar{K}_P, \comVar{K}_D \in \mathbb{R}^{6 \times 6}$ are positive definite symmetric feedback matrices.

Classical feedback linearization approach \cite{khalil2004modeling} helps us in finding the robot joint torques ${\robtau}$ needed to realize the control objective \eqref{eq:cartesian-control-objective} for Cartesian trajectory tracking and keeping the tracking error minimum. The robot control torques necessary for trajectory tracking with the desired dynamics directed by Eq.~\eqref{eq:cartesian-control-objective} are obtained using Eq.~\eqref{eq:linear-jacobian-mapping}. On differentiating $\dot{\comVar{x}}(t)$ we get the following relation,

\begin{equation}
	\ddot{\comVar{x}}(t) = \comVar{J}_{\comVar{x}} ~\dot{\robnu} + \dot{\comVar{J}}_{\comVar{x}} ~{\robnu}
	\label{eq:velocity-differentiaion}
\end{equation}

The quantity $\dot{\robnu}$ in the above equation is the robot's acceleration that can be derived from the equations of motion \eqref{eq:equations-of-motion} as $\dot{\robnu} = \comVar{M}^{-1}[\comVar{B} {\robtau} + \comVar{J}_c^\top {\comVar{ f}} - \comVar{h}]$ where, $\comVar{h} = \comVar{C}(\comVar{q},{\robnu}) {\robnu} + \comVar{G}(\comVar{q})$. Using this relation in \eqref{eq:velocity-differentiaion}, we get

\begin{subequations}
	\begin{equation}
	\ddot{\comVar{x}}(t) = \comVar{J}_{\comVar{x}} \comVar{M}^{-1}[\comVar{B} {\robtau} + \comVar{J}_c^\top {\comVar{ f}} - \comVar{h}] + \dot{\comVar{J}}_{\comVar{x}} ~{\robnu} \notag
	\end{equation}
	\begin{equation}
	\ddot{\comVar{x}}(t) = \comVar{J}_{\comVar{x}} \comVar{M}^{-1} \comVar{B} {\robtau} + \comVar{J}_{\comVar{x}} \comVar{M}^{-1} \comVar{J}_c^\top {\comVar{ f}} - \comVar{J}_{\comVar{x}} \comVar{M}^{-1} \comVar{h} + \dot{\comVar{J}}_{\comVar{x}} ~{\robnu} \notag
	\end{equation}
	\begin{equation}
	\comVar{J}_{\comVar{x}} \comVar{M}^{-1} \comVar{B} {\robtau} = \ddot{\comVar{x}}(t)  - \comVar{J}_{\comVar{x}} \comVar{M}^{-1} \comVar{J}_c^\top {\comVar{ f}} + \comVar{J}_{\comVar{x}} \comVar{M}^{-1} \comVar{h} - \dot{\comVar{J}}_{\comVar{x}} ~{\robnu} \notag
	\end{equation}
	\begin{equation}
	{\robtau} = [\comVar{J}_{\comVar{x}} \comVar{M}^{-1} \comVar{B}]^{\dagger} [\ddot{\comVar{x}}(t)  - \comVar{J}_{\comVar{x}} \comVar{M}^{-1} \comVar{J}_c^\top {\comVar{ f}} + \comVar{J}_{\comVar{x}} \comVar{M}^{-1} \comVar{h} - \dot{\comVar{J}}_{\comVar{x}} ~{\robnu}] \notag
	\end{equation}
\end{subequations}

Now, using the desired dynamics from Eq.~\eqref{eq:cartesian-control-objective}, we compute the control torques as

\begin{equation}
{\robtau} = [\comVar{J}_{\comVar{x}} \comVar{M}^{-1} \comVar{B}]^{\dagger} [\ddot{\comVar{x}}^*  - \comVar{J}_{\comVar{x}} \comVar{M}^{-1} \comVar{J}_c^\top {\comVar{ f}} + \comVar{J}_{\comVar{x}} \comVar{M}^{-1} \comVar{h} - \dot{\comVar{J}}_{\comVar{x}} ~{\robnu}]
\label{eq:normal-control-torques}
\end{equation}

On putting in a compact form, we have:

\begin{equation}
	{\robtau} = \mathbold{\Delta}^{\dagger} [\ddot{\comVar{x}}^*  - \mathbold{\Omega} {\comVar{ f}} + \mathbold{\Lambda}] + \comVar{N}_{\mathbold{\Delta}} {\robtau}_0
	\label{eq:normal-control-torques-compact}
\end{equation}

where 

\begin{itemize}
	\item $ \mathbold{\Delta} = \comVar{J}_{\comVar{x}} \comVar{M}^{-1} \comVar{B} \in \mathbb{R}^{6 \times n} $
	\item $ \mathbold{\Omega} = \comVar{J}_{\comVar{x}} \comVar{M}^{-1} \comVar{J}_c^\top \in \mathbb{R}^{6 \times 6n_c}$
	\item $ \mathbold{\Lambda} = \comVar{J}_{\comVar{x}} \ \comVar{M}^{-1} \comVar{h} - \dot{\comVar{J}}_{\comVar{x}} ~{\robnu} \in \mathbb{R}^{6}$
	\item $ \comVar{N}_{\mathbold{\Delta}} \in \mathbb{R}^{n \times n}$ is the nullspace projector of $\mathbold{\Delta}$ 
	\item $ {\robtau}_0 \in \mathbb{R}^{n}$ represent torques required to satisfy lower priority tasks in case of redundancy in joint torques
\end{itemize}

The control torques generated using the relation \eqref{eq:normal-control-torques-compact} eventually facilitate good trajectory tracking based on the proper choice of the gains $\comVar{K}_P, \comVar{K}_D \in \mathbb{R}^{6 \times 6}$. Another key detail to observer here is that the control torques generated completely cancel out any external wrench applied to the robot. Although it proves to be robust to unwanted external perturbations, classical feedback linearization control approach is fundamentally limiting for human-robot collaboration applications that require active collaboration between an external agent and a robot through intentional physical interactions. Unlike unexpected collisions, intentional physical interactions during a collaborative task are often helpful to achieve the common goal. Under such situations, it makes sense not to be agnostic to the external interaction but rather exploit when it is helpful. However, two important questions to be addressed mathematically are: 1) How to classify the interaction as helpful, and 2) How to exploit the helpful interaction. The following subsections answer these two important questions.

\subsection{Helpful Interaction}
\label{sec:helpful-interaction}

To explore the idea of mathematically characterizing an interaction to be helpful, let us consider a simple toy problem with an object of negligible mass $m$ as shown in Fig.~\ref{fig:mass-1d-case}. Assume a control input $u$ that is applied to the object and an external interaction force $f^{x}$ acting on the object.

\begin{figure}[ht]
	\centering
	\includegraphics[width=0.9\textwidth]{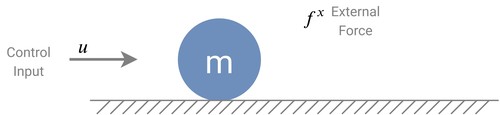}
	\caption{A object of negligible mass}
	\label{fig:mass-1d-case}
\end{figure}

The equations of motion for the system relates the object acceleration $\ddot{x}$ to the control input $u$ and the external force $f^{x}$, given by the following relation,

\begin{equation}
	\ddot{x} = u + f^{x}
	\label{eq:1d-equations-of-motion}
\end{equation}

Let us consider a task of moving the object along a 1D trajectory. A simple trajectory tracking control objective can be defined as,

\begin{equation}
\ddot{{x}} = \ddot{{x}}^* := \ddot{{x}}_d - {K}_D \ \dot{{\widetilde{x}}} - {K}_P \int_0^t \dot{{\widetilde{x}}} ~du, \quad {K}_D, {K}_P > 0
\end{equation}  

To understand the stability of the system, under the influence of the control input $u$ and under the influence of the external force $f^{x}$, let us consider a simple Lyapunov function as below,

\begin{equation}
\mathrm{V} = \frac{1}{2} \norm{\dot{{x}} - \dot{{x}}_d}^{2} + \frac{{K}_P}{2} \norm{\int_{0}^{t}(\dot{{x}} - \dot{{x}}_d) ~du}^{2}
\end{equation}

On time differentiating the above function we get the relation,

\begin{equation}
\dot{\mathrm{V}} = \dot{\widetilde{{x}}} \ [ \ddot{\widetilde{{x}}} +  {K}_P  \int_{0}^{t}\dot{\widetilde{{x}}} ~du ] \notag
\end{equation}

Using the equation of acceleration $\ddot{x}$ in the above equation we get,

\begin{equation}
\dot{\mathrm{V}} = \dot{\widetilde{{x}}} \ [ u + f^{x} - \ddot{x}_d +  {K}_P  \int_{0}^{t}\dot{\widetilde{{x}}} ~du ] \notag
\end{equation}

Given the task of trajectory tracking, we choose the control input $u$ to be the desired dynamics i.e., $u = \ddot{{x}}^*$. Based on this choice for the control input, the derivative of the Lyapunov function becomes, 

\begin{equation}
\dot{\mathrm{V}} = \dot{\widetilde{{x}}} f^{x} - \dot{\widetilde{{x}}} K_D \dot{\widetilde{{x}}} \notag
\end{equation}

The stability of the system is ensured as long as the derivative of the Lyapunov function is negative i.e., $\dot{\mathrm{V}} \le 0$. The term - $\dot{\widetilde{{x}}} K_D \dot{\widetilde{{x}}} \le 0$. So, the stability of the system under the influence of the external force $f^{x}$ is ensured as long as the term $\dot{\widetilde{{x}}} f^{x} \le 0$, and the interaction force can be characterized as a helpful interaction.

An intuitive understanding of characterizing an interaction force to be helpful can be better understood through the Fig.~\ref{fig:1d-helpful-interaction}. Consider the Case I, when the current velocity is lagging behind the desired velocity. The velocity error $\dot{\widetilde{x}} = \dot{x} - \dot{x}_d < 0$. Under such conditions, an external interaction force $f^{x} > 0$, applied in the positive direction, pushes the object towards the desired velocity. This interaction force $f^{x}$ is considered to be helpful as the term $\dot{\widetilde{x}} f^{x} < 0$. Similarly, in Case II, when the current velocity is leading the desired velocity. The velocity error $\dot{\widetilde{x}} = \dot{x} - \dot{x}_d > 0$. Under such conditions, an external interaction force $f^{x} < 0$, applied in the negative direction, pushes the object back towards the desired velocity. This interaction force $f^{x}$ is considered to be helpful as the term $\dot{\widetilde{x}} f^{x} < 0$.

\begin{figure}[ht]
	\centering
	\includegraphics[width=0.95\textwidth]{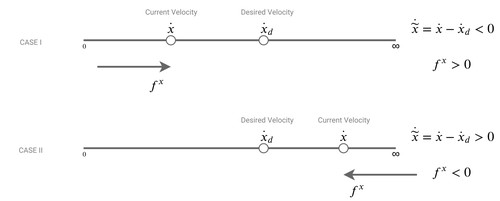}
	\caption{Helpful interaction intuition}
	\label{fig:1d-helpful-interaction}
\end{figure}

\subsection{Interaction Exploitation}
\label{sec:interaction-exploitation}

This section presents the idea and the intuition behind the exploitation of helpful interaction. Let us consider the Cartesian trajectory tracking task and consider the following control objective, 

\begin{equation}
\label{eq:updated-cartesian-control-objective}
\ddot{\comVar{x}} = \ddot{\comVar{x}}^* := \ddot{\comVar{x}}_d - \comVar{K}_D \ \dot{\widetilde{\comVar{x}}} - \comVar{K}_P \int_0^t \dot{\widetilde{\comVar{x}}} ~du + \mathbold{\Omega} {\comVar{ f}}
\end{equation}

Similar to the Lyapunov stability analysis presented in subsection~\ref{sec:helpful-interaction}, let us consider a simple Lypunov function as below,

\begin{equation}
\label{eq:classical-control-LyapunovFunction}
\mathrm{V} = \frac{1}{2} \norm{\dot{\comVar{x}} - \dot{\comVar{x}}_d}^{2} + \frac{\comVar{K}_P}{2} \norm{\int_{0}^{t}(\dot{\comVar{x}} - \dot{\comVar{x}}_d) ~du}^{2}
\end{equation}

On taking the time derivation of eq.~\eqref{eq:classical-control-LyapunovFunction}, and using the desired dynamics following the control objective of eq.~\eqref{eq:updated-cartesian-control-objective}, we get

\begin{equation}
\dot{\mathrm{V}} = - \dot{\widetilde{\comVar{x}}} \ \comVar{K}_D \ \dot{\widetilde{\comVar{x}}} + \dot{\widetilde{\comVar{x}}} \ \mathbold{\Omega} {\comVar{ f}} 
\label{eq:updated-lyapunov-function-time-derivative}
\end{equation}

Now, the stability of the system under the influence of the external interaction wrench $\comVar{f}$ depends on the second term being negative, i.e., $\dot{\widetilde{\comVar{x}}} \ \mathbold{\Omega} {\comVar{ f}} \le 0$. Let us decompose the term $\mathbold{\Omega} {\comVar{ f}}$ along the parallel and perpendicular directions of the velocity error as following,

\begin{subequations}
	
	\begin{equation}
	\mathbold{\Omega} {\comVar{ f}} =
	\alpha \ \dot{\widetilde{\comVar{x}}}^{\parallel} +
	\beta \ \dot{\widetilde{\comVar{x}}}^{\perp}
	\end{equation}
	
	\begin{equation}
	\dot{\widetilde{\comVar{x}}}^{\parallel} = \frac{\dot{\widetilde{\comVar{x}}}}{\norm{\dot{\widetilde{\comVar{x}}}}}, \
	\alpha = \frac{\dot{\widetilde{\comVar{x}}}^\top \mathbold{\Omega} \ {\comVar{f}}}{\norm{\dot{\widetilde{\comVar{x}}}}}
	\end{equation}
	\label{eq:wrench-decomposition}
\end{subequations}

Following the helpful interaction characterization idea of section~\ref{sec:helpful-interaction}, the parallel wrench component $\alpha$ can be classified as a helpful interaction when $\alpha \le 0$. Now, the idea of exploiting helpful interaction is achieved through considering the desired dynamics with the helpful interaction as following, 

\[
\ddot{\comVar{x}}^* :=
\begin{cases}

 \ddot{\comVar{x}}_d - \comVar{K}_D \ \dot{\widetilde{\comVar{x}}} - \comVar{K}_P \int_0^t \dot{\widetilde{\comVar{x}}} ~du & \text{if} \quad \alpha > 0 \\
  \ddot{\comVar{x}}_d - \comVar{K}_D \ \dot{\widetilde{\comVar{x}}} - \comVar{K}_P \int_0^t \dot{\widetilde{\comVar{x}}} ~du + \alpha \norm{\dot{\widetilde{\comVar{x}}}}  & \text{if} \quad \alpha \le 0
\end{cases}
\]

So, even when the feedback linearized robot torques generated through the relation from eq.~\eqref{eq:normal-control-torques-compact} cancel any external interaction wrench, we can consider the helpful interaction ($\alpha \le 0$) to be part of the desired dynamics and generate the robot torques by exploiting the helpful interaction. A preliminary investigation of the helpful interaction exploitation to generate robot control torques was carried in~\cite{8093992}. However, the external interaction wrench considered is not directly measured through the force-torque sensors, but are estimated quantities through the robot control torques. This results in an algebraic loop, when the estimated interaction wrench is considered to generate the robot control torques by exploiting the helpful interaction. This chapter presents a more integrated formulation to achieve partner-aware control by leveraging the interaction in terms of the external agent joint torques.

\section*{Notational Nuances}

The reader is advised to pay close attention to the notational nuances presented here in addition to the notion presented in the background notation~\eqref{sec:background-notation}. This notation is followed for the rest of this chapter and the associated Appendix~\ref{app:whole-body-momentum-control}.

\[
  \left[
      \begin{tabular}{@{\quad}m{.05\textwidth}@{\quad}m{.83\textwidth}}
        {\Huge \faInfoCircle} & \raggedright \textbf{} \par
          \begin{tabular}{@{}p{0.18\textwidth}p{0.60\textwidth}@{}}
            $^{r}(.)$                 & Robot quantity (non-bold small letter)\\
            $^{ea}(.)$                & External Agent quantity (non-bold small letter)\\
            \rowcolor{Gray} 
            \multicolumn{2}{l}{\large \textbf{Robot Agent Quantities}} \\
            \rowcolor{White}
            $\robVar p$                 & Vector (straight bold small letter)\\
            $\robVar A$                 & Matrix, tensor (straight bold capital letter)\\
            \rowcolor{Gray} 
            \multicolumn{2}{l}{\large \textbf{External Agent Quantities}} \\
            \rowcolor{White}
            $\ageVar p$                 & Vector (straight double-bold small letter)\\
            $\ageVar A$                 & Matrix, tensor (straight double-bold capital letter)\\
            \rowcolor{Gray} 
            \multicolumn{2}{l}{\large \textbf{Common/Coupled/Composite Quantities}} \\
            \rowcolor{White}
            $\comVar p$        & vector (slanted-bold small letter) \\
            $\comVar A$        & Matrix, tensor (slanted-bold capital letter) \\
          \end{tabular}
      \end{tabular}
    \right]
\]

\section{Coupled Modeling}
 
The external agent related terms are denoted with ''$\ageVar{straight \ double-bold}$'' notation, e.g., $\ageVar{M}$; robot related terms are denoted with ``$\robVar{straight \ bold}$'' notation, e.g., $\robVar{M}$; and terms that are common/coupled/composite are denoted with $\comVar{slanted \ bold}$ notation, e.g., $\comVar{M}$.

A typical physical agent-robot interaction scenario where the external agent is a human is shown in Fig. \ref{pHRI}. There are two agents: the external agent, and the robot. Both agents are physically interacting with the environment and, besides, are also engaged in physical interaction with each other.

\begin{figure}[ht]
	\centering
	\includegraphics[scale=0.4]{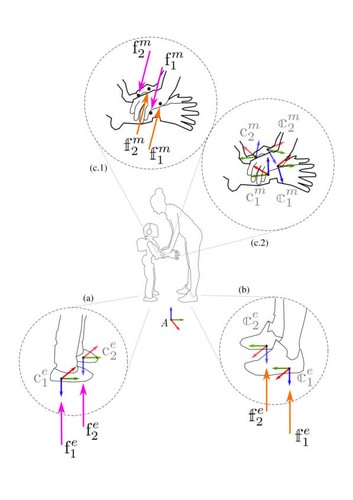}
	\caption{A typical agent-robot dynamic interaction scenario}
	\label{pHRI}
\end{figure}

In the first approximation, both the external agent (\emph{ea}) and the robot agent (\emph{r}) can be considered as multi rigid-body mechanical systems composed of  $^{ea}n+1$ and $^{r}n+1$ rigid bodies respectively, called links, connected through $^{ea}n \in \mathbb{N}$ and $^{r}n \in \mathbb{N}$ joints with one degree of freedom. In case of the human being an external agent, even though the assumption of a human body being modeled as rigid bodies is far from reality, it serves as a rough approximation when formulating physical human-robot interaction dynamics and allows us to synthesize robot controllers optimizing both human and robot variables. Further, we consider both the external agent and the robot to be \textit{free-floating} systems, i.e. none of the links have an \textit{a priori} constant pose with respect to the inertial frame.

To the purpose of finding mathematical models for both the external agent and the robot, we apply the Euler-Poincar\'{e} formalism to both multi-body systems~\cite{Marsden2010}. The equations of motion describing the dynamics of the external agent and the robotic agents respectively are:

\begin{subequations}
\begin{equation}
	\ageVar{M}(\ageVar{q}) \dot{\agenu} + \ageVar{C}(\ageVar{q},{\agenu}) {\agenu} + \ageVar{G}(\ageVar{q}) =
	 \ageVar{B} {\agetau} + \ageVar{J}^\top_{c} \ageVar{ f}
	\label{NEHuman2}
\end{equation}
\begin{equation}
	\robVar{M}(\robVar{q}) \dot{\robnu} + \robVar{C}(\robVar{q},\robnu) \robnu + \robVar{G}(\robVar{q}) =
	\robVar{B} {\robtau} + \robVar{J}^\top_{c} \robVar{ f}
	\label{NERobot2}
\end{equation}
\label{NE}
\end{subequations}

The number of degrees of freedom of all the joints for the external agent is denoted by  $^{ea}n \in \mathbb{N}$ and $^{r}n \in \mathbb{N}$ for the robot agent. For the sake of clarity, the superscript is ignored when referring to a quantity that corresponds to both the agents. Another assumption is that each agent is subject to a total of $n_c = n_m+n_e \in \mathbb{N}$ distinct wrenches. These wrenches are composed of two subsets: the wrenches due to \textit{mutual} interaction, denoted with the subscript $m$ and the wrenches exchanged between the \textit{external agent} and the \textit{environment}, denoted with the subscript $e$ respectively, In either case, the contact wrenches are represented by:

\begin{equation}
	\comVar{ f} = \begin{bmatrix}
		 	\comVar{ f}^{\scalebox{\hrsup_size}{\textit{m}}}_1; &&
		 	\comVar{ f}^{\scalebox{\hrsup_size}{\textit{m}}}_2; && .... && \comVar{ f}^{\scalebox{\hrsup_size}{\textit{m}}}_{n_m}; &&
		 	\comVar{ f}^{\scalebox{\hrsup_size}{\textit{e}}}_{{n_m}+1}; &&
		 	\comVar{ f}^{\scalebox{\hrsup_size}{\textit{e}}}_{{n_m}+2}; && .... && \comVar{ f}^{\scalebox{\hrsup_size}{\textit{e}}}_{n_m}
	      \end{bmatrix} \in \mathbb{R}^{6{n_c}} \notag
\end{equation}

Accordingly,
$\ageVar{ f} = \begin{bmatrix}
			        \ageVar{ f}_m && \ageVar{ f}_e
                 \end{bmatrix}^\top$
with $\ageVar{ f}_m$ the external wrenches applied on the external agent by the robotic agent and $\ageVar{ f}_e$ the external wrenches applied on the external agent by the environment. 
Similarly, $\robVar{ f} = \begin{bmatrix}
				    \robVar{ f}_m && \robVar{ f}_e
			   \end{bmatrix}^\top$
with $\robVar{ f}_m$ the external wrenches applied on the robotic agent by the external agent and $\robVar{ f}_e$ the external wrenches applied on the robotic agent by the environment.

A set of frames are defined: 
$\mathscr{C} = \{\mathcal{C}_1,\mathcal{C}_2,....\mathcal{C}_{n_m},\mathcal{C}_{n_{m+1}},\mathcal{C}_{n_{m+2}},....,\mathcal{C}_{n_e} \}$
and assume that the application point of the $k$-th external wrench on an agent is associated with a frame $\mathcal{C}_k \in \mathscr{C}$, attached to the agent's link on which the wrench acts, and has \textit{z}-axis pointing in the direction normal to the contact plane. Furthermore, the external wrench $\comVar{ f}^{\scalebox{\hrsup_size}{m/e}}_k$ is expressed in a frame whose orientation is that of the inertial frame $\mathcal{I}$, and whose origin is that of the frame $\mathcal{C}_k$.

Following the relation in Eq.~\eqref{eq:linear-jacobian-mapping}, the jacobian $\comVar{J}_{\mathcal{C}_k} = \comVar{J}_{\mathcal{C}_k}(\comVar{q})$ is the map between the agent's velocity $\nu = (^{\mathcal{I}}{}{\comVar{v}}_{\mathcal{F}}, \dot{\comVar{s}}) \in \mathbb{R}^{n+6}$ and the velocity of the frame $\mathcal{C}_k$ given by $^{\mathcal{I}}{\comVar{v}}_{\mathcal{C}_k} = [ ^{\mathcal{I}}\dot{\comVar{p}}_{\mathcal{C}_k}; \ ^{\mathcal{I}}{{\comVar{\omega}}_{\mathcal{C}_k}} ]$:

\begin{equation}
	^{\mathcal{I}}{\comVar{v}}_{\mathcal{C}_k} = \comVar{J}_{\mathcal{C}_k} \ {\nu}
	\label{EEVelocity}
\end{equation}

For a floating base system, the jacobian matrix has the following structure~\cite{Featherstone2007}:

\begin{subequations}
\begin{equation}
	{\comVar{J}_{\mathcal{C}_k}}(\comVar{q}) = \begin{bmatrix}
						{\comVar{J}^b_{\comVar{C}_k}}(\comVar{q}) && {\comVar{J}^j_{\mathcal{C}_k}}(\comVar{q})
				   \end{bmatrix} \in \mathbb{R}^{6 \times n+6}
\end{equation}

\begin{equation}
	{\comVar{J}^b_{\mathcal{C}_k}}(\comVar{q}) = \begin{bmatrix}
						\comVar{1}_3 && -\bm \skewOp({^{\mathcal{I}}{\comVar{p}}_{\mathcal{C}_k}} - {^{\mathcal{I}}{\comVar{p}}_{\mathcal{F}}}) \\
						\comVar{0}_{3 \times 3} && \comVar{1}_3
					\end{bmatrix} \in \mathbb{R}^{6 \times 6}
\end{equation}
\end{subequations}

\subsection{Contact Constraints}

We assume that holonomic constraints of the form $c(\comVar{q}) = 0$ act on both the external agent and the robot during their interaction with the environment. The links that are in contact with the ground can be considered as end-effector links that are rigidly fixed to the ground for the duration of the contact and hence have zero velocity. Following equation (\ref{EEVelocity}), this can be represented as follows for the human and the robot respectively:

\begin{subequations}
	\begin{equation}
		\ageVar{J}_{\ageVar{c}_k}	{\agenu} = 0
	\end{equation}
	\begin{equation}
		\robVar{J}_{\robVar{c}_k} \robnu = 0
	\end{equation}
\end{subequations}
Differentiating the above kinematic constraints yields:
\begin{subequations}
	\begin{equation}
		\begin{bmatrix}
			\ageVar{J}^b_{\ageVar{c}_k} && \ageVar{J}^j_{\ageVar{c}_k}
		\end{bmatrix}
		\begin{bmatrix}
			\dot{\ageVar{v}}_B \\
			\ddot{\ageVar{s}}
		\end{bmatrix} +
		\begin{bmatrix}
			\dot{\ageVar{J}}^b_{\ageVar{c}_k} && \dot{\ageVar{J}}^j_{\ageVar{c}_k}
		\end{bmatrix}
		\begin{bmatrix}
			\ageVar{v}_B \\
			\dot{\ageVar{s}}
		\end{bmatrix} = 0
		\label{holcon1}
	\end{equation}
	\begin{equation}
		\begin{bmatrix}
			\robVar{J}^b_{\robVar{c}_k} && \robVar{J}^j_{\robVar{c}_k}
		\end{bmatrix}
		\begin{bmatrix}
			\dot{\robVar{v}}_B \\
			\ddot{\robVar{s}}
		\end{bmatrix} +
		\begin{bmatrix}
			\dot{\robVar{J}}^b_{\robVar{c}_k} && \dot{\robVar{J}}^j_{\robVar{c}_k}
		\end{bmatrix}
		\begin{bmatrix}
			\robVar{v}_B \\
			\dot{\robVar{s}}
		\end{bmatrix} = 0
		\label{holcon2}
	\end{equation}
	\label{holcon}
\end{subequations}

Now, during physical agent-robot interaction, there is a contact between the external agent and the robot. These contacts can be modeled as holonomic constraints of the form $c(\ageVar{q},\robVar{q}) = 0$. To this purpose, a frame $\ageVar{c}_m \in {^{ea}\mathscr{C}}$ attached to the external agent link which is in contact with the robot is considered. More precisely, let $^{\mathcal{I}}H_{\ageVar{c}_m}(\ageVar{q})$ denote the homogeneous transformation from $\ageVar{c}_m$ to the inertial frame. Similarly, another frame $\robVar{c}_m \in ^{r}\mathscr{C}$ attached to the robot link in contact with the external agent is considered. Let $^{\mathcal{I}}H_{\robVar{c}_m}(\robVar{q})$ denote the homogeneous transformation from $\robVar{c}_m$ to the inertial frame. The relative transformation between the frames $\ageVar{c}_m$ and $\robVar{c}_m$ is given by:

\begin{equation}
	^{\ageVar{c}_m}H_{\robVar{c}_m} =
	{^{\mathcal{I}}H^{-1}_{\ageVar{c}_m}}(\ageVar{q}) \
	^{\mathcal{I}}H_{\robVar{c}_m}({\robVar{q}})
	\label{HRcon}
\end{equation}

When $^{\ageVar{c}_m}H_{\robVar{c}_m}$ (or a part of it) is constant, it means that the external agent and the robot are in contact. By setting $^{\ageVar{c}_m}H_{\robVar{c}_m}$  to a constant, the aforementioned holonomic constraint of the form $c(\ageVar{q},\robVar{q}) = 0$ is obtained. In general, there can be two kinds of contact between the external agent and the robot, either a point contact (Figure: \ref{point_contact}) or a plane contact (Figure: \ref{plane_contact}).

\begin{figure}[!hbt]
	\centering
	\begin{subfigure}{0.5 \textwidth}
		\centering
		\includegraphics[width = 0.75 \textwidth]{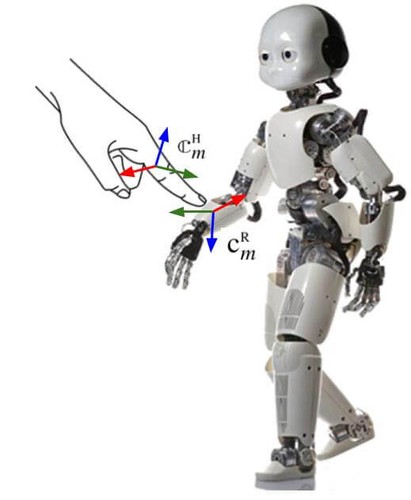}
		\caption{Point Contact}
		\label{point_contact}
	\end{subfigure}
	\begin{subfigure}{0.39\textwidth}
		\centering
		\includegraphics[width = 0.95 \textwidth]{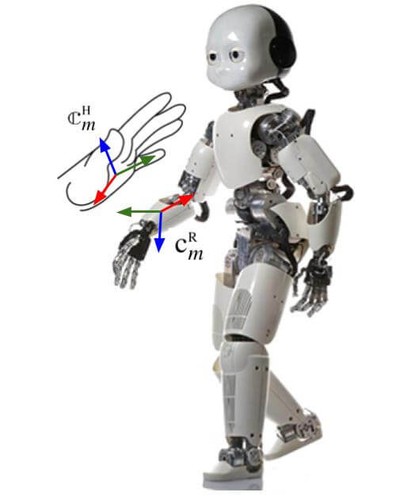}
		\caption{Plane Contact}
		\label{plane_contact}
	\end{subfigure}
	\caption{Different contact types between agent-robot partners}
\label{fig:mutual_contacts}
\end{figure}

A stable contact between the external agent and the robot during physical agent-robot interaction is considered, which leads to the condition that the relative velocity between the two frames $\ageVar{c}_m$ and $\robVar{c}_m$ is zero, i.e., the two contact frames move with the same velocity with respect to the inertial frame as given by the following relation:

\begin{equation}
	^{\mathcal{I}}{\ageVar{v}}_{\ageVar{c}_m} = \
	{^{{\ageVar{c}_m}}{X}_{\robVar{c}_m}} \ ^{\mathcal{I}}{\robVar{v}}_{\robVar{c}_m}
	\label{HRholcongeneral}
\end{equation}
where ${^{\ageVar{c}_m}{X}_{\robVar{c}_m}}$ is an adjoint transformation matrix for 6D motion vectors. We consider the contact frames to be coinciding and hence, the transformation matrix is Identity i.e. ${^{{\ageVar{c}_m}}{X}_{\robVar{c}_m}} = I_{6 \times 6}$

\vspace{0.3cm}
In light of the above, equation (\ref{HRholcongeneral}) can be written as:

\begin{equation}
	^{\mathcal{I}}{\ageVar{v}}_{\ageVar{c}_m} = \
	^{\mathcal{I}}{\robVar{v}}_{\robVar{c}_m},
	\label{HRholcon-specific}
\end{equation}

which can be represented as follows
\begin{equation}
	\ageVar{J}_{\ageVar{c}_m} {\agenu} =
	\robVar{J}_{\robVar{c}_m} {\robnu}
	\label{HRholcon}
\end{equation}

Differentiating the equation (\ref{HRholcon}) we get,
\begin{subequations}
\begin{equation}
	\begin{bmatrix}
		\ageVar{J}^b_{\ageVar{c}_m} && \ageVar{J}^j_{\ageVar{c}_m}
	\end{bmatrix}
	\begin{bmatrix}
		\dot{\ageVar{v}}_B \\
        \ddot{\ageVar{s}}
	\end{bmatrix} +
	\begin{bmatrix}
		\dot{\ageVar{J}}^b_{\ageVar{c}_m} && \dot{\ageVar{J}}^j_{\ageVar{c}_m}
	\end{bmatrix}
	\begin{bmatrix}
		\ageVar{v}_B \\
        \dot{\ageVar{s}}
	\end{bmatrix}  = \notag
\end{equation}

\begin{equation}
	\begin{bmatrix}
		\robVar{J}^b_{\robVar{c}_m} && \robVar{J}^j_{\robVar{c}_m}
	\end{bmatrix}
	\begin{bmatrix}
		\dot{\robVar{v}}_B \\
        \ddot{\robVar{s}}
	\end{bmatrix} +
	\begin{bmatrix}
		\dot{\robVar{J}}^b_{\robVar{c}_m} && \dot{\robVar{J}}^j_{\robVar{c}_m}
	\end{bmatrix}
	\begin{bmatrix}
		\robVar{v}_B \\
        \dot{\robVar{s}}
	\end{bmatrix}
	\label{HRholcon1}
\end{equation}

\begin{equation}
	\begin{bmatrix}
		\dot{\ageVar{J}}^b_{\ageVar{c}_m} && \dot{\ageVar{J}}^j_{\ageVar{c}_m} &&
		-\dot{\robVar{J}}^b_{\robVar{c}_m} && -\dot{\robVar{J}}^j_{\robVar{c}_m}
	\end{bmatrix}
	\begin{bmatrix}
		\ageVar{v}_B \\
        \dot{\ageVar{s}} \\
        \robVar{v}_B \\
        \dot{\robVar{s}}
	\end{bmatrix} + \notag
\end{equation}

\begin{equation}
	\begin{bmatrix}
		\ageVar{J}^b_{\ageVar{c}_m} && \ageVar{J}^j_{\ageVar{c}_m} &&
		-\robVar{J}^b_{\robVar{c}_m} && -\robVar{J}^j_{\robVar{c}_m}
	\end{bmatrix}
	\begin{bmatrix}
		\dot{\ageVar{v}}_B \\
        \ddot{\ageVar{s}} \\
        \dot{\robVar{v}}_B \\
        \ddot{\robVar{s}}
	\end{bmatrix} = 0
	\label{HRholcon2}
\end{equation}

\end{subequations}

Furthermore, the constraint equations~\eqref{holcon1}~\eqref{holcon2}~\eqref{HRholcon1} and \eqref{HRholcon2} can be represented in a compact form as follows:

\begin{equation}
    \comVar{P}
	\dot{\comVar{V}} +
	\dot{\comVar{P}}
	\comVar{V} 
	 = 0
	\label{HRholcon3}
\end{equation}

where,
\begingroup
    \fontsize{8pt}{15pt}\selectfont
    \begin{itemize}
    	\item $\comVar{P} =
    		   \begin{bmatrix}
    				\ageVar{J}^b_{\ageVar{c}_k} && \ageVar{J}^j_{\ageVar{c}_k} &&
    				0 && 0 \\
    					0 && 0 &&
    				\robVar{J}^b_{\robVar{c}_k} && \robVar{J}^j_{\robVar{c}_k} \\
    					\ageVar{J}^b_{\ageVar{c}_m} && \ageVar{J}^j_{\ageVar{c}_m} &&
    				-\robVar{J}^b_{\robVar{c}_m} && -\robVar{J}^j_{\robVar{c}_m}
    		   \end{bmatrix} \in \mathbb{R}^{6 \times ( {^{ea}n} + {^{r}{n}} + 12)}$ \vspace{0.2 cm}
    	\item  $\comVar{V} =
    			\begin{bmatrix}
    				\agenu && \robnu
    			\end{bmatrix}^\top \in \mathbb{R}^{ {^{ea}{n}} + {^{r}{n}}+12}$
    \end{itemize}
\endgroup

\subsection{Contact and Interaction Wrench}

First, observe that Eq.~\eqref{NEHuman2} and Eq.~\eqref{NERobot2} can be combined to obtain a single equation of motion for the composite system as presented in Eq.~\eqref{NE-HR1}

\begin{equation}
    \small
	\begin{bmatrix}
		\ageVar{M} && 0 \\
		0 && \robVar{M}
	\end{bmatrix}
	\begin{bmatrix}
		\dot{\agenu} \\
		\dot{\robnu}
	\end{bmatrix} +
	\begin{bmatrix}
		\ageVar{h} \\
		\robVar{h}
	\end{bmatrix} =
	\begin{bmatrix}
		\ageVar{B} && 0 \\
		0 && \robVar{B}
	\end{bmatrix}
	\begin{bmatrix}
		{\agetau} \\
		{\robtau}
	\end{bmatrix} +
	\begin{bmatrix}
		\ageVar{J}^\top_c && 0 \\
		0 && \robVar{J}^\top_c
	\end{bmatrix}
	\begin{bmatrix}
		\ageVar{ f} \\
		\robVar{ f}
	\end{bmatrix}
	\label{NE-HR1}
\end{equation}

\vspace{0.1 cm}
where, $\ageVar{h} = \ageVar{C}(\ageVar{q},{\agenu}){\agenu} + \ageVar{G}(\ageVar{q})$, $\robVar{h} = \robVar{C}(\robVar{q},{\robnu}){\robnu} + \robVar{G}(\robVar{q})$

According to Newtonian mechanics, in the case of rigid contacts, the perturbations exerted by the robot on the external agent is equal and opposite to the perturbation exerted by the external agent on the robot. As a consequence, when the external wrenches are expressed with respect to the inertial frame $\mathcal{I}$, the interaction wrenches $\comVar{ f}_m$ can be written as follows:

\begin{equation}
	\comVar{ f}_m \ = \ \ageVar{ f}_m \ = -\robVar{ f}_m
	\label{mutual-wrenches}
\end{equation}

Furthermore, Eq.~\eqref{NE-HR1} can be written in a compact form as follows:
\begin{equation}
	\comVar{M} \dot{\comVar{V}} + \comVar{h} = \comVar{B} \tau + \comVar{J}^\top_c \comVar{ f}
	\label{NE-HR2}
\end{equation}

where $\tau = \begin{bmatrix} \agetau && \robtau \end{bmatrix}^\top \in \mathbb{R}^{ {^{ea}n} + {^{r}n} }$; $\comVar{ f} = \begin{bmatrix}
							\comVar{ f}_m && \ageVar{ f}_e && \robVar{ f}_e
					    \end{bmatrix}^\top \in \mathbb{R}^{6 \times ({n}_m + {^{ea}{n}_e} + {^{r}{n}_e} )}$ and
$\comVar{J}_c$ is a proper contact jacobian  matrix. This equation implies that

\begin{equation}
	\dot{\comVar{V}} = {\comVar{M}^{-1}}[\comVar{B} {\tau}+\comVar{J}^\top_c \comVar{ f} - \comVar{h}]
	\label{HR-vel}
\end{equation}

Making use of Eq.~\eqref{HR-vel} in the constraint equation Eq.~\eqref{HRholcon3}

\begin{subequations}
	\begin{equation}
		\dot{\comVar{P}} \comVar{V} + \comVar{P} {\comVar{M}^{-1}}[\comVar{B} {\tau}+\comVar{J}^\top_c \comVar{ f} - \comVar{h}] = 0 \notag
	\end{equation}
	\begin{equation}
		\Rightarrow \comVar{P} {\comVar{M}}^{-1} \comVar{J}^\top_c \comVar{ f} = - [ \comVar{P} {\comVar{M}}^{-1} [\comVar{B} {\tau} - \comVar{h}] + \dot{\comVar{P}} \comVar{V}] \notag
	\end{equation}
\end{subequations}

\begin{equation}
		\Rightarrow \comVar{ f} = -\comVar{\Gamma}^{\dagger}[ \comVar{P} {\comVar{M}}^{-1} [\comVar{B} {\tau} - \comVar{h}] + \dot{\comVar{P}} \comVar{V}] \notag
\end{equation}
\noindent where, $\comVar{\Gamma} = \comVar{P} {\comVar{M}}^{-1} \comVar{J}^\top_c$
\vspace*{0.2 cm} \\
Furthermore,
\begin{equation}
	\mathbold{f} = - \comVar{\Gamma}^{\dagger} \comVar{P} {\comVar{M}}^{-1} \comVar{B} {\tau} +
					   \comVar{\Gamma}^{\dagger} \comVar{P} {\comVar{M}}^{-1} \comVar{h} -
					   \comVar{\Gamma}^{\dagger} \dot{\comVar{P}} \comVar{V}
	\label{external-wrenches}
\end{equation}

\vspace*{0.2 cm}
Through coupled-dynamics, Eq.~\eqref{external-wrenches} shows that the external wrenches are a function of the agent's configuration $\ageVar{q}$, $\robVar{q}$, velocity ${\agenu}$, ${\robnu}$, and joint torques ${\agetau}$, ${\robtau}$. This can be represented as a function $
	\comVar{ f} = g(\ageVar{q},\robVar{q},{\agenu},{\robnu},{\agetau}, {\robtau}) \notag
$. This relation can be further decomposed as,

\begin{equation}
	\comVar{ f} = \comVar{G}_1 {\agetau} + \comVar{G}_2 {\robtau} + {\comVar{G}_3}(\ageVar{q},\robVar{q},{\agenu},{\robnu})
	\label{external-wrenches-compact}
\end{equation}

\noindent where,
\begin{itemize}
	\item $\comVar{G}_1 \in \mathbb{R}^{ 6 ( {n}_m + {^{ea}{n}_e} + {^{r}{n}_e} ) \times {^{ea}{n}} }$
	\item $\comVar{G}_2 \in \mathbb{R}^{6 ( {n}_m + {^{ea}{n}_e} + {^{r}{n}_e} ) \times {^{r}{n}} }$
	\item $\comVar{G}_3 \in \mathbb{R}^{6 ( {n}_m + {^{ea}{n}_e} + {^{r}{n}_e} )}$
\end{itemize}

\section{Partner-Aware Robot Control}

As presented in section~\ref{sec:recourse-on-classical-control}, the robot joint torques from Eq.~\eqref{eq:normal-control-torques-compact} are agnostic to the interaction wrench from the external agent which are eventually compensated by considering the system dynamics. Hence, the robot behavior is invariant with respect to the external agent interaction. Certainly, instead of completely canceling out any external interaction by the feedback control action, it is gainful and desirable to \emph{exploit} external interaction to accomplish the robot's task. This poses, however, the question of characterizing and quantifying help from the external agent with respect to the robot's task. In a previous work \cite{8093992}, the interaction wrench \textit{estimates} from the robot is used as intent information from the external agent who is a human. However, in a coupled system, wrench information introduces an algebraic loop in the control design as the wrench estimates from the robot are computed using the robot joint torques \cite{Frontiers2015}. Instead, another beneficial approach is to leverage the joint torques of the external agent. In the case of human they are largely self-generated and self-regulated. Additionally, considering the joint torques of a human agent opens new possibilities for investigating and optimizing human ergonomics. Also, considering the joint torques of an external agent that is a robot will help formulate robust control techniques for the combined system to achieve a collaborative task.

The following lemma proposes a partner-aware control law that exploits the external agent contribution towards the achievement of the robot control objective, thus actively taking into account the physical agent-robot interaction. The associated analysis is based on the \emph{energy} of the robot's control task.

\begin{lemma}
	
	Assume that the control objective is to asymptotically stabilize the following point
	\begin{equation}
		\left(\widetilde{\comVar{\chi}},\int\widetilde{\comVar{\chi}} ~ds\right) = (0,0)
		\label{eq:partner-aware-control-equilibrium-point}
	\end{equation}
	Apply the following robot torques to the robot system~\eqref{NERobot2}
	\begin{equation}
		\robtau = - \comVar{\Delta}^{\dagger} \ [ \ \comVar{\Lambda} \ + \
										 \robVar{K}_D \ \widetilde{\comVar{\chi}} \ + \ max(0,\alpha) \ \widetilde{\comVar{\chi}}^{\parallel} \ ] + \robVar{N}_{\comVar{\Delta}} {\robtau}_0
		\label{eq:partner-aware-control-law}
	\end{equation}
	with
		\begin{itemize}\setlength\itemsep{0.8em}
		    \item $ \comVar{\Delta} = \robVar{K}_d \ \robVar{J}_{\comVar{\chi}} \robVar{M}^{-1} [\robVar{B} + \robVar{J}_c^\top \bar{\comVar{G}}_2 ] \in \mathbb{R}^{6 \times \mathrm{n}}$
		    \item $ \robVar{N}_{\comVar{\Delta}}$ the null-space projector of the matrix $\comVar{\Delta}$
		    \item ${\robtau}_0$ a free $^{r}n-$dimensional vector
		    \item $ \comVar{\Lambda} = \robVar{K}_d \ [ \ [ \ \robVar{J}_{\comVar{\chi}} \robVar{M}^{-1} \robVar{J}_c^\top \ ] {\bar{\comVar{G}}_3}(\ageVar{q},\robVar{q},{\agenu},{\robnu}) -     					  		    \robVar{J}_{\comVar{\chi}} \robVar{M}^{-1} \robVar{h} \ + \
						  			\dot{\robVar{J}}_{\comVar{\chi}} \ {\robnu}  - \dot{\comVar{\chi}}_d \ ] + \robVar{K}_p \ \int_{0}^{t}(\comVar{chi} - \comVar{\chi}_d) ~ds \in \mathbb{R}^{6}$
		    \item $\robVar{K}_D \in \mathbb{R}^{6 \times 6}$ is a symmetric, positive-definite matrix
		    \item $\alpha \in \mathbb{R}$ is a component proportional to the external agent joint torques $\agetau$ projected along $\widetilde{\comVar{\chi}}^{\parallel}$ i.e., the direction parallel to $\widetilde{\comVar{\chi}}$
	\end{itemize}
	Assume that the matrix $\comVar{\Delta}$ is full rank matrix $\forall \ t \in \mathbb{R}^{+}$,
	\begin{itemize}\setlength\itemsep{1em}
		\item The trajectories $(\widetilde{\comVar{\chi}},\int\widetilde{\comVar{\chi}} ~ds)$ are globally bounded
		\item The equilibrium point \eqref{eq:partner-aware-control-equilibrium-point} is stable
	\end{itemize}
    \label{partner-aware-control-lemma}
\end{lemma}

\begin{proof}

The stability of $\widetilde{\comVar{\chi}}$ from Eq.~\eqref{eq:partner-aware-control-equilibrium-point} can be analyzed by considering the following Lyapunov function:

\begin{equation}
\label{eq:partner-aware-control-LyapunovFunction}
\mathrm{V} = \frac{\robVar{K}_d}{2} \norm{\comVar{\chi} - \comVar{\chi}_d}^{2} + \frac{\robVar{K}_p}{2} \norm{\int_{0}^{t}(\comVar{\chi} - \comVar{\chi}_d) ~ds}^{2}
\end{equation}
where $\robVar{K}_d, \robVar{K}_p \in \mathbb{R}^{p \times p}$ are two symmetric, positive-definite matrices. Now, on differentiating $\mathrm{V}$, we get:
\begin{subequations}
	\begin{equation}
	\dot{\mathrm{V}} = \robVar{K}_d \ (\comVar{\chi} - \comVar{\chi}_d)^\top \ (\dot{\comVar{\chi}} - \dot{\comVar{\chi}}_d) + \robVar{K}_p \ \int_{0}^{t}(\comVar{\chi} - \comVar{\chi}_d)^\top ds \ (\comVar{\chi} - \comVar{\chi}_d) \notag
	\end{equation}
	\begin{equation}
	\Rightarrow \dot{\mathrm{V}} =  \robVar{K}_d \ \widetilde{\comVar{\chi}}^\top \ (\dot{\comVar{\chi}} - \dot{\comVar{\chi}}_d) + \robVar{K}_p \ (\comVar{\chi} - \comVar{\chi}_d)^\top \ \int_{0}^{t}(\comVar{\chi} - \comVar{\chi}_d) ~ds \  \notag
	\end{equation}
	\begin{equation}
	\Rightarrow \dot{\mathrm{V}} = \widetilde{\comVar{\chi}}^\top \ [ \robVar{K}_d \ (\dot{\comVar{\chi}} - \dot{\comVar{\chi}}_d) + \robVar{K}_p \ \int_{0}^{t}(\comVar{\chi} - \comVar{\chi}_d) ~ds] \  \notag
	\end{equation}
\end{subequations}
\begin{equation}
\Rightarrow \dot{\mathrm{V}} = \widetilde{\comVar{\chi}}^\top \ [\robVar{K}_d \ \dot{\comVar{\chi}} + \robVar{K}_p \ \int_{0}^{t}(\comVar{\chi} - \comVar{\chi}_d) ~ds - \robVar{K}_d \ \dot{\comVar{\chi}}_d]
\label{Vdot}
\end{equation}

Following the relation in Eq.~\eqref{eq:linear-jacobian-mapping}, on differentiating the link velocity $\comVar{\chi}$ we have the following relation:
\begin{equation}
\dot{\comVar{\chi}} = \robVar{J}_{\comVar{\chi}} ~\dot{\robnu} + \dot{\robVar{J}}_{\comVar{\chi}} ~\robnu
\label{chi_jacobian_diff}
\end{equation}

Using equation (\ref{chi_jacobian_diff}) in equation (\ref{Vdot}) we can write:
\begin{equation}
\dot{\mathrm{V}} = \widetilde{\comVar{\chi}}^\top \ [ \robVar{K}_d \ [\robVar{J}_{\comVar{\chi}} ~\dot{\robnu} +\dot{\robVar{J}}_{\comVar{\chi}} ~\robnu] + \robVar{K}_p \ \int_{0}^{t}(\comVar{\chi} - \comVar{\chi}_d) ~ds - \robVar{K}_d \ \dot{\comVar{\chi}}_d]
\label{Vdot_final}
\end{equation}

The quantity $\dot{\robnu}$ in equation (\ref{Vdot_final}) is the robot's acceleration. Following the equation~\eqref{NERobot2} $\dot{\robnu}$ can be written as follows:
\begin{equation}
\dot{\robnu} = \robVar{M}^{-1}[\robVar{B} {\robtau} + \robVar{J}_c^\top \robVar{ f} - \robVar{h}]
\label{robot_acc}
\end{equation}
In equation (\ref{robot_acc}) $\robVar{ f}$ can be extracted from the equation (\ref{external-wrenches-compact}) by selecting the first and the last rows which can be represented as $\robVar{f} = \bar{\comVar{G}}_1 {\agetau} + \bar{\comVar{G}}_2 {\robtau} + {\bar{\comVar{G}}}_3(\ageVar{q},\robVar{q},{\agenu},{\robnu})$. Now, equation (\ref{robot_acc}) can be written as,

\begin{equation}
\dot{\robnu} = \robVar{M}^{-1}[\robVar{B} {\robtau} +
\robVar{J}_c^\top [\bar{\comVar{G}}_1 {\agetau} + \bar{\comVar{G}}_2 {\robtau} + 								{\bar{\comVar{G}}}_3(\ageVar{q},\robVar{q},{\agenu},{\robnu})] - \robVar{h}]
\end{equation}

\begin{equation}
\Rightarrow	\dot{\robnu} = \robVar{M}^{-1}
[\robVar{B} {\robtau} +
\robVar{J}_c^\top \bar{\comVar{G}}_1 {\agetau} + \robVar{J}_c^\top \bar{\comVar{G}}_2 {\robtau} + \robVar{J}_c^\top {\bar{\comVar{G}}}_3(\ageVar{q},\robVar{q},{\agetau},{\robtau}) - \robVar{h}] \notag
\end{equation}

\begin{equation}
\Rightarrow	\dot{\robnu} = \robVar{M}^{-1}
[\robVar{J}_c^\top \bar{\comVar{G}}_1 {\agetau} +
[\robVar{B} + \robVar{J}_c^\top \bar{\comVar{G}}_2 ] {\robtau} +
\robVar{J}_c^\top {\bar{\comVar{G}}}_3(\ageVar{q},\robVar{q},{\agenu},{\robnu}) - \robVar{h}] \notag
\end{equation}

\begin{equation}
\Rightarrow	\dot{\robnu} = [\robVar{M}^{-1}\robVar{J}_c^\top \bar{\comVar{G}}_1] {\agetau} +
[\robVar{M}^{-1} [\robVar{B} + \robVar{J}_c^\top \bar{\comVar{G}}_2 ]] {\robtau} +
[\robVar{M}^{-1} \robVar{J}_c^\top] {\bar{\comVar{G}}}_3(\ageVar{q},\robVar{q},{\agenu},{\robnu}) - \robVar{M}^{-1} \robVar{h}
\label{robot_acc_final}
\end{equation}

We have $\dot{\robnu}$ in equation (\ref{Vdot_final}) which can be replaced with equation (\ref{robot_acc_final}) leading to the following relation:

\begin{equation}
\begin{split}
\dot{\mathrm{V}} = \widetilde{\comVar{\chi}}^\top  \
[ \robVar{K}_d \ \robVar{J}_{\comVar{\chi}} [ \ [\robVar{M}^{-1}\robVar{J}_c^\top \bar{\comVar{G}}_1] {\agetau} +
[ \ \robVar{M}^{-1} [\robVar{B} + \robVar{J}_c^\top \bar{\comVar{G}}_2 ] \ ] {\robtau} + \\
[\robVar{M}^{-1} \robVar{J}_c^\top] {\bar{\comVar{G}}}_3(\ageVar{q},\robVar{q},{\agenu},{\robnu}) - \robVar{M}^{-1} \robVar{h} \ ] \\
+ \robVar{K}_d \ \dot{\robVar{J}}_{\comVar{\chi}} {\robnu} + \robVar{K}_p \ \int_{0}^{t}(\comVar{\chi} - \comVar{\chi}_d) ~ds - \robVar{K}_d \ \dot{\comVar{\chi}}_d] \notag
\end{split}
\end{equation}

\begin{equation}
\begin{split}
\Rightarrow \dot{\mathrm{V}} = \widetilde{\comVar{\chi}}^\top \
[\robVar{K}_d \ [\robVar{J}_{\comVar{\chi}} \robVar{M}^{-1}\robVar{J}_c^\top \bar{\comVar{G}}_1] {\agetau} +
[\robVar{K}_d \ \robVar{J}_{\comVar{\chi}} \robVar{M}^{-1} [\robVar{B} + \robVar{J}_c^\top \bar{\comVar{G}}_2 ] \ ] {\robtau} + \\
[ \robVar{K}_d \ \robVar{J}_{\comVar{\chi}} \robVar{M}^{-1} \robVar{J}_c^\top]  {\bar{\comVar{G}}}_3(\ageVar{q},\robVar{q},{\agenu},{\robnu}) - \\
\robVar{K}_d \ \robVar{J}_{\comVar{\chi}} \robVar{M}^{-1} \robVar{h} + \\
\robVar{K}_d \ \dot{\robVar{J}}_{\comVar{\chi}} \ {\robnu} + K_p \ \int_{0}^{t}(\comVar{\chi} - \comVar{\chi}_d) ds - \robVar{K}_d \ \dot{\comVar{\chi}}_d] \notag
\end{split}
\end{equation}

This can be written in a compact form as follows:

\begin{equation}
\Rightarrow \dot{\mathrm{V}} =
\widetilde{\comVar{\chi}}^\top \ \comVar{\Omega} \ {\agetau} +
\widetilde{\comVar{\chi}}^\top \ [ \ \comVar{\Delta} \ {\robtau} +
\ \comVar{\Lambda} \ ]
\label{Vdot_final_compact}
\end{equation}
where,
\begin{itemize}
	\item $ \comVar{\Omega} = \robVar{K}_d \ \robVar{J}_{\comVar{\chi}} \robVar{M}^{-1}\robVar{J}_c^\top \bar{\comVar{G}}_1 \in \mathbb{R}^{p \times {^{ea}{n}}} $
	\item $ \comVar{\Delta} = \robVar{K}_d \ \robVar{J}_{\comVar{\chi}} \robVar{M}^{-1} [\robVar{B} + \robVar{J}_c^\top \bar{\comVar{G}}_2 ] \in \mathbb{R}^{p \times {^{r}n}}$
	\item $ \comVar{\Lambda} = \robVar{K}_d \ [[\robVar{J}_{\comVar{\chi}} \robVar{M}^{-1} \robVar{J}_c^\top] {\bar{\comVar{G}}}_3(\ageVar{q},\robVar{q},{\agenu},{\robnu}) - \robVar{J}_{\comVar{\chi}} \robVar{M}^{-1} \robVar{h} + \dot{\robVar{J}}_{\comVar{\chi}} \ {\robnu} - \ \dot{\comVar{\chi}}_d] + \
	\robVar{K}_p \ \int_{0}^{t}(\comVar{\chi} - \comVar{\chi}_d) ~ds \in \mathbb{R}^{p}$
\end{itemize}

The stability of $\comVar{\chi}$ is ensured when $\dot{\mathrm{V}} \ \leq \ 0$, .i.e.,
\begin{equation}
\widetilde{\comVar{\chi}}^\top \ \comVar{\Omega} \ {\agetau} +
\widetilde{\comVar{\chi}}^\top \ [ \ \comVar{\Delta} \ {\robtau} +
\ \comVar{\Lambda} \ ] \leq 0
\label{Vdot_final_compact_relation}
\end{equation}

Now, the quantity $\mathbold{\Omega} \ \tau^{\scalebox{\hrsup_size}{H}}$ comprising human joint torques can be decomposed into two orthogonal components according to the following relation:

\begin{subequations}
	
	\begin{equation}
	\comVar{\Omega} \ {\agetau} =
	\alpha \ \widetilde{\comVar{\chi}}^{\parallel} +
	\beta \ \widetilde{\comVar{\chi}}^{\perp}
	\end{equation}
	
	\begin{equation}
	\widetilde{\comVar{\chi}}^{\parallel} = \frac{\widetilde{\comVar{\chi}}}{\norm{\widetilde{\comVar{\chi}}}}, \
	\alpha = \frac{\widetilde{\comVar{\chi}}^\top \comVar{\Omega} \ {\agetau}}{\norm{\widetilde{\comVar{\chi}}}}
	\end{equation}
	\label{humanT_decomp}
\end{subequations}
where $\alpha$ and $\beta$ are the components of $\comVar{\Omega} \ {\agetau}$. On using the decomposition of equation (\ref{humanT_decomp}), we can write equation (\ref{Vdot_final_compact}) as:

\begin{equation}
\dot{\mathrm{V}} =
\widetilde{\comVar{\chi}}^\top \ \alpha \ \widetilde{\comVar{\chi}}^{\parallel} +
\widetilde{\comVar{\chi}}^\top \ [ \ \comVar{\Delta} \ {\robtau} +
\ \comVar{\Lambda} \ ]
\label{Vdot_decomp}
\end{equation}

Note that the component $\beta$ is projected in the perpendicular direction of $\widetilde{\comVar{\chi}}$ and hence it does not contribute any value to $\dot{\mathrm{V}}$. Using the above control law~\eqref{eq:partner-aware-control-law}) in equations~\eqref{Vdot_decomp} leads to the following relation:

\begin{equation}
\dot{\mathrm{V}} =
- \
\widetilde{\comVar{\chi}}^\top \ \robVar{K}_D \ \widetilde{\comVar{\chi}} \ - \
\widetilde{\comVar{\chi}}^\top \ max(0,\alpha) \ \widetilde{\comVar{\chi}}^{\parallel}
\label{Vdot_control}
\end{equation}

The fact that the human joint torques help the robot to perform a control action is encompassed in the right hand side of above equation: a positive $\alpha$ makes the Lyapunov function decrease faster. Thus the control law~\eqref{eq:partner-aware-control-law} ensures that $\dot{\mathrm{V}} \leq 0$ which proves that the trajectories are globally bounded. From lyapunov theory, as $\dot{\mathrm{V}} \leq 0 $ in the neighborhood of the point $(0,0)$ the equilibrium point~\eqref{eq:partner-aware-control-equilibrium-point} is stable.

It is important to observe that Eq.~\eqref{eq:partner-aware-control-law} depends upon the external agent joint torques. In fact, the scalar $\alpha$ represents the projection of the external agent joint torques into the desired robot direction. As a consequence, all human joint torques that will make the energy function~\eqref{eq:partner-aware-control-LyapunovFunction} decrease faster are not canceled out from the feedback control action.

Observe also that the results in Lemma 1 do not include the convergence to zero of the robot equilibrium point. This additional property of the control laws~\eqref{eq:partner-aware-control-law} is currently being investigated, and requires additional properties on the human control system for the application of the Barbalat's lemma. In fact, to ensure this convergence, the next step of the proof is to show that the Lyapunov function~\eqref{eq:partner-aware-control-LyapunovFunction} has a bounded second order time derivative, i.e. $|\ddot{V}|<c$, which requires additional assumptions on the human motion currently being investigated.

Finally, observe that the control laws~\eqref{eq:partner-aware-control-law} include a degree of redundancy under the assumption that the matrix $\comVar{\Delta}$ is fat, i.e. the dimension of the robot control task is  lower than the robot actuation. As a consequence, the free vector ${\robtau}_0$ can be used for other control purposes.

\end{proof}

\section{Whole-Body Standup Controller}
\label{sec:whole-body-standup-controller}

The whole-body standup controller is based on momentum control that is briefly described in Appendix~\ref{app:whole-body-momentum-control}. The main control variable is the momentum of the robot. So, Eq.~\eqref{eq:linear-jacobian-mapping} becomes,

\begin{equation}
\robVar{L} = \robVar{J}_{cmm}(\robVar{q}) \ \robnu
\label{Robot-momentum}
\end{equation}

\noindent where, $\robVar{J}_{cmm}$ is the centroidal momentum matrix. The primary task of the robot we considered is to perform a standup motion by moving its center of mass through the sit-to-stand transition with momentum control as the primary control objective. A Matlab Simulink controller using a whole-body toolbox\footnote{\href{https://github.com/robotology/wb-toolbox}{https://github.com/robotology/wb-toolbox}}~\citep{RomanoWBI17Journal} is implemented with the external agent and robot as the main sub systems as shown in Fig.~\ref{fig:standup_controller_full}.

\begin{figure}[H]
	\centering
	\includegraphics[width=0.85\textwidth]{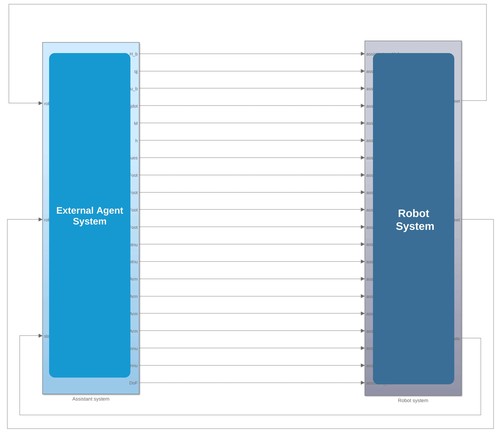}
	\caption{Whole-body standup controller implemented in Matlab Simulink}
	\label{fig:standup_controller_full}
\end{figure}

\subsection{Robot Subsystem}

The robot subsystem is composed of three main components as highlighted in Fig.~\ref{fig:standup_controller_robot}.

\begin{figure}[H]
	\centering
	\includegraphics[width=0.85\textwidth]{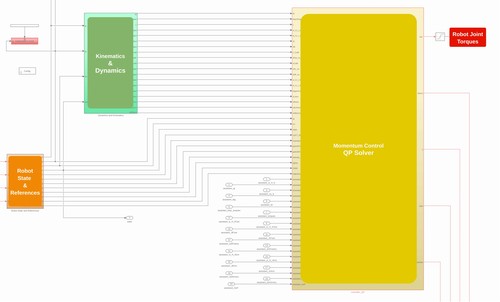}
	\caption{Main components of whole-body standup controller robot subsystem implemented in Matlab Simulink}
	\label{fig:standup_controller_robot}
\end{figure}

\subsubsection{Robot State \& References}

This subsystem constitutes of a state machine that guides the robot behavior. There are four states for the robot as highlighted in  Fig. \ref{standup-states}. At each state a reference trajectory to the center of mass is generated through a minimum-jerk trajectory generator~\cite{Pattacini2010}. During the state 1, the robot balances on a chair and enters to state 2 when an interaction wrench of a set threshold is detected at the hands indicating the start of pull-up assistance from an external agent. During state 2, the robot moves its center of mass forwards and enters state 3 when the external wrench experienced at the feet of the robot is above a set threshold. During state 3, the robot moves its center of mass both in forward and upward directions and enters state 4 when the external wrench experienced at the feet of the robot are above another set threshold. Finally, during state 4 the robot moves its center of mass further upward to stand fully erect on both the feet. Accordingly, during the states 1 and 2 the contacts with the environment (chair), accounted in the controller, are at the \emph{upper legs} of the robot. Similarly, for the states 3 and 4 the \emph{feet} contacts of the robot with the environment (ground) are accounted in the controller.

\begin{figure}[H]
	\centering
	\includegraphics[width=\textwidth]{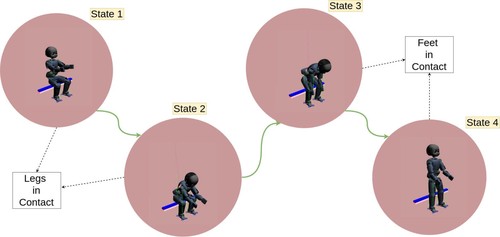}
	\caption{State machine of whole-body standup controller}
	\label{standup-states}
\end{figure}

\subsubsection{Kinematics \& Dynamics}

This subsystem computes the contact kinematics described by Eq.~\ref{HRholcon3} and dynamic quantities like mass matrix $\robVar{M}$ and bias forces $\robVar{h}$ depending on the state of the robot.

\subsubsection{Momentum Control QP Solver}

This subsystem implements the momentum control objective as an optimization problem represented by the system of equations~\eqref{optTorque}-\eqref{optPost} that is solved through an open source quadratic programming solver qpOASES~\cite{Ferreau2014}. The objective is to find the forces at the feet of the robot~\eqref{qpforces} while subject to the constraints of the feet friction cones~\eqref{frictionCones} and the momentum control objective~\eqref{momentumControl}. Having obtained the optimal contact forces $\robVar{f}^*_{feet}$ at the feet of the robot, the robot control torques $\robtau^*$ are evaluated according to~\eqref{optPost} subject to the constraints of the contact~\eqref{constraintsRigid}, robot dynamics~\eqref{robotDynamics} and the postural task~\eqref{posturalTask}.

\begin{IEEEeqnarray}{RCL}
	\IEEEyesnumber
	\label{optTorque}
	\robVar{f}^*_{feet} &=& \argmin_{\robVar{f}_{feet}}  |\robtau^*(\robVar{f})| \IEEEyessubnumber  \label{qpforces} \\
	&s.t.& \nonumber \\
	&& \robVar{A} \robVar{f} < \robVar{b} \IEEEyessubnumber  \label{frictionCones} \\
	&& \dot{\robVar{L}}(\robVar{f}, \alpha) = \dot{\robVar{L}}^*  \IEEEyessubnumber \label{momentumControl} \\
	&&\robtau^*(\robVar{f}) = \argmin_{\robtau}  |\robtau(\robVar{f}) - \robtau_0(\robVar{f})| 	\label{optPost} 
	\\
	&& \quad s.t.  \nonumber \\
	&& \quad \quad \ \comVar{P} \dot{\comVar{V}} + \dot{\comVar{P}} \comVar{V}  = 0
	\IEEEyessubnumber 	\label{constraintsRigid} \\
	&& \quad \quad \ \dot{\robnu} = \robVar{M}^{-1}[\robVar{B} {\robtau} + \robVar{J}_c^\top \robVar{ f} - \robVar{h}] \IEEEyessubnumber \label{robotDynamics} \\
	&& \quad \quad \ 	{\robtau}_0 = 
	\robVar{h}_j - \robVar{J}_j^\top \robVar{ f} {-}  \robVar{K}^j_{p}(\robVar{s}- \robVar{s}^d) {-} \robVar{K}^j_{d}\dot{\robVar{s}}		    \IEEEyessubnumber \label{posturalTask}
	\yesnumber
\end{IEEEeqnarray}

\section{Experiments}
\label{experiments}

Considering the human model as a multi-body mechanical system of rigid links allows us to use another humanoid robot in place of a human as an external agent for the experiment without the loss of integrity of the experiment to validate our framework. So, we designed an experimental scenario with two iCub humanoid robots as shown in Fig. \ref{fig:two-icubs}. 

\begin{figure}[!hbt]
	\centering
	\begin{subfigure}{0.49\textwidth}
		\centering
		\includegraphics[width=0.95\textwidth]{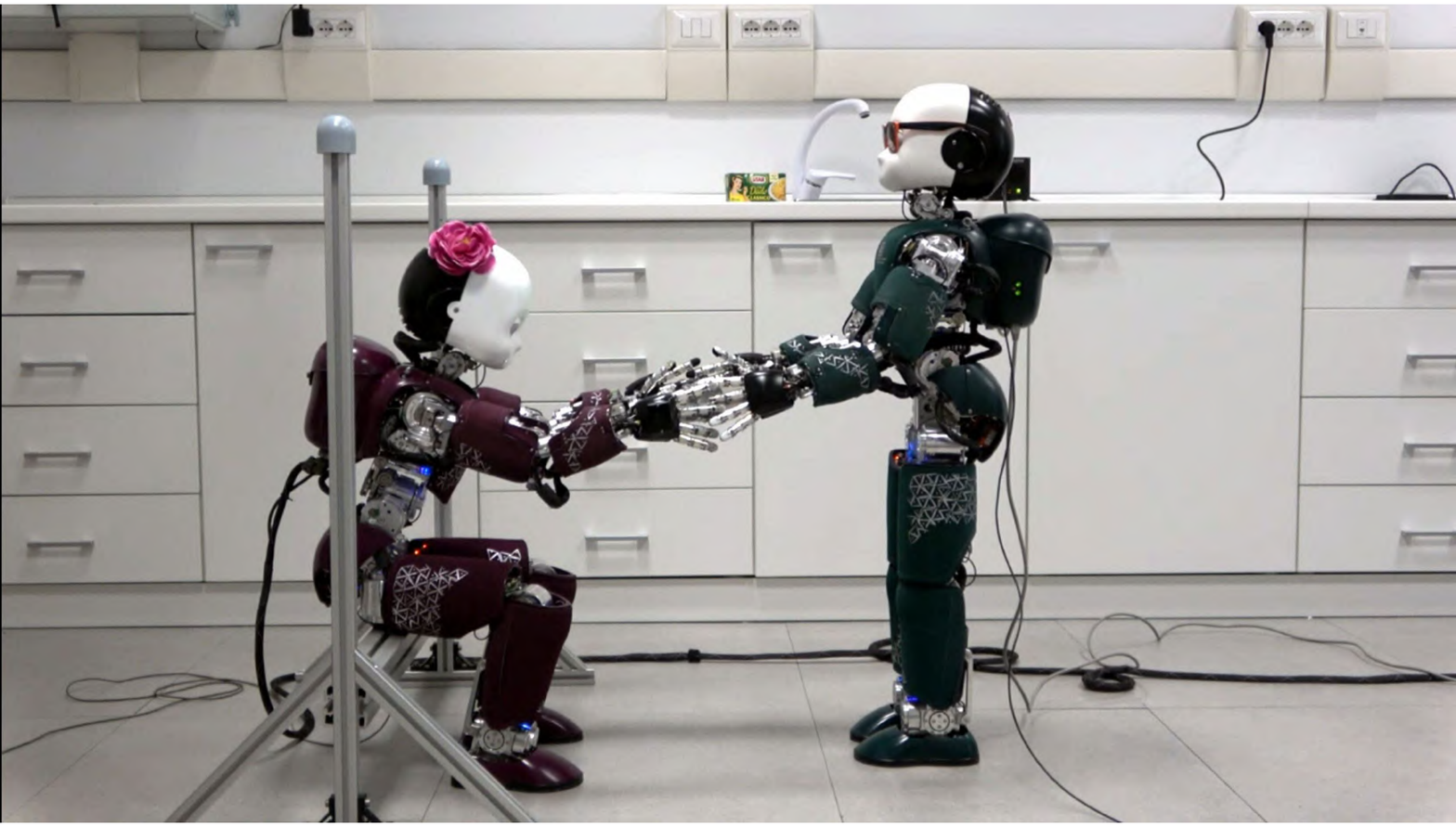}
		\label{fig:state1}
	\end{subfigure}%
	\begin{subfigure}{0.49\textwidth}
		\centering
		\includegraphics[width=0.95\textwidth]{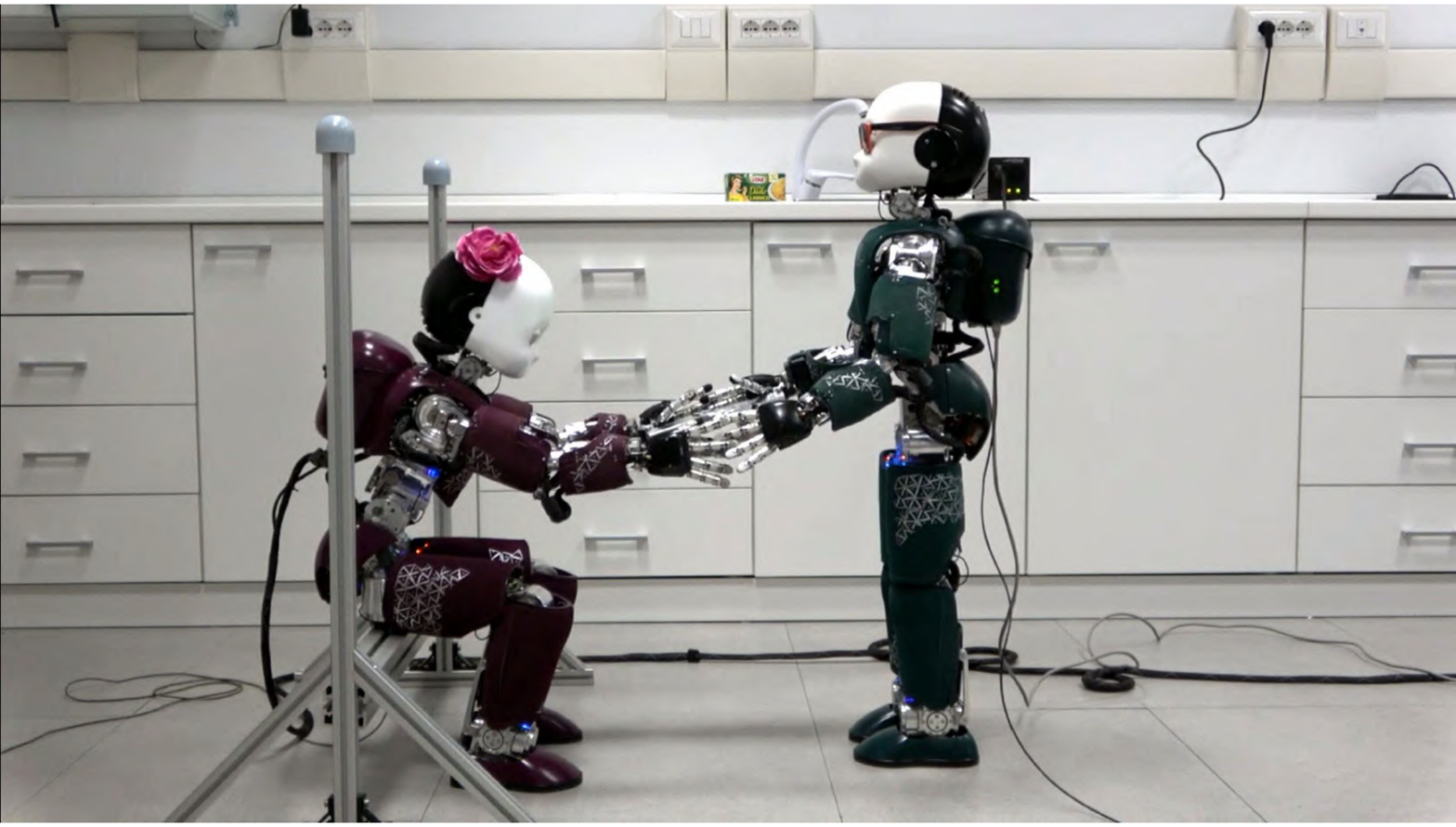}
		\label{fig:state2}
	\end{subfigure}
	\vspace{-0.5cm}
	\begin{subfigure}{0.49\textwidth}
		\centering
		\vspace{-1.25cm}\includegraphics[width=0.95\textwidth]{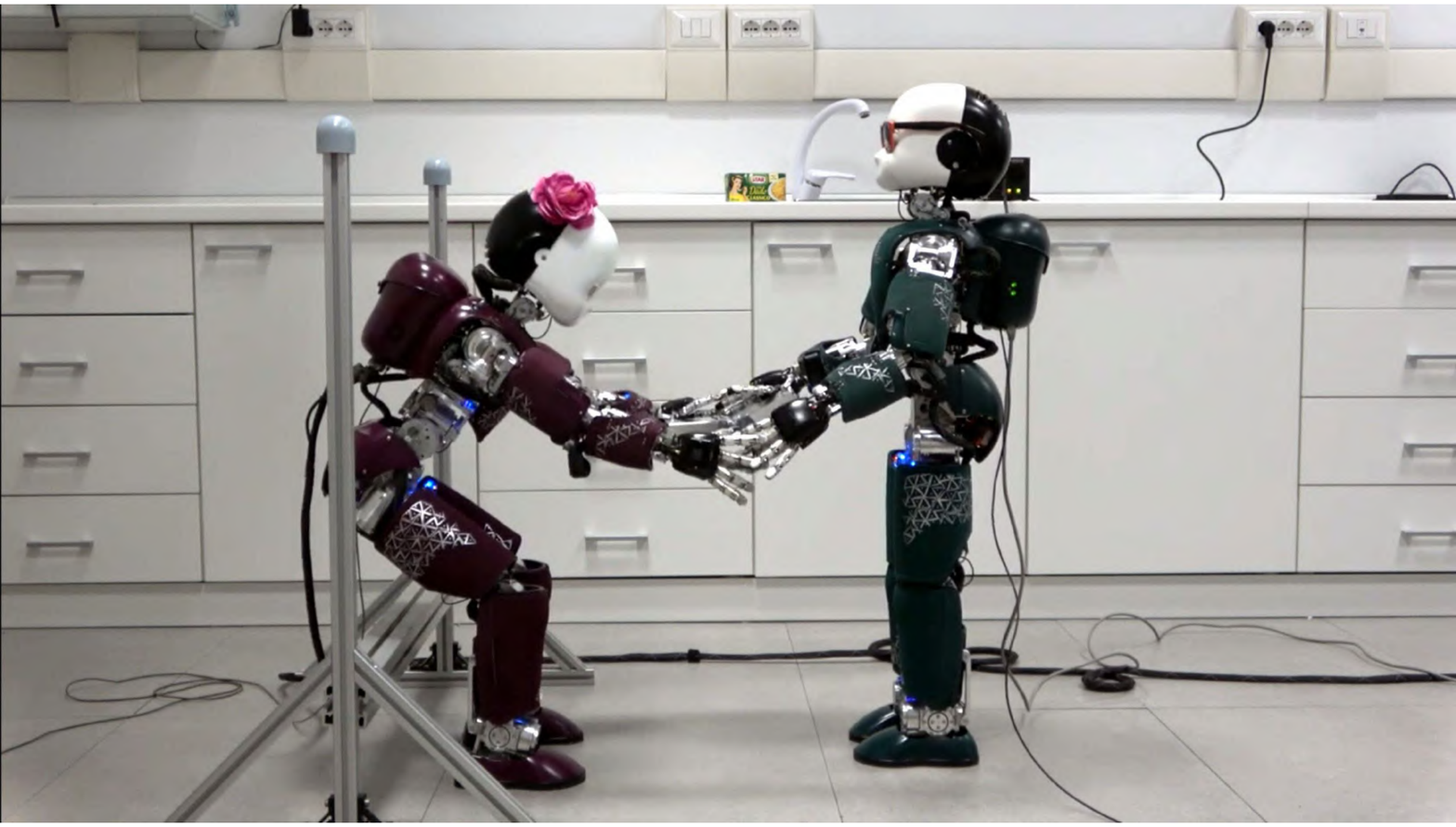}
		\label{fig:state3}
	\end{subfigure}%
	\begin{subfigure}{0.49\textwidth}
		\centering
		\vspace{-1.25cm}\includegraphics[width=0.95\textwidth]{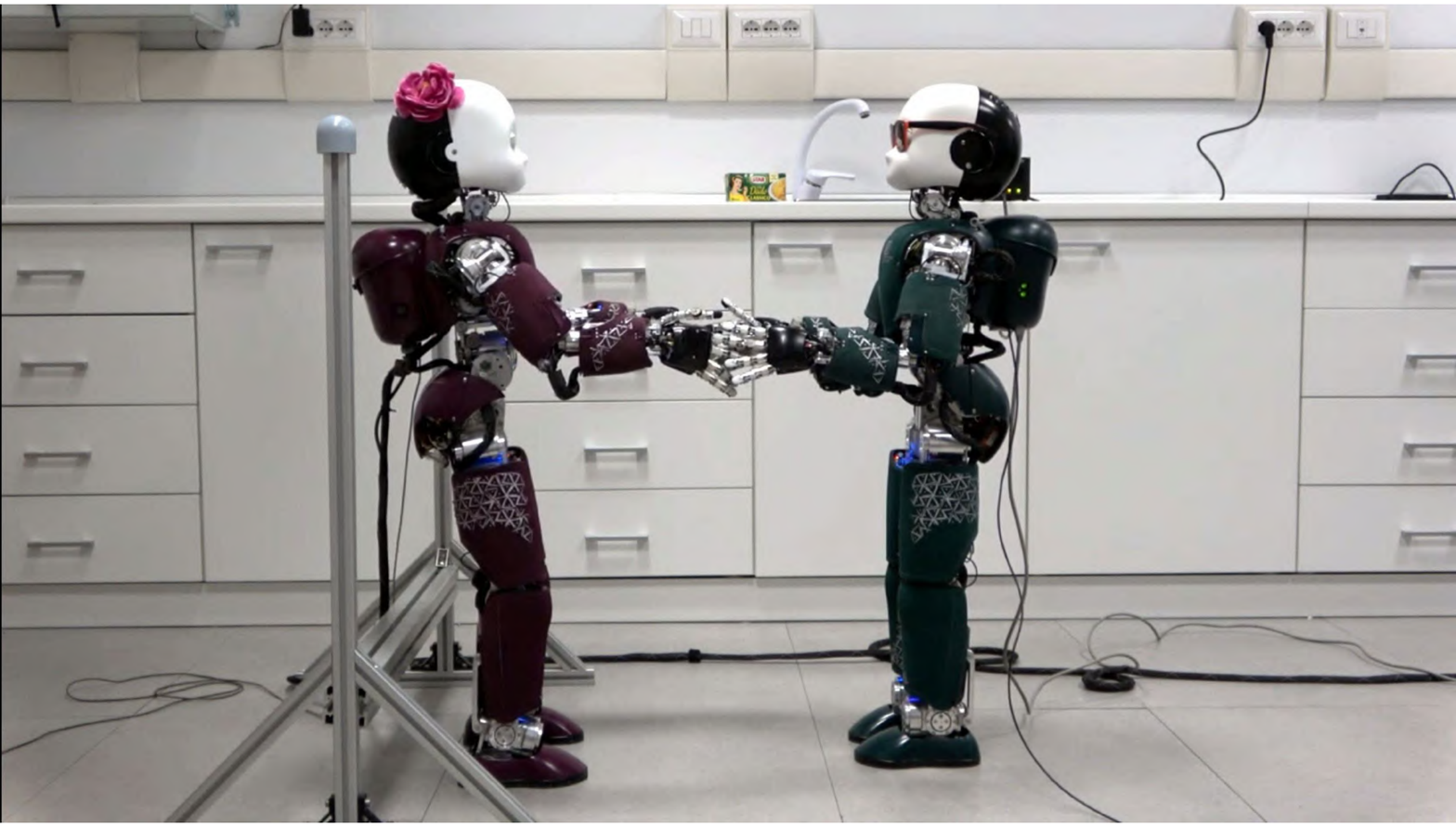}
		\label{fig:state4}
	\end{subfigure}
	\caption{Standup experimental scenario with two iCub robots involved in physical interaction}
	\label{fig:two-icubs}
\end{figure}

The \emph{purple} iCub robot is run in torque control mode and receives torque inputs from the controller for the standup task. The \emph{green} iCub robot is in position control mode. A predetermined trajectory generated using a minimum-jerk trajectory generator \cite{5650851} is given as a reference to the torso pitch, shoulder pitch and elbow joints of the \emph{green} iCub robot. The resulting motion mimics the pull-up assistance to the \emph{purple} iCub robot for performing the standup task. Hence, the \emph{green} iCub robot is considered as an external interacting agent whose joint motion is indicated in Fig. \ref{fig:humanQ} and the associated joint torques are shown in Fig. \ref{fig:humanTau}.

\begin{figure}[H]
	\centering
	\includegraphics[scale=0.2]{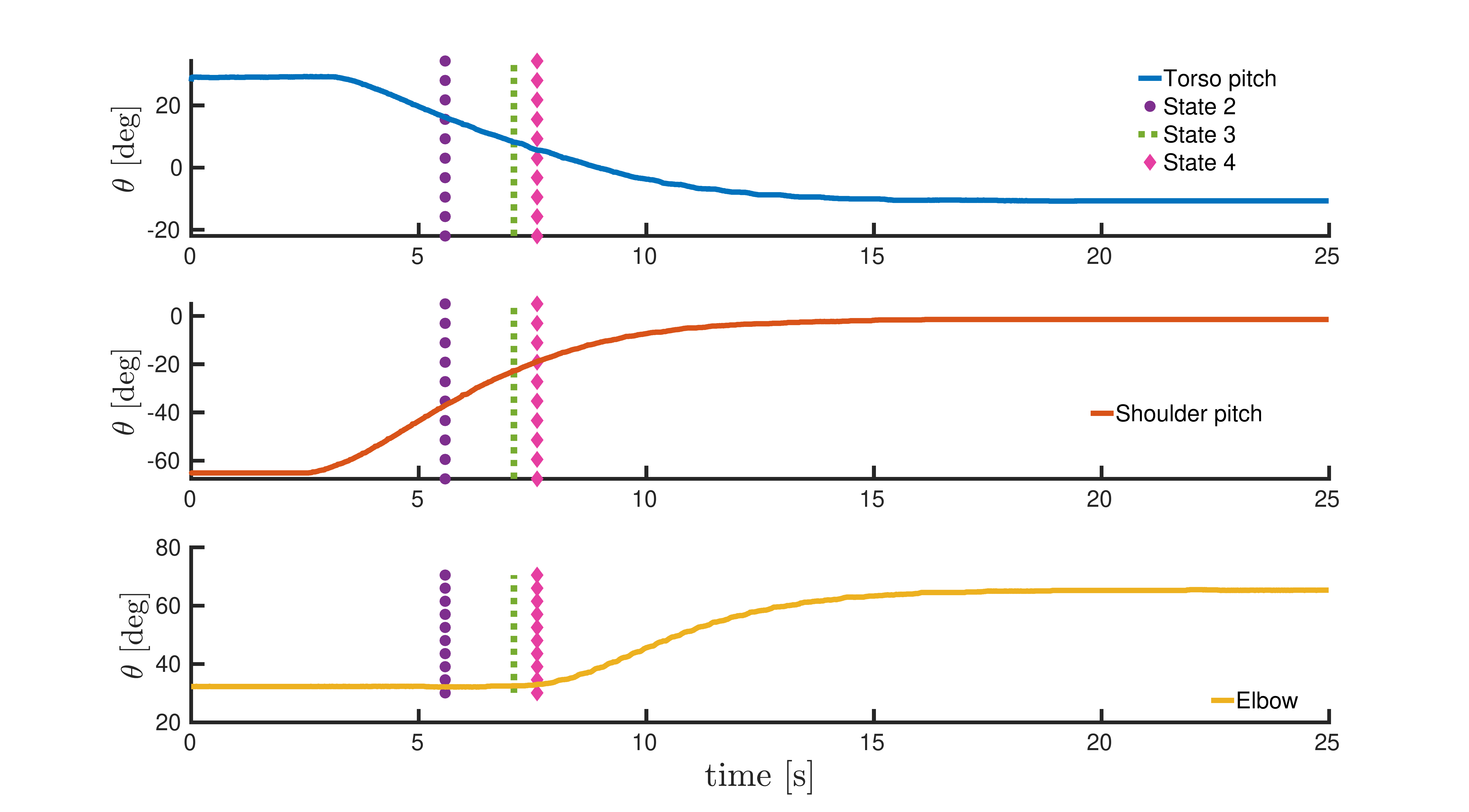}
	\caption{Interacting agent joint trajectories}
	\label{fig:humanQ}
\end{figure}

\begin{figure}[H]
	\centering
	\includegraphics[scale=0.2]{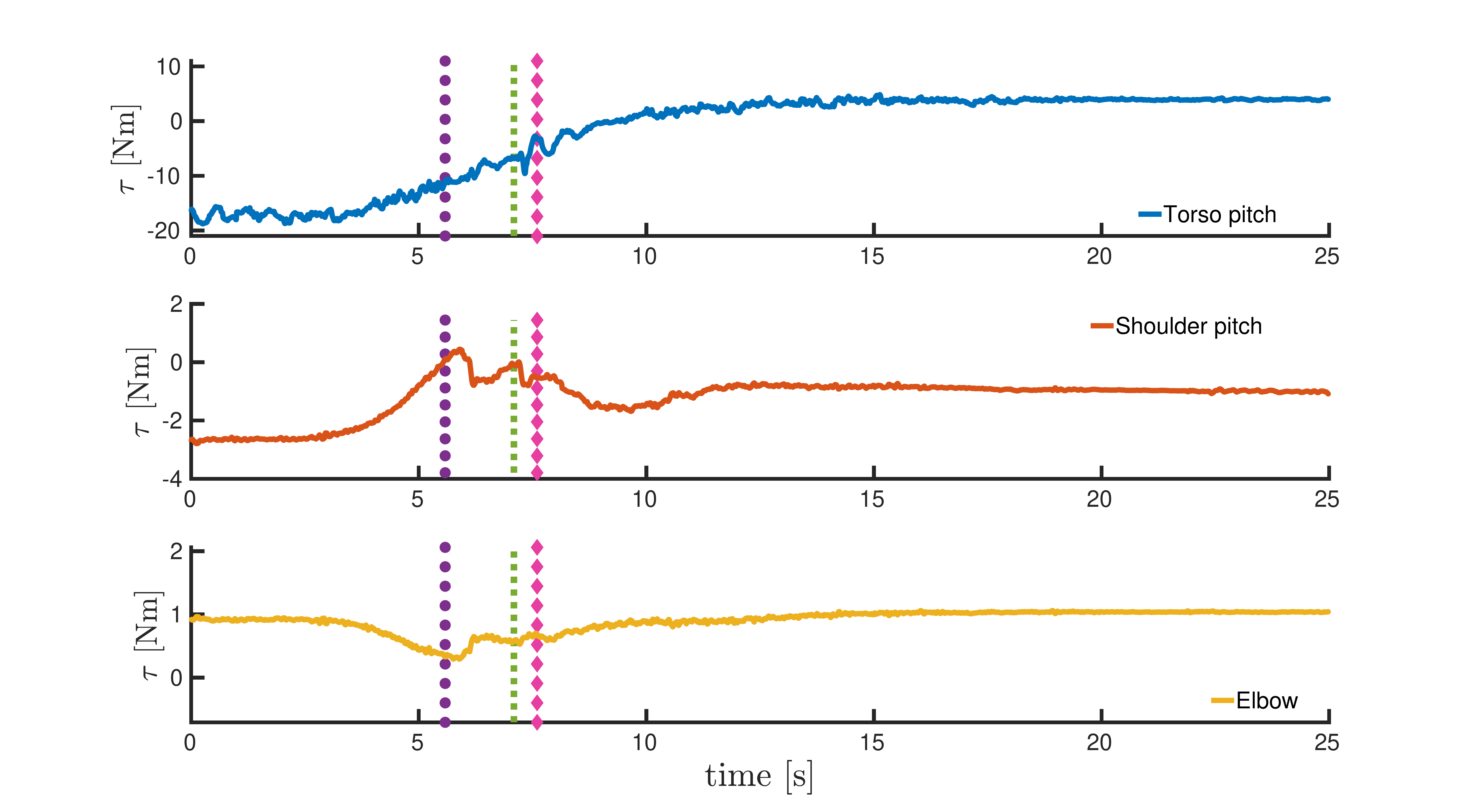}
	\caption{Interacting agent joint torques}
	\label{fig:humanTau}
\end{figure}

The hands of the iCub robot are quite fragile and are not designed to make \emph{sustained} power grasps. This posed quite a challenge during our experiments. So, we designed a new mechanical part highlighted in Figure.~\ref{fig:icub-hand-wrist-attachment} that is attached at the wrists of the iCub hands. Additionally, we designed new mechanical contraption shown in Figure.~\ref{fig:icub-hand-mechanical-contraption} that can be attached rigidly to the wrists of the robots ensuring rigid contacts between the hands during the entire duration of the experiment.

\begin{figure}[H]
	\centering
	\begin{subfigure}{0.65\textwidth}
		\centering
		\includegraphics[width=0.65\textwidth, angle=180]{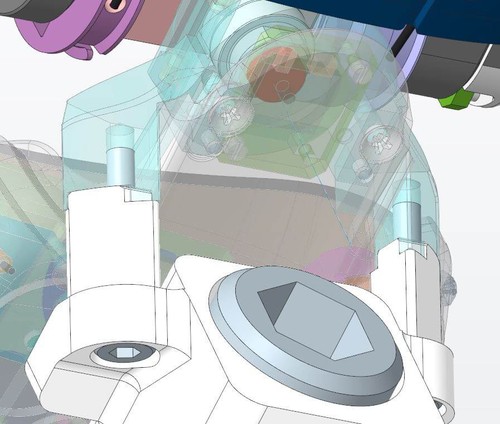}
		\caption{CAD Design}
	\end{subfigure}
	\begin{subfigure}{0.325\textwidth}
		\centering
		\includegraphics[width=0.825\textwidth]{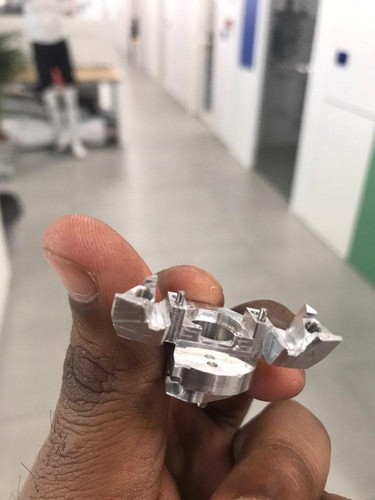}
		\caption{CNC Machined Part}
	\end{subfigure}
	\caption{New mechanical component (cyan) added at the wrist of iCub hand}
	\label{fig:icub-hand-wrist-attachment}
\end{figure}

\begin{figure}[H]
	\centering
	\begin{subfigure}{\textwidth}
		\centering
		\includegraphics[width=0.9\textwidth]{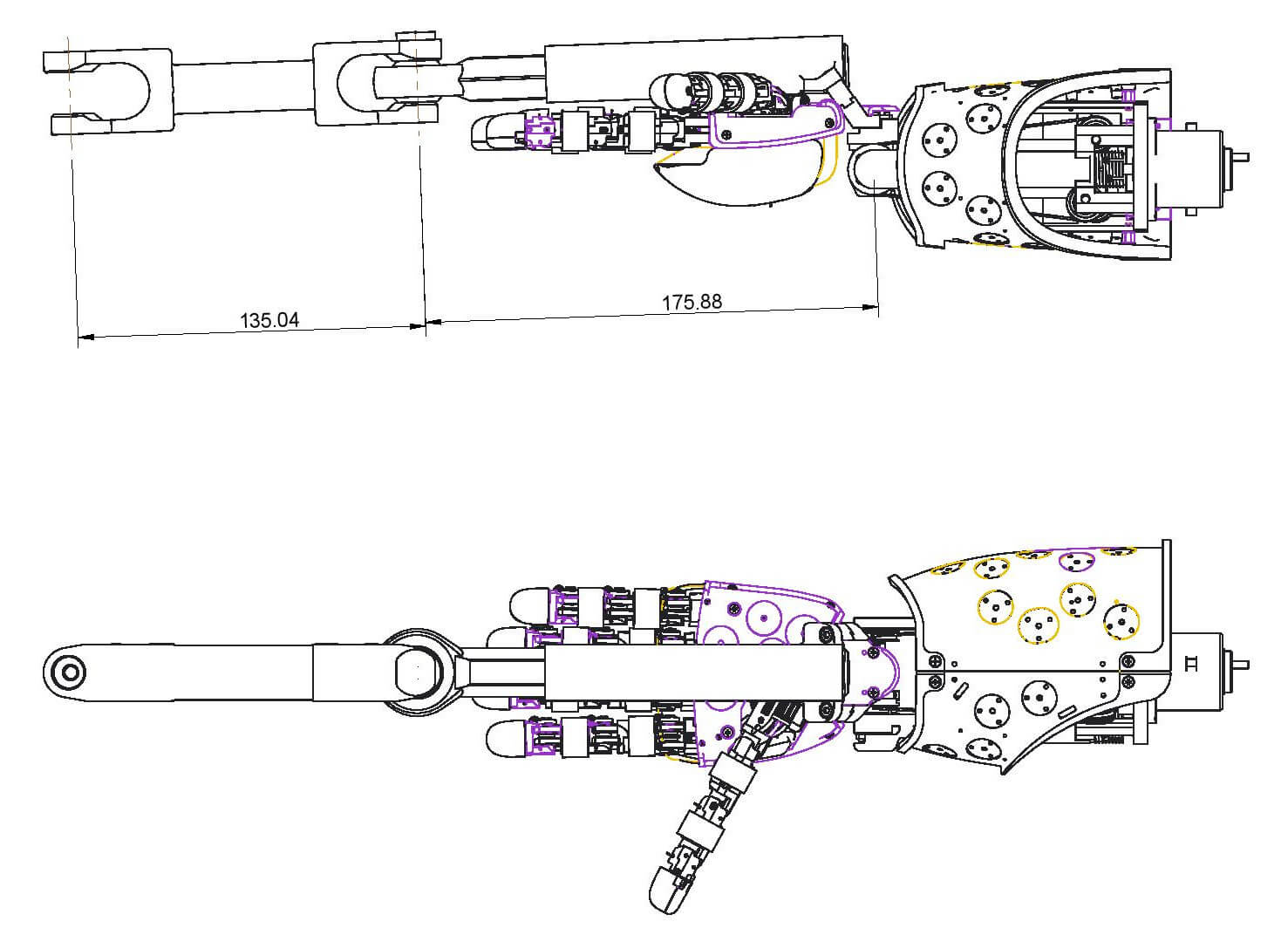}
	\end{subfigure}
	\begin{subfigure}{0.65\textwidth}
		\centering
		\includegraphics[width=0.85\textwidth]{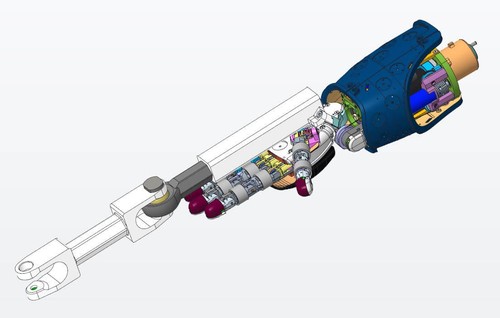}
	\end{subfigure}
	\begin{subfigure}{0.325\textwidth}
		\centering
		\includegraphics[width=0.95\textwidth]{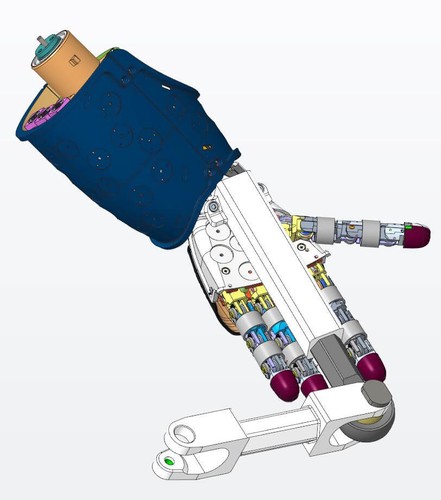}
	\end{subfigure}
	\begin{subfigure}{0.65\textwidth}
		\centering
		\includegraphics[width=0.64\textwidth, angle=90]{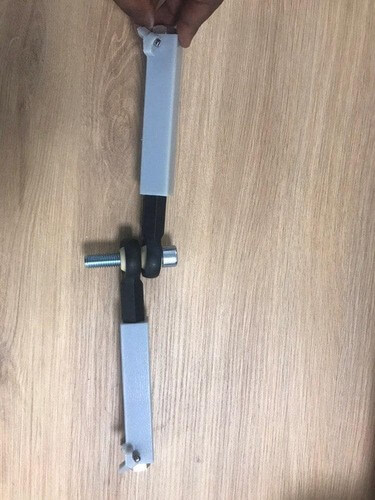}
	\end{subfigure}
	\begin{subfigure}{0.325\textwidth}
		\centering
		\includegraphics[width=0.95\textwidth]{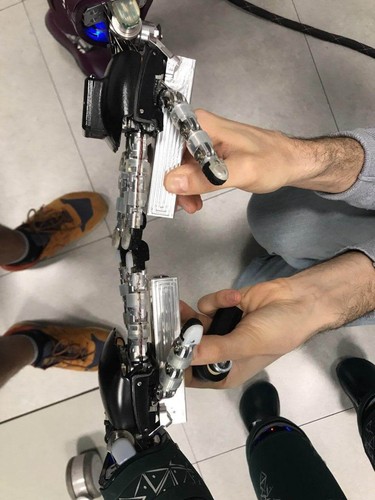}
	\end{subfigure}
	\caption{Mechanical contraption to ensure rigid contacts between iCub hands}
	\label{fig:icub-hand-mechanical-contraption}
\end{figure}

\section{Results}
\label{results}

A predetermined trajectory generated using a minimum-jerk trajectory generator \cite{5650851} is given as a reference to the center of mass of the \emph{purple} iCub robot to perform the sit-to-stand transition. At the start of the Simulink controller the \emph{purple} iCub robot is seated on the metallic structure that serves as a chair. Once the controller is started, it receives joint quantities as inputs from both the robots and actively generates joint torque inputs for the \emph{purple} iCub robot to maintain its momentum and track the center of mass.

The time evolution of the center of mass tracking is shown in Fig. \ref{comErr}. The vertical lines indicate the time instance at which a new state begins. Between states 2 and 3, the \emph{purple} iCub robot is seated on the chair with its upper leg as contact constraints. This seriously limits the robot motion along the $x$-direction and hence the tracking error of the center of mass along the $x$-direction is not negligible. Similarly, between states 3 and 4, the robot has to standup relative quickly as the trajectory smoothing is kept very low and the contact constraints change from the upper legs to the feet. This contact switching, along with unmodeled phenomena such as joint friction limits the robot motion along the $z$-direction and hence the tracking error of the center of mass along the $z$-direction is not negligible. Apart from these two instances, the overall center of mass tracking is good.

\begin{figure}[!hbt]
	\centering
	\includegraphics[scale=0.2]{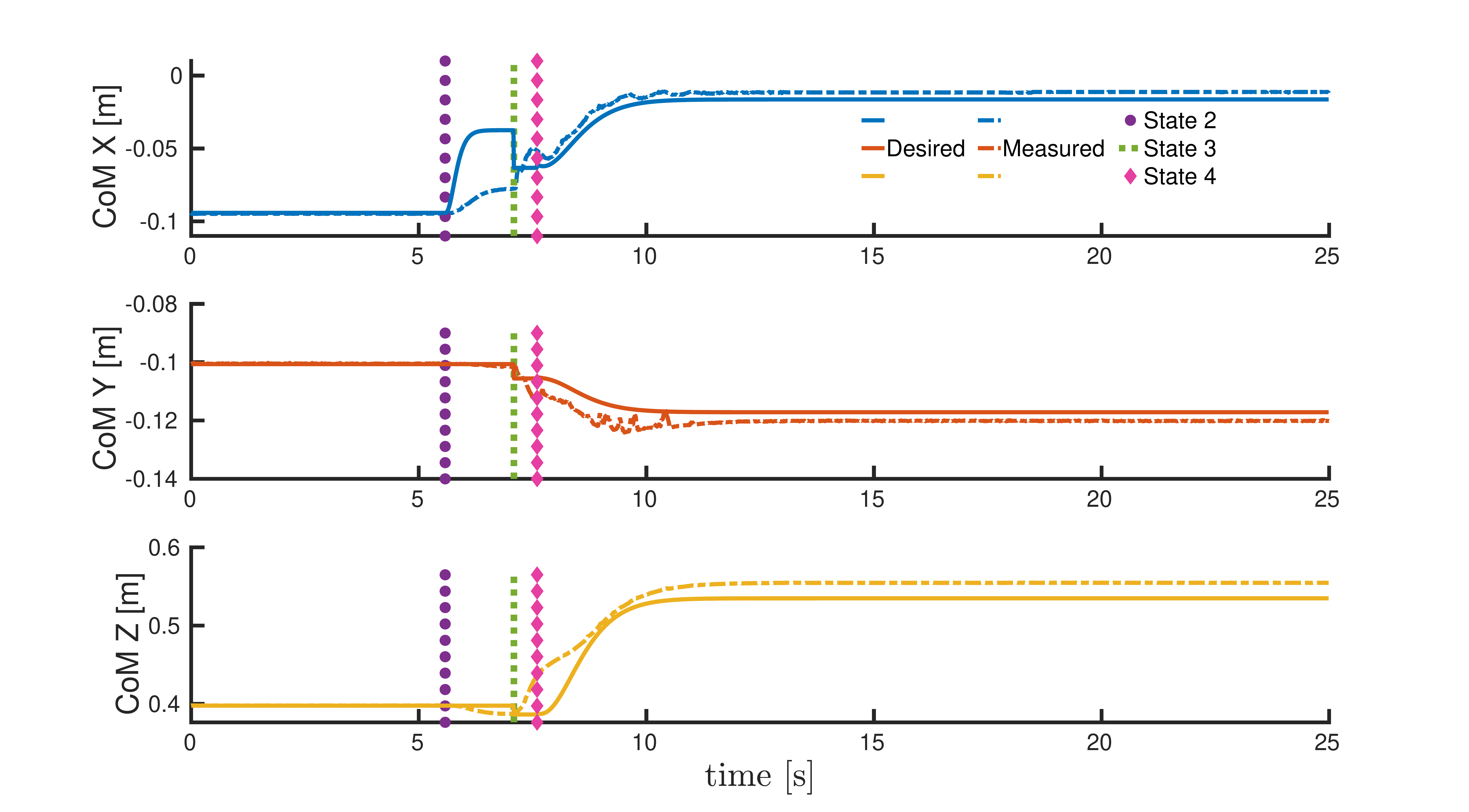}
	\caption{Time evolution of the desired and measured CoM trajectory while performing standup motion on application of the control law \eqref{eq:partner-aware-control-law}}
	\label{comErr}
\end{figure}

The primary control objective of momentum control is also realized well as highlighted by the time evolution of the linear and angular momentum of the robot as shown in Fig. \ref{momentumErr}. Between states 3 and 4, both the linear and angular momentum error increased momentarily. Understandably this results from the impact at the contact switching that occurs at the beginning of state 3. However, the overall robot momentum is maintained close to zero and eventually, the momentum error converges to zero when the robot becomes stable after standing fully erect.

\begin{figure}[!hbt]
	\centering
	\includegraphics[scale=0.2]{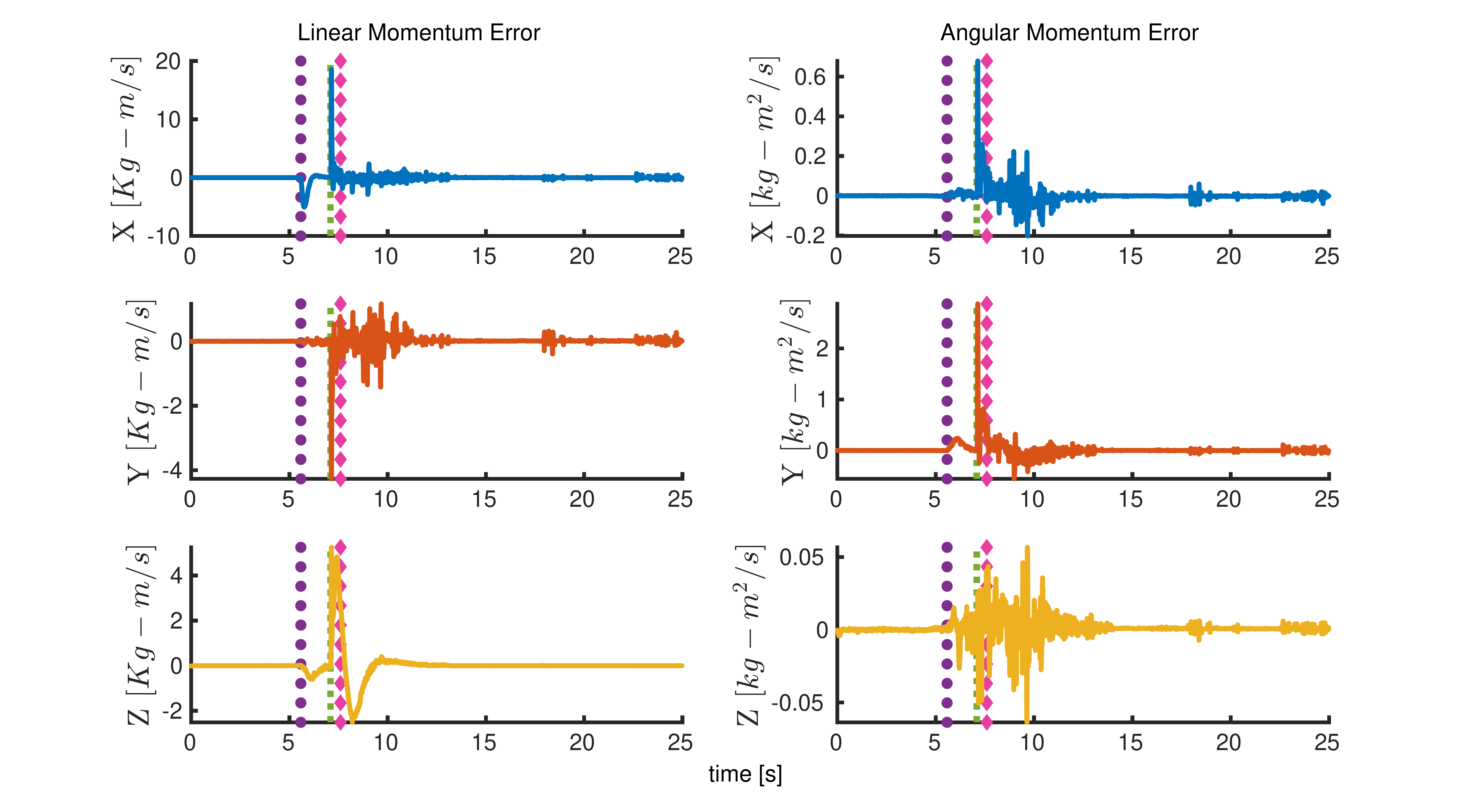}
	\caption{Time evolution of the linear and angular momentum while performing standup motion on application of the control law \eqref{eq:partner-aware-control-law}}
	\label{momentumErr}
\end{figure}

The time evolution of $\alpha$ i.e. the component of the interaction agent joint torques projected in the direction parallel to the task is shown in Fig. \ref{alpha}. The instantaneous values of $\alpha$ change throughout the experiment according to the joint torque values of the \emph{green} iCub robot. Between the states 2 and 3, the negative values of $\alpha$ contribute towards making the Lyapunov function decrease faster as indicated in Fig. \ref{vLyap}. This highlights the fact that the physical interaction with an external agent is exploited (in terms of the joint torques) by the \emph{purple} iCub robot to perform the standup task. 

\begin{figure}[!hbt]
	\centering
	\includegraphics[scale=0.195]{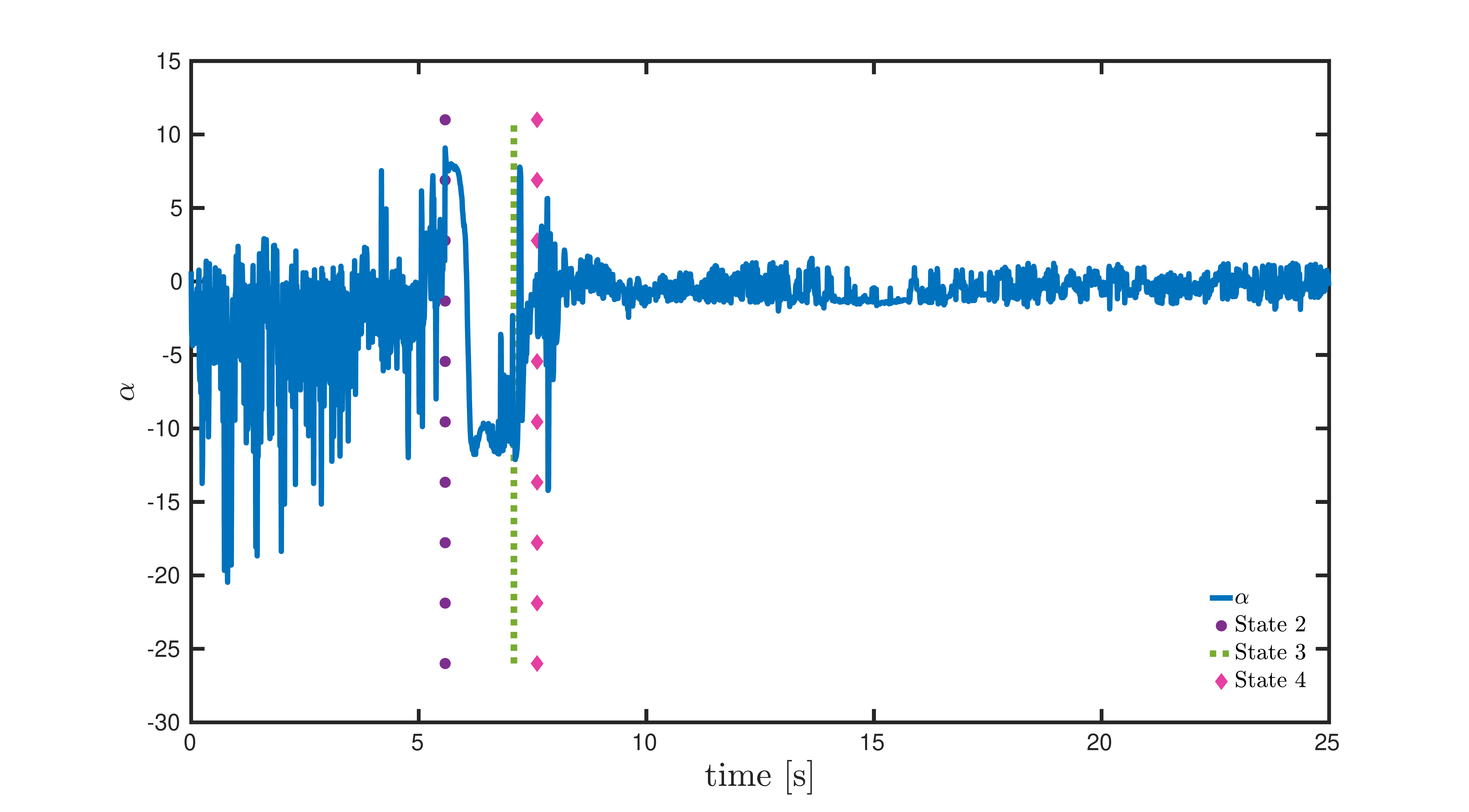}
	\caption{Time evolution of $\alpha$ under the influence of physical interaction}
	\label{alpha}
\end{figure}

The time evolution of the Lyapunov function $\mathrm{V}$ from equation (\ref{eq:partner-aware-control-LyapunovFunction}) is shown in Fig. \ref{vLyap}. After the controller is started, during state 1 the system has small energy while the robot is seated on the chair. This is highlighted in the inset plot shown for the duration between 1-2 seconds. Starting from state 2 as the robot starts moving, the total energy of the systems starts to increase as shown between states 2 and 3. As the robot enters state 3 the energy quickly drops during the contact switching from upper legs to feet. This is a direct reflection of exploiting the physical interaction with the \emph{green} iCub robot. Between states 3 and 4 while the robot is moving to a fully erect stance the energy rises slightly and eventually settles to a stable value. The inset plot during the duration between 16-17 seconds highlights the system energy when the robot is in a stable fully erect position.

\begin{figure}[!hbt]
	\centering
	\includegraphics[scale=0.27]{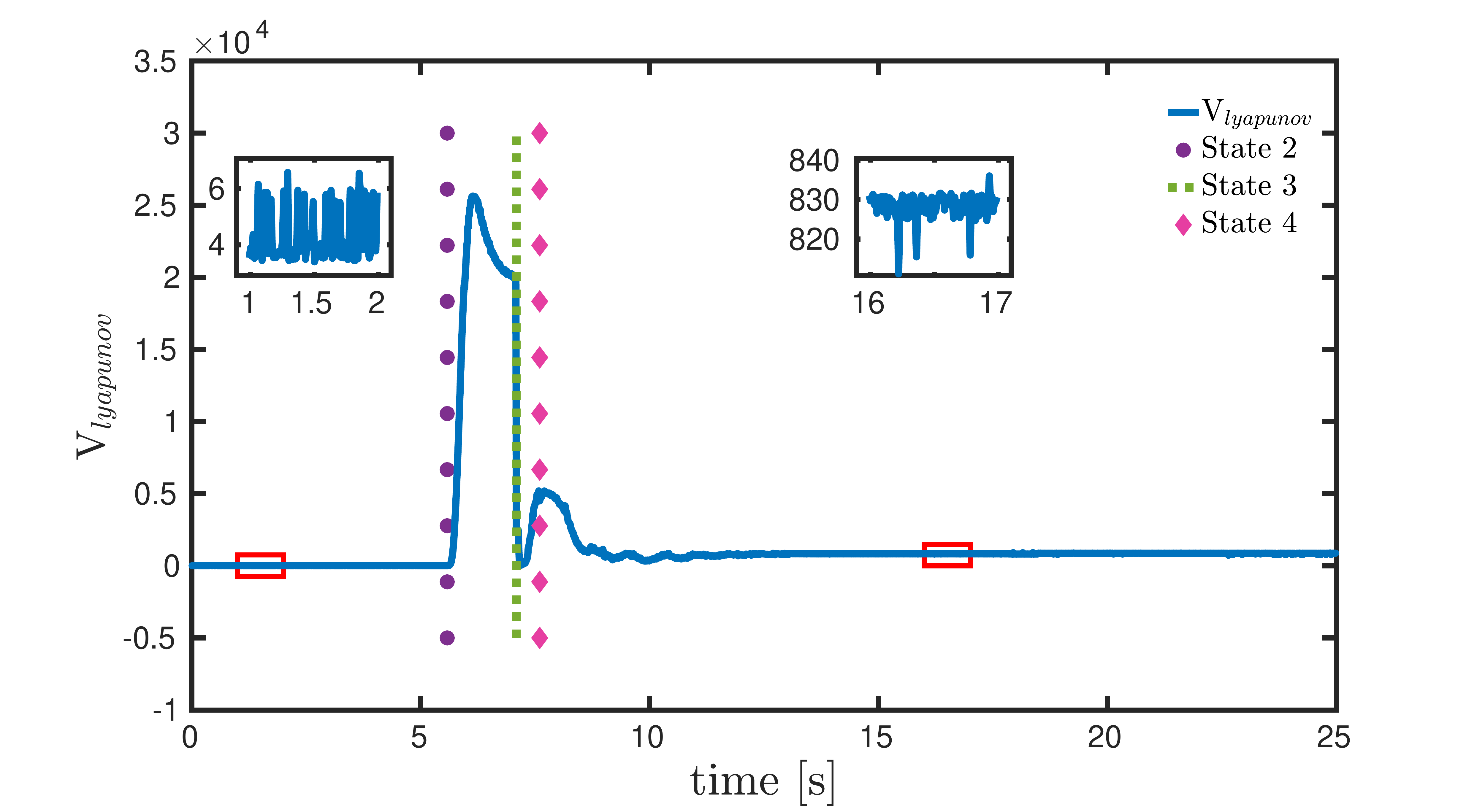}
	\caption{Time evolution of lyapunov function considered \eqref{eq:partner-aware-control-LyapunovFunction} on application of the control law \eqref{eq:partner-aware-control-law}}
	\label{vLyap}
\end{figure}

\chapter{Trajectory Advancement}
\label{cha:trajectory-advancement}

\chapreface{The previous chapter on partner-aware control proved to provide a stable robot behavior while following a given reference Cartesian trajectory for the robot's center of mass. Although the proposed partner-aware control law exploits help in terms of joint torques, the reference trajectory for the center of mass is not modified or updated under the influence of the help from an external agent. This chapter presents the concept of trajectory advancement through which the robot can advance along the reference trajectory leveraging assistance from physical interactions with a human partner.}
	
Considering the particular experiment of robot standing up~(Fig.~\ref{fig:example-icub-help}) using the assistance, an intuitive behavior one would expect for the robot is to stand up quicker by leveraging the assistance. Consider another example case of a manipulator robot moving along a given Cartesian reference trajectory performing a pick and place task~(Fig.~\ref{fig:example-cobot-help}). An intuitive interaction of a human with the intention to speed up the robot motion is to apply forces in the robot's desired direction. Under such circumstances, traditionally, the robot can either render a compliant behavior through impedance/admittance control or switch to gravity compensation mode that allows the human to move the robot freely (compromising task accuracy). Instead, a more intuitive behavior is to advance further along the reference trajectory and complete the task quicker.

\begin{figure}[H]
	\centering
	\begin{subfigure}{0.49\textwidth}
		\centering
		\includegraphics[clip, trim=0cm 2cm 0cm 0cm, scale=0.375]{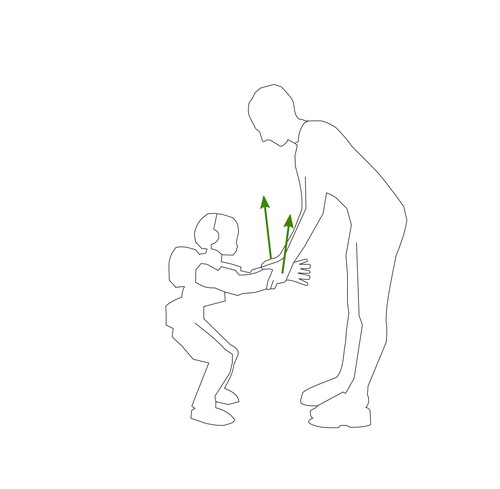}
		\caption{}
		\label{fig:example-icub-help}
	\end{subfigure}%
	\begin{subfigure}{0.49\textwidth}
		\centering
		\includegraphics[clip, trim=0cm 0cm 0cm 0cm, scale=0.35]{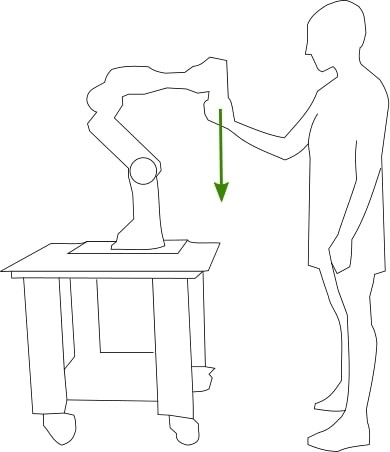}
		\caption{}
		\label{fig:example-cobot-help}
	\end{subfigure}
	\caption{Example scenarios of HRC with human and robot engaged in physical interactions}
	\label{fig:HRC-scenarios}
\end{figure}

\section{Parametrized Reference Trajectory}

Traditionally, motion control problems involving tracking of a reference trajectory has both spatial dimension, encapsulated in a geometric path, and temporal dimension, encapsulated in the dynamic evolution of the geometric path \citep{aguiar2004path}. Accordingly, the reference trajectory is a time ($t$) parametrized curve and the control design drives the system to a specific point in space at a specific predefined time. In contrast, the path following the problem involves converging to and following a geometric path without any temporal constraints \citep{breivik2005principles}. In this work, we bank on the concepts of path following and design a parametric curve parametrized with a \textit{free parameter} $\psi \in [0, \infty)$. The choice of $\mit{\psi}$ becomes clear in the subsequent sections. The resulting parametric curve $\comVar{x}_d(\psi)$ is the desired geometric path to be followed spatially by a link of the robot. Assuming that the free parameter is time-dependent i.e., $\psi = \psi(t)$, the first and second time derivatives of the path are given as, 

\begin{subequations}
    \begin{equation}
        \dot{\comVar{x}}_d(\psi, \dot{\psi}) = \partial_\psi \comVar{x}_d(\psi) \ \dot{\psi}
        \label{eq:trajectory-first-derivative}
    \end{equation}
    \begin{equation}
        \ddot{\comVar{x}}_d(\psi, \dot{\psi}, \ddot{\psi}) = \partial_{\psi}^2 \comVar{x}_d(\psi) \ \dot{\psi}^2 + \partial_\psi \comVar{x}_d(\psi) \ \ddot{\psi}
        \label{eq:trajectory-second-derivative}
    \end{equation}
\end{subequations}

\section{Interaction Exploitation}

Consider the control objective of trajectory tracking where at each time instant the reference position ($\comVar{x}_d(\psi))$, velocity $(\dot{\comVar{x}}_d(\psi, \dot{\psi}))$, and acceleration $(\ddot{\comVar{x}}_d(\psi, \dot{\psi}, \ddot{\psi}))$ are taken from the reference trajectory parametrized in $\psi$. The term $\mathbold{\Omega} {\comVar{ f}}$ in the robot control torques equation Eq.~\eqref{eq:normal-control-torques-compact} represents the Cartesian resultant acceleration that results under the influence of external interaction wrench ${\comVar{ f}}$. Instead of completely canceling out the effects of external interaction wrench, it is desirable to exploit any \textit{helpful} components to advance along the desired reference trajectory making an active collaboration possible during \textsc{hrc}. More specifically, let us define the \emph{helpful} interaction by decomposing the resultant acceleration into \textit{parallel} and \textit{perpendicular} components along the desired velocity as,

\begin{subequations}
    \begin{equation}
        \mathbold{\Omega} {\comVar{ f}} = \alpha \ \dot{\comVar{x}}_d^{\parallel} + \beta \ \dot{\comVar{x}}_d^{\perp}
        \label{eq:trajectory-advancement-wrench-decomposition}
    \end{equation}
    \begin{equation}
        \dot{\comVar{x}}_d^{\parallel} = \frac{\dot{\comVar{x}}_d}{\norm{\dot{\comVar{x}}_d}}, \quad \alpha = \frac{\dot{\comVar{x}}_d^T \mathbold{\Omega} {\comVar{ f}}}{\norm{\dot{\comVar{x}}_d}} \notag
    \end{equation}
\end{subequations}

where $\dot{\comVar{x}}_d^{\parallel} \in \mathbb{R}^{6}$ is the unit vector along the direction of the desired velocity, $\alpha \in \mathbb{R}$ is the resultant acceleration component projected along the direction parallel to the direction of the desired velocity. An intuitive choice for the component $\alpha$ is in the direction of the desired velocity i.e. $\alpha > 0$. Accordingly, we define a \textit{correction wrench}\footnote{The name \textit{correction wrench} is an abuse of notation but has an intuitive meaning in conveying the notion of helpful interaction wrench. Also, the units of wrench $[\si{\newton},\si{\newton\meter}]$ are used.} term given by $\alpha \ \dot{\comVar{x}}_d^{\parallel} \in \mathbb{R}^{6} \ \forall \ \alpha > 0$, which represents the \textit{helpful} interaction mathematically.

\section{Trajectory Advancement}

\begin{proposition}
\label{proposiiton-update-law}
The time evolution of the free parameter $\psi$ for trajectory advancement leveraging assistance is given by the following update rule,
\begin{equation}
    \dot{\psi} = min \left\{ \dot{\psi}_{upper}, max \left\{ 1, \frac{\dot{\comVar{x}}(t)^\top \ \partial_{\psi} \comVar{x}_d(\psi)}{\norm{\partial_{\psi} \comVar{x}_d(\psi)}^2}  \right\}\right\}
    \label{eq:update-rule}
\end{equation}

\end{proposition}

The update rule in Eq.~\eqref{eq:update-rule} reflects the time evolution of the free parameter $\psi$ which helps in advancing along the desired reference trajectory exploiting the external interaction wrenches with the robot. The lower bound value $1$ signifies that the new parametrization is exactly equal to the time parametrized trajectory i.e. $\psi = t$ until any external wrench ${\comVar{ f}}$ is applied such that it will help the robot's task.

\begin{proof}
Considering the task of Cartesian reference trajectory tracking, the desired dynamics for the control objective can be written as directed by Eq.~\eqref{eq:cartesian-control-objective}. Given the correction wrench term Eq.~\eqref{eq:trajectory-advancement-wrench-decomposition}, the desired dynamics for the trajectory tracking task is updated as,

\begin{equation}
\label{eq:control-objective-updated}
\ddot{\comVar{x}} = \ddot{\comVar{x}}^* := \ddot{\comVar{x}}_d - \comVar{K}_D \ \dot{\widetilde{\comVar{x}}} - \comVar{K}_P \int_0^t \dot{\widetilde{\comVar{x}}} ~du + \alpha \ \dot{\comVar{x}}_d^{\parallel} \ \ \forall \ \alpha > 0
\end{equation}

Using the above choice of the desired dynamics, the robot control torques defined in Eq.~\eqref{eq:normal-control-torques-compact} will only compensate for external wrench that is not helpful. Now, consider the following Lyapunov function candidate, 

\begin{equation}
\mathrm{V} = \frac{1}{2} \norm{\dot{\comVar{x}}(t) - \dot{\comVar{x}}_d(\psi, \dot{\psi})}^{2} + \frac{\comVar{K}_P}{2} \norm{\int_{0}^{t}(\dot{\comVar{x}}(t) - \dot{\comVar{x}}_d(\psi, \dot{\psi})) du}^{2}
\label{eq:lyapunov-function}
\end{equation}

On differentiating $\mathrm{V}$, we get: 

\begin{subequations}
	\begin{equation}
	\dot{\mathrm{V}} = \ \dot{\widetilde{\comVar{x}}}^\top \ddot{\widetilde{\comVar{x}}} + \int_{0}^{t} \dot{\widetilde{\comVar{x}}}^\top du \ \comVar{K}_p \dot{\widetilde{\comVar{x}}} \notag		
	\end{equation}
	\begin{equation}
	\dot{\mathrm{V}} =  \dot{\widetilde{\comVar{x}}}^\top \ddot{\widetilde{\comVar{x}}} +  \dot{\widetilde{\comVar{x}}}^\top \ \comVar{K}_p \int_{0}^{t}\dot{\widetilde{\comVar{x}}} \ du \  \notag
	\end{equation}
	\begin{equation}
	\dot{\mathrm{V}} = \dot{\widetilde{\comVar{x}}}^\top [\ddot{\widetilde{\comVar{x}}} + \comVar{K}_p \int_{0}^{t}\dot{\widetilde{\comVar{x}}} \ du] \notag
	\end{equation}
\end{subequations}

Given the updated desired dynamics in Eq.~\eqref{eq:control-objective-updated} we rearrange it as $\ddot{\widetilde{\comVar{x}}} + \comVar{K}_p \int_{0}^{t}\dot{\widetilde{\comVar{x}}} \ du = - \comVar{K}_D \ \dot{\widetilde{\comVar{x}}} + \alpha \ \dot{\comVar{x}}_d^{\parallel}$ and use it in the derivative of the Lyapunov function to obtain the following relation,

\begin{subequations}
	\begin{equation}
	\dot{\mathrm{V}} = \dot{\widetilde{\comVar{x}}}^\top [ - \comVar{K}_D \ \dot{\widetilde{\comVar{x}}} + \alpha \ \dot{\comVar{x}}_d^{\parallel}] \  \notag
	\end{equation}
	\begin{equation}
	\dot{\mathrm{V}} = - \dot{\widetilde{\comVar{x}}}^\top \comVar{K}_D \ \dot{\widetilde{\comVar{x}}} + \dot{\widetilde{\comVar{x}}}^T \alpha \ \dot{\comVar{x}}_d^{\parallel}\  \notag
	\end{equation}
\end{subequations}

According to Lyapunov theory, the stability of the system is ensured when $\dot{\mathrm{V}} \le 0$. Given that $\comVar{K}_D$ is a positive symmetric matrix, the term $- \dot{\widetilde{\comVar{x}}}^\top \ \comVar{K}_D \ \dot{\widetilde{\comVar{x}}} \le 0$. So, to ensure the stability of the system i.e. $\dot{\mathrm{V}} \le 0$, the following condition has to be satisfied,

\begin{equation}
\dot{\widetilde{\comVar{x}}}^\top \alpha \ \dot{\comVar{x}}_d^{\parallel} \le 0 \notag
\end{equation}

Considering that $\alpha > 0$ and $\norm{\dot{\comVar{x}}_d} > 0$, the above inequality is equivalent to $\dot{\widetilde{\comVar{x}}}^\top \dot{\comVar{x}}_d(\psi, \dot{\psi}) \le 0$,

\begin{subequations}
	\begin{equation}
	(\dot{\comVar{x}}(t) - \dot{\comVar{x}}_d(\psi, \dot{\psi}))^\top \dot{\comVar{x}}_d(\psi, \dot{\psi}) \le 0 \notag
	\end{equation}
	\begin{equation}
	\dot{\comVar{x}}(t)^\top \dot{\comVar{x}}_d(\psi, \dot{\psi}) - \norm{\dot{\comVar{x}}_d(\psi, \dot{\psi})}^2 \le 0 \notag
	\end{equation}
	\begin{equation}
	\dot{\comVar{x}}(t)^\top \partial_{\psi} \comVar{x}_d(\psi)  \dot{\psi} - \norm{\partial_{\psi} \comVar{x}_d(\psi)}^2 \dot{\psi}^2 \le 0 \notag
	\end{equation}
	\begin{equation}
	(\dot{\comVar{x}}(t)^\top \partial_{\psi} \comVar{x}_d(\psi) - \norm{\partial_{\psi} \comVar{x}_d(\psi)}^2 \dot{\psi}) \dot{\psi} \le 0 \notag
	\end{equation}
\end{subequations}

Assuming the lower bound $\dot{\psi} \ge 1$, we obtain

\begin{equation}
\dot{\psi} \ge \frac{\dot{\comVar{x}}(t)^\top \partial_{\psi} \comVar{x}_d(\psi)}{\norm{\partial_{\psi} \comVar{x}_d(\psi)}^2}
\label{eq:sdot-condition}
\end{equation}

The condition in Eq.~\eqref{eq:sdot-condition} reflects the time evolution of the free parameter $\psi$ which helps in advancing along the desired reference trajectory exploiting the external interaction wrenches with the robot. The lower bound value $1$ signifies that the new parametrization is exactly equal to the time parametrized trajectory i.e. $\psi = t$ until any external wrench $\comVar{ f}$ is applied such that it will help the robot's task. Under the influence of \textit{helpful} external wrench, the value of $\dot{\psi}$ becomes greater than $1$. On integrating/differentiating $\dot{\psi}$ we determine the advancement along the desired reference trajectory,
\begin{equation}
\psi^* = \int_{t_1}^{t_2} \dot{\psi} \ du, \quad \ddot{\psi}^* = \frac{d\dot{\psi}}{dt} \nonumber
\end{equation}

Now, the updated references for trajectory tracking becomes,

\begin{equation}
\comVar{x}_d(\psi^*), \dot{\comVar{x}}_d(\psi^*,\dot{\psi}), \ddot{\comVar{x}}_d(\psi^*,\dot{\psi},\ddot{\psi}^*) \notag
\end{equation}

Besides, an upper limit $\dot{\psi}_{upper}$ is set to bound the length of advancement along the reference trajectory ensuring safe physical interactions. 

The updated equation of $\dot{\psi}$ becomes,

\begin{equation}
\dot{\psi} = min \left\{ \dot{\psi}_{upper}, max \left\{ 1, \frac{\dot{\comVar{x}}(t)^\top \ \partial_{\psi} \comVar{x}_d(\psi)}{\norm{\partial_{\psi} \comVar{x}_d(\psi)}^2}  \right\}\right\}
\label{eq:sdot-final-equation-with-both-limits}
\end{equation}

\textbf{Remark:} Strictly speaking, the choice of $\dot{\psi}$ as stated in Eq.~\eqref{eq:sdot-final-equation-with-both-limits} induces an algebraic loop when applied with the control law Eq.~\eqref{eq:control-objective-updated}. In fact, the updated reference acceleration $\ddot{\comVar{x}}_d(\psi^*,\dot{\psi},\ddot{\psi}^*)$ does depend on the Cartesian acceleration $\ddot{\comVar{x}}$, and, consequently, on the joint torques ${\robtau}$. For this reason, no formal stability statement was claimed in Proposition~\ref{proposiiton-update-law}. From the theoretical point of view, the algebraic loop can be avoided by designing an update rule for $\ddot{\psi}$ rather than $\dot{\psi}$, and by modifying the control law \eqref{eq:control-objective-updated} so that the reference Cartesian acceleration is not compensated anymore. This choice, however, would imply the calculation of $\psi^*$ through double numerical integration of  $\ddot{\psi}$, which may lead to a fast divergence of the reference trajectory due to numerical drifts. For this reason, the proposed control solution \eqref{eq:sdot-final-equation-with-both-limits}-\eqref{eq:control-objective-updated}, despite not being fully theoretically sound, resulted to be more robust when applied in practice. Furthermore, the algebraic loop can be resolved at the implementation level by computing the numerical derivative $\ddot{\psi}^* = \frac{d\dot{\psi}}{dt}$ with one time step of delay, and/or by low-pass filtering the signal to also attenuate the effect of numerical noise. Driven by these motivations we used Eq.~\eqref{eq:sdot-final-equation-with-both-limits}-\eqref{eq:control-objective-updated} for controlling the robot and verified the closed-loop system stability numerically.

\end{proof}

\section{End-Effector Trajectory Advancement Experiments}

The robotic platform considered for the experiment is the iCub humanoid robot described in Section~\eqref{sec:background-icub}. The control objective is to move the \textit{ right foot} of the robot along the desired reference trajectory. The leg of the robot has $3$ joints at the hip, $1$ joint at the knee, and $2$ joints at the ankle. For the sake of intuition, only one dimensional (1D) trajectory in the $x$-direction is considered. The reference trajectory is a sinusoidal function of amplitude $0.05 \si{\meter}$ with frequency $0.1 \si{\hertz}$ and is designed to have a minimum jerk profile \citep{kyriakopoulos1988minimum}. Concerning the task of trajectory tracking with the leg, the robot base is fixed on a pole as shown in Fig.~\ref{fig:icub-on-pole}. The link frame associated with the right foot of the robot and the inertial frame of reference (shown under the base/pelvis of the robot) is highlighted in Fig.~\ref{fig:icub-foot-frames}.

\begin{figure}[!hbt]
	\centering
	\begin{subfigure}{0.5\textwidth}
		\centering
		\includegraphics[clip, trim=0 0.5cm 0 0cm, scale=0.475]{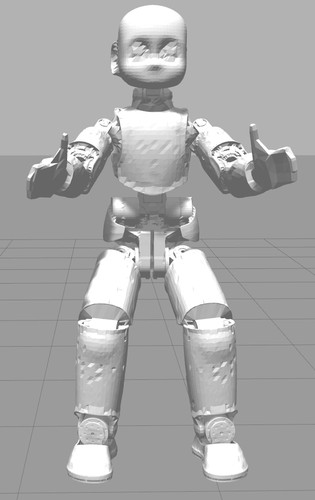}
		\caption{}
		\label{fig:icub-on-pole}
	\end{subfigure}%
	\begin{subfigure}{0.5\textwidth}
		\centering
		\includegraphics[clip, trim=0 0.5cm 0 0cm, scale=0.475]{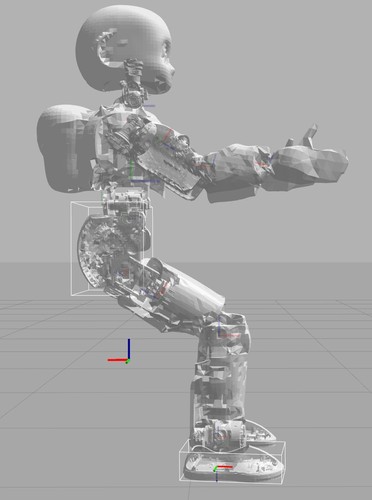}
		\caption{}
		\label{fig:icub-foot-frames}
	\end{subfigure}
	\caption{iCub humanoid robot with base fixed on a pole in gazebo simulation}
	\label{fig:icub}
\end{figure}

Experiments are carried out in both Gazebo simulation and on the real robot. The controller is implemented in Matlab Simulink, using whole-body toolbox\footnote{\href{https://github.com/robotology/wb-toolbox}{https://github.com/robotology/wb-toolbox}}~\citep{RomanoWBI17Journal}, as a stack-of-tasks controller with trajectory tracking as the primary objective. The controller gains are tuned to achieve good trajectory tracking both in simulation and on the real robot as highlighted in Fig.~\ref{fig:normal-trajectory-tracking-1d-x}. The trajectory tracking error in the case of simulation is very small and can be attributed to numerical instability of the dynamics integration in Gazebo simulation and numerical noise in measurements. On the other hand, the trajectory tracking error on the real robot is certainly higher than in simulations owing to several unmodeled effects such as joint friction which are prominent on the real robot. Additionally, friction induces phase delays in following the desired trajectory resulting in higher tracking error. The upper limit $\dot{\psi}_{upper}$ is set to $10$ for experiments both in simulation and on the real robot.

\begin{figure}[!hbt]
	\centering
	\begin{subfigure}{0.49\textwidth}
		\centering
		\includegraphics[clip, trim=0.75cm 2.75cm 4.75cm 2.75cm, scale=0.105]{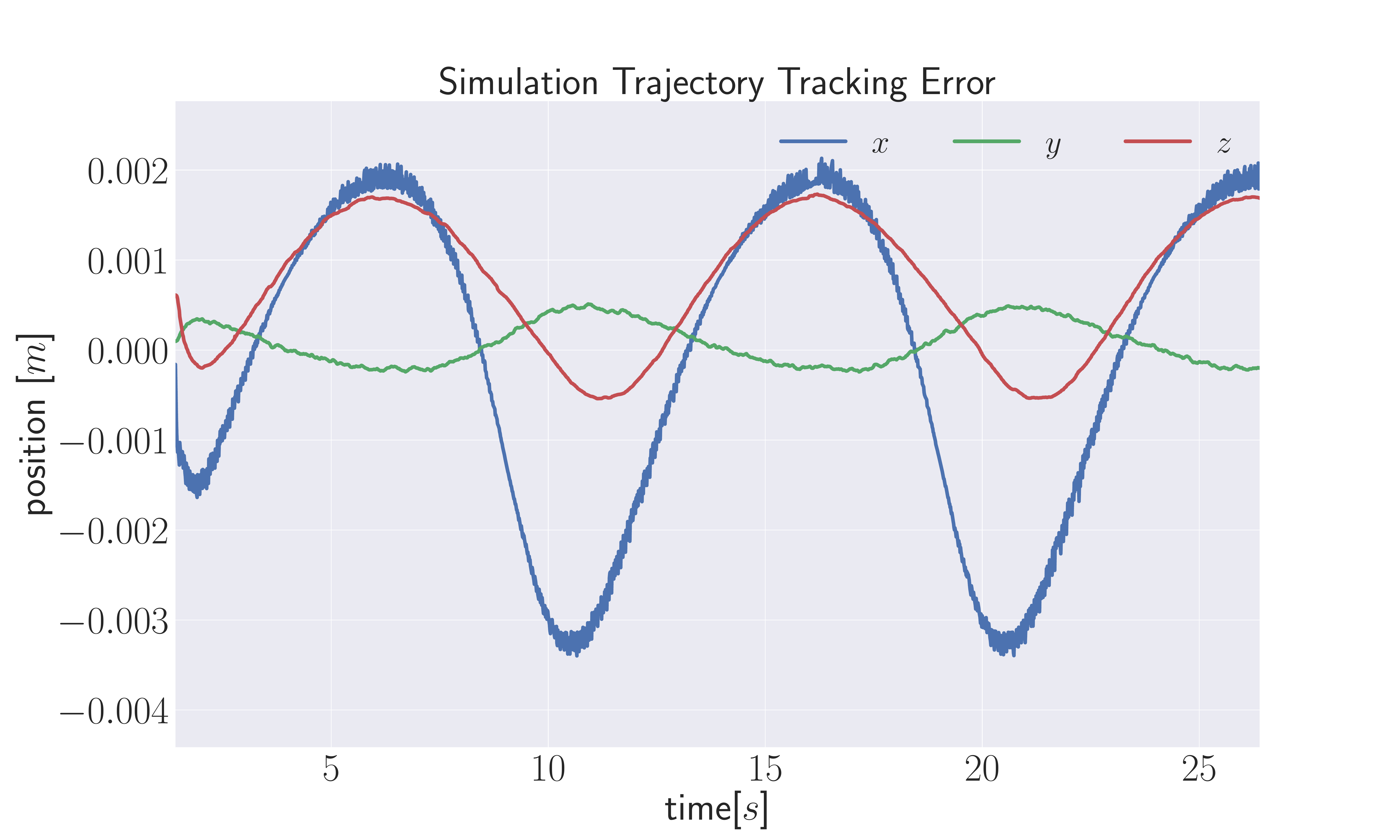}
	\end{subfigure}%
	\begin{subfigure}{0.49\textwidth}
		\centering 
		\includegraphics[clip, trim=0.75cm 0.5cm 4.75cm 2.75cm, scale=0.105]{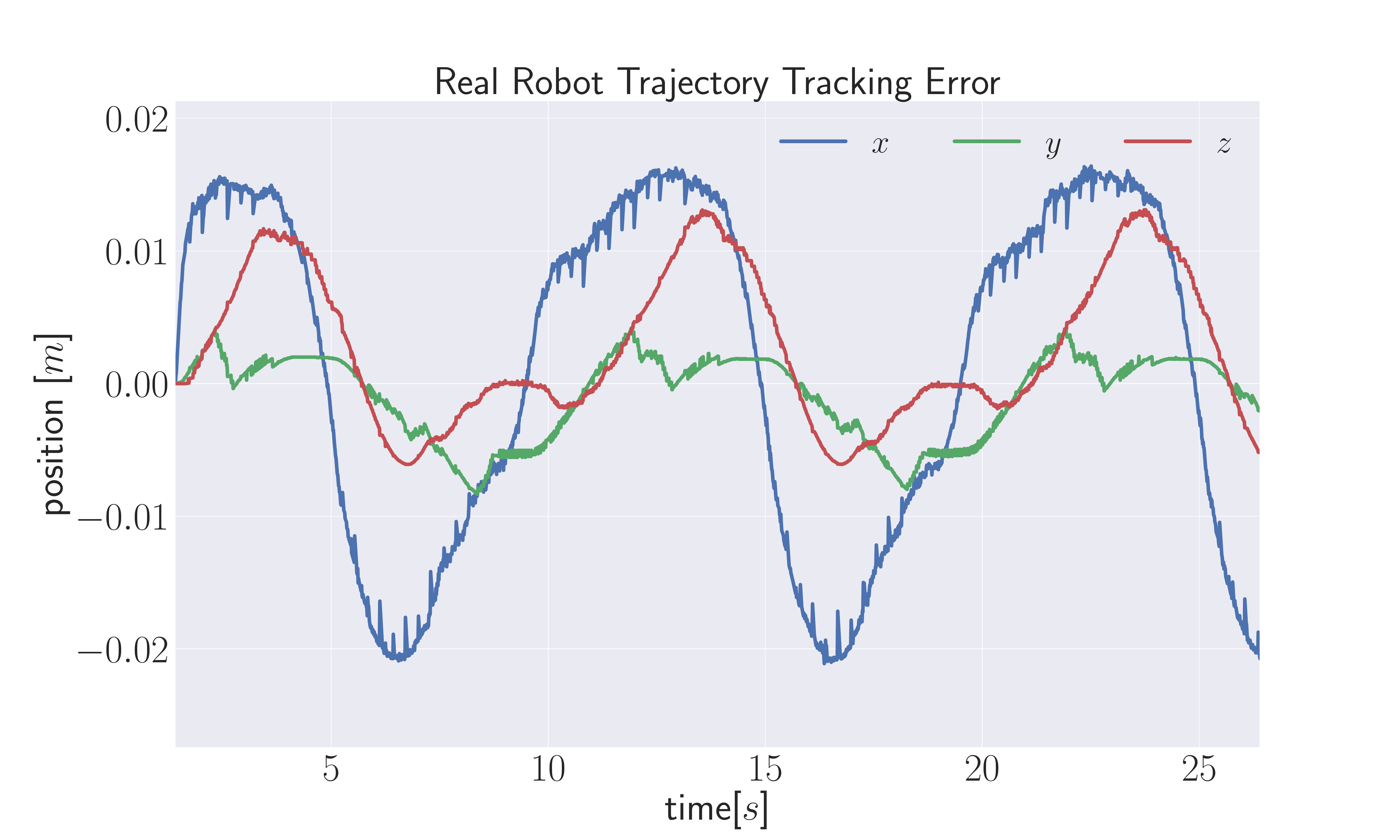}
	\end{subfigure}
	\caption{Trajectory tracking error under no external wrenches}
	\label{fig:normal-trajectory-tracking-1d-x}
\end{figure}

\subsection{Wrench Classification}

The iCub robot has a force-torque sensor embedded at the end-effector considered for the experiments i.e. the right foot. Instead of reading the sensor measurements directly in the sensor frame, the wrench measurements are expressed with a frame that has the origin of the end-effector frame and the orientation of the inertial frame of reference \citep{nori2015}. An external wrench applied to the link of the robot is classified, in this work, in two ways: 

\begin{itemize}
	\item \textit{Assistive wrench} if the external wrench has a vector component along the desired direction of motion
	\item \textit{Agnostic wrench} if the external wrench does \textit{not} have vector components along the desired direction of motion
\end{itemize}

Examples of external wrench classification are highlighted in Fig.~\ref{fig:wrench-classification}. Considering that the desired direction of motion for the foot is in \textit{positive} $x$-direction with respect to the inertial frame of reference, the external wrenches shown in Fig.~\ref{fig:assistive-wrench-1}~\ref{fig:assistive-wrench-2}~\ref{fig:assistive-wrench-3} are assistive wrenches as they have a vector component along the positive $x$-direction. Similarly, the external wrenches shown in Fig.~\ref{fig:agnostic-wrench-1}~\ref{fig:agnostic-wrench-2}~\ref{fig:agnostic-wrench-3} are agnostic wrenches as they do \textit{not} have any vector component along the positive $x$-direction.

\begin{figure}[H]
	\centering
	\begin{subfigure}{0.33\textwidth}
		\centering
		\includegraphics[width=0.55\textwidth]{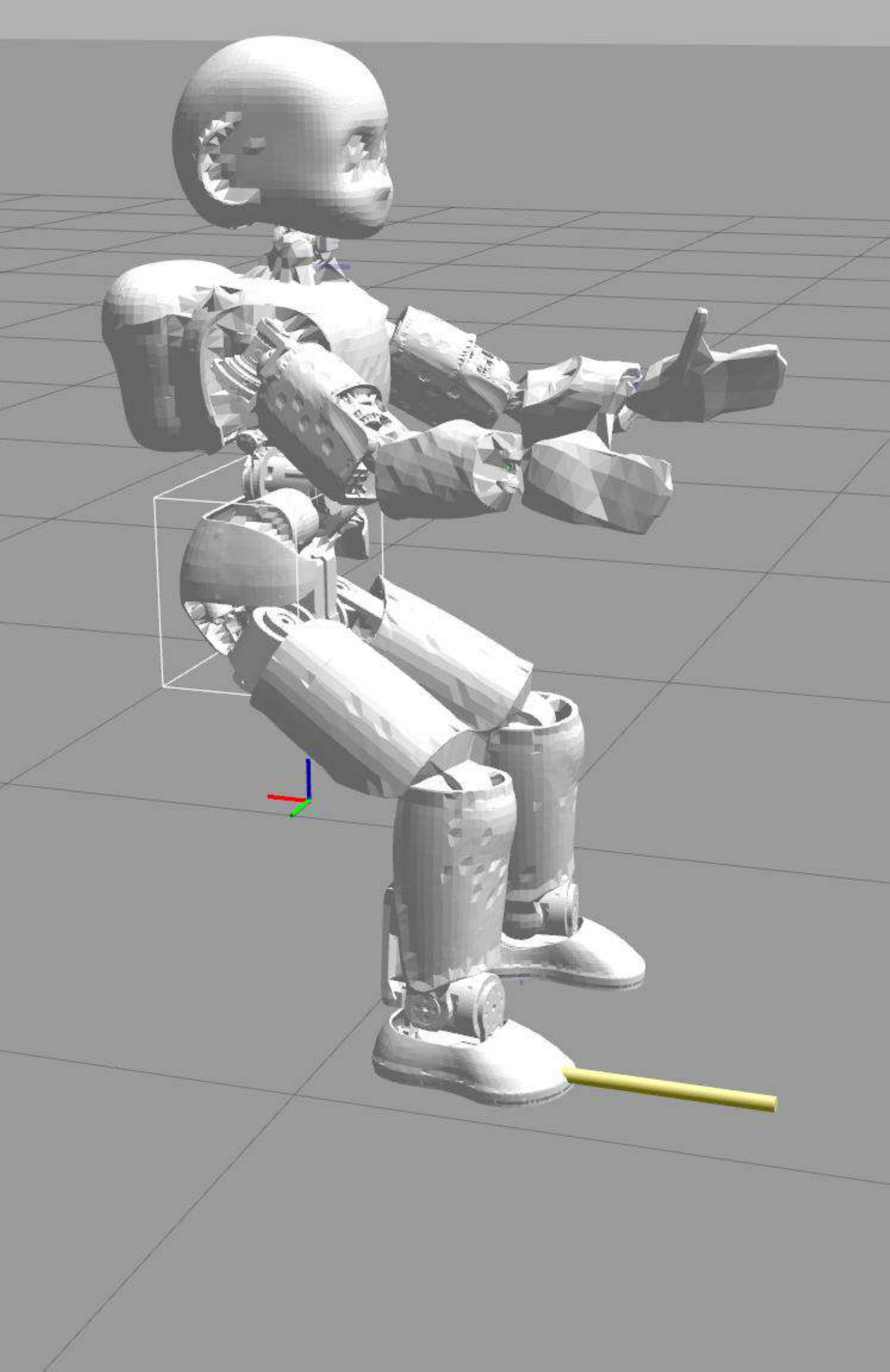}
		\caption{}
		\label{fig:assistive-wrench-1}
	\end{subfigure}%
	\begin{subfigure}{0.33\textwidth}
		\centering
		\includegraphics[width=0.55\textwidth]{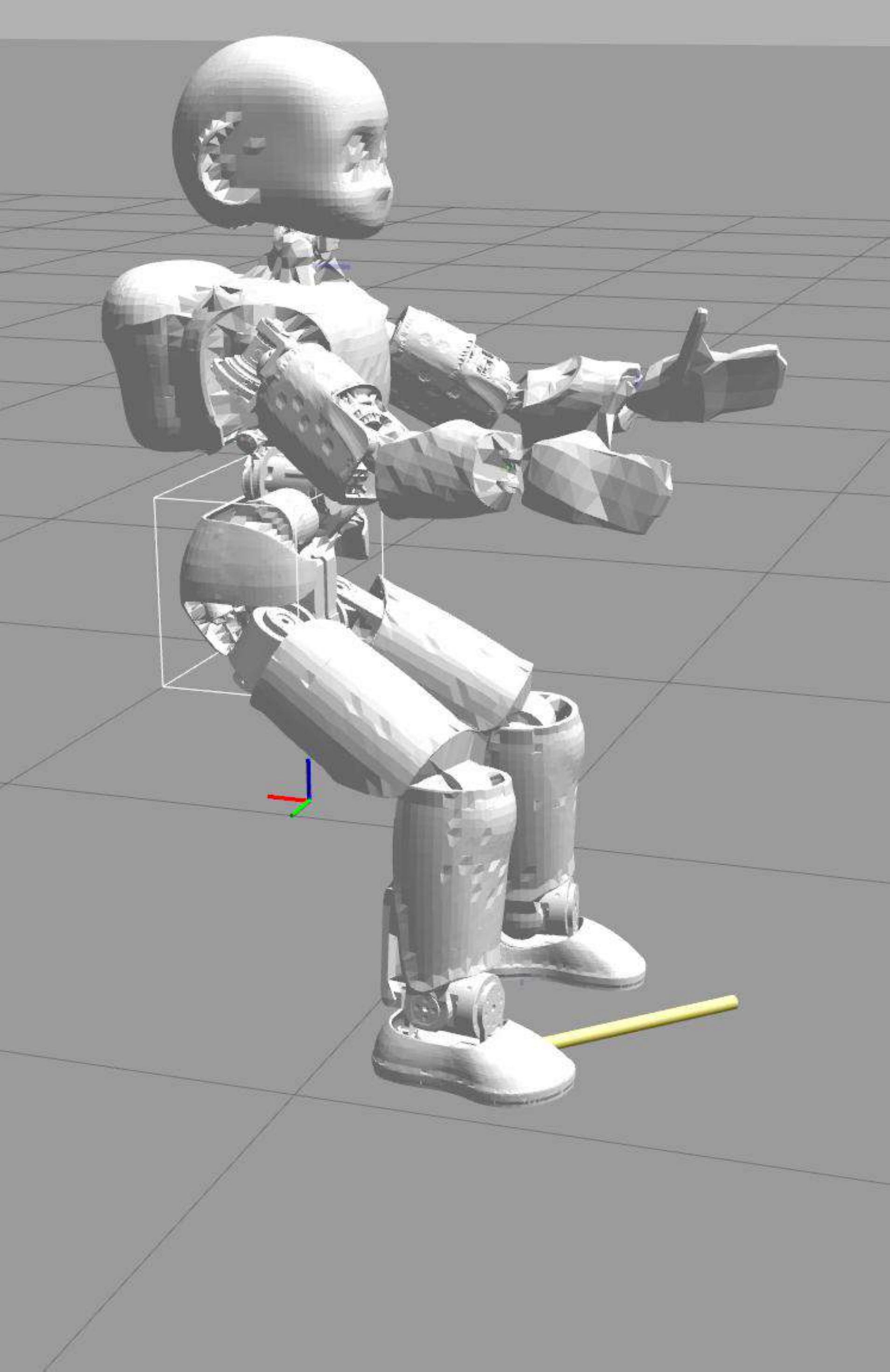}
		\caption{}
		\label{fig:assistive-wrench-2}
	\end{subfigure}%
	\begin{subfigure}{0.33\textwidth}
		\centering
		\includegraphics[width=0.55\textwidth]{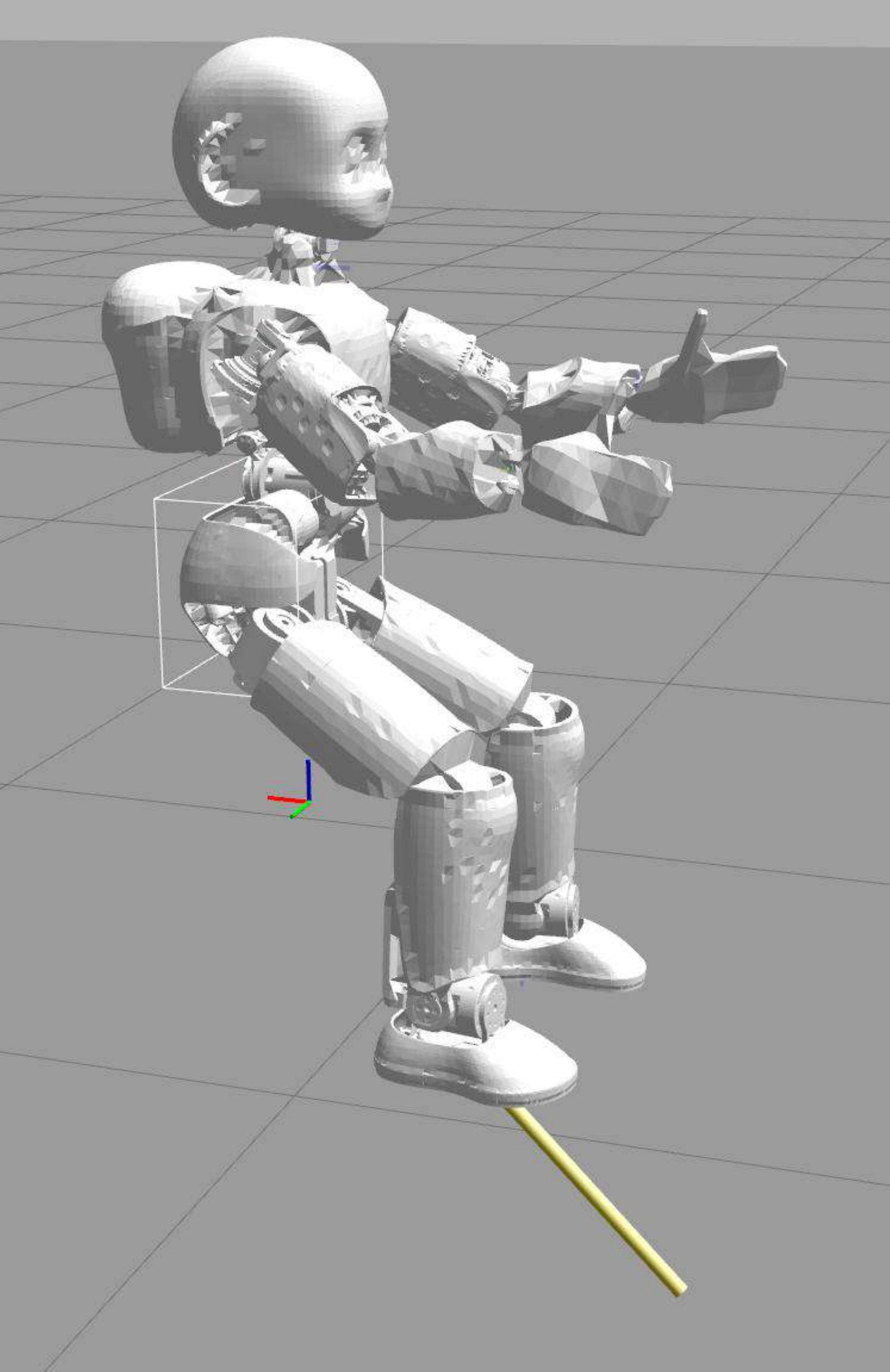}
		\caption{}
		\label{fig:assistive-wrench-3}
	\end{subfigure}
	\begin{subfigure}{0.33\textwidth}
		\centering
		\includegraphics[width=0.55\textwidth]{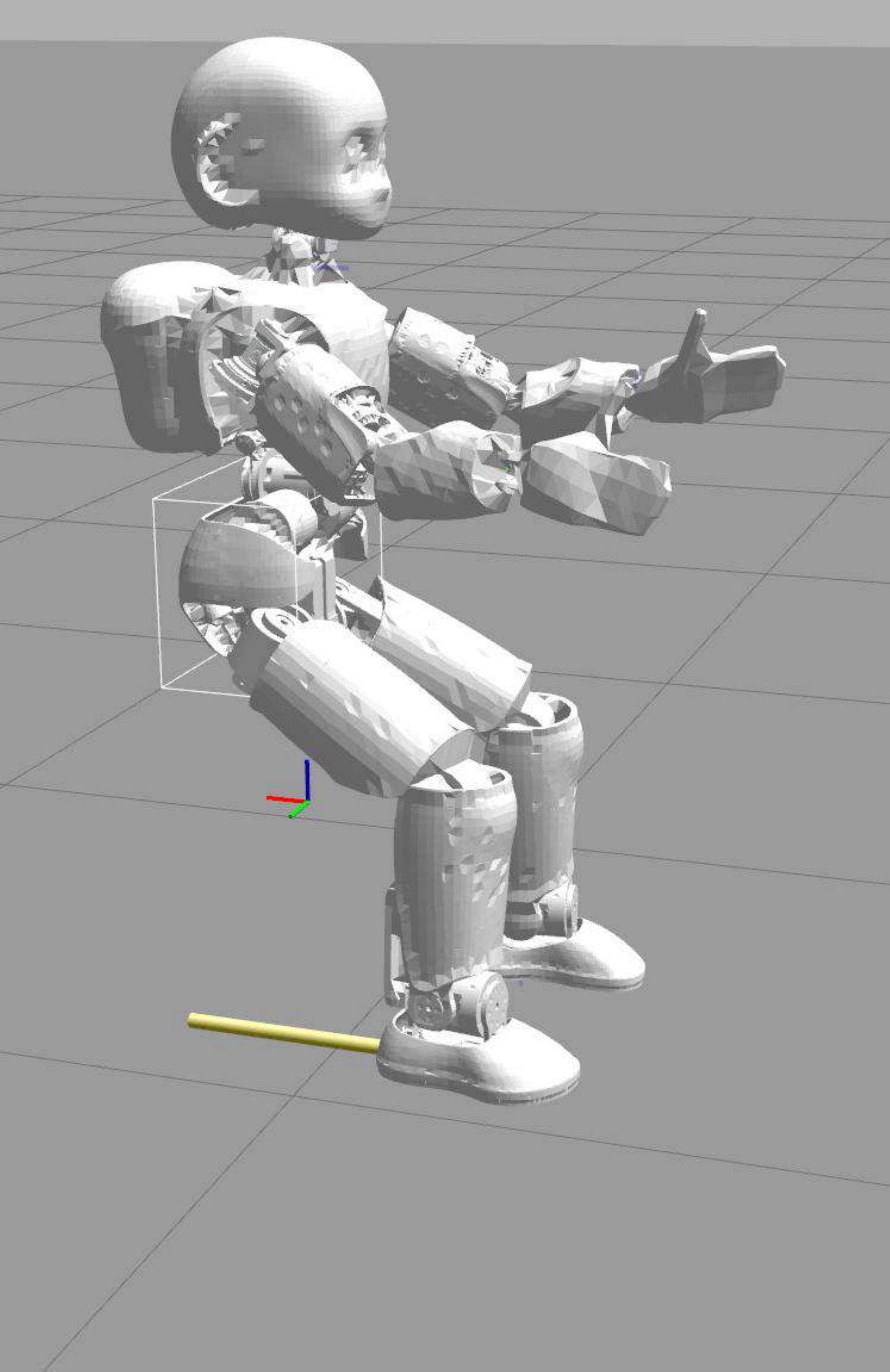}
		\caption{}
		\label{fig:agnostic-wrench-1}
	\end{subfigure}%
	\begin{subfigure}{0.33\textwidth}
		\centering
		\includegraphics[width=0.55\textwidth]{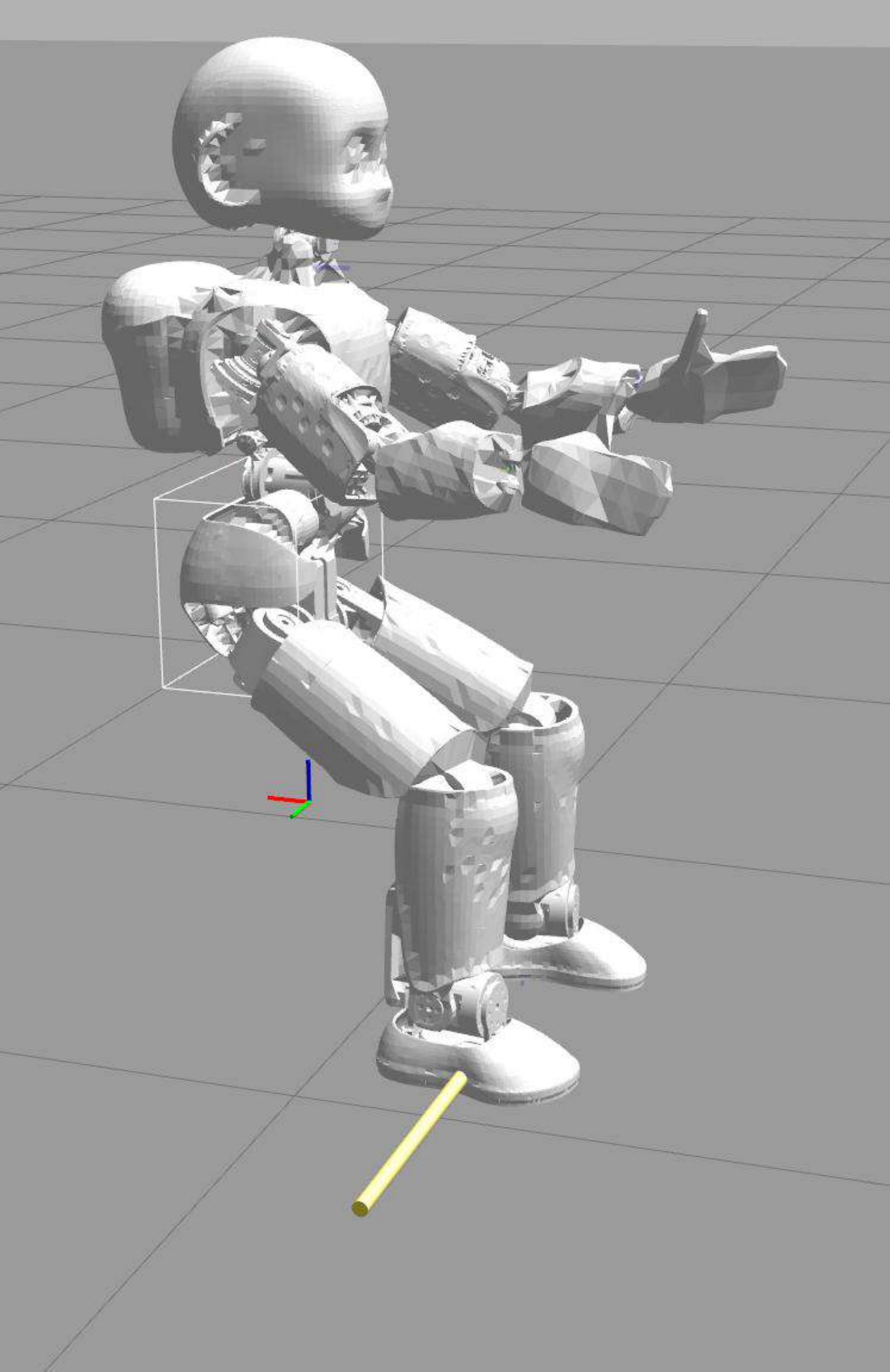}
		\caption{}
		\label{fig:agnostic-wrench-2}
	\end{subfigure}%
	\begin{subfigure}{0.33\textwidth}
		\centering
		\includegraphics[width=0.55\textwidth]{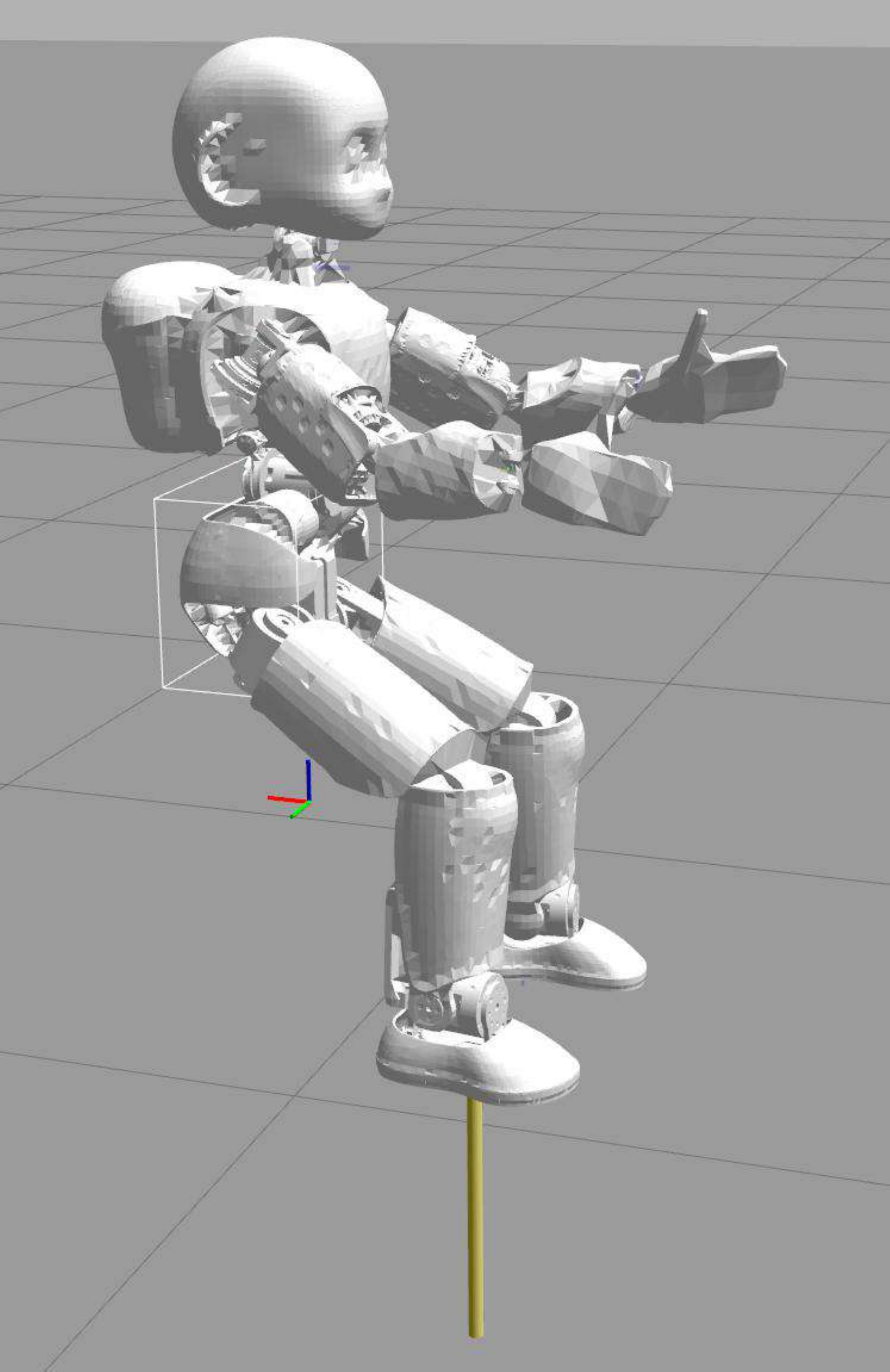}
		\caption{}
		\label{fig:agnostic-wrench-3}
	\end{subfigure}
	\caption{External interaction wrench classification examples when the desired direction of motion in positive $x$-direction}
	\label{fig:wrench-classification}
\end{figure}

Concerning the experiments conducted in Gazebo simulation environment, wrench is applied through a plugin\footnote{\href{https://github.com/robotology/gazebo-yarp-plugins}{https://github.com/robotology/gazebo-yarp-plugins}} \citep{hoffman2014yarp}. Due to the limited operational space of the robot we chose a fixed duration of $0.75 \si{\second}$ for the wrench. Furthermore, the wrench applied has a smooth profile rather than an impulse profile. This is an experimental design choice made to mimic the intentional interaction wrench applied by a human on the real robot during HRC scenarios.

\subsection{Simulation Results}

The first set of experiments are performed in Gazebo simulation environment. A set of test wrenches listed in the table \ref{table:simulation-test-wrenches} are applied when the desired direction of motion is along the positive $x$-direction. These test wrench vectors are similar in \textit{direction} to the wrench vectors highlighted in the wrench classification examples Fig.~\ref{fig:wrench-classification}. The first three wrench vectors (a)(b)(c) are classified as assistive wrenches as they have a vector component (highlighted in \textit{blue}) along the desired direction of motion i.e. positive $x$-direction. The next three wrench vectors (d)(e)(f) are classified as agnostic wrenches as they do not have any vector component along the desired direction of motion.

\begin{table}[!ht]
	\centering
	\begin{tabular}{|P{0.5cm}|P{0.75cm}|P{0.75cm}|P{0.75cm}|P{0.75cm}|P{0.75cm}|P{0.75cm}|}
		\cline{2-7}
		\multicolumn{1}{c|}{\cellcolor{white}} & $f_x$ & $f_y$  & $f_z$  & $\tau_x$  & $\tau_y$  & $\tau_z$ \\
		\hline
		\rowcolor{GreenYellow} (a) & \colorbox{Periwinkle}{$10$} & $0$  & $0$  & $0$  & $0$  & $0$ \\
		\hline
		\rowcolor{GreenYellow} (b) & \colorbox{Periwinkle}{$5$} & $10$  & $0$  & $0$  & $0$  & $0$ \\
		\hline
		\rowcolor{GreenYellow} (c) & \colorbox{Periwinkle}{$5$} & $0$  & $10$  & $0$  & $0$  & $0$ \\
		\hline
		\rowcolor{Apricot} (d) & $-10$ & $0$  & $0$  & $0$  & $0$  & $0$ \\
		\hline
		\rowcolor{Apricot} (e) & $0$ & $-10$  & $0$  & $0$  & $0$  & $0$ \\
		\hline
		\rowcolor{Apricot} (f) & $0$ & $0$  & $10$  & $0$  & $0$  & $0$ \\
		\hline
	\end{tabular}
	\caption{Test wrenches applied in gazebo simulation for end-effector trajectory advancement experiments}
	\label{table:simulation-test-wrenches}
\end{table}

The results of experiments in Gazebo simulation under the application of the test wrenches listed in table \ref{table:simulation-test-wrenches} are highlighted in Fig.~\ref{fig:simulation-velocity-modulation-1d-x}. The external interaction wrench experienced by the right foot link of the robot are shown in Fig.~\ref{fig:simulation-external-interaction-wrench-x} and the \textit{correction wrench} that is considered towards trajectory advancement is shown in Fig.~\ref{fig:simulation-correction-wrench-x}. In the case of agnostic wrenches, the correction wrench terms are insignificant and they are present due to the noise in the wrench estimation \cite{nori2015}. The reference trajectory is similar to a time parametrized trajectory i.e., $\psi = t$ until any helpful wrench is applied to the end-effector. Under the influence of assistive wrenches, the derivative of the trajectory free parameter $\dot{\psi}$ changes as highlighted in Fig.~\ref{fig:simulation-sdotvalue-x} and the corresponding trajectory advancement is reflected as an increase in $\psi$ as seen in Fig.~\ref{fig:simulation-svalue-x}. Accordingly, the reference is advanced further along the reference trajectory as shown in Fig.~\ref{fig:simulation-reference-trajectory-x}.

The trajectory tracking error is slightly more when the reference trajectory is updated under the influence of the assistive wrenches however the error magnitude is of low order as highlighted in Fig.~\ref{fig:simulation-trajectory-tracking-error-x} proving that the task of trajectory tracking is achieved reliably by the controller. Another important observation is that the magnitude of change in $\dot{\psi}$ is related to the magnitude of the interaction wrench. The length of advancement under the influence of assistive wrench vector (a) is more than under the influence of assistive wrench vector (b) or (c) from table \ref{table:simulation-test-wrenches}. The time evolution of the desired leg joint torques generated by the controller for the duration of the experiment is highlighted in Fig.~\ref{fig:ee_simulation_leg_toqures_plot}.

\subsection{Real Robot Results}

The results of experiments on the real iCub robot with 1D reference trajectory along the $x$-axis are shown in Fig.~\ref{fig:real-robot-velocity-modulation-1d-x}. The external interaction wrenches experienced by the right foot of the robot are highlighted in Fig.~\ref{fig:real-robot-external-interaction-wrench-x} and the \textit{correction wrench} that is considered towards trajectory advancement is shown in Fig.~\ref{fig:real-robot-correction-wrench-x}. The reference trajectory is similar to a time parametrized trajectory i.e., $\psi = t$ until any helpful wrench is applied to the end-effector. Under the influence of assistive wrenches, the derivative of the trajectory free parameter $\dot{\psi}$ changes as shown in Fig.~\ref{fig:real-robot-sdotvalue-x} and the corresponding trajectory advancement is reflected as an increase in $\psi$ as seen in Fig.~\ref{fig:real-robot-svalue-x}. Accordingly, the reference is advanced further along the reference trajectory as shown in Fig.~\ref{fig:real-robot-reference-trajectory-x}. Furthermore, starting from $t = 30 \si{\second}$ wrench is applied in the positive $x$-direction continuously. While the reference trajectory is in the positive $x$-direction, this wrench is considered assistive but as the reference trajectory is changed to the negative $x$-direction the wrench becomes agnostic and the reference trajectory is unchanged. Although there are some noisy wrenches that are considered to be correction wrench, they are tuned out by a regularization parameter in computing $\dot{\psi}$ to not have any direct effect on trajectory advancement.

The trajectory tracking error on the real robot is highlighted in Fig.~\ref{fig:real-robot-trajectory-tracking-error-x}. Although the tracking error is higher due to phase delays induced by joint friction, the desired amplitude of the reference trajectory is reached. The time evolution of the leg desired joint torques generated by the controller for the duration of the experiment is highlighted in Fig.~\ref{fig:ee_real_robot_leg_toqures_plot}.

\begin{figure}[H]
	\centering
	\begin{subfigure}{0.5\textwidth}
		\centering
		\includegraphics[clip, trim=1cm 0.5cm 4.5cm 1.5cm, scale=0.105]{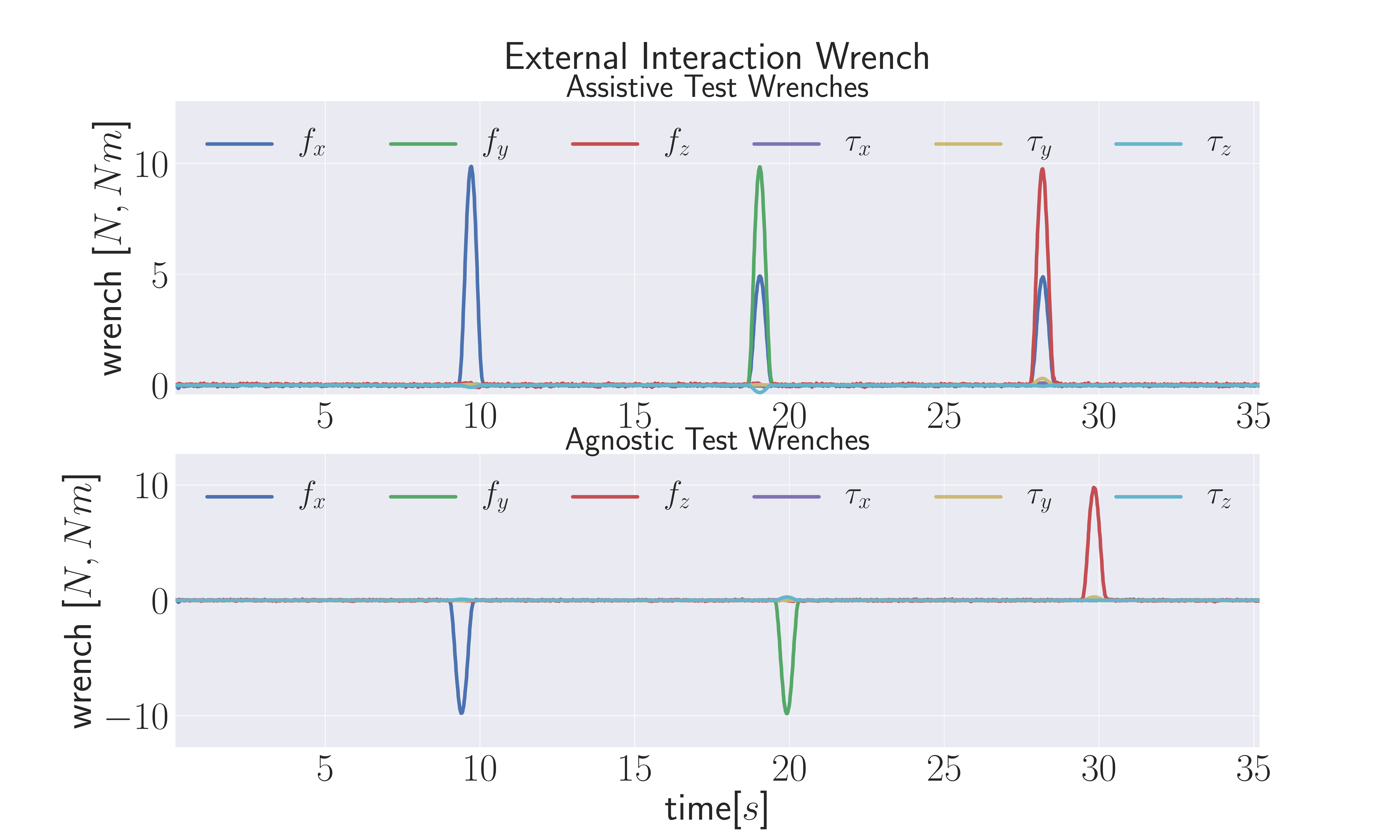}
		\caption{\hspace*{-7.5mm}}
		\label{fig:simulation-external-interaction-wrench-x}
	\end{subfigure}%
	\begin{subfigure}{0.5\textwidth}
		\centering
		\includegraphics[clip, trim=1cm 0.5cm 4.5cm 1.5cm, scale=0.105]{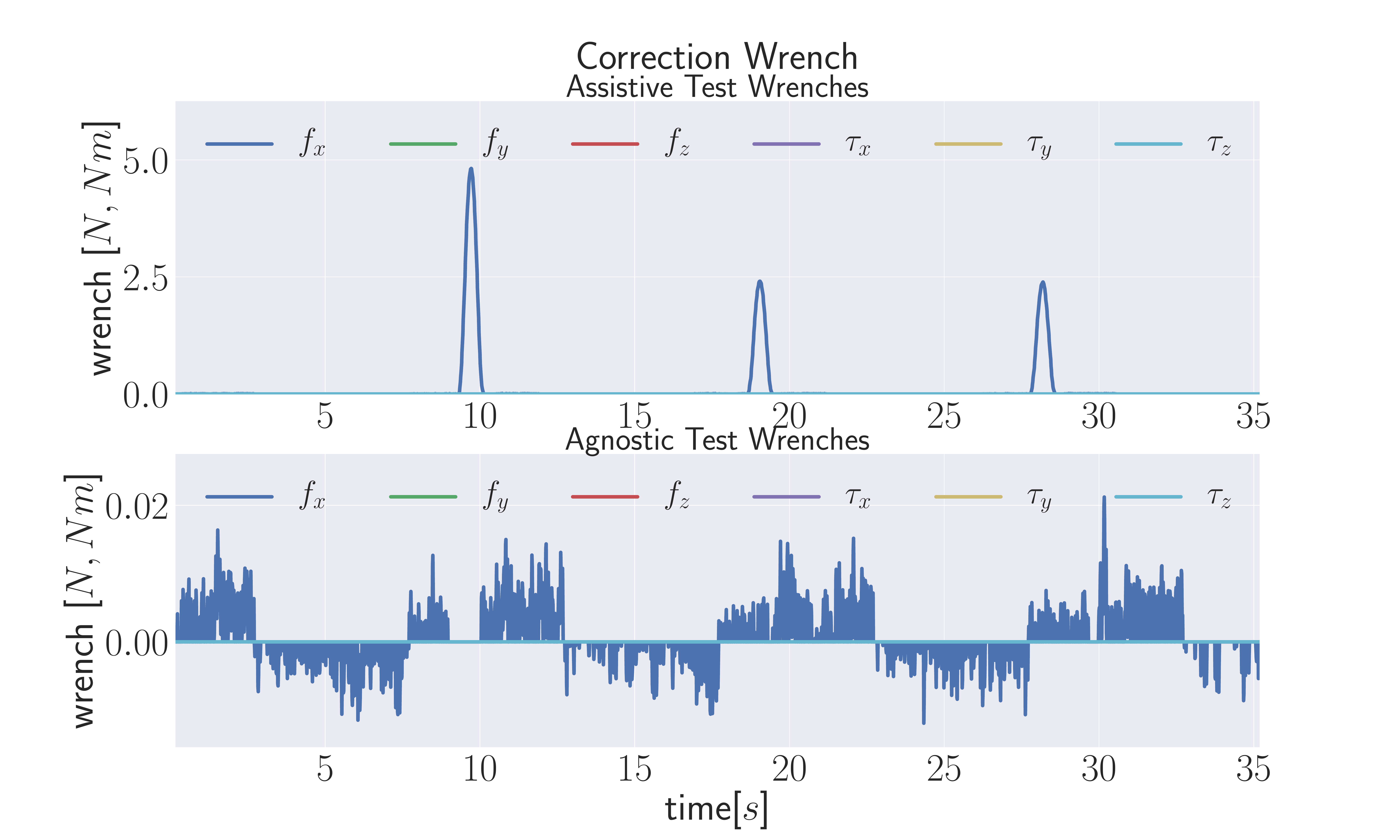}
		\caption{\hspace*{-7.5mm}}
		\label{fig:simulation-correction-wrench-x}
	\end{subfigure}
	\begin{subfigure}{0.5\textwidth}
		\centering
		\includegraphics[clip, trim=1cm 0.5cm 4.5cm 2.5cm, scale=0.105]{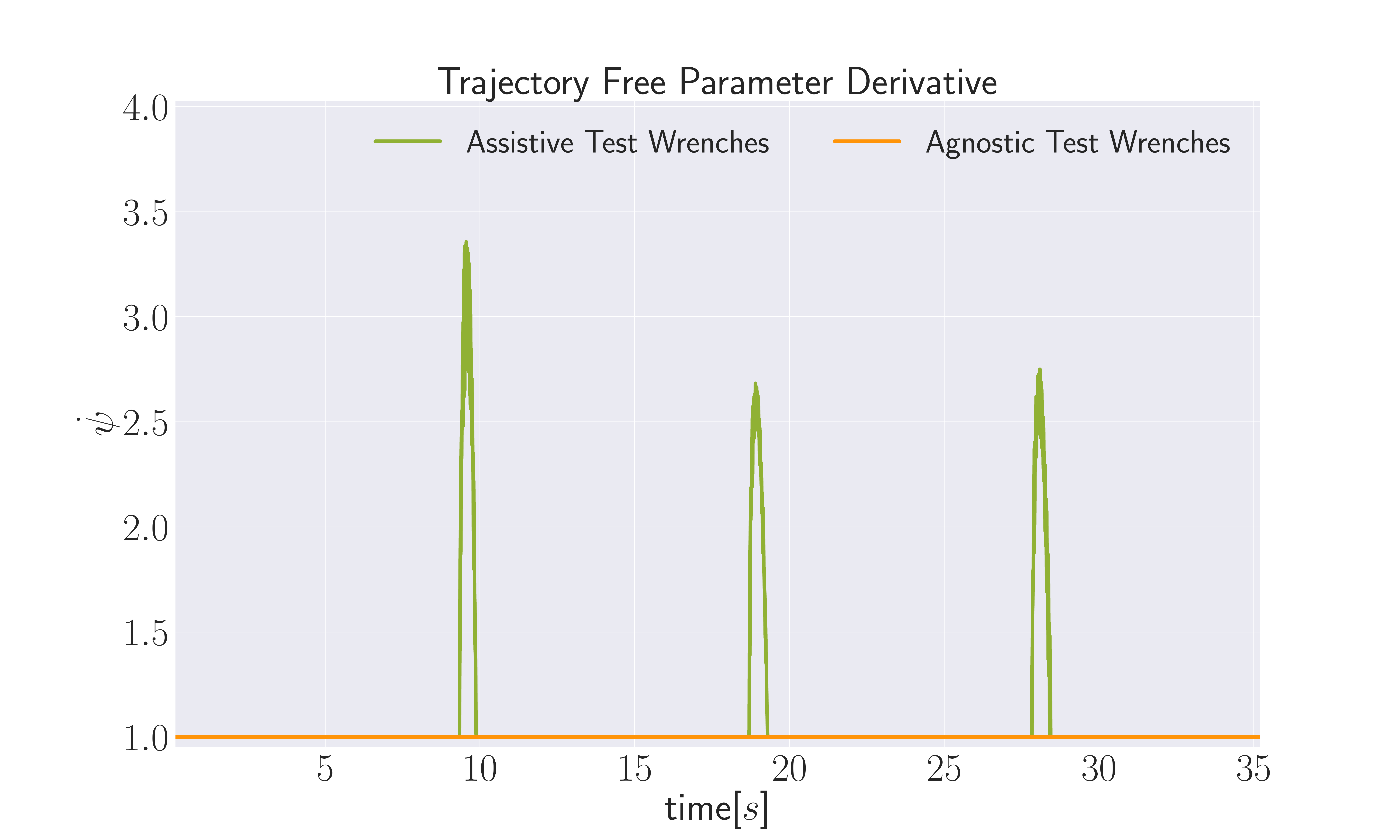}
		\caption{\hspace*{-7.5mm}}
		\label{fig:simulation-sdotvalue-x}
	\end{subfigure}%
	\begin{subfigure}{0.5\textwidth}
		\centering
		\includegraphics[clip, trim=1cm 0.5cm 4.5cm 2.5cm, scale=0.105]{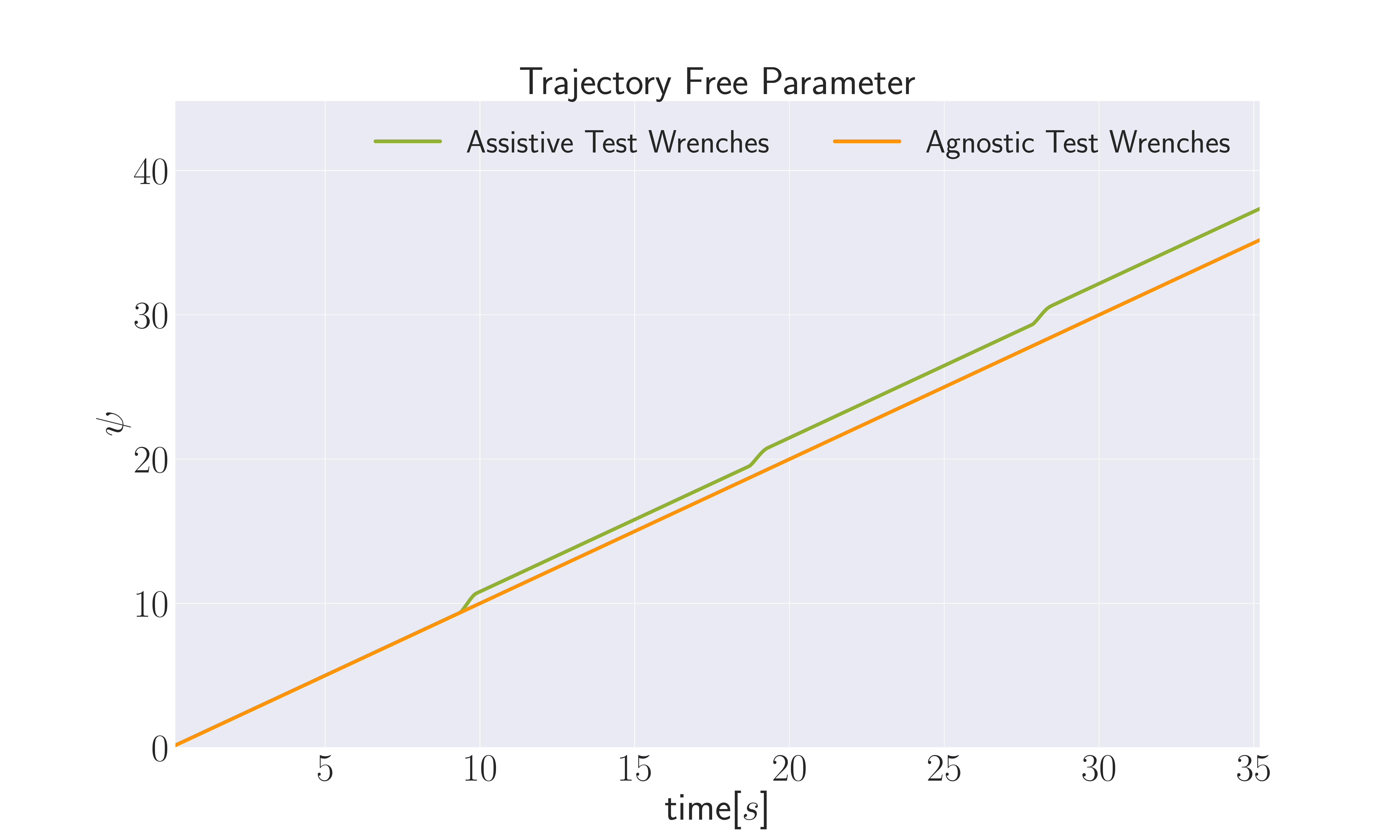}
		\caption{\hspace*{-7.5mm}}
		\label{fig:simulation-svalue-x}
	\end{subfigure}
	\begin{subfigure}{0.5\textwidth}
		\centering
		\includegraphics[clip, trim=1cm 0.5cm 4.5cm 2.5cm, scale=0.105]{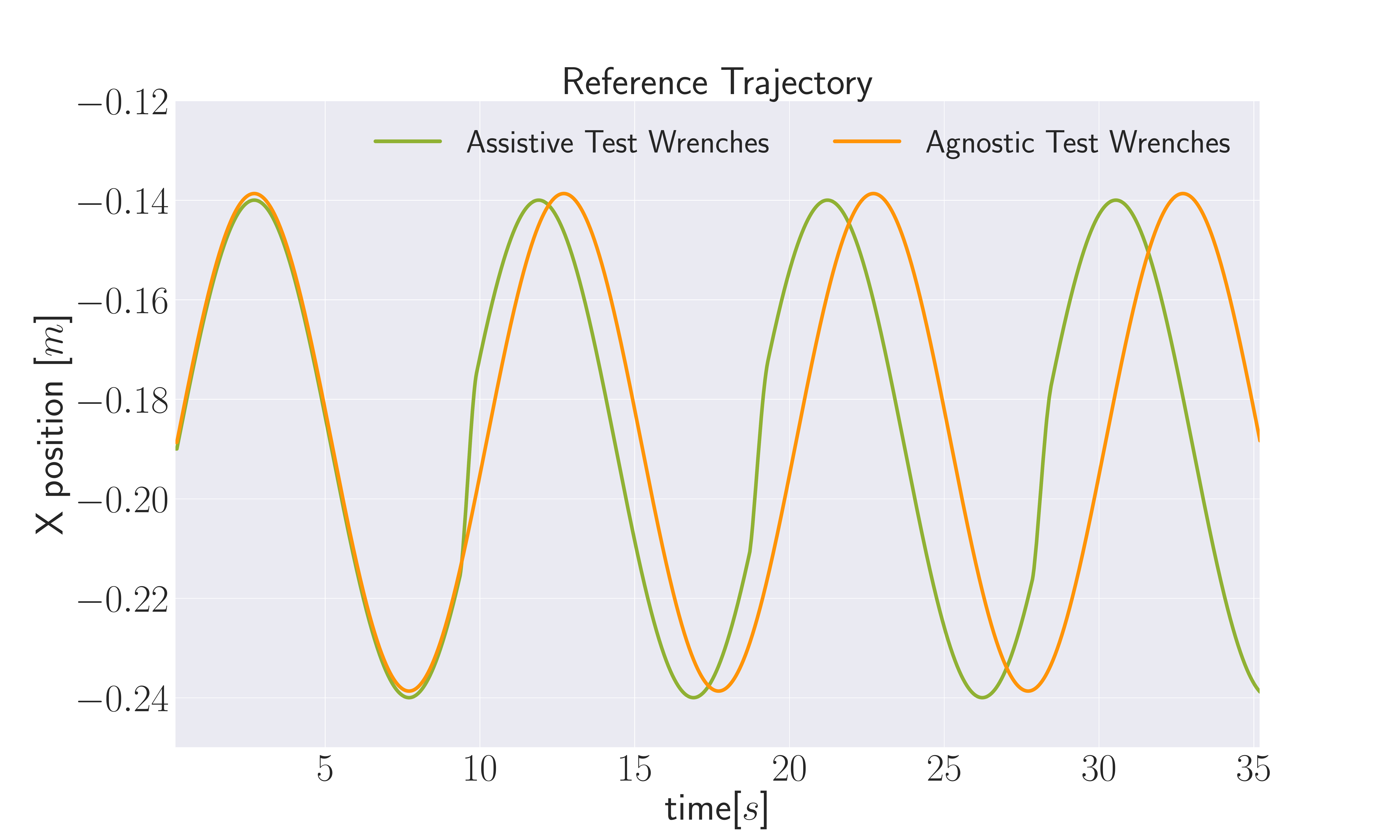}
		\caption{\hspace*{-7.5mm}}
		\label{fig:simulation-reference-trajectory-x}
	\end{subfigure}%
	\begin{subfigure}{0.5\textwidth}
		\centering
		\includegraphics[clip, trim=1cm 0.5cm 4.5cm 2.5cm, scale=0.105]{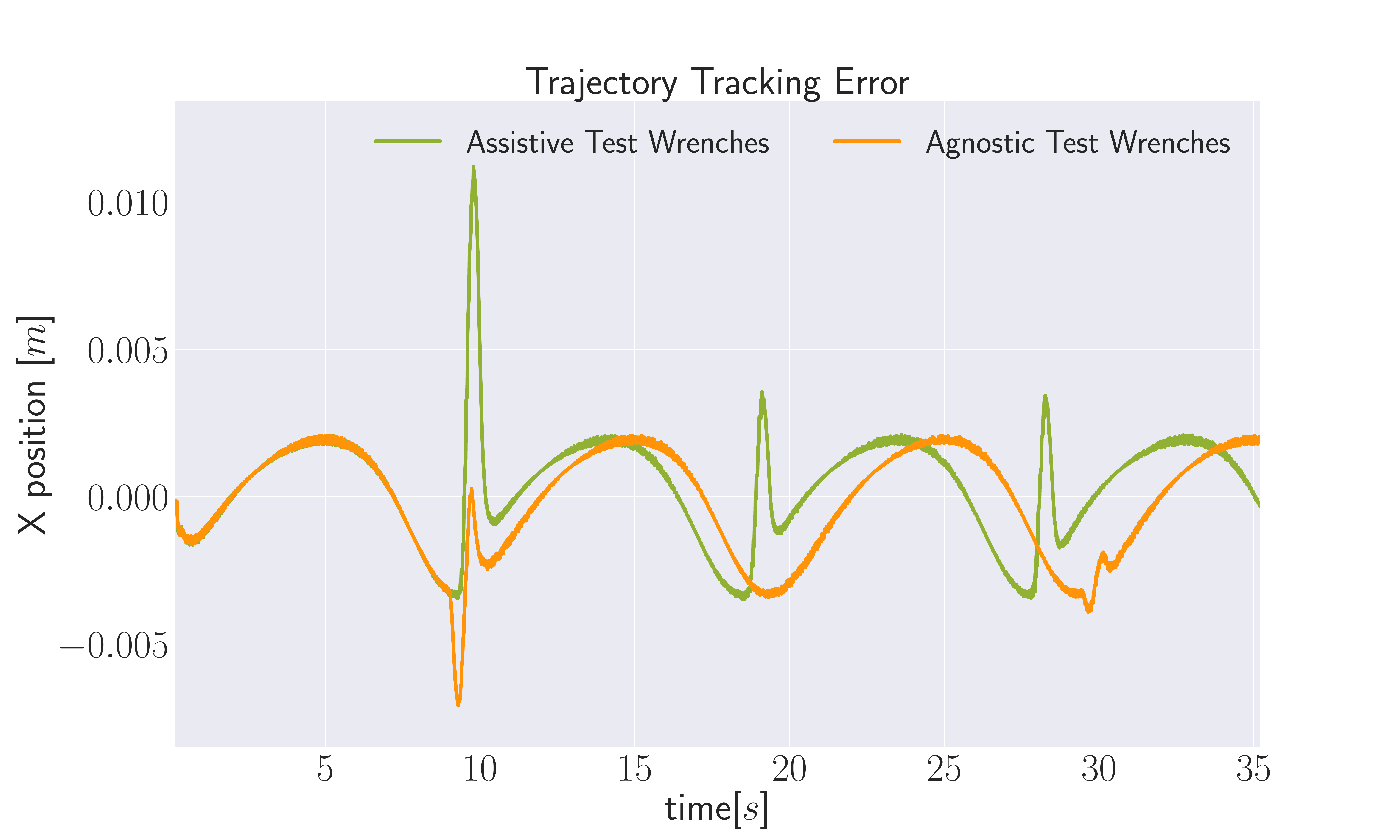}
		\caption{\hspace*{-7.5mm}}
		\label{fig:simulation-trajectory-tracking-error-x}
	\end{subfigure}
	\caption{Gazebo simulation 1D trajectory advancement along $X$-direction}
	\label{fig:simulation-velocity-modulation-1d-x}
\end{figure}

\begin{figure}[H]
	\hspace{-0.5cm}
	\includegraphics[width=\textwidth]{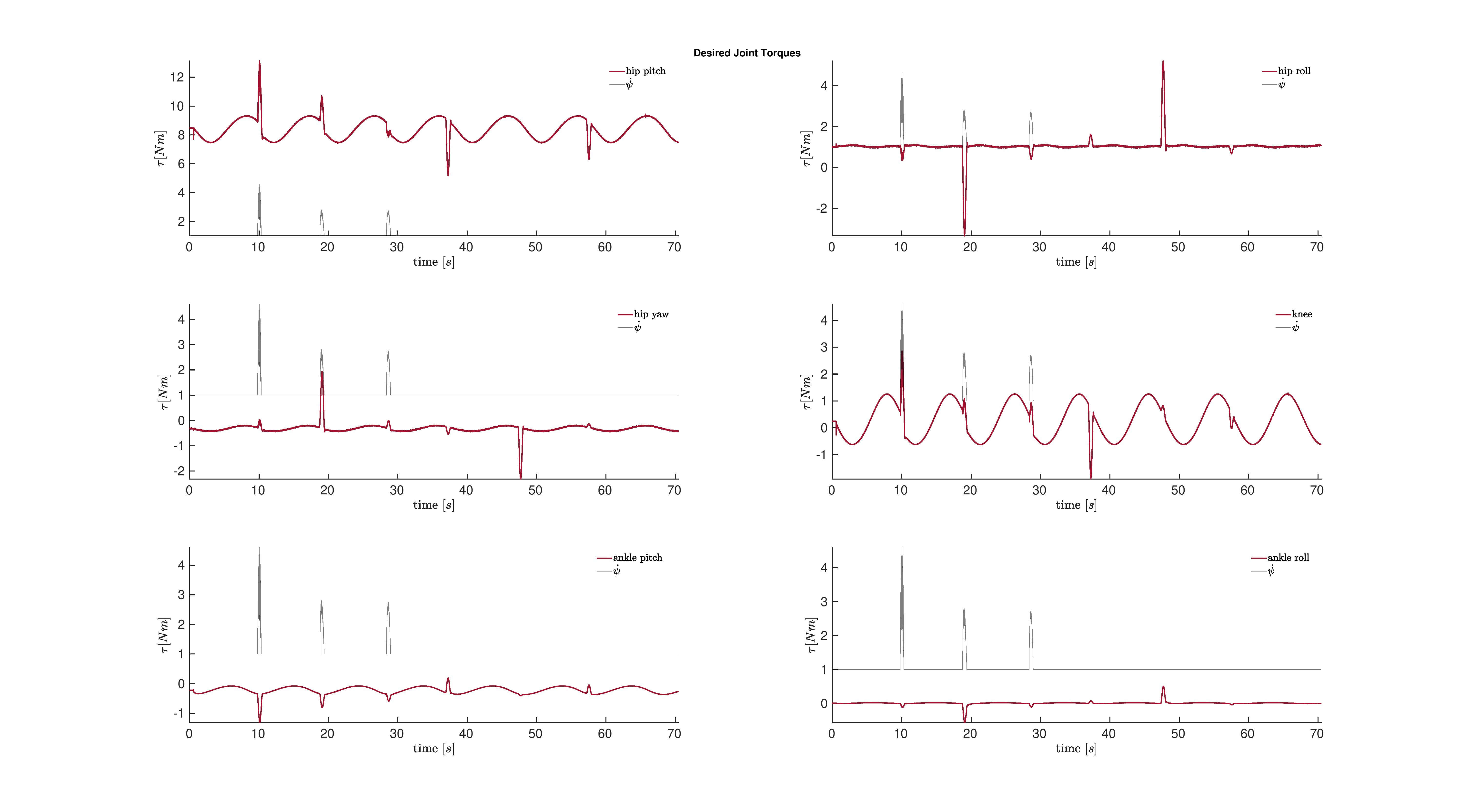}
	\caption{Gazebo simulation 1D trajectory advancement leg joint torques}
	\label{fig:ee_simulation_leg_toqures_plot}
\end{figure}

\begin{figure}[H]
	\centering
	\begin{subfigure}{0.5\textwidth}
		\centering
		\includegraphics[clip, trim=1cm 0.5cm 4.5cm 3.5cm, scale=0.105]{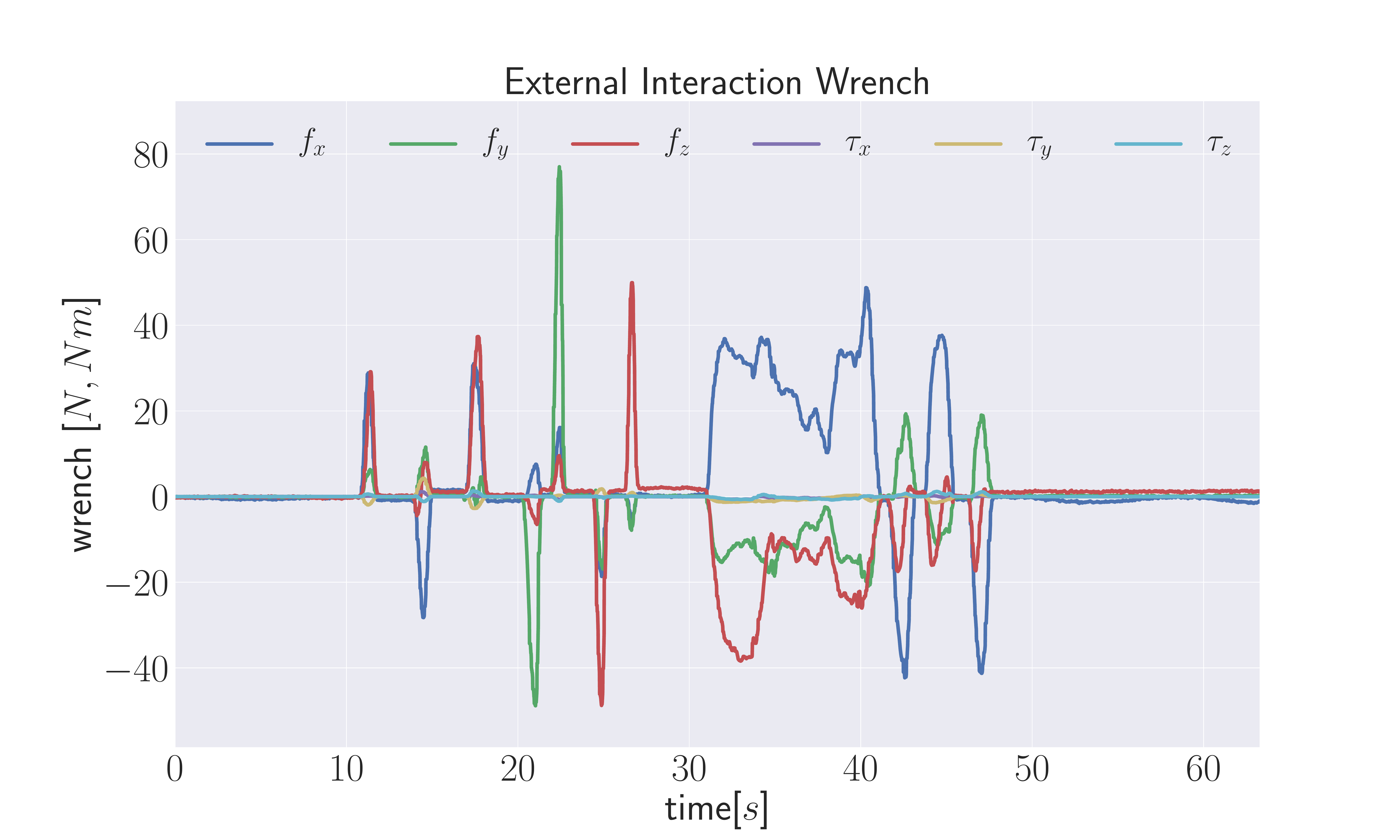}
		\caption{\hspace*{-7.5mm}}
		\label{fig:real-robot-external-interaction-wrench-x}
	\end{subfigure}%
	\begin{subfigure}{0.5\textwidth}
		\centering
		\includegraphics[clip, trim=1cm 0.5cm 4.5cm 3.5cm, scale=0.105]{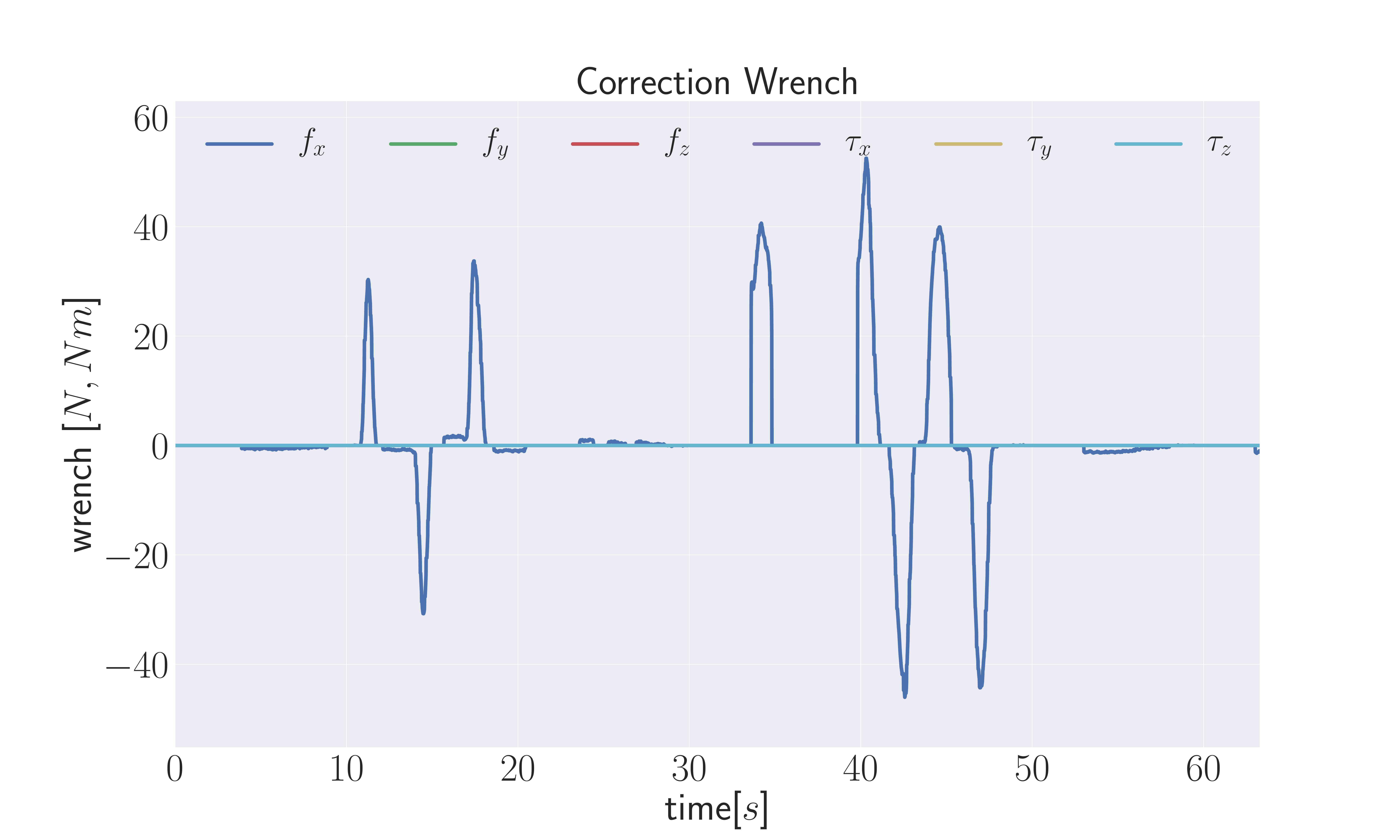}
		\caption{\hspace*{-7.5mm}}
		\label{fig:real-robot-correction-wrench-x}
	\end{subfigure}
	\begin{subfigure}{0.5\textwidth}
		\centering
		\includegraphics[clip, trim=1cm 0.5cm 4.5cm 2.5cm, scale=0.105]{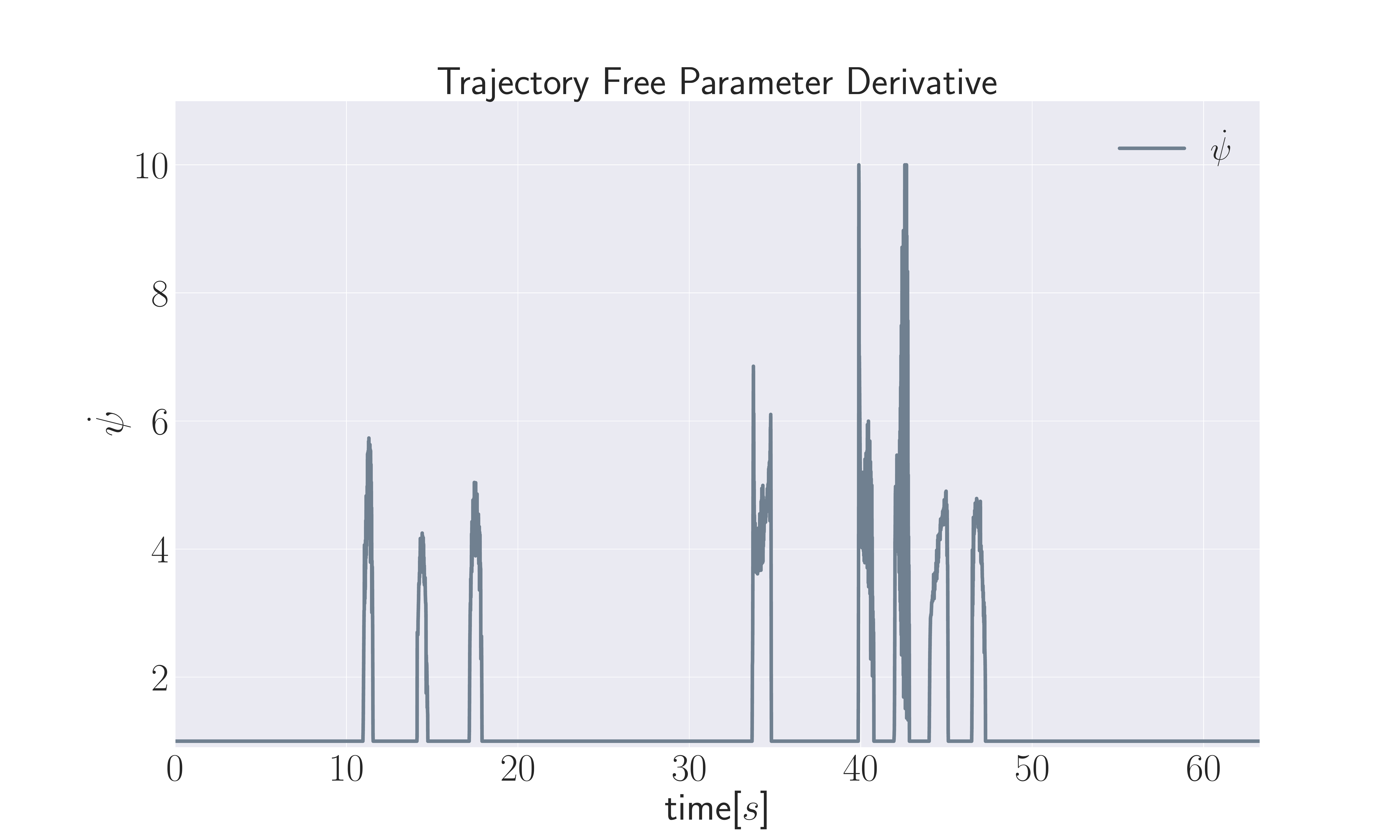}
		\caption{\hspace*{-7.5mm}}
		\label{fig:real-robot-sdotvalue-x}
	\end{subfigure}%
	\begin{subfigure}{0.5\textwidth}
		\centering
		\includegraphics[clip, trim=1cm 0.5cm 4.5cm 2.5cm, scale=0.105]{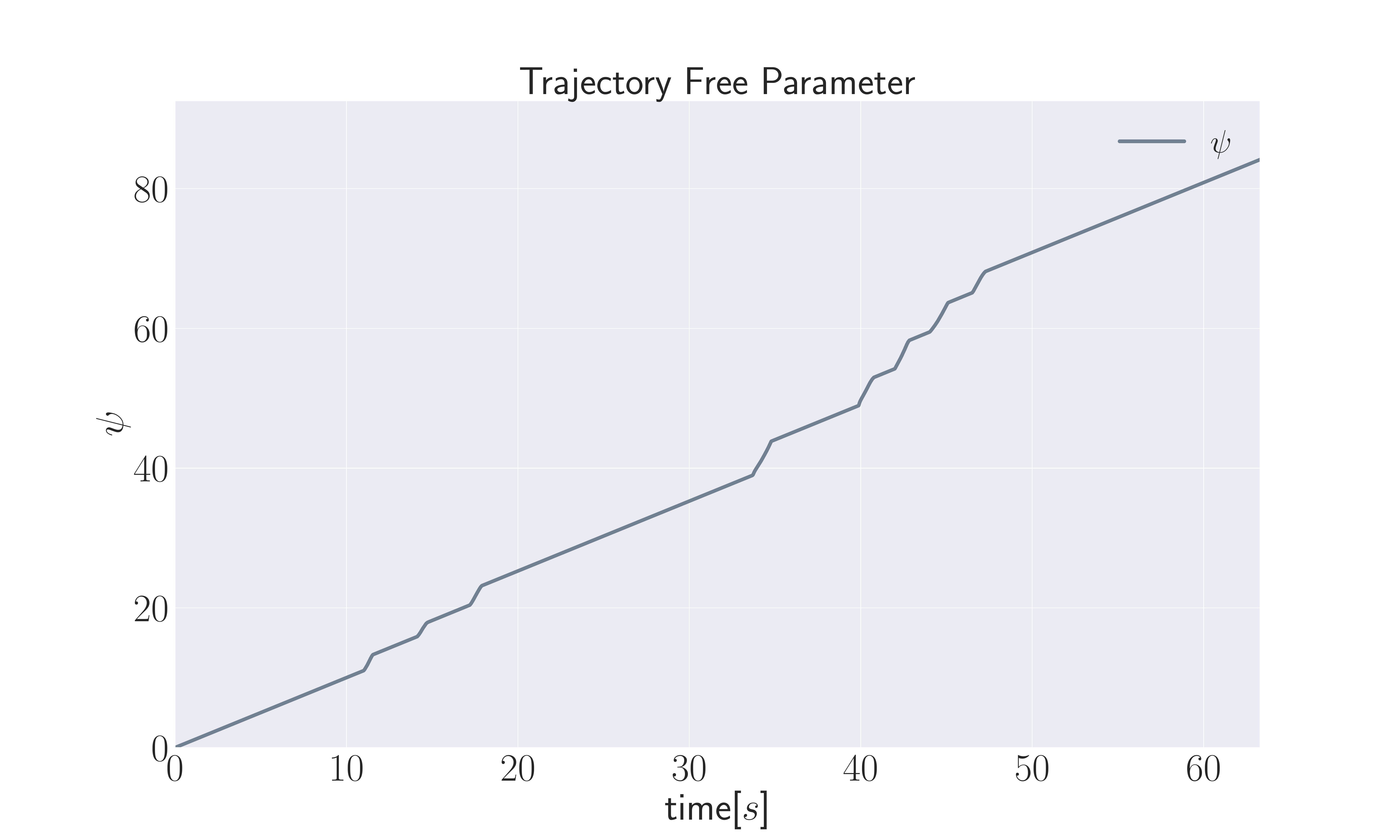}
		\caption{\hspace*{-7.5mm}}
		\label{fig:real-robot-svalue-x}
	\end{subfigure}
	\begin{subfigure}{0.5\textwidth}
		\centering
		\includegraphics[clip, trim=1cm 0.5cm 4.5cm 2.5cm, scale=0.105]{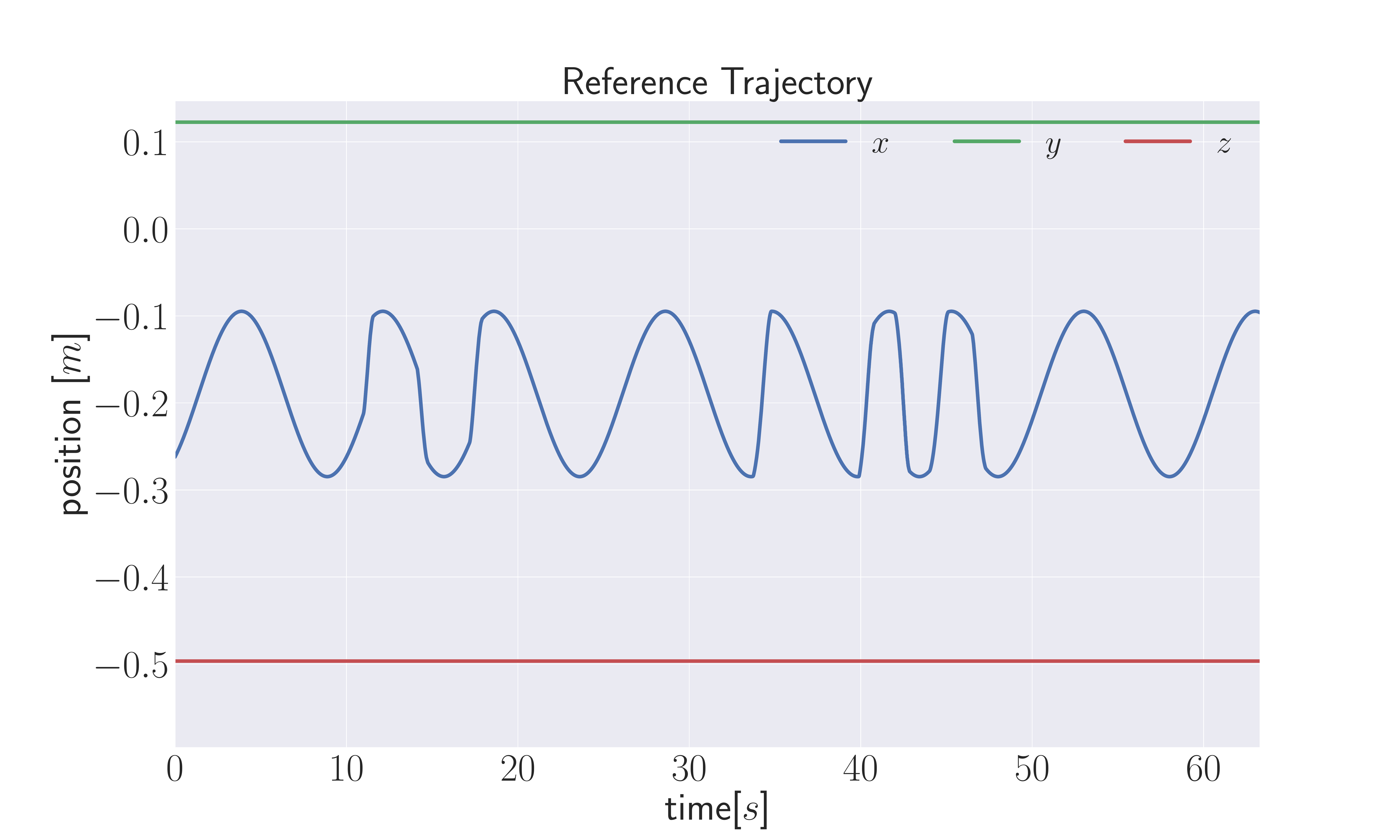}
		\caption{\hspace*{-7.5mm}}
		\label{fig:real-robot-reference-trajectory-x}
	\end{subfigure}%
	\begin{subfigure}{0.5\textwidth}
		\centering
		\includegraphics[clip, trim=1cm 0.5cm 4.5cm 2.5cm, scale=0.105]{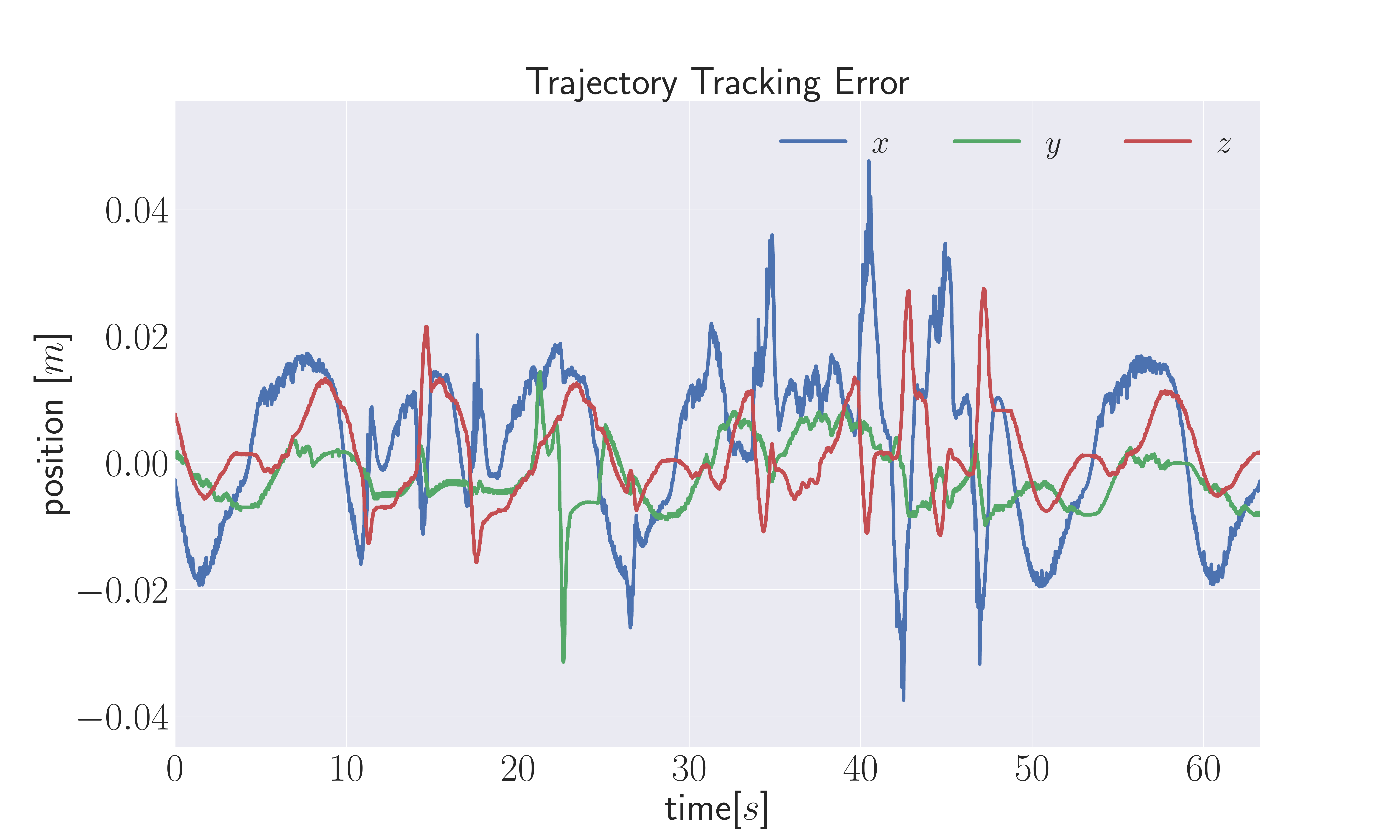}
		\caption{\hspace*{-7.5mm}}
		\label{fig:real-robot-trajectory-tracking-error-x}
	\end{subfigure}
	\caption{Real robot 1D trajectory advancement along $X$-direction}
	\label{fig:real-robot-velocity-modulation-1d-x}
\end{figure}

\begin{figure}[H]
	\hspace{-2cm}
	\includegraphics[scale=0.18]{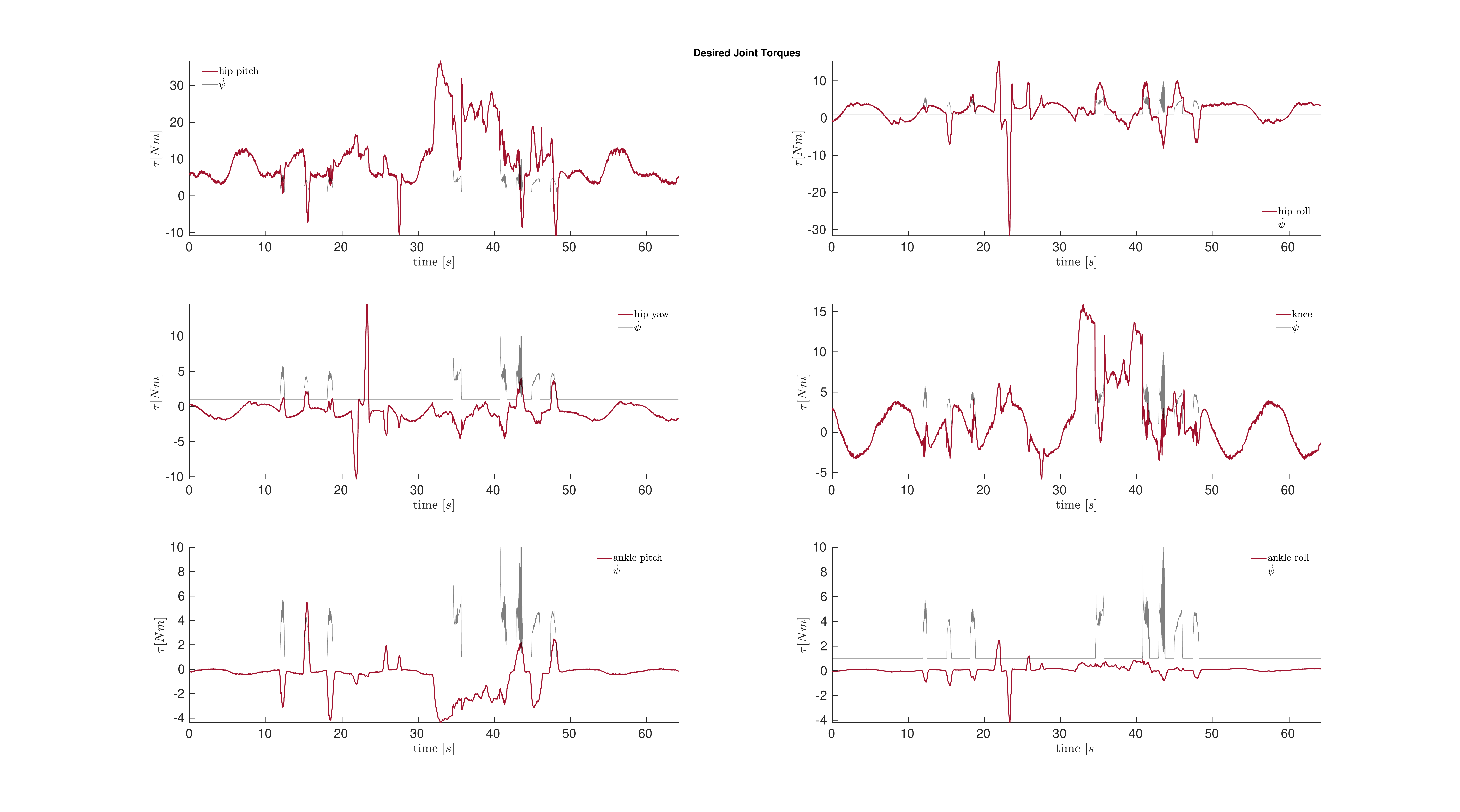}
	\caption{Real robot 1D trajectory advancement leg joint torques}
	\label{fig:ee_real_robot_leg_toqures_plot}
\end{figure}

\newpage

\section{Whole-Body Standup Trajectory Advancement Experiments}
\label{sec:whole-body-standup-trajectory-advancement-experiments}

Extended experiments of trajectory advancement are carried with the whole-body standup controller that is described in section~\ref{sec:whole-body-standup-controller}. Unlike the use of external agent joint torques, the controller is updated to use the interaction wrench estimates from the hands of the robot to advance along the trajectory. The reference trajectory to be followed by the robot is given for the center of mass. Intuitively using the approach of trajectory advancement, when a human is helping the robot to standup, the interaction wrench is used to stand faster.

\subsection{Simulation Results}

Simulation experiments are carried using Gazebo, a popular 3D dynamics simulator. Different states of the standup controller are highlighted in Fig.~\ref{fig:trajectory-advancement-standup-simulation-standup-states}. External interaction is mimicked through an application of wrench at the hands of the robot a plugin \cite{hoffman2014yarp}. The cylinders at the hands in Fig.~\ref{fig:trajectory-advancement-standup-simulation-standup-states} indicate the interaction wrench applied at the hands and the plots in Fig.~\ref{fig:standup-simulation-left-external-wrench} and Fig.~\ref{fig:standup-simulation-right-external-wrench} indicate the left and right hands wrench estimates respectively~\cite{nori2015}.

\begin{figure}[H]
	\centering
	\begin{subfigure}{0.25\textwidth}
		\centering
		\includegraphics[width=0.85\textwidth]{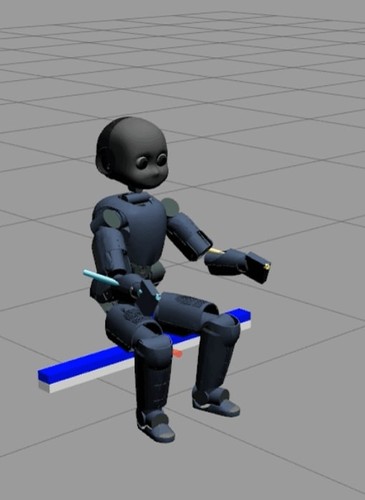}
		\caption{}
		\label{fig:trajectory-advancement-standup-simulation-state-1.png}
	\end{subfigure}%
	\begin{subfigure}{0.25\textwidth}
		\centering
		\includegraphics[width=0.85\textwidth]{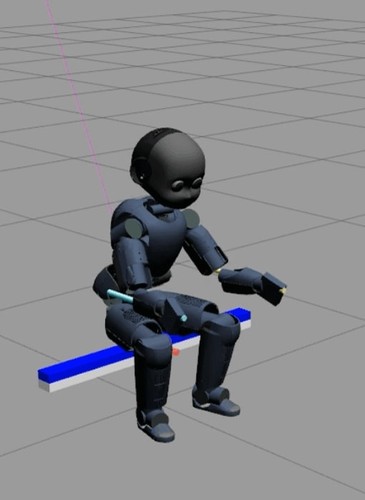}
		\caption{}
		\label{fig:trajectory-advancement-standup-simulation-state-2.png}
	\end{subfigure}%
	\begin{subfigure}{0.25\textwidth}
		\centering
		\includegraphics[width=0.85\textwidth]{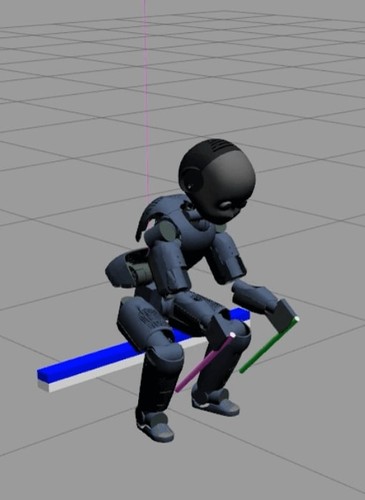}
		\caption{}
		\label{fig:trajectory-advancement-standup-simulation-state-3.png}
	\end{subfigure}%
	\begin{subfigure}{0.25\textwidth}
		\centering
		\includegraphics[width=0.85\textwidth]{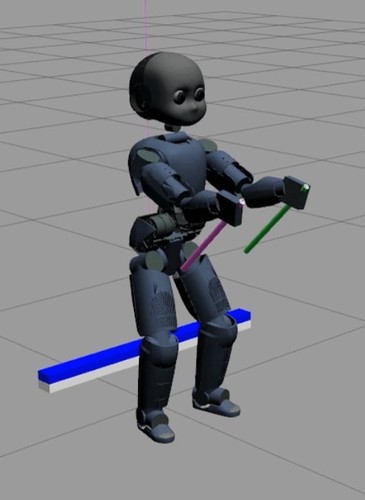}
		\caption{}
		\label{fig:trajectory-advancement-standup-simulation-state-4.png}
	\end{subfigure}
	\caption{iCub at different states during sit-to-stand transition with the whole-body standup controller~(Sec~\ref{sec:whole-body-standup-controller})}
	\label{fig:trajectory-advancement-standup-simulation-standup-states}
\end{figure}

\vspace{-0.75cm}

The results of trajectory advancement in simulation are shown in Fig.~\ref{fig:standup-simulation-trajectory-advancement}. The wrench estimated at the hands of the robot under the influence of simulated assistance is highlighted in Fig.~\ref{fig:standup-simulation-left-external-wrench}~\ref{fig:standup-simulation-right-external-wrench}. The external wrench that is helpful to achieve the task is shown in Fig.~\ref{fig:standup-simulation-correction-wrench}. The reference trajectory is similar to a time parametrized trajectory i.e., $\psi = t$ until any helpful wrench is applied at the hands of the robot. Under the influence of helpful wrenches, the derivative of the trajectory free parameter $\dot{\psi}$ changes as shown in Fig.~\ref{fig:standup-simulation-sdotvalue} and the corresponding trajectory advancement is reflected as an increase in $\psi$ as seen in Fig.~\ref{fig:standup-simulation-svalue}. Accordingly, the reference is advanced further along the  CoM reference trajectory as shown in Fig.~\ref{fig:standup-simulation-reference-trajectory}. The original CoM reference trajectory without trajectory advancement is shown in the same figure with reduced transparency. The time evolution of the leg desired joint torques generated by the controller for the duration of the experiment is highlighted in Fig.~\ref{fig:standup_simulation_leg_toqures_plot}.

As highlighted in section~\ref{sec:whole-body-standup-controller}, the robot behavior is guided through a state machine and the center of mass trajectory is generated between the states through a smoothing parameter. The experiments for trajectory advancement are conducted using the same state machine of the controller and as the time duration between the sates is small, the amount of trajectory advancement is also limited. As soon as the end of a trajectory between the states is reached, the trajectory advancement is disabled to keep the robot stabilized.

\begin{figure}[H]
	\centering
	\begin{subfigure}{0.49\textwidth}
		\centering
		\includegraphics[width=\textwidth]{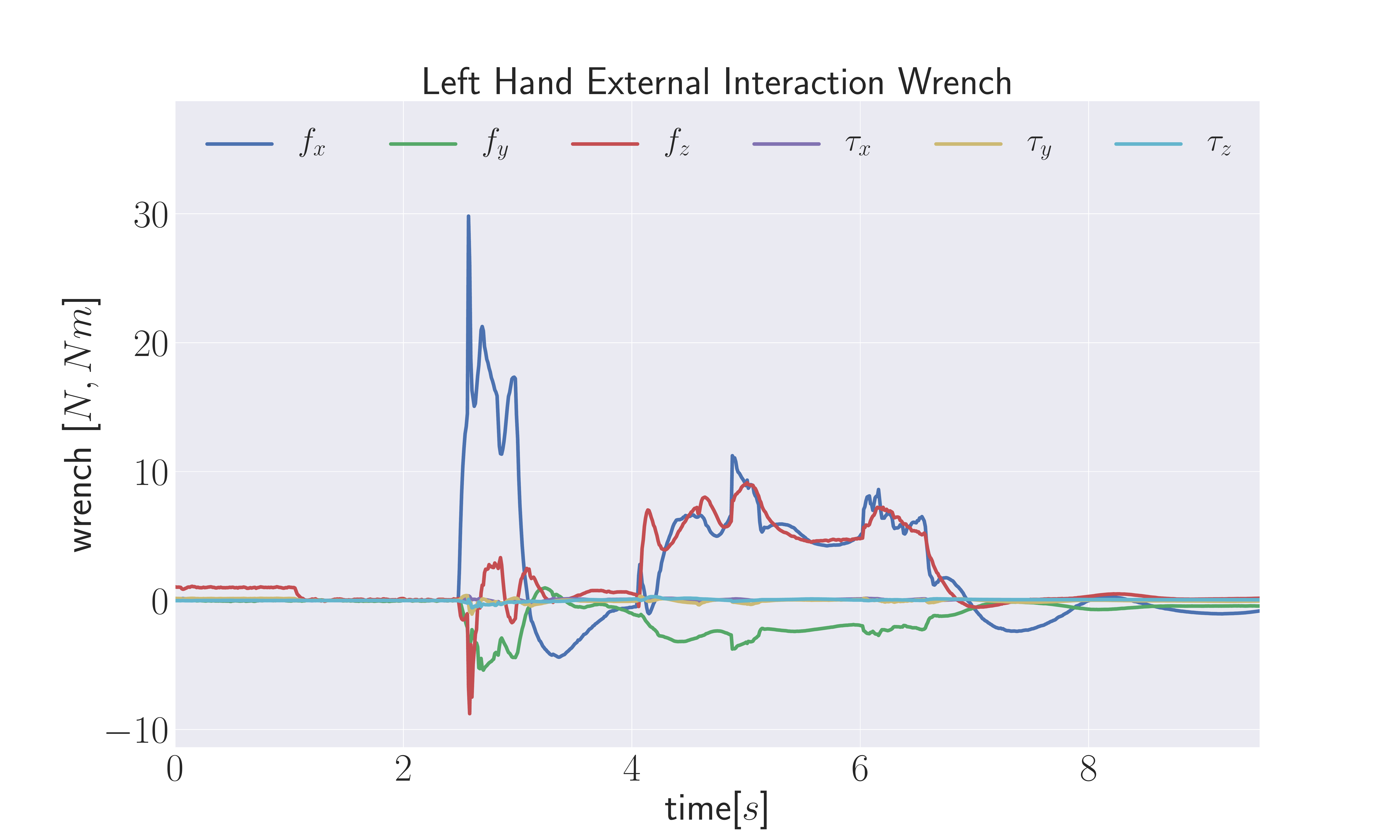}
		\caption{\hspace*{-4mm}}
		\label{fig:standup-simulation-left-external-wrench}
	\end{subfigure}%
	\begin{subfigure}{0.49\textwidth}
		\centering
		\includegraphics[width=\textwidth]{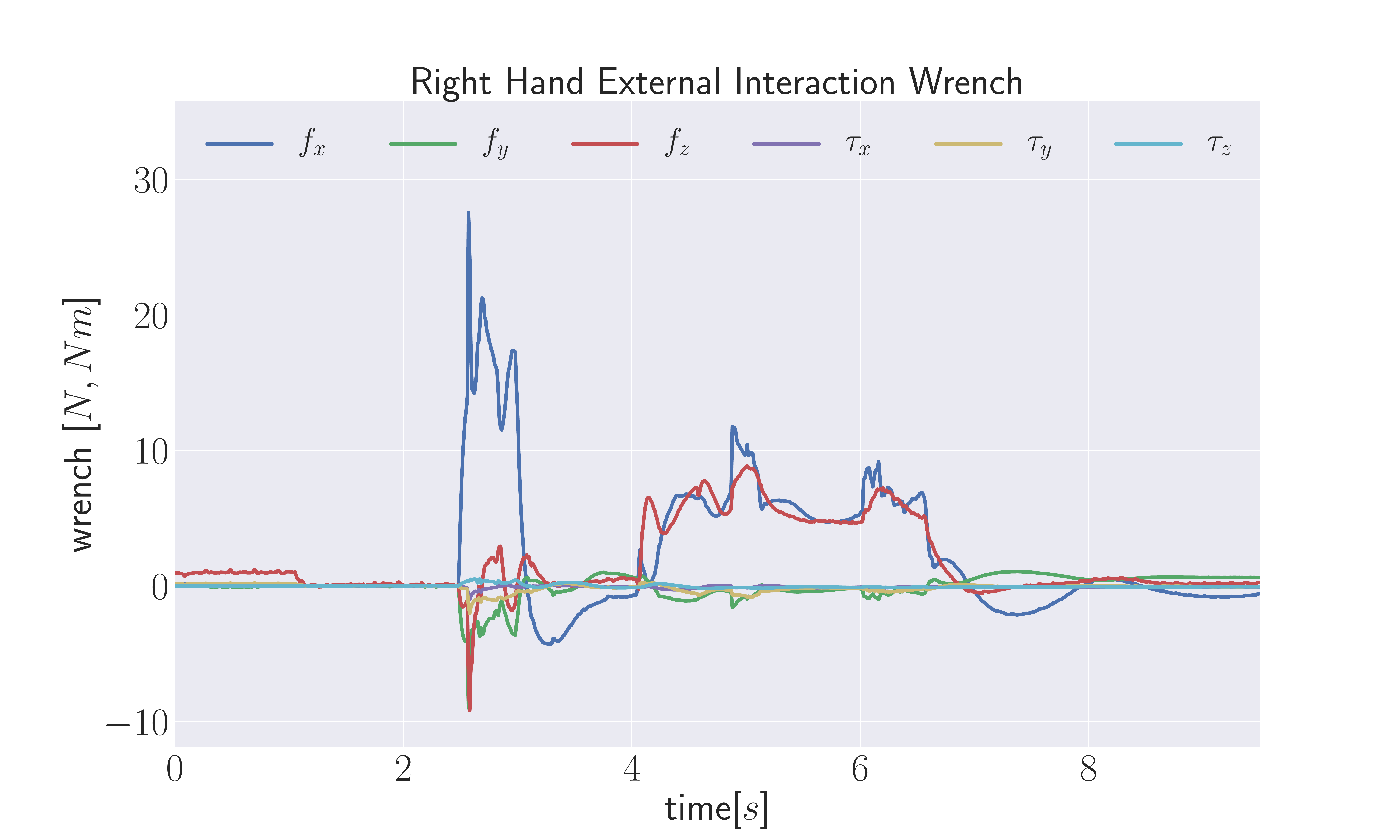}
		\caption{\hspace*{-4mm}}
		\label{fig:standup-simulation-right-external-wrench}
	\end{subfigure}
	~
	\begin{subfigure}{0.85\textwidth}
		\centering
		\includegraphics[width=0.65\textwidth]{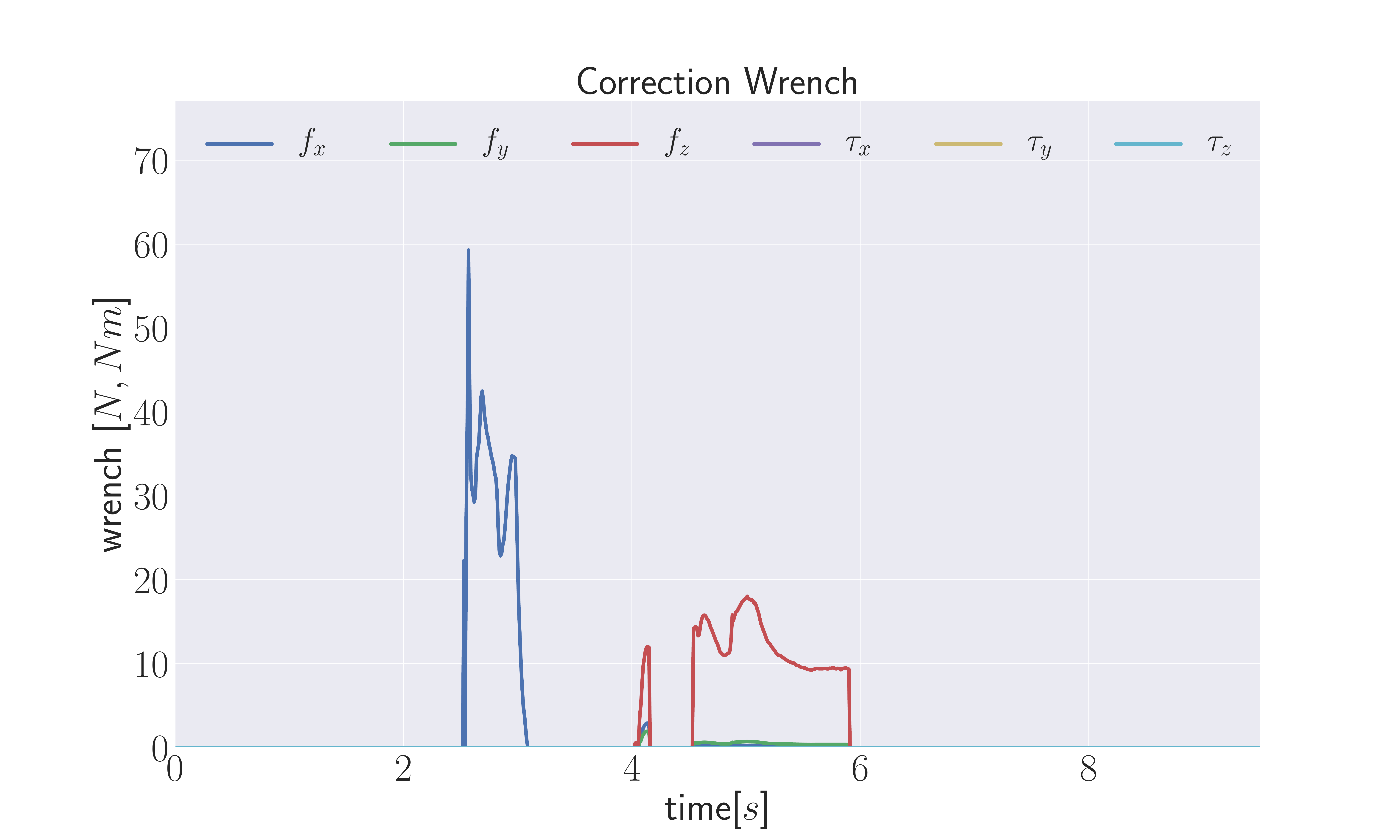}
		\caption{\hspace*{-4mm}}
		\label{fig:standup-simulation-correction-wrench}
	\end{subfigure}
	\begin{subfigure}{0.49\textwidth}
		\centering
		\includegraphics[width=\textwidth]{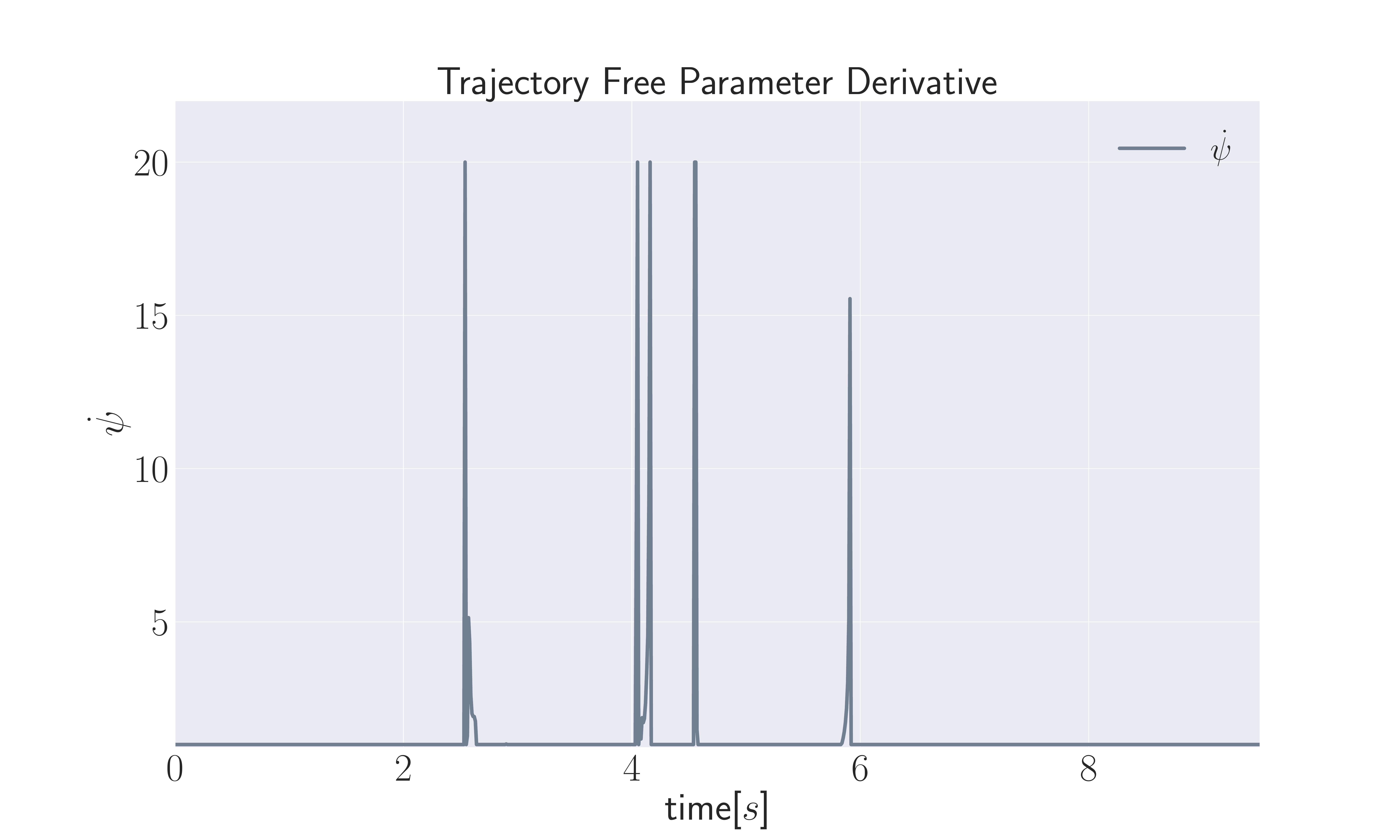}
		\caption{\hspace*{-4mm}}
		\label{fig:standup-simulation-sdotvalue}
	\end{subfigure}%
	\begin{subfigure}{0.49\textwidth}
		\centering
		\includegraphics[width=\textwidth]{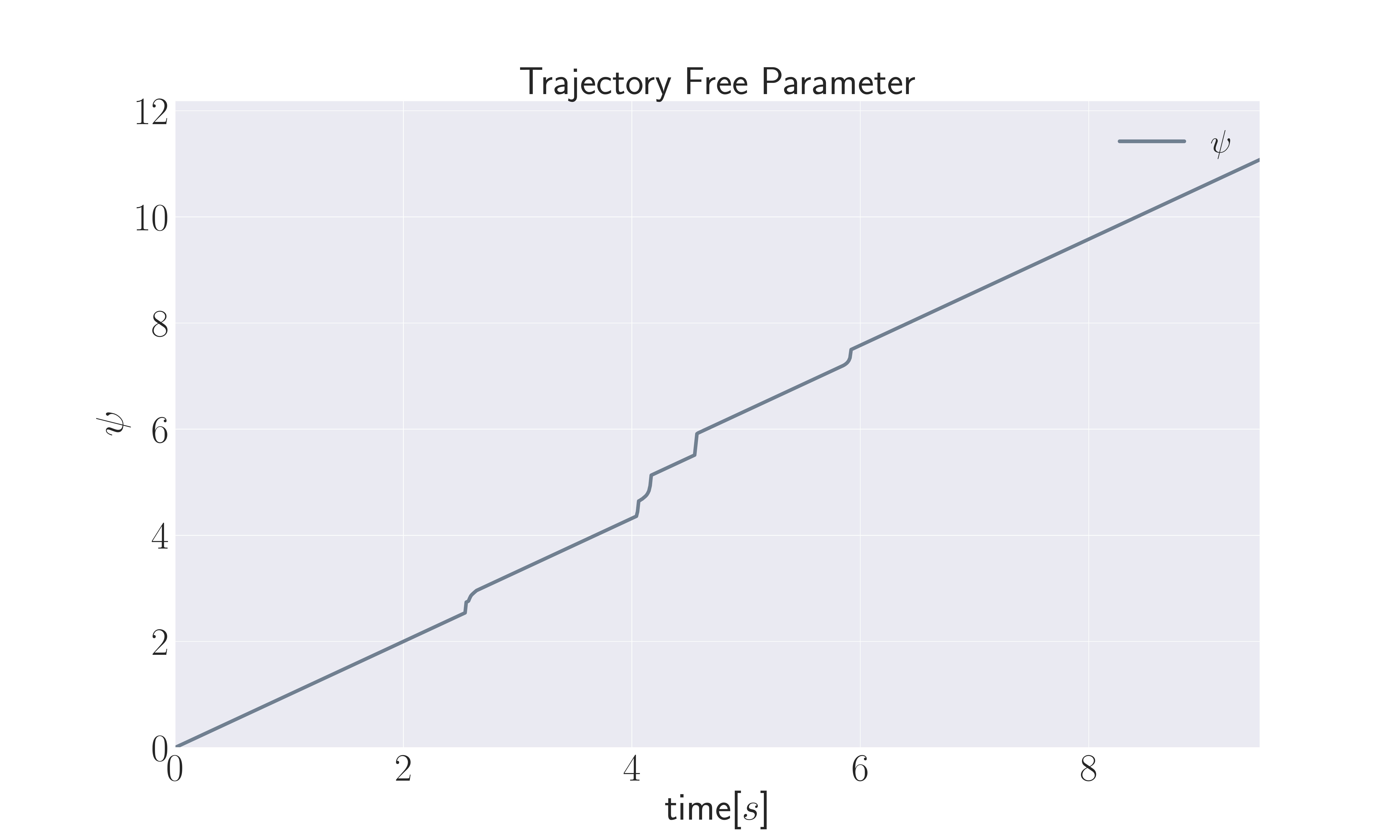}
		\caption{\hspace*{-4mm}}
		\label{fig:standup-simulation-svalue}
	\end{subfigure}
	\begin{subfigure}{0.85\textwidth}
		\centering
		\includegraphics[width=0.65\textwidth]{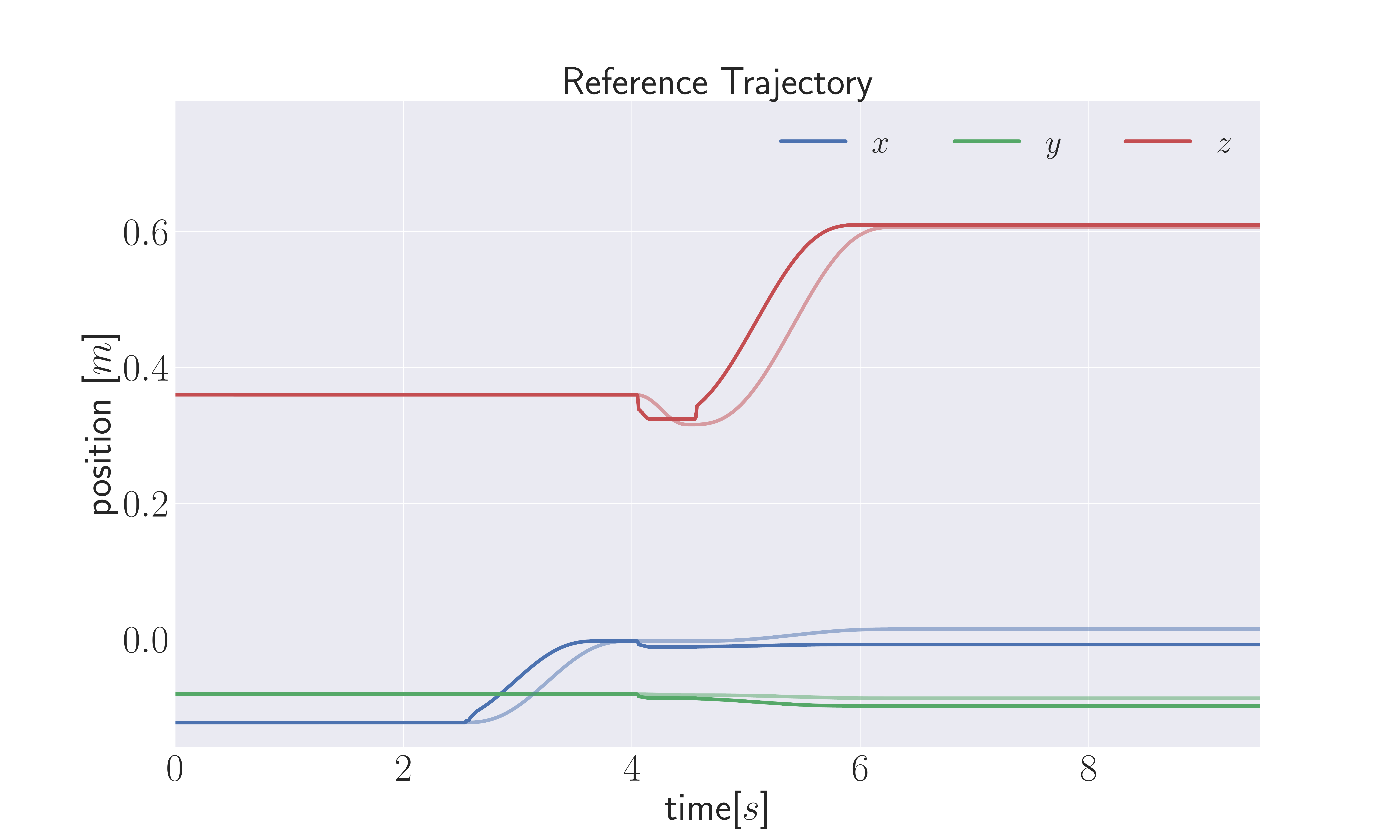}
		\caption{\hspace*{-4mm}}
		\label{fig:standup-simulation-reference-trajectory}
	\end{subfigure}
	\caption{Gazebo simulation trajectory advancement during sit-to-stand transition with whole-body standup controller}
	\label{fig:standup-simulation-trajectory-advancement}
\end{figure}

\begin{figure}[H]
	\begin{subfigure}{\textwidth}
		\hspace{-2cm}
		\includegraphics[scale=0.18]{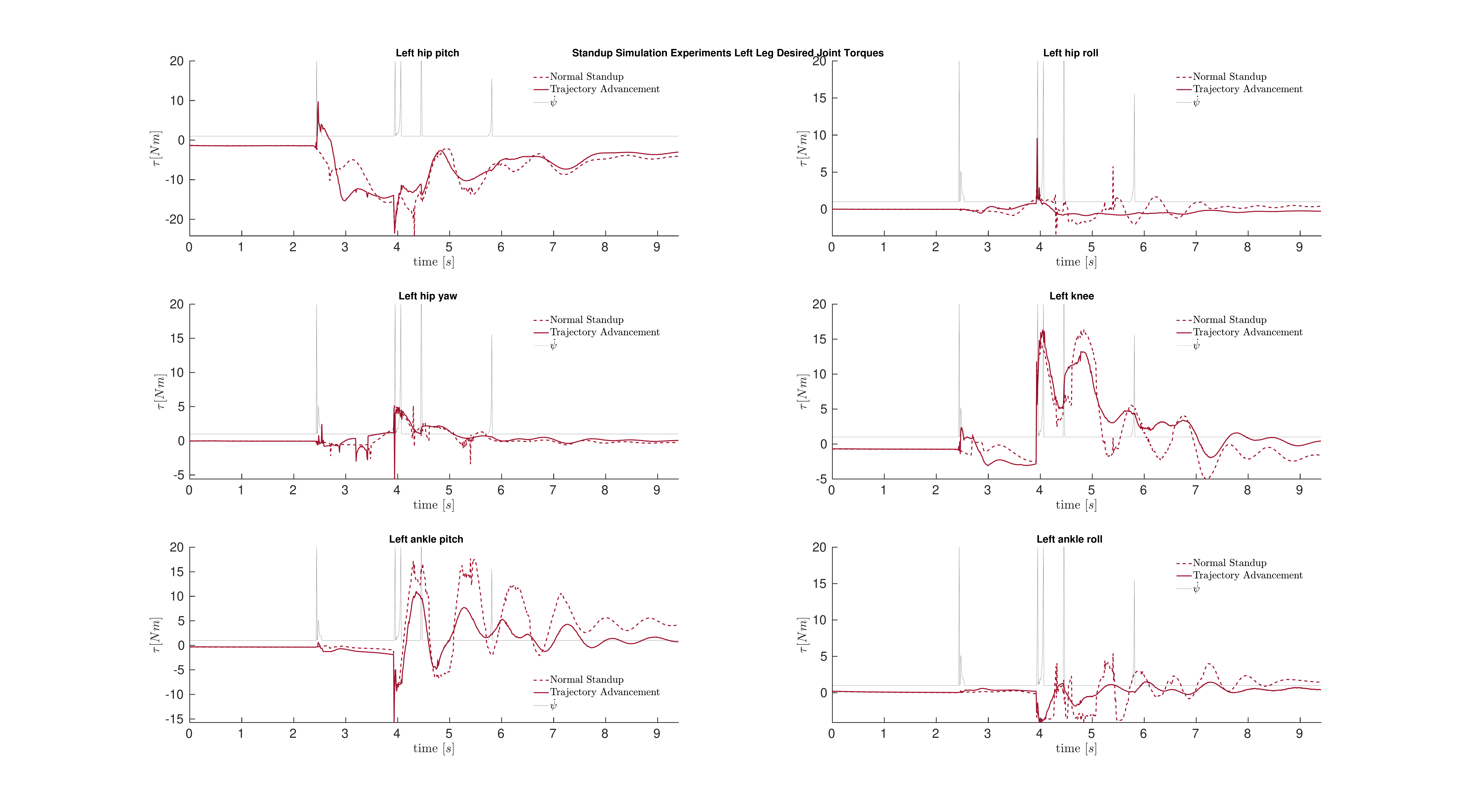}
		\caption{}
		\label{fig:standup_simulation_left_leg_toqures_plot}
	\end{subfigure}
	\begin{subfigure}{\textwidth}
		\hspace{-2cm}
		\includegraphics[scale=0.18]{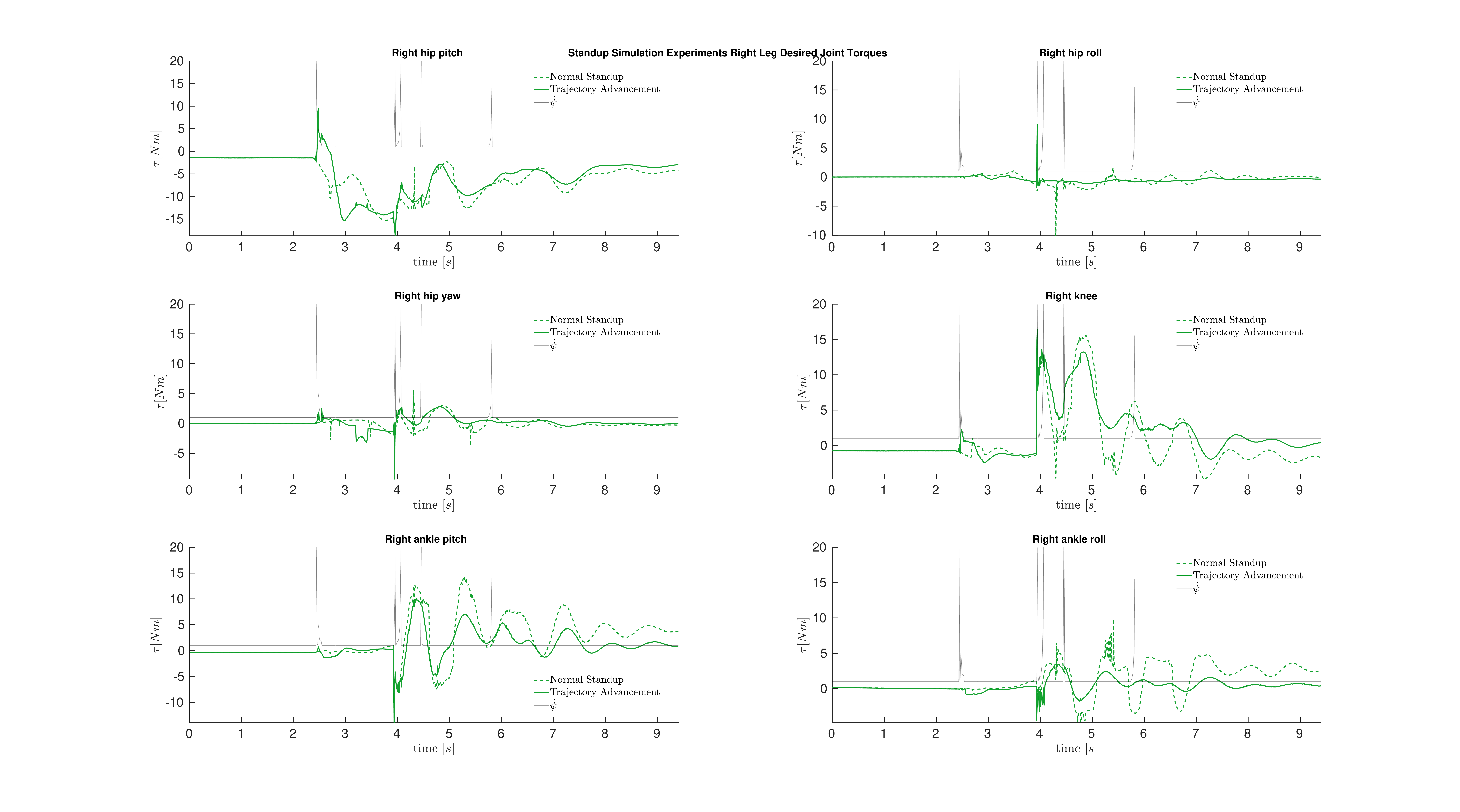}
		\caption{}
		\label{fig:standup_simulation_right_leg_toqures_plot}
	\end{subfigure}
    \caption{Gazebo simulation standup experiment leg joint torques}
    \label{fig:standup_simulation_leg_toqures_plot}
\end{figure}

\newpage

\subsection{Real Robot Results}

The trajectory advancement experiments with the whole body standup controller are also carried with a real iCub humanoid robot as shown in Fig.~\ref{fig:trajectory-advancement-standup-real-robot-standup-states}. A human partner assists the robot by pulling it up and forward while the robot is performing the sit-to-stand transition.

\begin{figure}[ht!]
	\centering
	\begin{subfigure}{0.33\textwidth}
		\centering
		\includegraphics[width=0.975\textwidth]{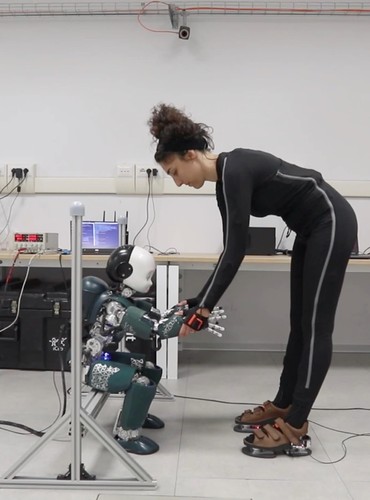}
		\label{fig:trajectory-advancement-standup-real-robot-state-2.png}
	\end{subfigure}%
	\begin{subfigure}{0.33\textwidth}
		\centering
		\includegraphics[width=0.975\textwidth]{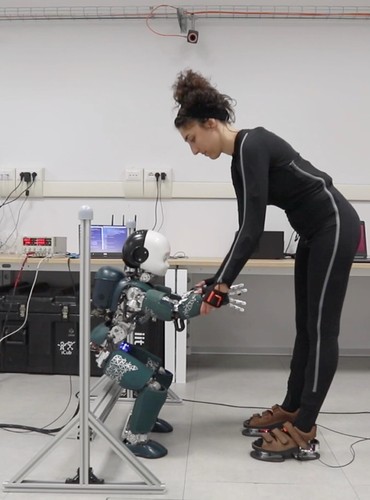}
		\label{fig:trajectory-advancement-standup-real-robot-state-3.png}
	\end{subfigure}%
	\begin{subfigure}{0.33\textwidth}
		\centering
		\includegraphics[width=0.975\textwidth]{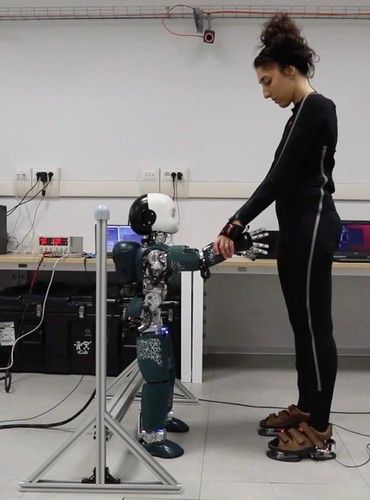}
		\label{fig:trajectory-advancement-standup-real-robot-state-4.png}
	\end{subfigure}
	\caption[iCub at different states during sit-to-stand transition with the whole-body standup controller~(Sec~\ref{sec:whole-body-standup-controller}) using assistance from a human partner]{iCub at different states during sit-to-stand transition with the whole-body standup controller~(Sec~\ref{sec:whole-body-standup-controller}) using assistance\footnotemark[1] from a human partner}
	\label{fig:trajectory-advancement-standup-real-robot-standup-states}
\end{figure}

\footnotetext[1]{Human assistance is considered in terms of the interaction wrench}

The results of trajectory advancement on the real robot are shown in Fig.~\ref{fig:real-robot-standup-trajectory-advancement}. Fig.~\ref{fig:real-robot-standup-left-external-wrench}~\ref{fig:real-robot-standup-right-external-wrench} highlights the wrench estimates at the hands of the robot under the influence of the external wrench applied while the human partner is assisting the robot. The external wrench that is helpful to achieve the task is shown in Fig.~\ref{fig:real-robot-standup-correction-wrench}. The reference trajectory is similar to a time parametrized trajectory i.e., $\psi = t$ until any helpful wrench is considered at the hands of the robot. Under the influence of helpful wrenches, the derivative of the trajectory free parameter $\dot{\psi}$ changes as shown in Fig.~\ref{fig:real-robot-standup-sdotvalue} and the corresponding trajectory advancement is reflected as an increase in $\psi$ as seen in Fig.~\ref{fig:real-robot-standup-svalue}. Accordingly, the reference is advanced further along the  CoM reference trajectory as shown in Fig.~\ref{fig:real-robot-standup-reference-trajectory}. The original CoM reference trajectory without trajectory advancement is shown in the same figure with reduced transparency. Furthermore, the center of mass trajectory tracking error is presented in Fig.~\ref{fig:real-robot-standup-com-error}. The overall tracking error is kept low except during the moments of trajectory advancement, where there is a momentary increase in the error which decreases quickly. 

\begin{figure}[hbt!]
	\centering
	\begin{subfigure}{0.49\textwidth}
		\centering
		\includegraphics[width=\textwidth]{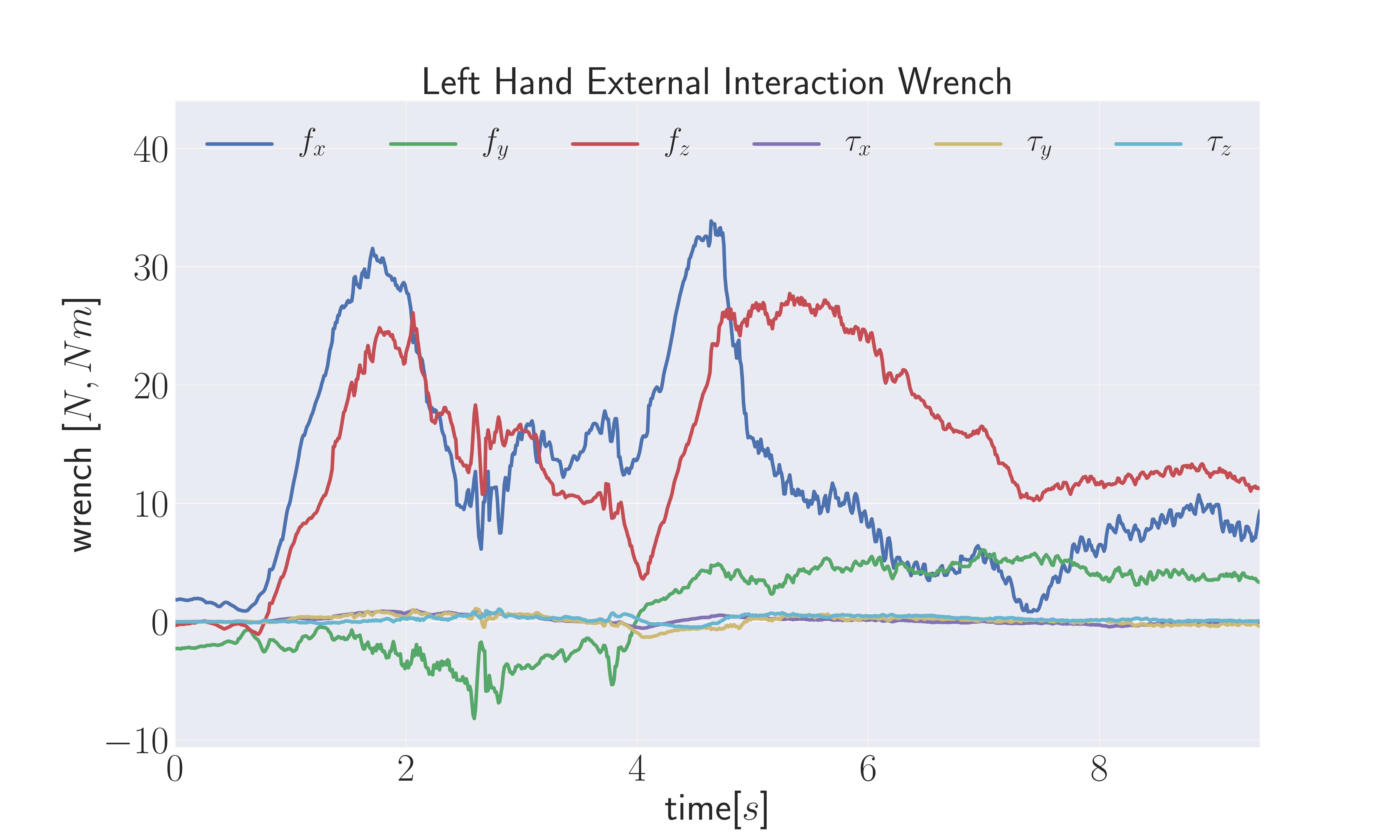}
		\caption{\hspace*{-4mm}}
		\label{fig:real-robot-standup-left-external-wrench}
	\end{subfigure}%
	\begin{subfigure}{0.49\textwidth}
		\centering
		\includegraphics[width=\textwidth]{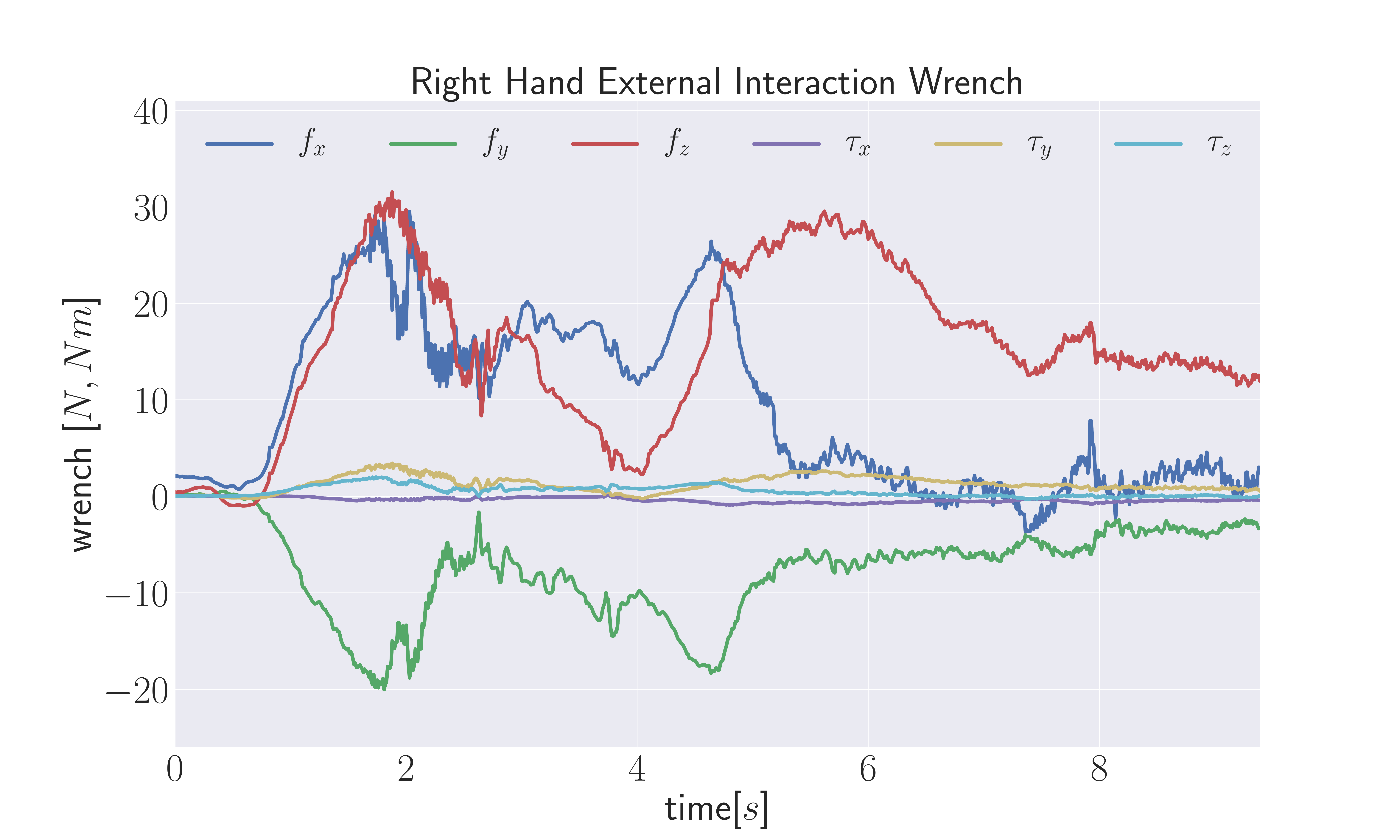}
		\caption{\hspace*{-4mm}}
		\label{fig:real-robot-standup-right-external-wrench}
	\end{subfigure}
	~
	\begin{subfigure}{0.85\textwidth}
		\centering
		\includegraphics[width=0.65\textwidth]{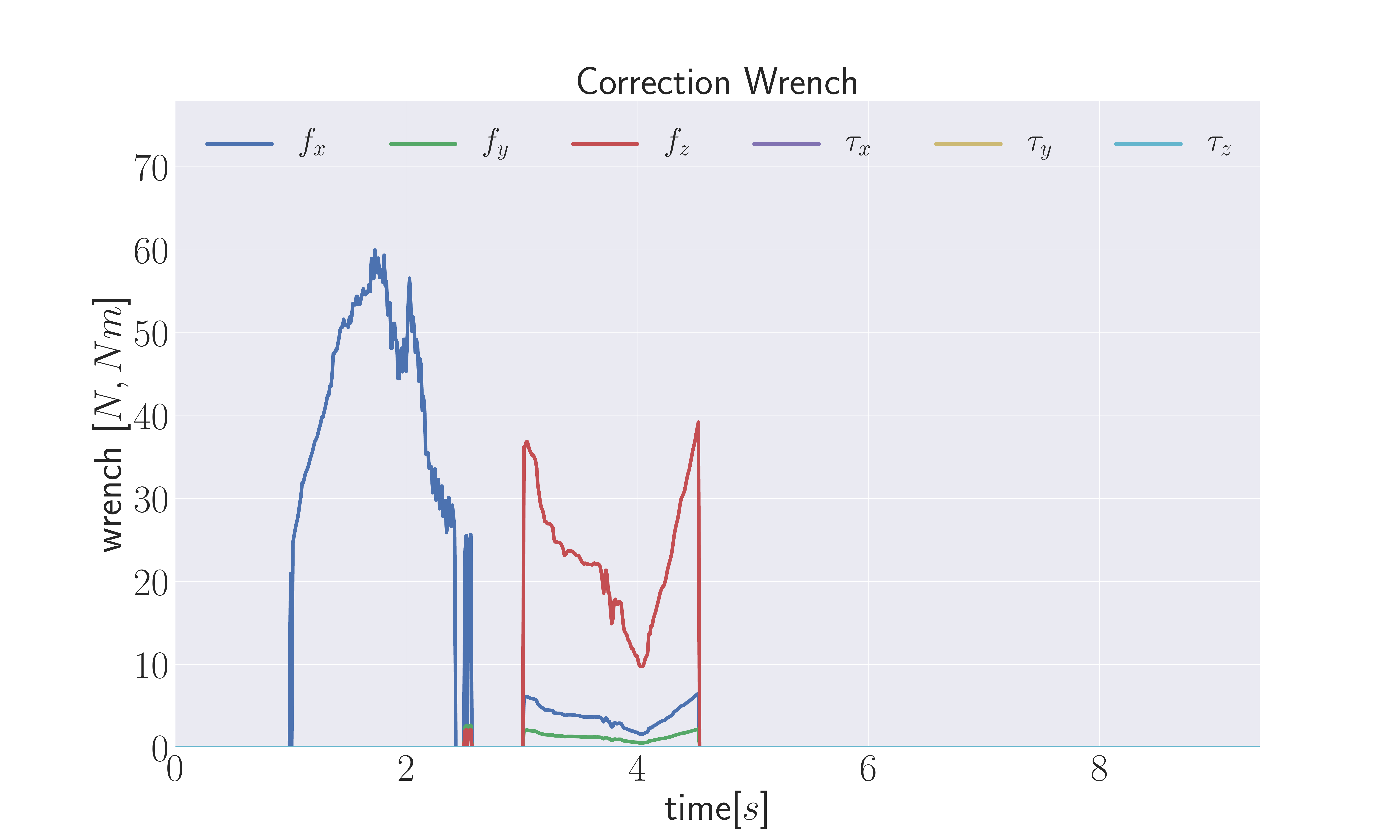}
		\caption{\hspace*{-4mm}}
		\label{fig:real-robot-standup-correction-wrench}
	\end{subfigure}
	\begin{subfigure}{0.49\textwidth}
		\centering
		\includegraphics[width=\textwidth]{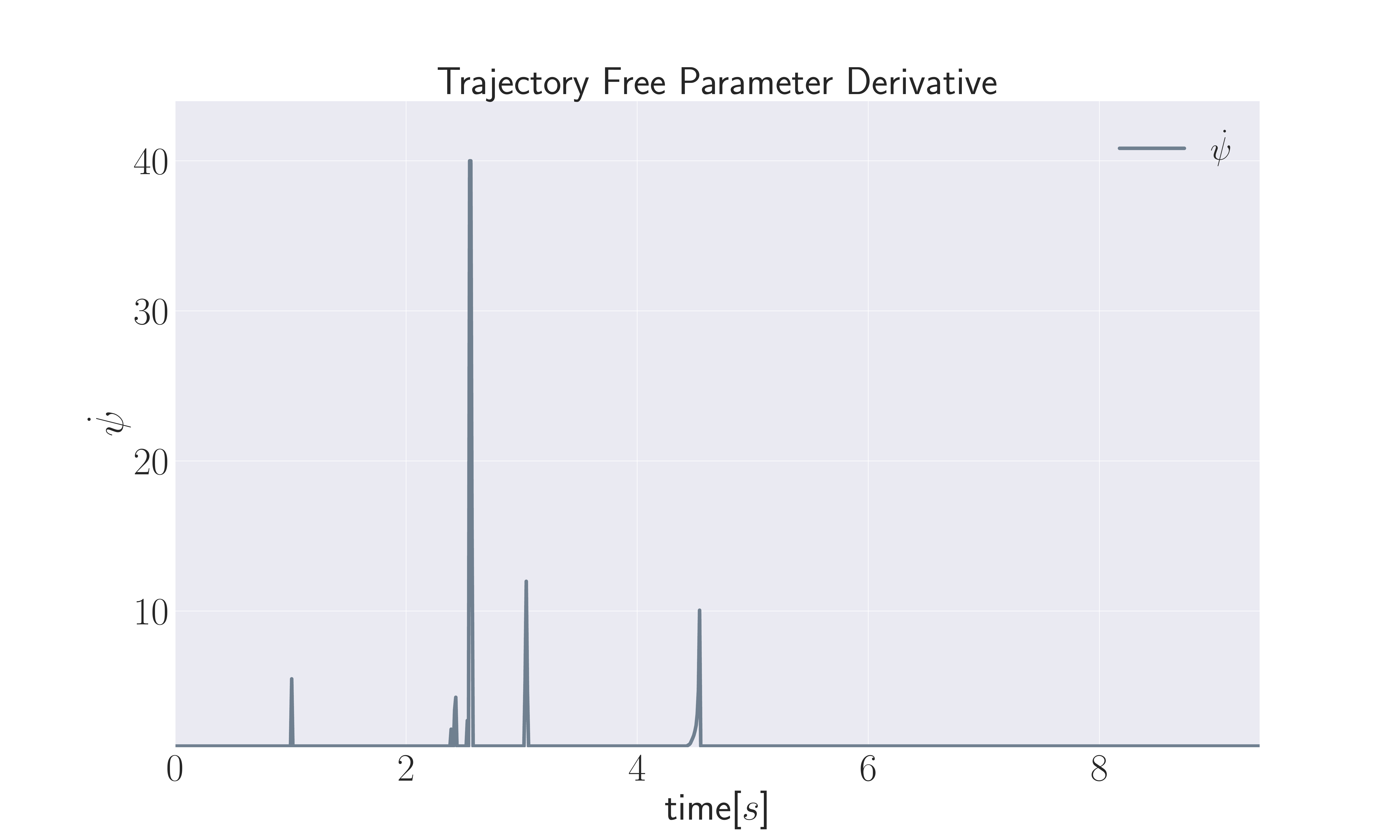}
		\caption{\hspace*{-4mm}}
		\label{fig:real-robot-standup-sdotvalue}
	\end{subfigure}%
	\begin{subfigure}{0.49\textwidth}
		\centering
		\includegraphics[width=\textwidth]{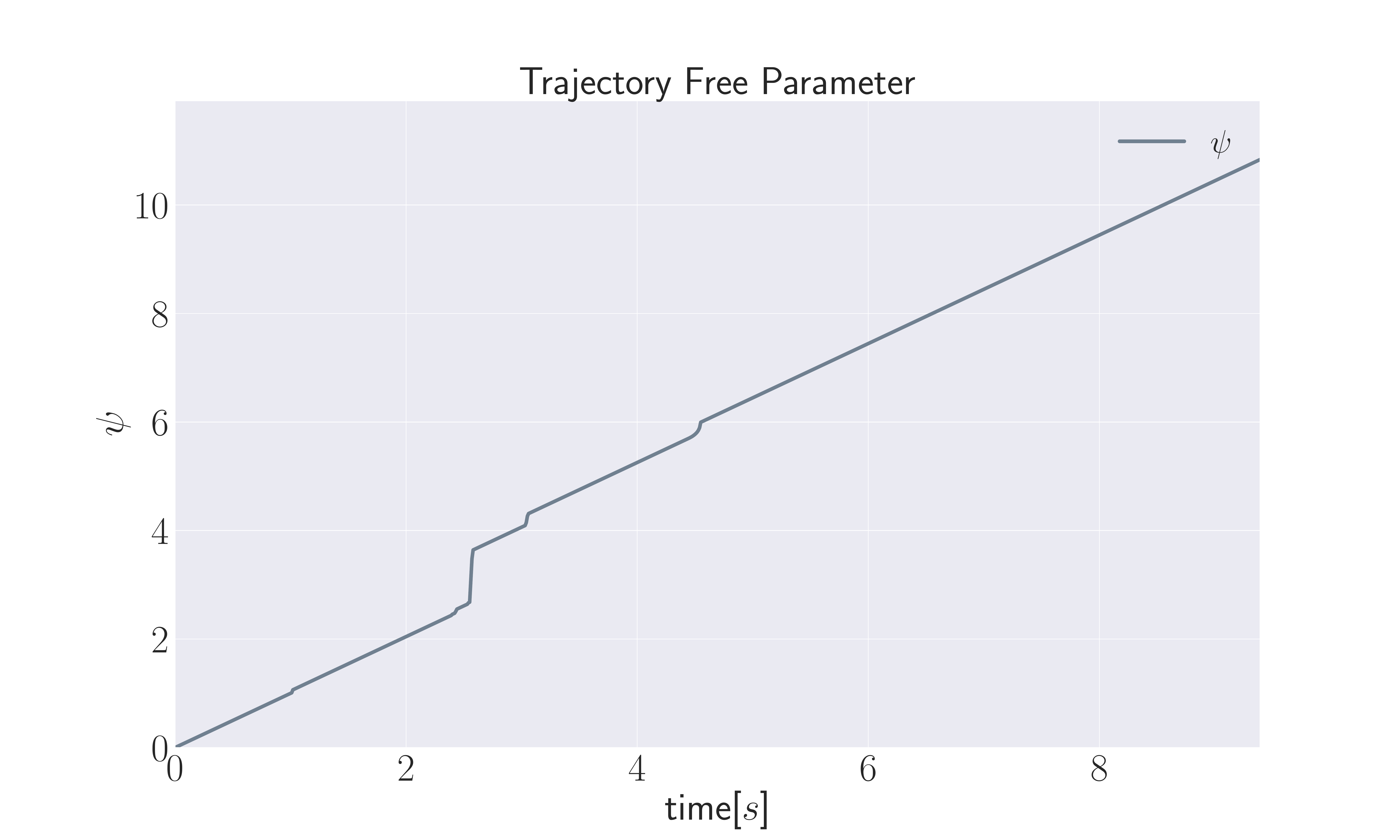}
		\caption{\hspace*{-4mm}}
		\label{fig:real-robot-standup-svalue}
	\end{subfigure}
	\begin{subfigure}{0.49\textwidth}
		\centering
		\includegraphics[width=\textwidth]{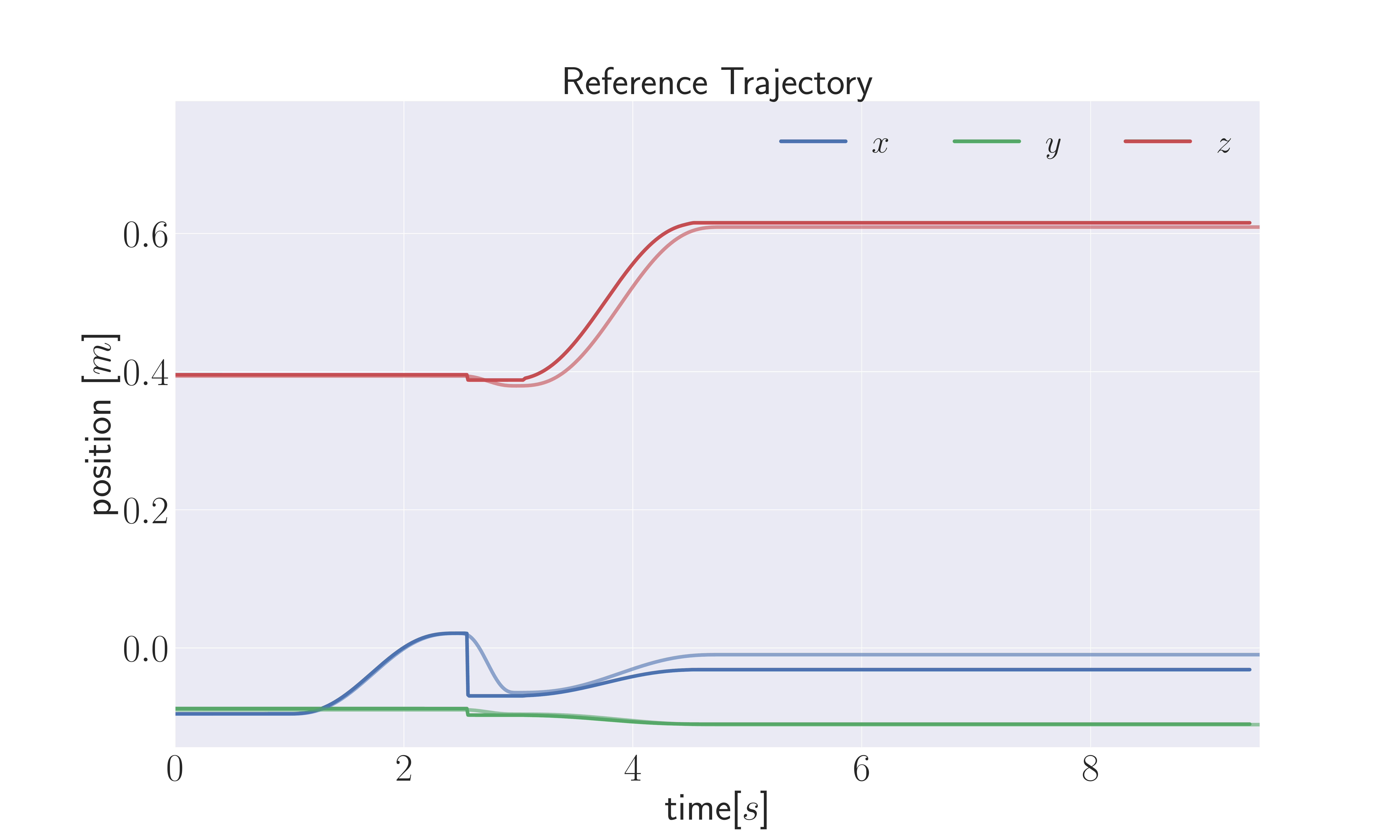}
		\caption{\hspace*{-4mm}}
		\label{fig:real-robot-standup-reference-trajectory}
	\end{subfigure}
    \begin{subfigure}{0.49\textwidth}
    	\centering
    	\includegraphics[width=\textwidth]{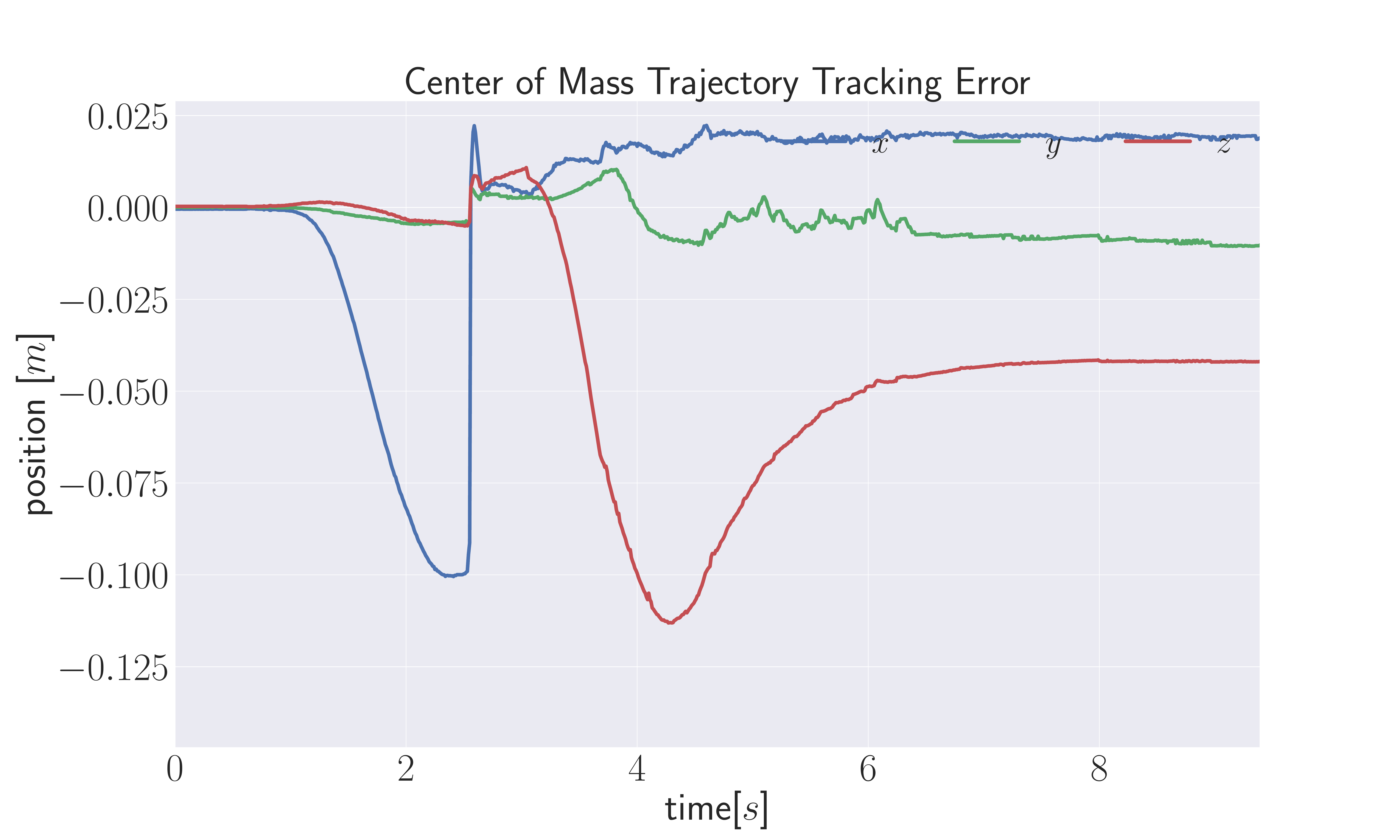}
    	\caption{\hspace*{-4mm}}
    	\label{fig:real-robot-standup-com-error}
    \end{subfigure}
	\caption{Real robot trajectory advancement during sit-to-stand transition with whole-body standup controller}
	\label{fig:real-robot-standup-trajectory-advancement}
\end{figure}

\chapter{Whole-Body Retargeting \& Teleoperation}
\label{cha:whole-body-retargeting}

\chapreface{The concept of telexistence of a human through a robotic avatar is a very challenging endeavor. It involves complex system integration of many technologies from research fields of perception, manipulation and control in robotics. This chapter presents the details of whole-body human motion retargeting to a humanoid avatar platform, and teleoperation experiments with state-of-the-art whole-body robot controllers for balancing and locomotion.}

Some of the technologies involved on the human operator side and the robotic avatar side of a sophisticated telexistence setup are highlighted in Fig.~\ref{fig:telexistence_technologies}. Potential applications of telexistence through a robotic avatar range from disaster response scenarios to providing health care and assistance remotely.

\begin{figure}[H]
	\centering
	\begin{subfigure}[b]{0.5\textwidth}
		\centering
		\includegraphics[width=\textwidth]{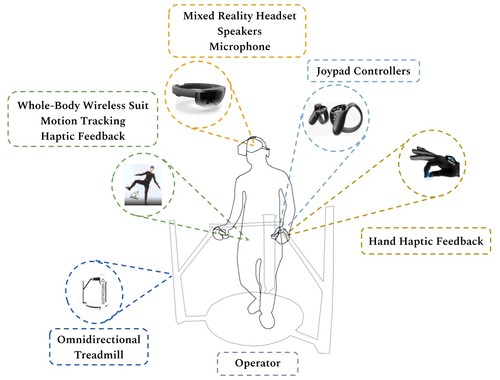}
		\label{fig:operator_technologies}
	\end{subfigure}%
	\begin{subfigure}[b]{0.5\textwidth}
		\centering
		\includegraphics[width=\textwidth]{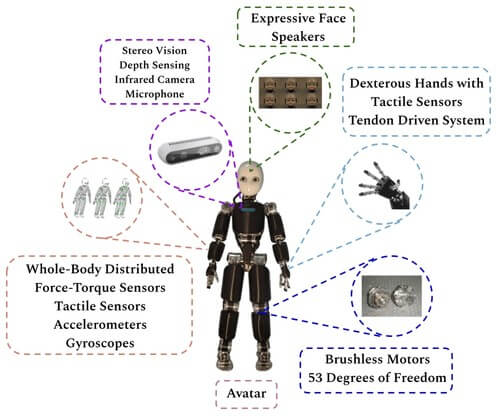}
		\label{fig:avatar_technologies}
	\end{subfigure}
	\caption{Technologies involved on the human operator side and the robotic avatar side of a sophisticated telexistence setup}
	\label{fig:telexistence_technologies}
\end{figure}

This chapter primarily focuses on the aspect of humanoid robot teleoperation using a novel framework of whole-body human motion retargeting to a humanoid robot avatar. Teleoperation is one of the core competencies for the successful realization of teleexistence of a human through a robotic avatar. Teleoperation stands for operating a robot from a distance that enables the extension of human capabilities to scenarios that are not reachable by humans due to the time and space constraints or the dangers posed by hazardous environments \cite{Hokayem2006}. The current robotic systems are still limited in terms of their perception and decision making capabilities to operate fully autonomous in real-world conditions. So, teleoperation facilitates an establishment of a human-robot team where the goal of the teleoperated robot is the same as the human operator and the chances of successful task completion are enhanced by bringing together the excellent cognitive capabilities of humans and the physical capabilities of robotic systems \cite{Zucker2015}. Teleoperation plays an important role in a wide range of real-world applications including  manipulation in hazardous environments \cite{Shimoga1993,Trevelyan2016}, telepresence \cite{Tachi1989}, telesurgery \cite{Burgner2015}, and space exploration \cite{Pedersen2003}.
 
Concerning the robotic avatars, humanoid robots are an ideal platform as they are designed and built based on the principles of Anthropomorphism. Unlike other robotic systems, they often have higher maneuverability and manipulation capabilities that are similar to a human being \cite{Ishiguro2018} which provides intuitive capabilities during telexistence. On the other hand, owing to their inherent complexity, humanoid robots are more challenging technologically for teleoperation in unstructured dynamic environments designed for humans. The level of autonomy, team organization and, the information exchange between the operator and the robot are some of the vital aspects in teleoperation performance to ensure successful task completion \cite{Beer2014, Steinfeld2006, goodrich2008human}. The level of autonomy ranges from being a semi-autonomous robot at the symbolic or the action level (high-level teleoperation) \cite{goodrich2008human, Hokayem2006} to complete control of the robot at the kinematic and the dynamic level (low-level teleoperation), either in the robot's configuration space or task space. A core component of the low-level teleoperation system is the human motion retargeting to a robot.

Two of the most studied teleoperation paradigms in literature are: 1) master-slave; and 2) bilateral systems. Under master-slave teleoperation paradigm, the flow of information is unidirectional from the human to the robot, while under the bilateral teleoperation paradigm there is an exchange of information between the human and the robot. In particular, haptic feedback to the human from the robot \cite{Ishiguro2017, Wang2015}. Teleoperation systems that involve humans in the control loop at the kinematic and dynamic level should have the prime objectives of situational awareness and transparency, i.e., the human operator experiencing the remote environment of the teleoperated robot as holistically as possible while maintaining the stability of the closed-loop system \cite{Lichiardopol2007, Hokayem2006}. Delays and information loss are some of the crucial problems with this approach that affect the transparency and stability of teleoperation greatly \cite{Hokayem2006, Lichiardopol2007}.  Different approaches such as Lyapunov stability analysis \cite{Islam2015, Chopra2003}, passivity based control \cite{Chopra2003} have been employed to address these limitations. However, these methods are studied extensively with manipulators and the stability measures for humanoid robot teleoperation are not well established \cite{Ramos2018TRO}. 

The research on teleoperation of humanoid robots can be broadly classified into three categories: upper body teleoperation, lower body teleoperation, and whole-body teleoperation. In upper body teleoperation, some of the research works consider mapping the human motion at the configuration space \cite{Liarokapis2013, Ayusawa2017, Stanton2012} while some others consider the task space \cite{Elobaid2018, Liarokapis2013}. Furthermore, the effect of change in Center of Mass (\textsc{com}) of the robot is also important to ensure the robot stability \cite{Elobaid2018}. So, concerning the lower body teleoperation of humanoid robots, the aspects of stability and locomotion have higher precedence over retargeting of all the lower limbs. A more detailed description of such approaches is discussed in \cite{Romualdi2018, Feng2015}. Coming to the topic of whole-body teleoperation of humanoid robots, the key challenge is to control the robot such that it does not fall while keeping its maneuverability and manipulability high, ensuring task completion by the human-robot team. Typically, the balance of the robot is achieved by either keeping the center of mass inside the support polygon or maintaining the net momentum about the Center of Pressure (\textsc{cop}) to zero \cite{Penco2018, Ishiguro2018}. Tasks involving multi-link dynamic contacts such as locomotion which involves careful monitoring and regulation of force exchange with the environment pose higher levels of complexity for teleoperation.

Whole-body human motion retargeting to a humanoid robot avatar is a critical component in the design of teleoperation technologies. One of the obvious shortcomings of the teleoperation systems proposed in the literature is the lack of ability to quickly and easily adapt the system for different human users and humanoid robots with different geometries, kinematics, and dynamics. The system designer often spends time and effort in careful consideration of changes in the human models and the robot models to ensure successful teleoperation which limits \textit{usability} and \textit{scalability}. 

This chapter aims at addressing this problem through a novel framework for whole-body human motion retargeting that requires minimal changes to use with different human operators or different robot avatars. Furthermore, experiments on whole-body teleoperation of humanoid robots with two state-of-the-art whole-body controllers for humanoid robots are presented.

\subsubsection{Notation}
\label{sec:background}

This chapter contains new notations presented in the following table in addition to the details presented in the nomenclature~\eqref{nomenclature_table} and background notation~\eqref{sec:background-notation}.

\[
  \left[
      \begin{tabular}{@{\quad}m{.05\textwidth}@{\quad}m{.83\textwidth}}
        {\Huge \faInfoCircle} & \raggedright \textbf{} \par
          \begin{tabular}{@{}p{0.18\textwidth}p{0.60\textwidth}@{}}
            $(.)^R$                 & Robot quantity\\
            $(.)^{H}$               & Human quantity\\
          \end{tabular}
      \end{tabular}
    \right]
\]

\section{Human Motion Retargeting}
\label{sec:kinematic-retargeting}

Human motion measurements are acquired through the Xsens whole-body motion tracking suit with distributed inertial measurement units as described in Section~\ref{sec:technologies-human-motion-perception}. The two approaches of retargeting human motion can be categorized as configuration space retargeting and task space retargeting.

\subsection{Configuration space retargeting}

Given a human model and the measurements of various limbs in terms of position and orientation, an inverse kinematics algorithm is employed to retrieve the human joint positions and velocities. A configuration space mapping between the human operator model and the robotic avatar is used as a reference to compute the joint position references for controlling the robot motion.

The architecture shown in Fig.~\ref{fig:configuration-retargeting} represents a typical configuration space retargeting scheme \cite{Ayusawa2017, Penco2018}. 

\begin{figure}[H]
	\centering
	\includegraphics[width=0.9\columnwidth]{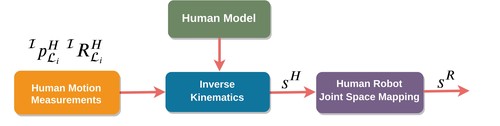}
	\caption{Typical configuration space retargeting scheme}
	\label{fig:configuration-retargeting}
\end{figure}

A key component in this approach is the customized mapping between the human operator and the robotic avatar in the configuration space. Oftentimes, the complexity and range of motion of certain human joints like the shoulder joint are very complex. So, the configuration space mapping must ensure these complexities and furthermore the limits of the robot joints to enable safe configuration space control references to the robot.

\subsection{Task space retargeting}

Task space retargeting approach depends on the task space mapping between the human operator limbs and the robotic avatar limbs. Leveraging the task space mapping, the human limb motion measurements are converted to appropriate robot limb motion references. Given a robot model and the computed robot limb motion references, an inverse kinematics algorithm is employed to compute the robot references in the joint space considering the robot joint limitations. The architecture shown in Fig.~\ref{fig:task-retargeting} represents a typical task space retargeting scheme  \cite{Elobaid2018, Ishiguro2018}.

\begin{figure}[H]
	\centering
	\includegraphics[width=0.9\columnwidth]{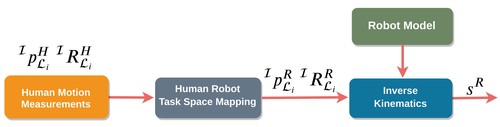}
	\caption{Typical task space retargeting scheme}
	\label{fig:task-retargeting}
\end{figure}

Clearly, a key component in this approach is finding the task space mapping between the human operator limbs and the robotic avatar limbs. Oftentimes, the scale of the human limbs are very different than that of the robot. Furthermore, considering the task space motion as a reference for the robot may lead to a robot joint configuration that may be dissimilar from the human joint configuration. This may often lead to psychological discomfort for the human operator or the people who are interacting with the robot as one cannot predict the robot's motions owing to non-anthropomorphic motions that depend on the inverse kinematics algorithm employed \cite{Liarokapis2013}. Moreover, the precise control of the robot's joint configuration becomes essential when the robot is deployed in cluttered environments where avoiding obstacles is crucial of mission success.

\subsection{Our Approach}
\label{sec:our-approach}

In our approach to whole-body kinematic motion retargeting, we first identify the task space mapping between the human links and the robot links in a neutral pose as shown in Figure.~\ref{fig:retargeting_link_mapping}. The frame equivalence from the human to the robot links is indicated by the numbering. The task space mapping is identified in terms of the rotation between the human links and the robot links i.e. $^{H_{\mathcal{L}_i}}{\comVar{R}}_{R_{\mathcal{L}_i}}$. The robot model is updated with new frames (attached to the robot links) that are identical to the human link frames. Given the rotation measurement from the human link frames to the inertial frame, $^{{\mathcal{I}}}{\comVar{R}}_{H_{\mathcal{L}_i}}$, and the constant rotation from the robot link frames to human link frames,  $^{H_{\mathcal{L}_i}}{\comVar{R}}_{R_{\mathcal{L}_i}}$, the rotation measurement of the robot frame with respect to the inertial frame is computed as,

\begin{equation}
	^{{\mathcal{I}}}{\comVar{R}}_{R_{\mathcal{L}_i}}^{*}  = \ ^{{\mathcal{I}}}{\comVar{R}}_{H_{\mathcal{L}_i}} \ ^{H_{\mathcal{L}_i}}{\comVar{R}}_{R_{\mathcal{L}_i}}
	\label{eq:human_robot_transformation}
\end{equation}

\begin{figure}[H]
	\centering
	\includegraphics[width=0.7\columnwidth]{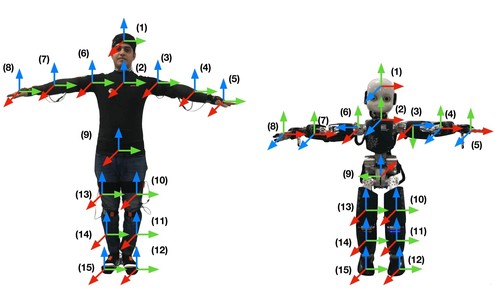}
	\caption{Task space mapping between the human links and the robot links in neutral pose}
	\label{fig:retargeting_link_mapping}
\end{figure}

Any measurements coming from the human are easily converted to the measurements of the robot links through the rotation mapping using Eq.~\ref{eq:human_robot_transformation}. Once the desired robot link orientation is computed, thanks to our modular software architecture, we leverage the dynamical optimization inverse kinematics \cite{Rapetti2019} outlined in Section~\ref{sec:dynamical-optimization} on the updated robot model to compute the robot joint positions $\comVar{s}^R$. Figure~\ref{fig:combined_retargeting} highlights our whole-body kinematic human motion retargeting approach.

\begin{figure}[H]
	\centering
	\includegraphics[width=0.9\columnwidth]{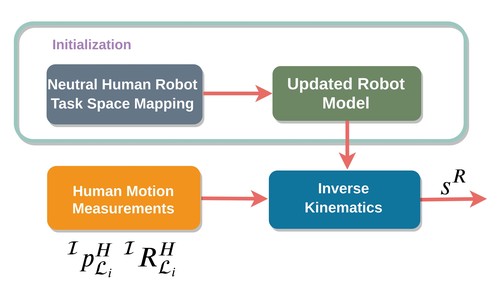}
	\caption{Block diagram of our kinematic whole-body motion retargeting}
	\label{fig:combined_retargeting}
\end{figure}

\section{Whole-Body Retargeting Experiments}
\label{sec:retargeting-experiments-results}

The software architecture for performing whole-body retargeting experiments is highlighted in Fig.~\ref{fig:hde-whole-body-retargeting}. Thanks to our modular infrastructure for holistic human perception presented in Chapter~\ref{cha:software-architecture}, only minimal changes are required to obtain the robot joint states for performing the retargeting. As highlighted in Fig.~\ref{fig:hde-whole-body-retargeting}, the key modification is to use the updated robot model from section~\ref{sec:our-approach} instead of the human model.

\begin{figure}[ht!]
	\centering
	\includegraphics[width=0.7\columnwidth]{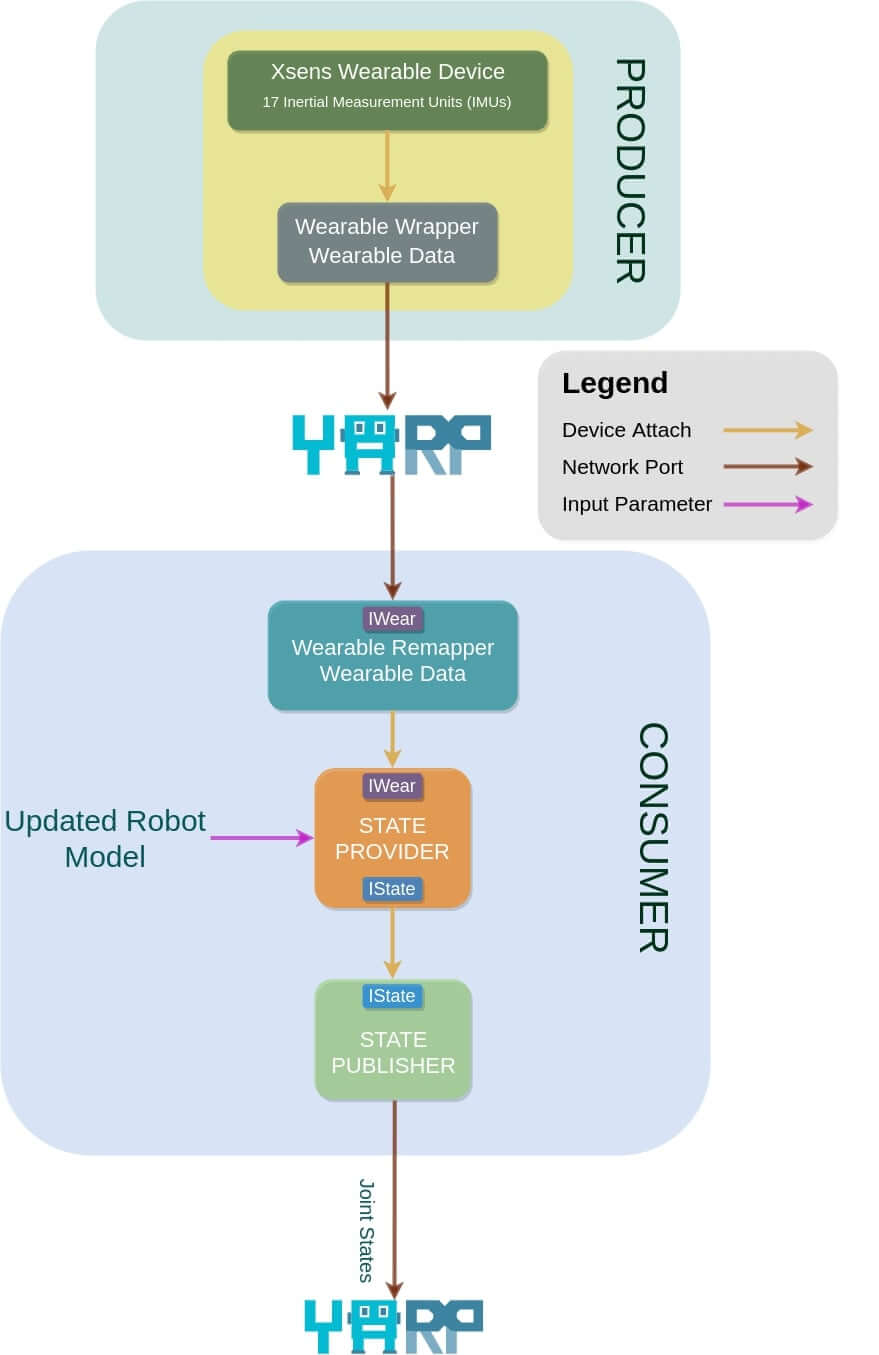}
	\caption{Software architecture from Chapter~\ref{cha:software-architecture} for whole-body retargeting using the updated robot model}
	\label{fig:hde-whole-body-retargeting}
\end{figure}

The whole-body retargeting experiments are performed with motion data captured for two human subjects. To demonstrate the scalability and usability of our proposed method, we perform whole-body kinematic retargeting with multiple robots having different degrees of freedom (DoFs). The robot models we considered are a) iCub humanoid robot with 32 DoFs, b) NAO humanoid robot with 24 DoFs, c) Atlas humanoid robot with 30 DoFs. To show that our method is not limited to humanoid robots, we perform a retargeting scenario with Baxter dual arm 15 DoFs robot. Additionally, we show the retargeting with a human model that has 66 DoFs which is described in Chapter~\ref{cha:human-modeling}.

The Fig.~\ref{fig:scalability} highlights our whole-body kinematic retargeting with different models and human subjects in Rviz, a 3D visualizer. The first row corresponds to the human motion of standing on the right foot from the first subject and the second row corresponds to the human motion of standing on the left foot by the second subject. Concerning the Baxter robot, the retargeting is done only for the arms and the head, as it is a fixed base robot.

\begin{figure}[ht!]
	\centering
	\begin{subfigure}[b]{0.2\textwidth}
	    \centering
		\includegraphics[clip, trim=0cm 2cm 0cm 2cm, scale=0.2215]{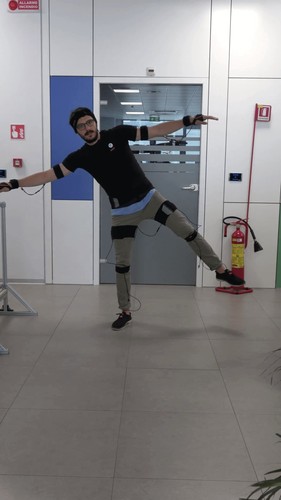}
		\label{fig:rviz-subject-one-foot}
	\end{subfigure}
	~
	\begin{subfigure}[b]{0.775\textwidth}
	    \centering
		\includegraphics[scale=0.6525]{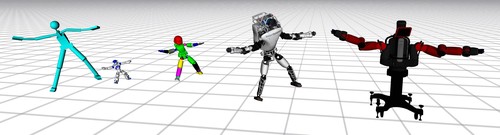}
		\label{fig:rviz-human-one-foot}
	\end{subfigure}
	\begin{subfigure}[b]{0.2\textwidth}
	    \centering
		\includegraphics[clip, trim=0cm 2cm 0cm 2cm, scale=0.2215]{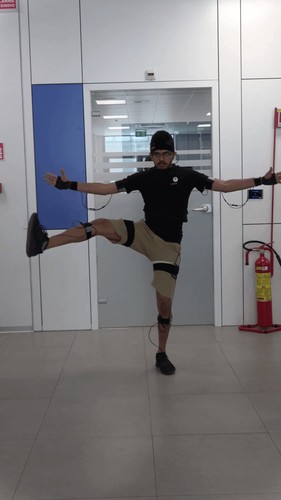}
		\label{fig:rviz-subject-jump}
	\end{subfigure}
	~
	\begin{subfigure}[b]{0.775\textwidth}
	    \centering
		\includegraphics[scale=0.6685]{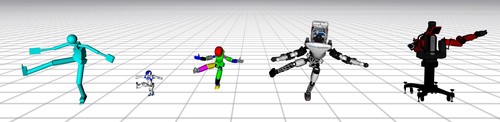}
		\label{fig:rviz-human-jump}
	\end{subfigure}
	\caption{Rviz visualization of whole-body retargeting of different human subjects motion to different models: a) Human Model b) Nao c) iCub d) Atlas e) Baxter; top: human subject stands on the right foot, bottom:  human subject stands on the left foot}
	\label{fig:scalability}
\end{figure}

\begin{figure}[!ht]
	\centering
	\begin{subfigure}{\textwidth}
	    \centering
		\includegraphics[width=0.8\textwidth]{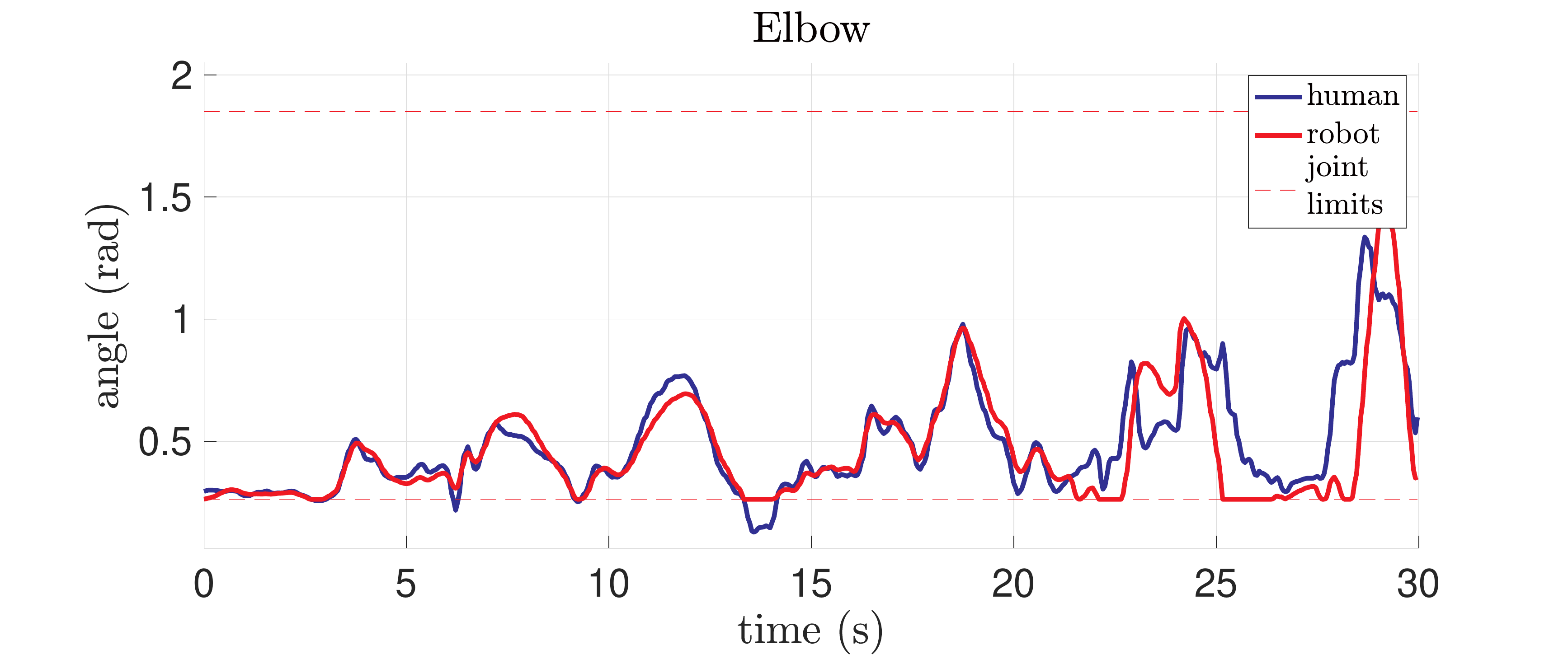}
		\caption{}
		\label{fig:error-joint-value} 
	\end{subfigure}
	\begin{subfigure}{\textwidth}
	    \centering
		\includegraphics[width=0.8\textwidth]{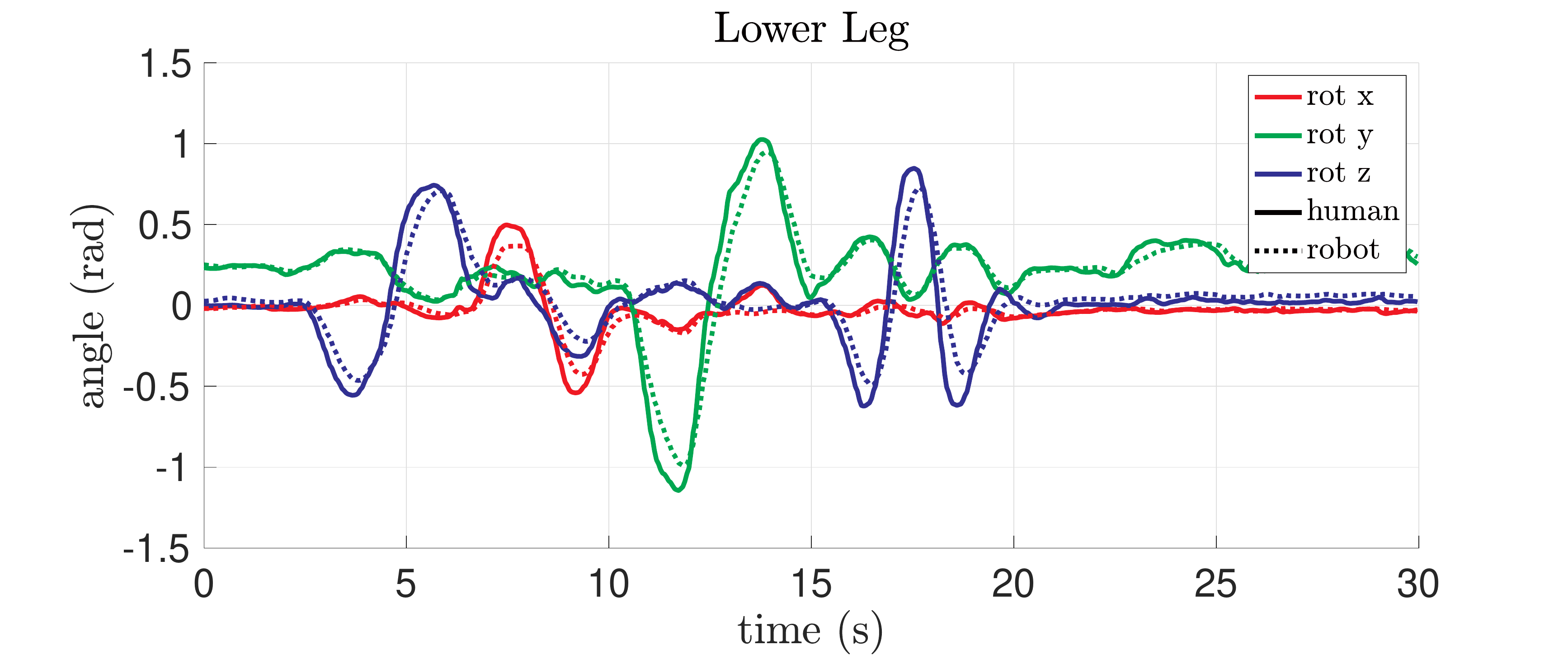}
		\caption{}
		\label{fig:error-link-value}
	\end{subfigure}
	\caption{Performance of whole-body retargeting of human motions to iCub}
	\label{fig:link_orientation_error_angles_robot_limits}
\end{figure}

The joint angles values of the right arm elbow joint for the human and the iCub robot model are highlighted in Fig.~\ref{fig:error-joint-value}. The overall retargeting of the joint position is good except for some configurations of the human where the robot is constrained by the joint limits e.g., time instant $t\sim \SI{13}{\second}$. The lower leg link orientation (in terms of \textit{Euler angles} ) of the human and the iCub robot model is highlighted in Fig.~\ref{fig:error-link-value}.

\section{Whole-Body Teleoperation Architecture}
\label{sec:RetargetingStructure}

At the current state, our teleoperation architecture highlighted in Fig.~\ref{fig:architecture} is composed of the following technologies:

\begin{itemize}
	\item Oculus Virtual Reality Headset: Visual feedback from the robot environment by streaming the robot camera images to the human operator
	\item Joypads: The robot hands are controlled via the joypads for opening and closing the fingers during manipulation scenarios
	\item Cyberith Virtualizer Virtual Reality Treadmill: The human locomotion information i.e., the linear velocity in $x$ and $y$ directions and the angular velocities about the $z$ direction
	\item Whole-Body Motion Tracking Suit: Human motion perception through inertial tracking technology described in Section~\ref{sec:technologies-human-motion-perception}
\end{itemize}

\begin{figure}[H]
	\centering
	\includegraphics[width=\textwidth]{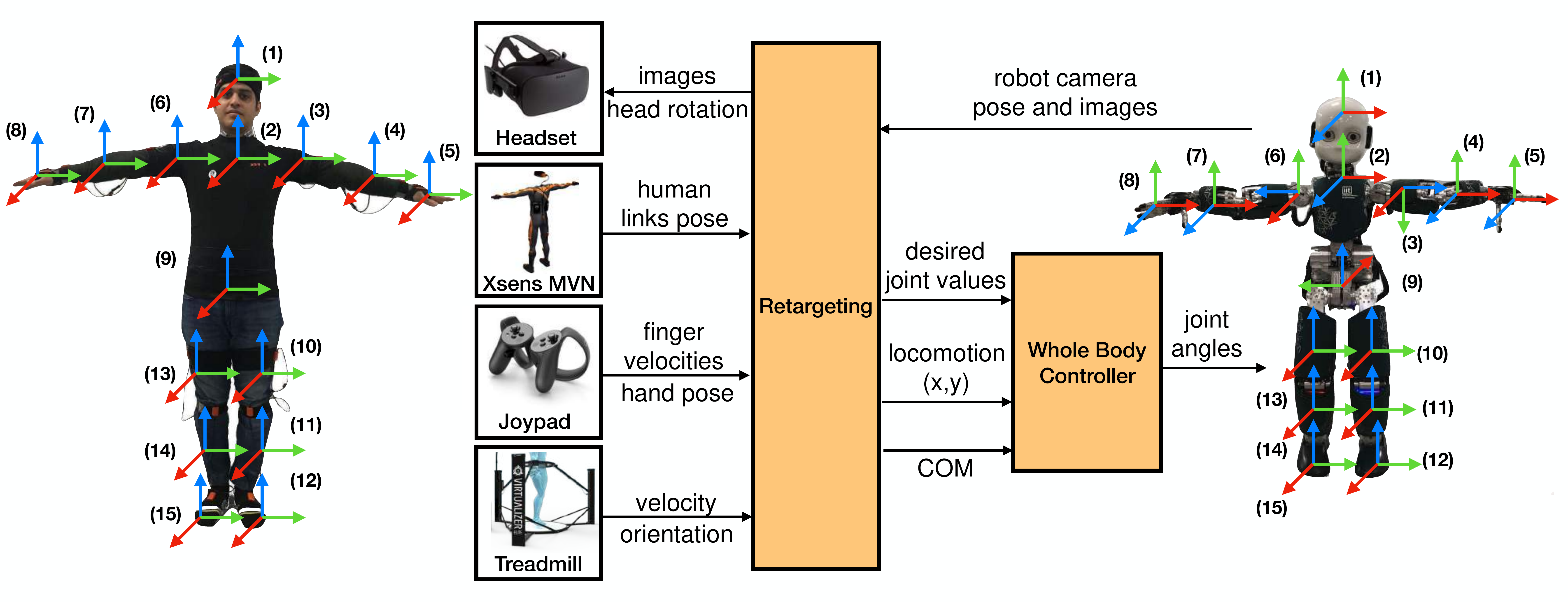}
	\caption{Whole-body teleoperation architecture with real-time human motion retargeting}
	\label{fig:architecture}
\end{figure}

\section{Whole-Body Teleoperation Experiments}
\label{sec:teleoperation-experiments-results}

Towards demonstrating the capabilities of our whole-body retargeting, we perform teleoperation experiments using two state-of-the-art \textit{whole-body controllers} for humanoid robots. The whole-body teleoperation experiments are carried with the 53 degrees of freedom iCub robot that is $\SI{104}{\centi \meter}$ tall \cite{natale2017icub}.
The controllers run at $\SI{100}{\hertz}$ while the retargeting application runs at $\SI{200}{\hertz}$.
The average walking speed of the robot is $\SI{0.23}{\meter \per \second}$.
Both the applications are run on a machine of $4$th generation Intel Core i7@$\SI{1.7}{\giga \hertz}$ with 8GB of RAM.

\subsection{Whole-Body Teleoperation with Balancing Controller}
\label{sec:balancing-controller}

Momentum-based control \cite{nava2016stability,herzog2014balancing} proved to be effective for maintaining the robot's stability by controlling the robot's momentum as the primary objective. Additionally, a \textit{postural task} projected into the nullspace of the primary task can be used for performing additional tasks like manipulation while ensuring the stability of the robot. The control problem is formulated as an optimization problem to achieve the two tasks while carefully monitoring and regulating the contact wrenches, considering the associated feasible domains by resorting to quadratic programming (\textsc{qp}) solvers.

We considered one such momentum-based balancing controller \cite{nava2016stability} and extended the postural task by giving the joint references from whole-body retargeting. The technologies involved in the teleoperation with the whole-body balancing controller are highlighted in Fig.~\ref{fig:balancing_updated}. The snapshots from the experiments of the whole-body retargeting with the balancing controller are shown in Fig.~\ref{fig:balancing-controller-snapshots}.

\begin{figure}[!ht]
	\centering
	\includegraphics[width=0.9\textwidth]{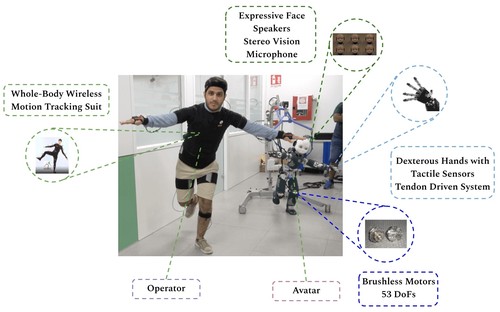}
	\caption{Technologies involved in teleoperation with whole-body balancing controller}
	\label{fig:balancing_updated}
\end{figure}

\begin{figure*}[!ht]
	\centering
	\begin{subfigure}[b]{0.45\textwidth}
	    \centering
		\includegraphics[width=\textwidth]{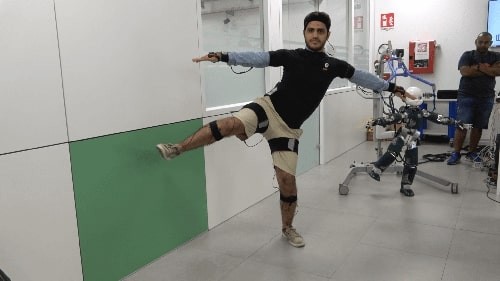}
		\caption{}
		\label{fig:balancing1} 
	\end{subfigure}
	~
	\begin{subfigure}[b]{0.45\textwidth}
	    \centering
		\includegraphics[width=\textwidth]{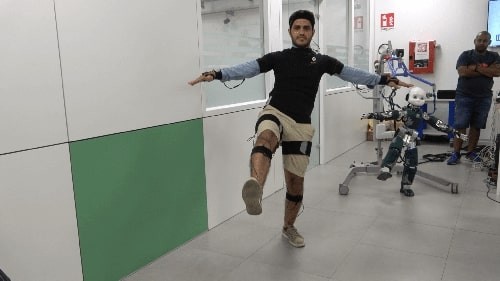}
		\caption{}
		\label{fig:balancing2}
	\end{subfigure}
	\begin{subfigure}[b]{0.45\textwidth}
	    \centering
		\includegraphics[width=\textwidth]{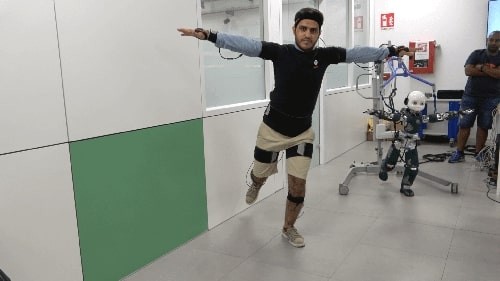}
		\caption{}
		\label{fig:balancing3}
	\end{subfigure}
	~
	\begin{subfigure}[b]{0.45\textwidth}
	    \centering
		\includegraphics[width=\textwidth]{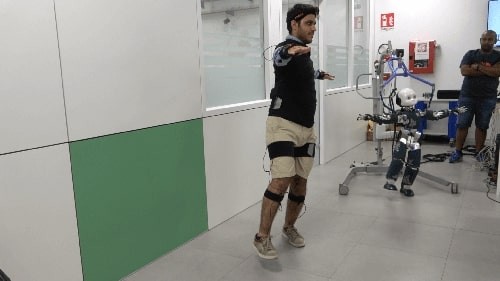}
		\caption{}
		\label{fig:balancing4}
	\end{subfigure}
	\caption{Snapshots of real-time motion retargeting with whole-body balancing controller at different time instances}
	\label{fig:balancing-controller-snapshots}
\end{figure*}

In this experiment, the robot is balancing on the left foot and maintaining the stability of its center of mass as shown in Fig.~\ref{fig:balancing_joint_tracking}. Additionally, it tracks all the joints with the references coming from whole-body retargeting.
The vertical dashed lines correspond to the experimental snapshots indicated in Fig.~\ref{fig:balancing-controller-snapshots}.
The references to the $x$ and $y$ components of the CoM are close to zero to maintain the stability of the robot by keeping the CoM inside the support polygon and the gains are tuned to achieve good tracking. The CoM motion along the $z$-axis does not affect the stability of the robot and the gain value of the $z$ components is kept lower in order to allow the vertical movements of the robot during retargeting. The input joint references from retargeting are smoothed through a minimum-jerk trajectory  \cite{Pattacini2010a}.

A smoothing time parameter is tuned in order to achieve a good balance between postural tracking and stability. Accordingly, the joints such as torso pitch, torso roll, and left knee for which the human does not move fast while balancing on left foot, achieve good tracking. On the other hand, the joints such as right shoulder pitch, right shoulder roll, and left ankle pitch are moved frequently while performing the retargeting and hence the tracking is not close owing to the delay from the smoothing time involved in producing minimum-jerk trajectory joint references for the robot joints. Ideally, the smoothing time can be kept lower considering that we receive continuous joint references from retargeting. At this point, we did not conduct exhaustive tests to find the lower threshold for the smoothing parameter that ensures fast and accurate retargeting of dynamic motions from the human while maintaining the robot's stability.

\begin{figure}[H]
	\centering
	\begin{subfigure}[b]{\textwidth}
		\centering
		\includegraphics[scale=0.275]{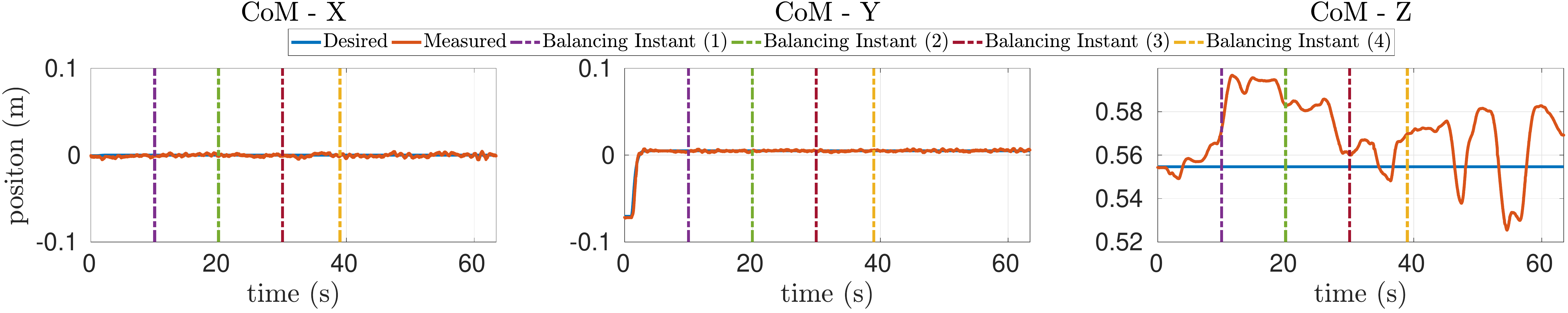}
		\label{fig:balancing_com_tracking}
	\end{subfigure}
	\vskip 0.1cm 
	\begin{subfigure}[b]{\textwidth}
		\centering
		\includegraphics[scale=0.275]{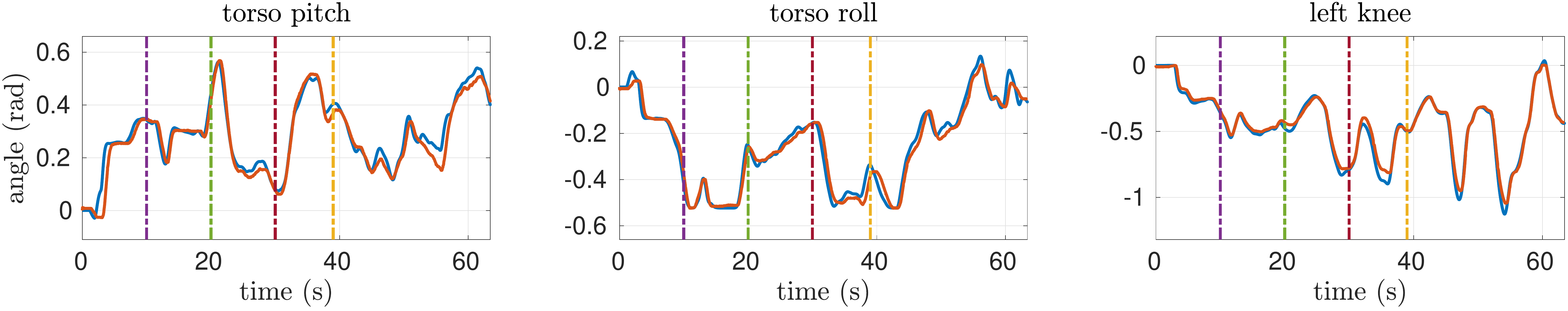}
		\label{fig:balancing_torso_tracking}
	\end{subfigure}
	\vskip 0.1cm 
	\begin{subfigure}[b]{\textwidth}
		\centering
		\includegraphics[scale=0.275]{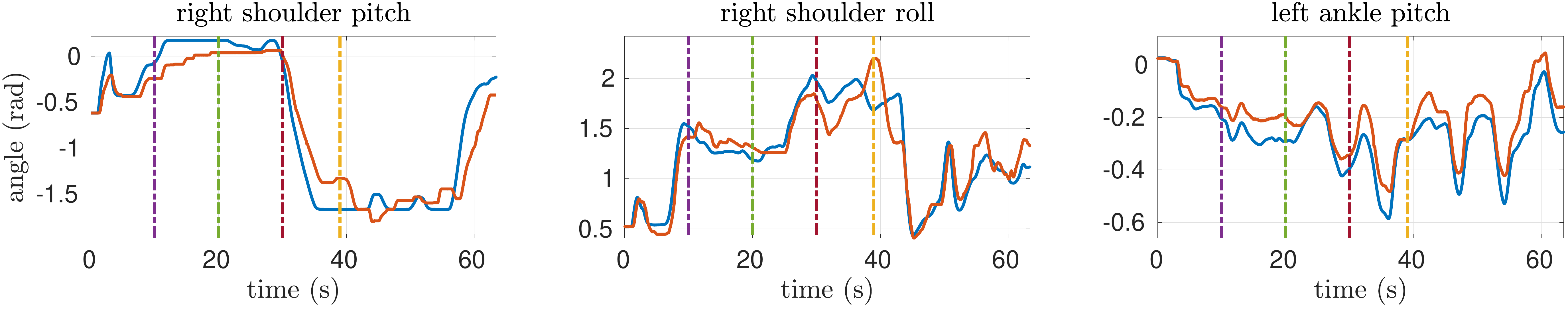}
		\label{fig:balancing_right_arm_tracking}
	\end{subfigure}
	\caption{Center of mass tracking and the time evolution of joint angles during real-time motion retargeting with whole-body balancing controller; blue line represents the desired quantity, orange line is the actual robot quantity. The vertical lines correspond to the time instances representing the snapshots from Fig.~\ref{fig:balancing-controller-snapshots} }
	\label{fig:balancing_joint_tracking}
\end{figure}

\newpage

\subsection{Whole-Body Teleoperation with Walking Controller}
\label{sec:walking-controller}

Humanoid robot walking is another challenging control paradigm. Divergent-Component-of-Motion (\textsc{dcm}) based control architectures proved promising for humanoid robot locomotion \cite{Romualdi2018, Englsberger2015}. The architecture typically consists of three layers: 1) Trajectory generation and optimization layer that generates the desired footsteps and the \textsc{dcm} trajectories \cite{Englsberger2015}; 2) Simplified model control layer that implements an \emph{instantaneous} control law with the objective of stabilizing the unstable \textsc{dcm} dynamics; and 3) Whole-body control layer that guarantees the tracking of the robot's set of tasks, including the Cartesian tasks and the postural tasks, using the stack-of-tasks paradigm implemented through a quadratic programming (\textsc{qp}) formalism.

We considered one such \textsc{dcm} based walking controller \cite{Romualdi2018} and extended the postural task by giving the joint references from whole-body retargeting. The technologies involved in the teleoperation with the whole-body walking controller are highlighted in Fig.~\ref{fig:locomotion_manipulation}.

\begin{figure}[H]
	\centering
	\includegraphics[width=\textwidth]{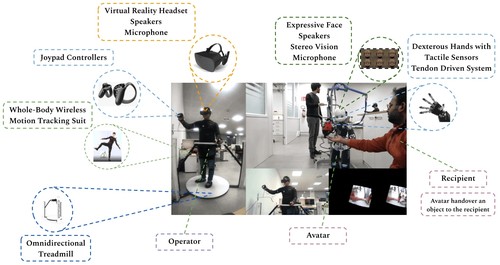}
	\caption{Technologies involved in teleoperation with the whole-body walking controller for locomotion and manipulation}
	\label{fig:locomotion_manipulation}
\end{figure}

The three different experimental stages of whole-body retargeting with the walking controller are shown in Fig.~\ref{fig:walking-controller-experimental-stages}. During the first and the third stages the robot is in double support standstill phase while during the second stage the robot is in walking phase.

\begin{figure}[H]
	\centering
	\begin{subfigure}[b]{0.55\textwidth}
	    \centering
		\includegraphics[width=\textwidth]{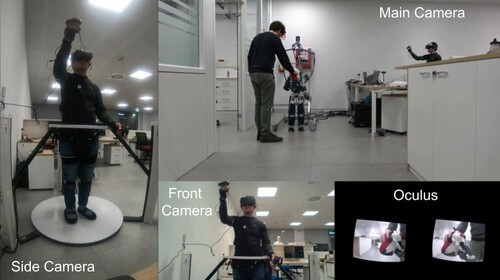}
		\caption{first stage}
		\label{fig:first-stage} 
	\end{subfigure}
	\begin{subfigure}[b]{0.55\textwidth}
	    \centering
		\includegraphics[width=\textwidth]{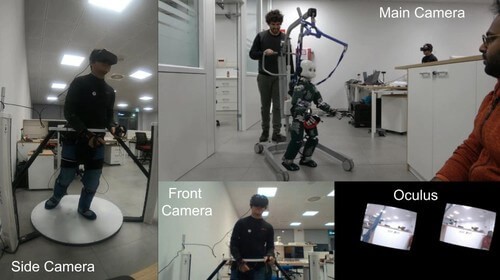}
		\caption{second stage}
		\label{fig:second-stage}
	\end{subfigure}
	\begin{subfigure}[b]{0.55\textwidth}
	    \centering
		\includegraphics[width=\textwidth]{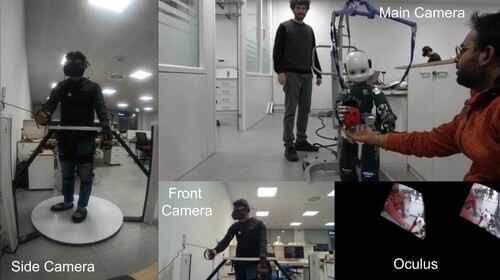}
		\caption{third stage}
		\label{fig:third-stage}
	\end{subfigure}
	\caption{Experimental stages during real-time retargeting with whole-body walking controller}
	\label{fig:walking-controller-experimental-stages}
\end{figure}

The walking controller's primary objective is to track the center of mass $x$ and $y$ components along the desired trajectory. The overall center of mass tracking of the $x$ and $y$ components is very good for the entire duration of the experiment as shown in Fig.~\ref{fig:walking_joint_tracking}.

Currently, we engage only the upper body retargeting, and the lower body is controlled by the walking controller. During our experiments, we observed that the weights for achieving satisfactory upper body retargeting of the postural task during the double support standstill phase and the walking phase are different. Having the same retargeting gains while walking leads to uncoordinated movements eventually compromising the robot's stability while walking. So, we choose higher retargeting gains during the double support standstill phase and the gain values are set to zero during the walking phase. The transition between the two sets of weights is achieved smoothly through minimum jerk trajectories \cite{Pattacini2010a}. Fig.~\ref{fig:walking_joint_tracking} highlights tracking for some of the upper-body joints. The blue line represents the desired joint position provided by human motion retargeting and the orange line is the actual robot joint position. The purple vertical dashed line indicates the starting instance of the \emph{second stage}, i.e., walking, and the green vertical dashed line indicates the stopping instance of the walking phase. During the \emph{first stage}, human motion retargeting is good and the joint position error is low. Instead, during the \emph{second stage}, as the robot starts walking the joint position error is higher as the retargeting gains are set to zero.

\begin{figure}[!t]
	\centering
	\begin{subfigure}[b]{\textwidth}
		\centering
		\includegraphics[scale=0.275]{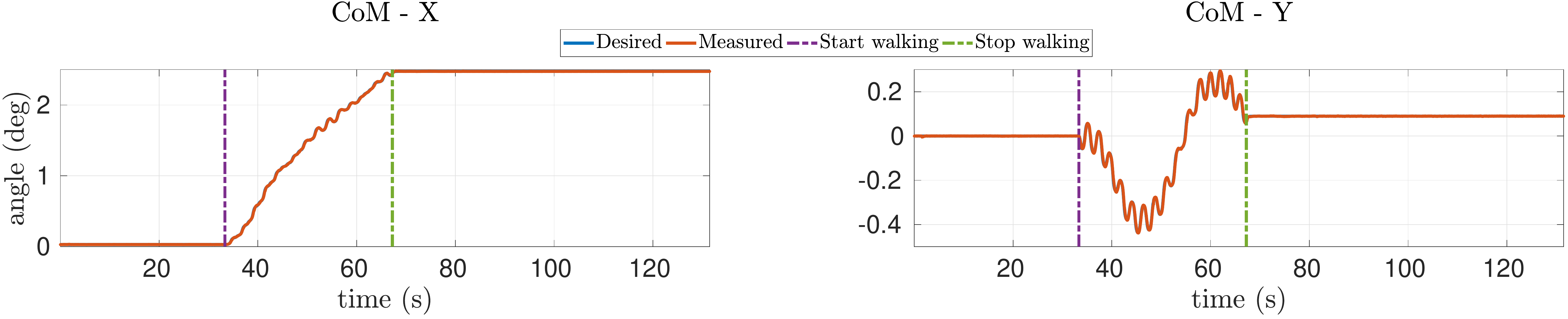}
		\label{fig:walking_com_tracking}
	\end{subfigure}
	\vskip 0.1cm
	\begin{subfigure}[b]{\textwidth}
		\centering
		\includegraphics[scale=0.275]{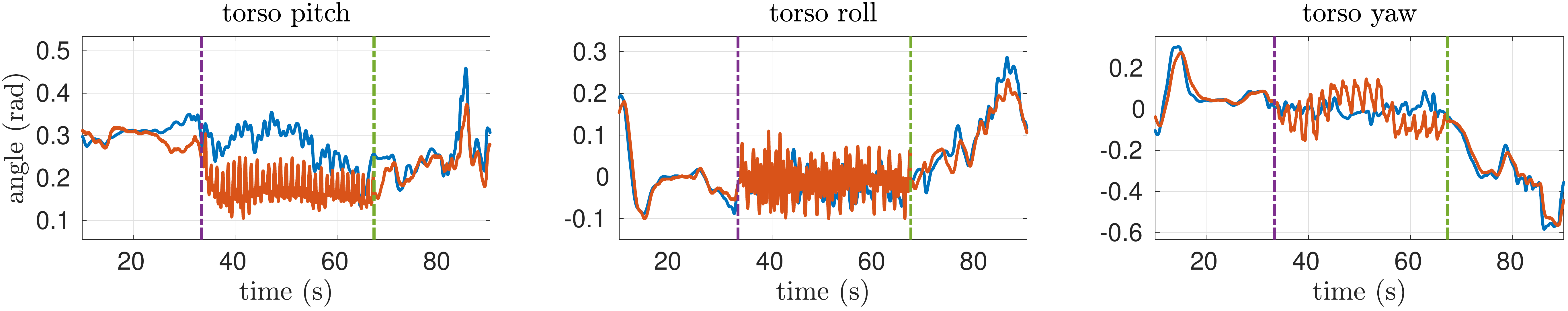}
		\label{fig:walking_torso_tracking}
	\end{subfigure}
	\vskip 0.1cm
	\begin{subfigure}[b]{\textwidth}
		\centering
		\includegraphics[scale=0.275]{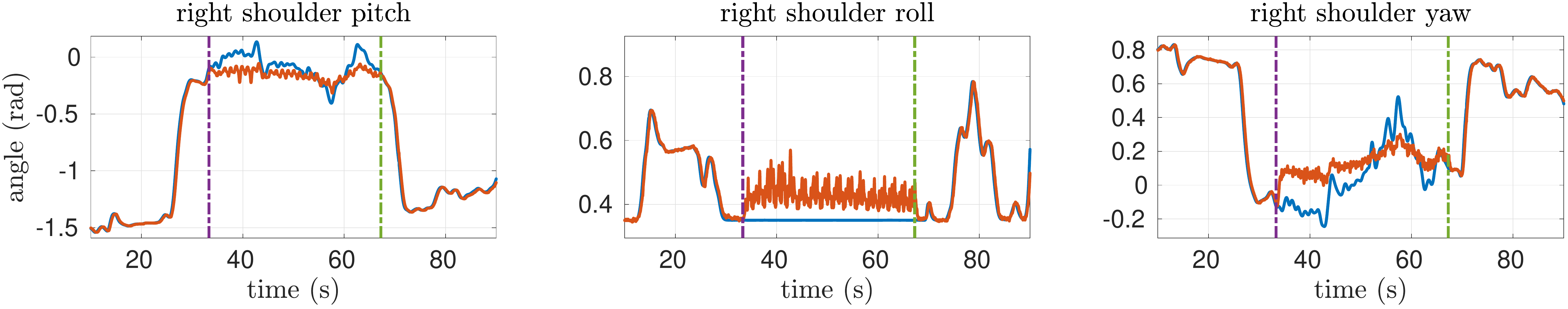}
		\label{fig:walking_right_shoulder_tracking}
	\end{subfigure}
	\caption{Center of mass tracking and the time evolution of joint angles during whole-body retargeting with walking controller}
	\label{fig:walking_joint_tracking}
\end{figure}

\bookmarksetup{startatroot}
\chapter*{Epilogue}
\addcontentsline{toc}{chapter}{Epilogue}

This thesis attempts at addressing some of the challenges for enabling human-robot collaboration. It presents a holistic human perception framework for real-time monitoring of whole-body human motion and dynamics. Furthermore, it attempts at methodically defining what constitutes an assistance from a human partner and propose partner-aware robot control strategies to endow robots with the capacity to meaningfully engage in a collaborative task.

The manuscript is divided into three main parts:

\begin{itemize}
	\item  The first part describes the role of human and robot partners in future socio-technical systems, followed by the mathematical preliminaries and background on rigid body systems. Furthermore, it recalls the importance of the human modeling concerning human-robot collaboration and presents the different enabling technologies that are employed in carrying out the research presented in this thesis.
	\item The second part focuses on the challenge of holistic human perception and presents promising approaches to real-time human motion tracking of a highly articulated human model. A novel sensorless external force estimation on the human links through a stochastic human dynamics estimation for a floating base human model is presented with experimental validation. Lastly, a modular and extensible software architecture towards realizing a holistic human perception framework is explained in detail. 
	\item The final part is dedicated to reactive robot control, where we attempt at methodically defining what constitutes an assistance from a human partner and propose partner-aware robot control strategies to endow robots with the capacity to meaningfully engage in a collaborative task. Furthermore, the concept of trajectory advancement for intuitive robot behavior leveraging interaction from a human partner is presented with experimental validation. Finally, our research towards sophisticated teleexistence setup is explained in detail followed by extensive experimental validation.
	
\end{itemize}
 
Before bringing down the curtains on this manuscript, we would like to reflect on some of the open questions and challenges that will guide future research aspirations. Firstly, inclusion of prediction and classification components in our human perception framework will facilitate new application scenarios. Concerning the aspect of human modeling, currently we precompute the inertial parameters of the human segments based on anthropometric tables and model it explicitly based on the measured parameters of the human subject. This approach however is not easily scalable and adaptable for different population~\cite{matrangola2008changes}. A promising research in this direction is the identification of body segment inertial parameters~\cite{hansen2014individual,Venture2019}. Automatic whole-body human model generation that is adaptable for multiple applications is still an open challenge~\cite{Venture2019} and there is a need for bridging the gaps between digital human models employed in different fields. Furthermore, higher level planning for task sharing between the human partner and the robot partner~\cite{raessa2019human} is an another interesting and potential research area that can facilitate new insights.

Concerning the robot control, the investigations presented in this thesis primarily focus on how the human partner assistance is leveraged by the robot partner for instantaneous control. On the other side of this question is how the robot partner can assist the human partner through holistic human perception and predictive control. Possible research endeavors in this direction are to build whole-body robot controllers that monitor the human joint torques and improve human ergonomics while performing complicated tasks such as collaborative transportation. Furthermore, the concept of trajectory advancement can be applied to exoskeleton control, where the physical interaction from the human guides the robotic exoskeleton to assist. Finally, adding a planning component will extend the applicability to collaborative scenarios that need deliberation between human and robot partners.

\begin{spacing}{0.9}

\bibliographystyle{apalike}
\cleardoublepage
\bibliography{thesis} %

\begin{thebibliography}{}

\bibitem[Aggarwal and Cai, 1999]{aggarwal1999human}
Aggarwal, J.~K. and Cai, Q. (1999).
\newblock Human motion analysis: A review.
\newblock {\em Computer vision and image understanding}, 73(3):428--440.

\bibitem[Agravante et~al., 2014]{agravante2014collaborative}
Agravante, D.~J., Cherubini, A., Bussy, A., Gergondet, P., and Kheddar, A.
  (2014).
\newblock Collaborative human-humanoid carrying using vision and haptic
  sensing.
\newblock In {\em Robotics and Automation (ICRA), 2014 IEEE International
  Conference on}, pages 607--612. IEEE.

\bibitem[Aguiar et~al., 2004]{aguiar2004path}
Aguiar, A.~P., Da{\v{c}}i{\'c}, D.~B., Hespanha, J.~P., and Kokotovi{\'c}, P.
  (2004).
\newblock Path-following or reference tracking?: An answer relaxing the limits
  to performance.
\newblock {\em IFAC Proceedings Volumes}, 37(8):167--172.

\bibitem[Ajoudani et~al., 2018]{ajoudani2018progress}
Ajoudani, A., Zanchettin, A.~M., Ivaldi, S., Albu-Sch{\"a}ffer, A., Kosuge, K.,
  and Khatib, O. (2018).
\newblock Progress and prospects of the human--robot collaboration.
\newblock {\em Autonomous Robots}, 42(5):957--975.

\bibitem[Alami, 2013]{alami2013human}
Alami, R. (2013).
\newblock On human models for collaborative robots.
\newblock In {\em 2013 International Conference on Collaboration Technologies
  and Systems (CTS)}, pages 191--194. IEEE.

\bibitem[Alami et~al., 2006]{alami2006safe}
Alami, R., Albu-Sch{\"a}ffer, A., Bicchi, A., Bischoff, R., Chatila, R.,
  De~Luca, A., De~Santis, A., Giralt, G., Guiochet, J., Hirzinger, G., et~al.
  (2006).
\newblock Safe and dependable physical human-robot interaction in anthropic
  domains: State of the art and challenges.
\newblock In {\em 2006 IEEE/RSJ International Conference on Intelligent Robots
  and Systems}, pages 1--16. IEEE.

\bibitem[Alexopoulos et~al., 2013]{Alexopoulos2013ErgoToolkitAE}
Alexopoulos, K., Mavrikios, D., and Chryssolouris, G. (2013).
\newblock Ergotoolkit: an ergonomic analysis tool in a virtual manufacturing
  environment.
\newblock {\em Int. J. Computer Integrated Manufacturing}, 26:440--452.

\bibitem[Andrade~Chavez et~al., 2019]{andrade2019six}
Andrade~Chavez, F.~J., Traversaro, S., and Pucci, D. (2019).
\newblock Six-axis force torque sensor model-based in situ calibration method
  and its impact in floating-based robot dynamic performance.
\newblock {\em Sensors}, 19(24):5521.

\bibitem[An.Dy, 2017]{andy}
An.Dy (2017).
\newblock Advancing anticipatory behaviors in dyadic human-robot collaboration,
  h2020 project.

\bibitem[Aristidou et~al., 2018]{aristidou2018inverse}
Aristidou, A., Lasenby, J., Chrysanthou, Y., and Shamir, A. (2018).
\newblock Inverse kinematics techniques in computer graphics: A survey.
\newblock In {\em Computer Graphics Forum}, volume~37, pages 35--58. Wiley
  Online Library.

\bibitem[Au et~al., 2008]{au2008powered}
Au, S., Berniker, M., and Herr, H. (2008).
\newblock Powered ankle-foot prosthesis to assist level-ground and
  stair-descent gaits.
\newblock {\em Neural Networks}, 21(4):654--666.

\bibitem[Ayusawa and Yoshida, 2017]{Ayusawa2017}
Ayusawa, K. and Yoshida, E. (2017).
\newblock Motion retargeting for humanoid robots based on simultaneous morphing
  parameter identification and motion optimization.
\newblock {\em IEEE Transactions on Robotics}, 33(6):1343--1357.

\bibitem[Bartneck et~al., 2009]{bartneck2009measurement}
Bartneck, C., Kuli{\'c}, D., Croft, E., and Zoghbi, S. (2009).
\newblock Measurement instruments for the anthropomorphism, animacy,
  likeability, perceived intelligence, and perceived safety of robots.
\newblock {\em International journal of social robotics}, 1(1):71--81.

\bibitem[Beer et~al., 2014]{Beer2014}
Beer, J.~M., Fisk, A.~D., and Rogers, W.~A. (2014).
\newblock Toward a framework for levels of robot autonomy in human-robot
  interaction.
\newblock {\em Journal of Human-Robot Interaction}, 3(2):74--99.

\bibitem[Belhassein et~al., 2019]{belhassein2019towards}
Belhassein, K., Buisan, G., Clodic, A., and Alami, R. (2019).
\newblock Towards methodological principles for user studies in human-robot
  interaction.

\bibitem[{Bharatkumar} et~al., 1994]{Bharatkumar1994}
{Bharatkumar}, A.~G., {Daigle}, K.~E., {Pandy}, M.~G., {Qin Cai}, and
  {Aggarwal}, J.~K. (1994).
\newblock Lower limb kinematics of human walking with the medial axis
  transformation.
\newblock In {\em Proceedings of 1994 IEEE Workshop on Motion of Non-rigid and
  Articulated Objects}, pages 70--76.

\bibitem[Blender, 2019]{Blender}
Blender (2019).
\newblock {\em Blender - a 3D modelling and rendering package}.
\newblock Blender Foundation, Stichting Blender Foundation, Amsterdam.

\bibitem[Bragan{\c{c}}a et~al., 2019]{Braganca2019}
Bragan{\c{c}}a, S., Costa, E., Castellucci, I., and Arezes, P.~M. (2019).
\newblock {\em A Brief Overview of the Use of Collaborative Robots in Industry
  4.0: Human Role and Safety}, pages 641--650.
\newblock Springer International Publishing, Cham.

\bibitem[Breivik and Fossen, 2005]{breivik2005principles}
Breivik, M. and Fossen, T.~I. (2005).
\newblock Principles of guidance-based path following in 2d and 3d.
\newblock In {\em Proceedings of the 44th IEEE Conference on Decision and
  Control}, pages 627--634. IEEE.

\bibitem[Burgner-Kahrs et~al., 2015]{Burgner2015}
Burgner-Kahrs, J., Rucker, D.~C., and Choset, H. (2015).
\newblock Continuum robots for medical applications: A survey.
\newblock {\em IEEE Transactions on Robotics}, 31(6):1261--1280.

\bibitem[Buss, 2004]{buss2004}
Buss, S.~R. (2004).
\newblock Introduction to inverse kinematics with jacobian transpose,
  pseudoinverse and damped least squares methods.
\newblock {\em IEEE Journal of Robotics and Automation}, 17(1-19):16.

\bibitem[Bussy et~al., 2012a]{bussy2012proactive}
Bussy, A., Gergondet, P., Kheddar, A., Keith, F., and Crosnier, A. (2012a).
\newblock Proactive behavior of a humanoid robot in a haptic transportation
  task with a human partner.
\newblock In {\em RO-MAN, 2012 IEEE}, pages 962--967. IEEE.

\bibitem[Bussy et~al., 2012b]{bussy2012human}
Bussy, A., Kheddar, A., Crosnier, A., and Keith, F. (2012b).
\newblock Human-humanoid haptic joint object transportation case study.
\newblock In {\em Intelligent Robots and Systems (IROS), 2012 IEEE/RSJ
  International Conference on}, pages 3633--3638. IEEE.

\bibitem[Caux et~al., 1998]{caux1998balance}
Caux, S., Mateo, E., and Zapata, R. (1998).
\newblock Balance of biped robots: special double-inverted pendulum.
\newblock In {\em Systems, Man, and Cybernetics, 1998. 1998 IEEE International
  Conference on}, volume~4, pages 3691--3696. IEEE.

\bibitem[Chavez et~al., 2016]{chavez2016model}
Chavez, F. J.~A., Traversaro, S., Pucci, D., and Nori, F. (2016).
\newblock Model based in situ calibration of six axis force torque sensors.
\newblock In {\em 2016 IEEE-RAS 16th International Conference on Humanoid
  Robots (Humanoids)}, pages 422--427. IEEE.

\bibitem[Cho et~al., 2016]{cho2016force}
Cho, E., Chen, R., Merhi, L.-K., Xiao, Z., Pousett, B., and Menon, C. (2016).
\newblock Force myography to control robotic upper extremity prostheses: a
  feasibility study.
\newblock {\em Frontiers in bioengineering and biotechnology}, 4:18.

\bibitem[{Chopra} et~al., 2003]{Chopra2003}
{Chopra}, N., {Spong}, M.~W., {Hirche}, S., and {Buss}, M. (2003).
\newblock Bilateral teleoperation over the internet: the time varying delay
  problem.
\newblock In {\em Proceedings of the 2003 American Control Conference},
  volume~1, pages 155--160, Denver, CO, USA.

\bibitem[Craig, 2009]{craig2009introduction}
Craig, J.~J. (2009).
\newblock {\em Introduction to robotics: mechanics and control, 3/E}.
\newblock Pearson Education India.

\bibitem[Dafarra et~al., 2016]{dafarra2016torque}
Dafarra, S., Romano, F., and Nori, F. (2016).
\newblock Torque-controlled stepping-strategy push recovery: Design and
  implementation on the icub humanoid robot.
\newblock In {\em 2016 IEEE-RAS 16th International Conference on Humanoid
  Robots (Humanoids)}, pages 152--157. IEEE.

\bibitem[Damiano et~al., 2015]{damiano2015towards}
Damiano, L., Dumouchel, P., and Lehmann, H. (2015).
\newblock Towards human--robot affective co-evolution overcoming oppositions in
  constructing emotions and empathy.
\newblock {\em International Journal of Social Robotics}, 7(1):7--18.

\bibitem[De~Santis et~al., 2007]{de2007human}
De~Santis, A., Lippiello, V., Siciliano, B., and Villani, L. (2007).
\newblock Human-robot interaction control using force and vision.
\newblock In {\em Advances in Control Theory and Applications}, pages 51--70.
  Springer.

\bibitem[Delp et~al., 2007]{delp2007opensim}
Delp, S.~L., Anderson, F.~C., Arnold, A.~S., Loan, P., Habib, A., John, C.~T.,
  Guendelman, E., and Thelen, D.~G. (2007).
\newblock Opensim: open-source software to create and analyze dynamic
  simulations of movement.
\newblock {\em IEEE transactions on biomedical engineering}, 54(11):1940--1950.

\bibitem[Demirel and Duffy, 2007]{demirel2007applications}
Demirel, H.~O. and Duffy, V.~G. (2007).
\newblock Applications of digital human modeling in industry.
\newblock In {\em International Conference on Digital Human Modeling}, pages
  824--832. Springer.

\bibitem[Denavit and Hartenberg, 1955]{Denavit1955}
Denavit, J. and Hartenberg, R.~S. (1955).
\newblock A kinematic notation for lower-pair mechanisms based on matrices.
\newblock {\em Trans. of the ASME. Journal of Applied Mechanics}, 22:215--221.

\bibitem[Donner and Buss, 2016]{donner2016cooperative}
Donner, P. and Buss, M. (2016).
\newblock Cooperative swinging of complex pendulum-like objects: Experimental
  evaluation.
\newblock {\em IEEE Transactions on Robotics}, 32(3):744--753.

\bibitem[Duffy, 2016]{duffy2016handbook}
Duffy, V.~G. (2016).
\newblock {\em Handbook of digital human modeling: research for applied
  ergonomics and human factors engineering}.
\newblock CRC press.

\bibitem[Elobaid et~al., 2019]{Elobaid2018}
Elobaid, M., Hu, Y., Romualdi, G., Dafarra, S., Babic, J., and Pucci, D.
  (2019).
\newblock Telexistence and teleoperation for walking humanoid robots.
\newblock In {\em Proceedings of SAI Intelligent Systems Conference}, pages
  1106--1121. Springer.

\bibitem[Endo et~al., 2014]{endo2014dhaiba}
Endo, Y., Tada, M., and Mochimaru, M. (2014).
\newblock Dhaiba: development of virtual ergonomic assessment system with human
  models.
\newblock In {\em Proceedings of The 3rd International Digital Human
  Symposium}.

\bibitem[Englsberger et~al., 2015]{Englsberger2015}
Englsberger, J., Ott, C., and Albu-Sch{\"{a}}ffer, A. (2015).
\newblock {Three-Dimensional Bipedal Walking Control Based on Divergent
  Component of Motion}.
\newblock {\em IEEE Transactions on Robotics}, 31(2):355--368.

\bibitem[Featherstone, 2007]{Featherstone2007}
Featherstone, R. (2007).
\newblock {\em Rigid Body Dynamics Algorithms}.
\newblock Springer-Verlag New York, Inc., Secaucus, NJ, USA.

\bibitem[Feng et~al., 2015]{Feng2015}
Feng, S., Xinjilefu, X., Atkeson, C., and Kim, J. (2015).
\newblock Optimization based controller design and implementation for the atlas
  robot in the darpa robotics challenge finals.
\newblock In {\em Humanoid Robots (Humanoids), 2015 IEEE-RAS 15th International
  Conference on}, pages 1028--1035.

\bibitem[Ferreau et~al., 2014]{Ferreau2014}
Ferreau, H., Kirches, C., Potschka, A., Bock, H., and Diehl, M. (2014).
\newblock {qpOASES}: A parametric active-set algorithm for quadratic
  programming.
\newblock {\em Mathematical Programming Computation}, 6(4):327--363.

\bibitem[Fitzpatrick et~al., 2014]{paul2014middle}
Fitzpatrick, P., Elena, C., Daniele, D., Ali, P., Giorgio, M., and Lorenzo, C.
  (2014).
\newblock A middle way for robotics middleware.

\bibitem[Frey and Osborne, 2017]{frey2017future}
Frey, C.~B. and Osborne, M.~A. (2017).
\newblock The future of employment: How susceptible are jobs to
  computerisation?
\newblock {\em Technological forecasting and social change}, 114:254--280.

\bibitem[Fritzsche, 2010]{fritzsche2010ergonomics}
Fritzsche, L. (2010).
\newblock Ergonomics risk assessment with digital human models in car assembly:
  Simulation versus real life.
\newblock {\em Human Factors and Ergonomics in Manufacturing \& Service
  Industries}, 20(4):287--299.

\bibitem[Galin and Meshcheryakov, 2019]{galin2019review}
Galin, R. and Meshcheryakov, R. (2019).
\newblock Review on human--robot interaction during collaboration in a shared
  workspace.
\newblock In {\em International Conference on Interactive Collaborative
  Robotics}, pages 63--74. Springer.

\bibitem[Gall et~al., 2009]{gall2009motion}
Gall, J., Stoll, C., De~Aguiar, E., Theobalt, C., Rosenhahn, B., and Seidel,
  H.-P. (2009).
\newblock Motion capture using joint skeleton tracking and surface estimation.
\newblock In {\em 2009 IEEE Conference on Computer Vision and Pattern
  Recognition}, pages 1746--1753. IEEE.

\bibitem[Gasparetto and Scalera, 2019]{gasparetto2019unimate}
Gasparetto, A. and Scalera, L. (2019).
\newblock From the unimate to the delta robot: the early decades of industrial
  robotics.
\newblock In {\em Explorations in the History and Heritage of Machines and
  Mechanisms}, pages 284--295. Springer.

\bibitem[Geravand et~al., 2013]{geravand2013human}
Geravand, M., Flacco, F., and De~Luca, A. (2013).
\newblock Human-robot physical interaction and collaboration using an
  industrial robot with a closed control architecture.
\newblock In {\em 2013 IEEE International Conference on Robotics and
  Automation}, pages 4000--4007. IEEE.

\bibitem[{Goldenberg} et~al., 1985]{goldenberg1985}
{Goldenberg}, A., {Benhabib}, B., and {Fenton}, R. (1985).
\newblock A complete generalized solution to the inverse kinematics of robots.
\newblock {\em IEEE Journal on Robotics and Automation}, 1(1):14--20.

\bibitem[Goodrich et~al., 2008]{goodrich2008human}
Goodrich, M.~A., Schultz, A.~C., et~al. (2008).
\newblock Human--robot interaction: a survey.
\newblock {\em Foundations and Trends{\textregistered} in Human--Computer
  Interaction}, 1(3):203--275.

\bibitem[Gopinathan et~al., 2017]{gopinathan2017user}
Gopinathan, S., Otting, S., and Steil, J. (2017).
\newblock A user study on personalized adaptive stiffness control modes for
  human-robot interaction.
\newblock In {\em 2017 26th IEEE International Symposium on Robot and Human
  Interactive Communication (RO-MAN)}, pages 831--837. IEEE.

\bibitem[Gribovskaya et~al., 2011]{gribovskaya2011motion}
Gribovskaya, E., Kheddar, A., and Billard, A. (2011).
\newblock Motion learning and adaptive impedance for robot control during
  physical interaction with humans.
\newblock In {\em 2011 IEEE International Conference on Robotics and
  Automation}, pages 4326--4332. IEEE.

\bibitem[Guo et~al., 2020]{GUO2020106544}
Guo, A., Ma, J., Sun, G., and Tan, S. (2020).
\newblock A personal character model of affect, behavior and cognition for
  individual-like research.
\newblock {\em Computers \& Electrical Engineering}, 81:106544.

\bibitem[H-Anim, 2006]{HAnim}
H-Anim (2006).
\newblock Human animation working group.
\newblock {\em http://www.h-anim.org/}.

\bibitem[Haddadin, 2015]{Haddadin2015}
Haddadin, S. (2015).
\newblock {\em Physical Safety in Robotics}, pages 249--271.
\newblock Springer Fachmedien Wiesbaden, Wiesbaden.

\bibitem[Haddadin et~al., 2017]{haddadin2017robot}
Haddadin, S., De~Luca, A., and Albu-Sch{\"a}ffer, A. (2017).
\newblock Robot collisions: A survey on detection, isolation, and
  identification.
\newblock {\em IEEE Transactions on Robotics}, 33(6):1292--1312.

\bibitem[Hanavan, 1964]{Hanavan1964}
Hanavan, E.~P. (1964).
\newblock A mathematical model of human body.
\newblock Technical report, Air force aerospace medical research lab
  Wright-Patterson AFB OH.

\bibitem[Hansen et~al., 2014]{hansen2014individual}
Hansen, C., Venture, G., Rezzoug, N., Gorce, P., and Isableu, B. (2014).
\newblock An individual and dynamic body segment inertial parameter validation
  method using ground reaction forces.
\newblock {\em Journal of biomechanics}, 47(7):1577--1581.

\bibitem[Herzog et~al., 2014]{herzog2014balancing}
Herzog, A., Righetti, L., Grimminger, F., Pastor, P., and Schaal, S. (2014).
\newblock Balancing experiments on a torque-controlled humanoid with
  hierarchical inverse dynamics.
\newblock In {\em Intelligent Robots and Systems (IROS 2014), 2014 IEEE/RSJ
  International Conference on}, pages 981--988. IEEE.

\bibitem[Hirai et~al., 1998]{hirai1998development}
Hirai, K., Hirose, M., Haikawa, Y., and Takenaka, T. (1998).
\newblock The development of honda humanoid robot.
\newblock In {\em Robotics and Automation, 1998. Proceedings. 1998 IEEE
  International Conference on}, volume~2, pages 1321--1326. IEEE.

\bibitem[Hoffman et~al., 2014]{hoffman2014yarp}
Hoffman, E.~M., Traversaro, S., Rocchi, A., Ferrati, M., Settimi, A., Romano,
  F., Natale, L., Bicchi, A., Nori, F., and Tsagarakis, N.~G. (2014).
\newblock Yarp based plugins for gazebo simulator.
\newblock In {\em International Workshop on Modelling and Simulation for
  Autonomous Systems}, pages 333--346. Springer.

\bibitem[Hofmann et~al., 2009]{hofmann2009exploiting}
Hofmann, A., Popovic, M., and Herr, H. (2009).
\newblock Exploiting angular momentum to enhance bipedal center-of-mass
  control.
\newblock In {\em Robotics and Automation, 2009. ICRA'09. IEEE International
  Conference on}, pages 4423--4429. IEEE.

\bibitem[Hogan, 1984]{hogan1984impedance}
Hogan, N. (1984).
\newblock Impedance control: An approach to manipulation.
\newblock In {\em 1984 American control conference}, pages 304--313. IEEE.

\bibitem[Hokayem and Spong, 2006]{Hokayem2006}
Hokayem, P.~F. and Spong, M.~W. (2006).
\newblock Bilateral teleoperation: An historical survey.
\newblock {\em Automatica}, 42(12):2035--2057.

\bibitem[Hyon et~al., 2007]{hyon2007full}
Hyon, S.-H., Hale, J.~G., Cheng, G., et~al. (2007).
\newblock Full-body compliant human-humanoid interaction: Balancing in the
  presence of unknown external forces.
\newblock {\em IEEE Trans. Robotics}, 23(5):884--898.

\bibitem[Ikemoto et~al., 2012]{ikemoto2012physical}
Ikemoto, S., Amor, H.~B., Minato, T., Jung, B., and Ishiguro, H. (2012).
\newblock Physical human-robot interaction: Mutual learning and adaptation.
\newblock {\em IEEE robotics \& automation magazine}, 19(4):24--35.

\bibitem[Imamura et~al., 2018]{imamura2018evaluation}
Imamura, Y., Ayusawa, K., Yoshida, E., and Tanaka, T. (2018).
\newblock Evaluation framework for passive assistive device based on humanoid
  experiments.
\newblock {\em International Journal of Humanoid Robotics}, 15(03):1750026.

\bibitem[Irshad et~al., 2019]{irshad2019coupling}
Irshad, L., Ahmed, S., Demirel, O., and Tumer, I.~Y. (2019).
\newblock Coupling digital human modeling with early design stage human error
  analysis to assess ergonomic vulnerabilities.
\newblock In {\em AIAA Scitech 2019 Forum}, page 2349.

\bibitem[{Ishiguro} et~al., 2017]{Ishiguro2017}
{Ishiguro}, Y., {Kojima}, K., {Sugai}, F., {Nozawa}, S., {Kakiuchi}, Y.,
  {Okada}, K., and {Inaba}, M. (2017).
\newblock Bipedal oriented whole body master-slave system for dynamic secured
  locomotion with lip safety constraints.
\newblock In {\em 2017 IEEE/RSJ International Conference on Intelligent Robots
  and Systems (IROS)}, pages 376--382.

\bibitem[{Ishiguro} et~al., 2018]{Ishiguro2018}
{Ishiguro}, Y., {Kojima}, K., {Sugai}, F., {Nozawa}, S., {Kakiuchi}, Y.,
  {Okada}, K., and {Inaba}, M. (2018).
\newblock High speed whole body dynamic motion experiment with real time
  master-slave humanoid robot system.
\newblock In {\em 2018 IEEE International Conference on Robotics and Automation
  (ICRA)}, pages 1--7.

\bibitem[{Islam} et~al., 2015]{Islam2015}
{Islam}, S., {Liu}, P.~X., {Saddik}, A.~E., and {Yang}, Y.~B. (2015).
\newblock Bilateral control of teleoperation systems with time delay.
\newblock {\em IEEE/ASME Transactions on Mechatronics}, 20(1):1--12.

\bibitem[Ito et~al., 2018]{ito2018evaluation}
Ito, T., Ayusawa, K., Yoshida, E., and Kobayashi, H. (2018).
\newblock Evaluation of active wearable assistive devices with human posture
  reproduction using a humanoid robot.
\newblock {\em Advanced Robotics}, 32(12):635--645.

\bibitem[Karatsidis et~al., 2017]{karatsidis2017estimation}
Karatsidis, A., Bellusci, G., Schepers, H., de~Zee, M., Andersen, M., and
  Veltink, P. (2017).
\newblock Estimation of ground reaction forces and moments during gait using
  only inertial motion capture.
\newblock {\em Sensors}, 17(1):75.

\bibitem[Khalil and Dombre, 2004]{khalil2004modeling}
Khalil, W. and Dombre, E. (2004).
\newblock {\em Modeling, identification and control of robots}.
\newblock Butterworth-Heinemann.

\bibitem[Koolen et~al., 2016]{koolen2016design}
Koolen, T., Bertrand, S., Thomas, G., De~Boer, T., Wu, T., Smith, J.,
  Englsberger, J., and Pratt, J. (2016).
\newblock Design of a momentum-based control framework and application to the
  humanoid robot atlas.
\newblock {\em International Journal of Humanoid Robotics}, 13(01):1650007.

\bibitem[Krotkov et~al., 2017]{Krotkov:2017:DRC:3074644.3074647}
Krotkov, E., Hackett, D., Jackel, L., Perschbacher, M., Pippine, J., Strauss,
  J., Pratt, G., and Orlowski, C. (2017).
\newblock The darpa robotics challenge finals: Results and perspectives.
\newblock {\em J. Field Robot.}, 34(2):229--240.

\bibitem[Kyriakopoulos and Saridis, 1988]{kyriakopoulos1988minimum}
Kyriakopoulos, K.~J. and Saridis, G.~N. (1988).
\newblock Minimum jerk path generation.
\newblock In {\em Proceedings. 1988 IEEE International Conference on Robotics
  and Automation}, pages 364--369. IEEE.

\bibitem[Kyrkjeb{\o}, 2018]{kyrkjebo2018inertial}
Kyrkjeb{\o}, E. (2018).
\newblock Inertial human motion estimation for physical human-robot interaction
  using an interaction velocity update to reduce drift.
\newblock In {\em Companion of the 2018 ACM/IEEE International Conference on
  Human-Robot Interaction}, pages 163--164. ACM.

\bibitem[Latella, 2018]{latella2018thesis}
Latella, C. (2018).
\newblock {\em Human Whole-Body Dynamics Estimation for Enhancing Physical
  Human-Robot Interaction}.
\newblock PhD thesis, University of Genoa.

\bibitem[Latella et~al., 2019a]{Latella2019}
Latella, C., Lorenzini, M., Lazzaroni, M., Romano, F., Traversaro, S., Akhras,
  M.~A., Pucci, D., and Nori, F. (2019a).
\newblock Towards real-time whole-body human dynamics estimation through
  probabilistic sensor fusion algorithms.
\newblock {\em Autonomous Robots}, 43(6):1591--1603.

\bibitem[Latella et~al., 2019b]{latella2019simultaneous}
Latella, C., Traversaro, S., Ferigo, D., Tirupachuri, Y., Rapetti, L.,
  Andrade~Chavez, F.~J., Nori, F., and Pucci, D. (2019b).
\newblock Simultaneous floating-base estimation of human kinematics and joint
  torques.
\newblock {\em Sensors}, 19(12):2794.

\bibitem[Lecours et~al., 2012]{lecours2012variable}
Lecours, A., Mayer-St-Onge, B., and Gosselin, C. (2012).
\newblock Variable admittance control of a four-degree-of-freedom intelligent
  assist device.
\newblock In {\em 2012 IEEE International Conference on Robotics and
  Automation}, pages 3903--3908. IEEE.

\bibitem[{Leung} and {Yee-Hong Yang}, 1995]{Leung1995}
{Leung}, M.~K. and {Yee-Hong Yang} (1995).
\newblock First sight: A human body outline labeling system.
\newblock {\em IEEE Transactions on Pattern Analysis and Machine Intelligence},
  17(4):359--377.

\bibitem[Li and Ge, 2013]{li2013impedance}
Li, Y. and Ge, S.~S. (2013).
\newblock Impedance learning for robots interacting with unknown environments.
\newblock {\em IEEE Transactions on Control Systems Technology},
  22(4):1422--1432.

\bibitem[Liarokapis et~al., 2013]{Liarokapis2013}
Liarokapis, M.~V., Artemiadis, P., Bechlioulis, C., and Kyriakopoulos, K.
  (2013).
\newblock Directions, methods and metrics for mapping human to robot motion
  with functional anthropomorphism: A review.
\newblock {\em School of Mechanical Engineering, National Technical University
  of Athens, Tech. Rep}.

\bibitem[Lichiardopol, 2007]{Lichiardopol2007}
Lichiardopol, S. (2007).
\newblock A survey on teleoperation.
\newblock {\em Technische Universitat Eindhoven, DCT report}.

\bibitem[Losey et~al., 2018]{losey2018review}
Losey, D.~P., McDonald, C.~G., Battaglia, E., and O'Malley, M.~K. (2018).
\newblock A review of intent detection, arbitration, and communication aspects
  of shared control for physical human--robot interaction.
\newblock {\em Applied Mechanics Reviews}, 70(1):010804.

\bibitem[Magrini et~al., 2015]{magrini2015control}
Magrini, E., Flacco, F., and De~Luca, A. (2015).
\newblock Control of generalized contact motion and force in physical
  human-robot interaction.
\newblock In {\em 2015 IEEE international conference on robotics and automation
  (ICRA)}, pages 2298--2304. IEEE.

\bibitem[Marsden and Ratiu, 2010]{Marsden2010}
Marsden, J.~E. and Ratiu, T.~S. (2010).
\newblock {\em Introduction to Mechanics and Symmetry: A Basic Exposition of
  Classical Mechanical Systems}.
\newblock Springer Publishing Company, Incorporated.

\bibitem[Marsden and Scheurle, 1993]{marsden1993reduced}
Marsden, J.~E. and Scheurle, J. (1993).
\newblock The reduced euler-lagrange equations.
\newblock {\em Fields Institute Comm}, 1:139--164.

\bibitem[Matrangola et~al., 2008]{matrangola2008changes}
Matrangola, S.~L., Madigan, M.~L., Nussbaum, M.~A., Ross, R., and Davy, K.~P.
  (2008).
\newblock Changes in body segment inertial parameters of obese individuals with
  weight loss.
\newblock {\em Journal of biomechanics}, 41(15):3278--3281.

\bibitem[Mattar et~al., 2018]{mattar2018biomimetic}
Mattar, E.~A., Al-Junaid, H.~J., and Al-Seddiqi, H.~H. (2018).
\newblock Biomimetic based eeg learning for robotics complex grasping and
  dexterous manipulation.
\newblock In {\em Biomimetic Prosthetics}. InTech.

\bibitem[McMullen et~al., 2014]{mcmullen2014demonstration}
McMullen, D.~P., Hotson, G., Katyal, K.~D., Wester, B.~A., Fifer, M.~S., McGee,
  T.~G., Harris, A., Johannes, M.~S., Vogelstein, R.~J., Ravitz, A.~D., et~al.
  (2014).
\newblock Demonstration of a semi-autonomous hybrid brain--machine interface
  using human intracranial eeg, eye tracking, and computer vision to control a
  robotic upper limb prosthetic.
\newblock {\em IEEE Transactions on Neural Systems and Rehabilitation
  Engineering}, 22(4):784--796.

\bibitem[Metta et~al., 2006]{yarp2006}
Metta, G., Fitzpatrick, P., and Natale, L. (2006).
\newblock Yarp: yet another robot platform.
\newblock {\em International Journal of Advanced Robotic Systems}, 3(1):8.

\bibitem[Moes, 2010]{moes2010digital}
Moes, N. (2010).
\newblock Digital human models: an overview of development and applications in
  product and workplace design.
\newblock In {\em Proceedings of TMCE (Tools and Methods of Competitive
  Engineering) International Conference, Ancona}.

\bibitem[Mourtzis, 2019]{doi:10.1080/00207543.2019.1636321}
Mourtzis, D. (2019).
\newblock Simulation in the design and operation of manufacturing systems:
  state of the art and new trends.
\newblock {\em International Journal of Production Research}, 0(0):1--23.

\bibitem[Natale et~al., 2017]{natale2017icub}
Natale, L., Bartolozzi, C., Pucci, D., Wykowska, A., and Metta, G. (2017).
\newblock icub: The not-yet-finished story of building a robot child.
\newblock {\em Science Robotics}, 2(13):eaaq1026.

\bibitem[Nava et~al., 2016]{nava2016stability}
Nava, G., Romano, F., Nori, F., and Pucci, D. (2016).
\newblock Stability analysis and design of momentum-based controllers for
  humanoid robots.
\newblock In {\em Intelligent Robots and Systems (IROS), 2016 IEEE/RSJ
  International Conference on}, pages 680--687. IEEE.

\bibitem[{Niyogi} and {Adelson}, 1994]{Niyogi1994}
{Niyogi} and {Adelson} (1994).
\newblock Analyzing and recognizing walking figures in xyt.
\newblock In {\em 1994 Proceedings of IEEE Conference on Computer Vision and
  Pattern Recognition}, pages 469--474.

\bibitem[Nori et~al., 2015a]{nori2015simultaneous}
Nori, F., Kuppuswamy, N., and Traversaro, S. (2015a).
\newblock Simultaneous state and dynamics estimation in articulated structures.
\newblock In {\em 2015 IEEE/RSJ International Conference on Intelligent Robots
  and Systems (IROS)}, pages 3380--3386. IEEE.

\bibitem[Nori et~al., 2015b]{nori2015}
Nori, F., Traversaro, S., Eljaik, J., Romano, F., {Del Prete}, A., and Pucci,
  D. (2015b).
\newblock {iCub Whole-Body Control through Force Regulation on Rigid
  Non-Coplanar Contacts}.
\newblock {\em Frontiers in Robotics and AI}, 2:6.

\bibitem[Nori et~al., 2015c]{Frontiers2015}
Nori, F., Traversaro, S., Eljaik, J., Romano, F., Del~Prete, A., and Pucci, D.
  (2015c).
\newblock i{C}ub whole-body control through force regulation on rigid
  noncoplanar contacts.
\newblock {\em Frontiers in Robotics and AI}, 2(6).

\bibitem[Norton et~al., 2017]{doi:10.1177/0278364916688254}
Norton, A., Ober, W., Baraniecki, L., McCann, E., Scholtz, J., Shane, D.,
  Skinner, A., Watson, R., and Yanco, H. (2017).
\newblock Analysis of human–robot interaction at the darpa robotics challenge
  finals.
\newblock {\em The International Journal of Robotics Research},
  36(5-7):483--513.

\bibitem[Ogenyi et~al., 2019]{ogenyi2019physical}
Ogenyi, U.~E., Liu, J., Yang, C., Ju, Z., and Liu, H. (2019).
\newblock Physical human-robot collaboration: Robotic systems, learning
  methods, collaborative strategies, sensors, and actuators.
\newblock {\em IEEE transactions on cybernetics}.

\bibitem[Ott et~al., 2011]{ott2011posture}
Ott, C., Roa, M.~A., and Hirzinger, G. (2011).
\newblock Posture and balance control for biped robots based on contact force
  optimization.
\newblock In {\em Humanoid Robots (Humanoids), 2011 11th IEEE-RAS International
  Conference on}, pages 26--33. IEEE.

\bibitem[Parmiggiani et~al., 2012]{parmiggiani2012}
Parmiggiani, A., Maggiali, M., Natale, L., Nori, F., Schmitz, A., Tsagarakis,
  N.~G., Victor, J.~S., Becchi, F., Sandini, G., and Metta, G. (2012).
\newblock {The Design of the iCub Humanoid Robot}.
\newblock {\em International Journal of Humanoid Robotics}, 09(04).

\bibitem[Parviainen and Pirhonen, 2017]{parviainen2017vulnerable}
Parviainen, J. and Pirhonen, J. (2017).
\newblock Vulnerable bodies in human--robot interactions: embodiment as ethical
  issue in robot care for the elderly.

\bibitem[Pattacini et~al., 2010a]{Pattacini2010}
Pattacini, U., Nori, F., Natale, L., Metta, G., and Sandini, G. (2010a).
\newblock {An experimental evaluation of a novel minimum-jerk cartesian
  controller for humanoid robots}.
\newblock In {\em Intelligent Robots and Systems (IROS), IEEE/RSJ International
  Conference on}, pages 1668--1674. IEEE.

\bibitem[Pattacini et~al., 2010b]{Pattacini2010a}
Pattacini, U., Nori, F., Natale, L., Metta, G., and Sandini, G. (2010b).
\newblock {An experimental evaluation of a novel minimum-jerk Cartesian
  controller for humanoid robots}.
\newblock In {\em 2010 IEEE/RSJ International Conference on Intelligent Robots
  and Systems (IROS)}.

\bibitem[Pattacini et~al., 2010c]{5650851}
Pattacini, U., Nori, F., Natale, L., Metta, G., and Sandini, G. (2010c).
\newblock An experimental evaluation of a novel minimum-jerk cartesian
  controller for humanoid robots.
\newblock In {\em 2010 IEEE/RSJ International Conference on Intelligent Robots
  and Systems}, pages 1668--1674.

\bibitem[Pedersen et~al., 2003]{Pedersen2003}
Pedersen, L., Kortenkamp, D., Wettergreen, D., Nourbakhsh, I., and Korsmeyer,
  D. (2003).
\newblock A survey of space robotics.
\newblock In {\em Proceedings of 7th International Symposium on Artificial
  Intelligence, Robotics and Automation in space (i-SAIRAS-03)}.

\bibitem[{Penco} et~al., 2018]{Penco2018}
{Penco}, L., {Clement}, B., {Modugno}, V., {Mingo Hoffman}, E., {Nava}, G.,
  {Pucci}, D., {Tsagarakis}, N.~G., {Mouret}, J.~., and {Ivaldi}, S. (2018).
\newblock Robust real-time whole-body motion retargeting from human to
  humanoid.
\newblock In {\em 2018 IEEE-RAS 18th International Conference on Humanoid
  Robots (Humanoids)}.

\bibitem[Peternel and Babi{\v{c}}, 2013]{peternel2013learning}
Peternel, L. and Babi{\v{c}}, J. (2013).
\newblock Learning of compliant human--robot interaction using full-body haptic
  interface.
\newblock {\em Advanced Robotics}, 27(13):1003--1012.

\bibitem[Peternel et~al., 2018]{peternel2018robot}
Peternel, L., Tsagarakis, N., Caldwell, D., and Ajoudani, A. (2018).
\newblock Robot adaptation to human physical fatigue in human--robot
  co-manipulation.
\newblock {\em Autonomous Robots}, pages 1--11.

\bibitem[Pfeiffer, 2016]{pfeiffer2016robots}
Pfeiffer, S. (2016).
\newblock Robots, industry 4.0 and humans, or why assembly work is more than
  routine work.
\newblock {\em Societies}, 6(2):16.

\bibitem[Pons-Moll et~al., 2011]{pons2011outdoor}
Pons-Moll, G., Baak, A., Gall, J., Leal-Taixe, L., Mueller, M., Seidel, H.-P.,
  and Rosenhahn, B. (2011).
\newblock Outdoor human motion capture using inverse kinematics and von
  mises-fisher sampling.
\newblock In {\em 2011 International Conference on Computer Vision}, pages
  1243--1250. IEEE.

\bibitem[Pucci et~al., 2016a]{pucci2016video}
Pucci, D., Romano, F., Traversaro, S., and Nori, F. (2016a).
\newblock {Highly dynamic balancing via force control}.
\newblock In {\em 2016 IEEE-RAS 16th International Conference on Humanoid
  Robots (Humanoids)}, pages 141--141. IEEE.

\bibitem[Pucci et~al., 2016b]{7803266}
Pucci, D., Romano, F., Traversaro, S., and Nori, F. (2016b).
\newblock Highly dynamic balancing via force control.
\newblock In {\em 2016 IEEE-RAS 16th International Conference on Humanoid
  Robots (Humanoids)}, pages 141--141.

\bibitem[Radmand et~al., 2014]{radmand2014characterization}
Radmand, A., Scheme, E., and Englehart, K. (2014).
\newblock A characterization of the effect of limb position on emg features to
  guide the development of effective prosthetic control schemes.
\newblock In {\em Engineering in Medicine and Biology Society (EMBC), 2014 36th
  Annual International Conference of the IEEE}, pages 662--667. IEEE.

\bibitem[Raessa et~al., 2019]{raessa2019human}
Raessa, M., Chen, J. C.~Y., Wan, W., and Harada, K. (2019).
\newblock Human-in-the-loop robotic manipulation planning for collaborative
  assembly.
\newblock {\em arXiv preprint arXiv:1909.11280}.

\bibitem[Rajesh and Srinath, 2016]{rajesh2016review}
Rajesh, R. and Srinath, R. (2016).
\newblock Review of recent developments in ergonomic design and digital human
  models.
\newblock {\em Ind Eng Manage}, 5(186):2169--0316.

\bibitem[{Ramos} and {Kim}, 2018]{Ramos2018TRO}
{Ramos}, J. and {Kim}, S. (2018).
\newblock Humanoid dynamic synchronization through whole-body bilateral
  feedback teleoperation.
\newblock {\em IEEE Transactions on Robotics}, 34(4):953--965.

\bibitem[Ranatunga et~al., 2016]{ranatunga2016adaptive}
Ranatunga, I., Lewis, F.~L., Popa, D.~O., and Tousif, S.~M. (2016).
\newblock Adaptive admittance control for human--robot interaction using model
  reference design and adaptive inverse filtering.
\newblock {\em IEEE Transactions on Control Systems Technology},
  25(1):278--285.

\bibitem[{Rapetti} et~al., 2019]{Rapetti2019}
{Rapetti}, L., {Tirupachuri}, Y., {Darvish}, K., {Latella}, C., and {Pucci}, D.
  (2019).
\newblock {Model-Based Real-Time Motion Tracking using Dynamical Inverse
  Kinematics}.
\newblock {\em arXiv e-prints}, page arXiv:1909.07669.

\bibitem[Rasouli et~al., 2016]{rasouli2016towards}
Rasouli, M., Chellamuthu, K., Cabibihan, J.-J., and Kukreja, S.~L. (2016).
\newblock Towards enhanced control of upper prosthetic limbs: A
  force-myographic approach.
\newblock In {\em Biomedical Robotics and Biomechatronics (BioRob), 2016 6th
  IEEE International Conference on}, pages 232--236. IEEE.

\bibitem[Rautaray and Agrawal, 2015]{rautaray2015vision}
Rautaray, S.~S. and Agrawal, A. (2015).
\newblock Vision based hand gesture recognition for human computer interaction:
  a survey.
\newblock {\em Artificial Intelligence Review}, 43(1):1--54.

\bibitem[Reily et~al., 2018]{reily2018skeleton}
Reily, B., Han, F., Parker, L.~E., and Zhang, H. (2018).
\newblock Skeleton-based bio-inspired human activity prediction for real-time
  human--robot interaction.
\newblock {\em Autonomous Robots}, 42(6):1281--1298.

\bibitem[Riemer and Hsiao-Wecksler, 2008]{riemer2008improving}
Riemer, R. and Hsiao-Wecksler, E.~T. (2008).
\newblock Improving joint torque calculations: Optimization-based inverse
  dynamics to reduce the effect of motion errors.
\newblock {\em Journal of biomechanics}, 41(7):1503--1509.

\bibitem[Robotic-Operating-System, 2018]{ros2018}
Robotic-Operating-System (2018).
\newblock Robotic operating system(ros).

\bibitem[Roetenberg et~al., 2009]{Roetenberg2009}
Roetenberg, D., Luinge, H., and Slycke, P. (2009).
\newblock Xsens mvn: full 6dof human motion tracking using miniature inertial
  sensors.
\newblock Technical report, Xsens Motion Technologies BV.

\bibitem[{Romano} et~al., 2018]{8093992}
{Romano}, F., {Nava}, G., {Azad}, M., {Čamernik}, J., {Dafarra}, S., {Dermy},
  O., {Latella}, C., {Lazzaroni}, M., {Lober}, R., {Lorenzini}, M., {Pucci},
  D., {Sigaud}, O., {Traversaro}, S., {Babič}, J., {Ivaldi}, S., {Mistry}, M.,
  {Padois}, V., and {Nori}, F. (2018).
\newblock The codyco project achievements and beyond: Toward human aware
  whole-body controllers for physical human robot interaction.
\newblock {\em IEEE Robotics and Automation Letters}, 3(1):516--523.

\bibitem[Romano et~al., 2017]{RomanoWBI17Journal}
Romano, F., Traversaro, S., Pucci, D., and Nori, F. (2017).
\newblock A whole-body software abstraction layer for control design of
  free-floating mechanical systems.
\newblock {\em Journal of Software Engineering for Robotics}.

\bibitem[Romualdi et~al., 2018]{Romualdi2018}
Romualdi, G., Dafarra, S., Hu, Y., and Pucci, D. (2018).
\newblock {A Benchmarking of DCM Based Architectures for Position and Velocity
  Controlled Walking of Humanoid Robots}.
\newblock In {\em 2018 IEEE 18th International Conference on Humanoid Robots
  (Humanoids)}, pages 1--9.

\bibitem[Saccon et~al., 2017]{saccon2017centroidal}
Saccon, A., Traversaro, S., Nori, F., and Nijmeijer, H. (2017).
\newblock On centroidal dynamics and integrability of average angular velocity.
\newblock {\em IEEE Robotics and Automation Letters}, 2(2):943--950.

\bibitem[Saenz et~al., 2018]{saenz2018survey}
Saenz, J., Elkmann, N., Gibaru, O., and Neto, P. (2018).
\newblock Survey of methods for design of collaborative robotics
  applications-why safety is a barrier to more widespread robotics uptake.
\newblock In {\em Proceedings of the 2018 4th International Conference on
  Mechatronics and Robotics Engineering}, pages 95--101. ACM.

\bibitem[Saikia and Hazarika, 2017]{Saikia2017}
Saikia, A. and Hazarika, S.~M. (2017).
\newblock cbdi: Towards an architecture for human--machine collaboration.
\newblock {\em International Journal of Social Robotics}, 9(2):211--230.

\bibitem[Samy et~al., 2020]{samy2020fusion}
Samy, V., Ayusawa, K., Yoshiyasu, Y., Sagawa, R., and Yoshida, E. (2020).
\newblock Fusion of multiple motion capture systems for musculoskeletal
  analysis.
\newblock In {\em 2020 IEEE/SICE International Symposium on System
  Integration}.

\bibitem[Sandini et~al., 2004]{sandini2014}
Sandini, G., Metta, G., and Vernon, D. (2004).
\newblock {RobotCub: an open framework for research in embodied cognition}.
\newblock In {\em 4th IEEE/RAS International Conference on Humanoid Robots,
  2004.}, volume~1, pages 13--32. IEEE.

\bibitem[Sarac et~al., 2013]{sarac2013brain}
Sarac, M., Koyas, E., Erdogan, A., Cetin, M., and Patoglu, V. (2013).
\newblock Brain computer interface based robotic rehabilitation with online
  modification of task speed.
\newblock In {\em Rehabilitation Robotics (ICORR), 2013 IEEE International
  Conference on}, pages 1--7. IEEE.

\bibitem[Sciavicco and Siciliano, 1988]{sciavicco1988}
Sciavicco, L. and Siciliano, B. (1988).
\newblock A solution algorithm to the inverse kinematic problem for redundant
  manipulators.
\newblock {\em IEEE Journal on Robotics and Automation}, 4(4):403--410.

\bibitem[Shahrokhi and Bernard, 2009]{SHAHROKHI200955}
Shahrokhi, M. and Bernard, A. (2009).
\newblock A framework to develop an analysis agent for evaluating human
  performance in manufacturing systems.
\newblock {\em CIRP Journal of Manufacturing Science and Technology}, 2(1):55
  -- 60.

\bibitem[Shimoga, 1993]{Shimoga1993}
Shimoga, K.~B. (1993).
\newblock A survey of perceptual feedback issues in dexterous telemanipulation.
  ii. finger touch feedback.
\newblock In {\em Proceedings of IEEE Virtual Reality Annual International
  Symposium}, Seattle, WA, USA.

\bibitem[{Shio} and {Sklansky}, 1991]{shio1991}
{Shio}, A. and {Sklansky}, J. (1991).
\newblock Segmentation of people in motion.
\newblock In {\em Proceedings of the IEEE Workshop on Visual Motion}, pages
  325--332.

\bibitem[Siciliano and Khatib, 2007]{handbook_robotics}
Siciliano, B. and Khatib, O. (2007).
\newblock {\em Springer Handbook of Robotics}.
\newblock Springer-Verlag New York, Inc.

\bibitem[Siciliano et~al., 2008]{Siciliano2009}
Siciliano, B., Sciavicco, L., Villani, L., and Oriolo, G. (2008).
\newblock {\em Robotics: Modelling, Planning and Control}.
\newblock Springer Publishing Company, Incorporated, 1st edition.

\bibitem[Siciliano and Villani, 2012]{siciliano2012robot}
Siciliano, B. and Villani, L. (2012).
\newblock {\em Robot force control}, volume 540.
\newblock Springer Science \& Business Media.

\bibitem[Skals et~al., 2017]{skals2017prediction}
Skals, S., Jung, M.~K., Damsgaard, M., and Andersen, M.~S. (2017).
\newblock Prediction of ground reaction forces and moments during
  sports-related movements.
\newblock {\em Multibody system dynamics}, 39(3):175--195.

\bibitem[Song et~al., 2008]{song2008assistive}
Song, R., Tong, K.-y., Hu, X., Li, L., et~al. (2008).
\newblock Assistive control system using continuous myoelectric signal in
  robot-aided arm training for patients after stroke.
\newblock {\em IEEE transactions on neural systems and rehabilitation
  engineering}, 16(4):371--379.

\bibitem[Stanton et~al., 2012]{Stanton2012}
Stanton, C., Bogdanovych, A., and Ratanasena, E. (2012).
\newblock Teleoperation of a humanoid robot using full-body motion capture,
  example movements, and machine learning.
\newblock In {\em Proc. Australasian Conference on Robotics and Automation}.

\bibitem[Steinfeld et~al., 2006]{Steinfeld2006}
Steinfeld, A., Fong, T., Kaber, D., Lewis, M., Scholtz, J., Schultz, A., and
  Goodrich, M. (2006).
\newblock Common metrics for human-robot interaction.
\newblock In {\em Proceedings of the 1st ACM SIGCHI/SIGART conference on
  Human-robot interaction}, pages 33--40. ACM.

\bibitem[Stephens and Atkeson, 2010]{stephens2010dynamic}
Stephens, B.~J. and Atkeson, C.~G. (2010).
\newblock Dynamic balance force control for compliant humanoid robots.
\newblock In {\em Intelligent Robots and Systems (IROS), 2010 IEEE/RSJ
  International Conference on}, pages 1248--1255. IEEE.

\bibitem[Tachi et~al., 1989]{Tachi1989}
Tachi, S., Arai, H., and Maeda, T. (1989).
\newblock Development of an anthropomorphic tele-existence slave robot.
\newblock In {\em Proceedings of the International Conference on Advanced
  Mechatronics (ICAM)}, volume 385, page 390.

\bibitem[Traversaro, 2017]{traversaro2017thesis}
Traversaro, S. (2017).
\newblock {\em {Modelling, Estimation and Identification of Humanoid Robots
  Dynamics}}.
\newblock PhD thesis, University of Genoa.

\bibitem[Traversaro et~al., 2015]{traversaro2015situ}
Traversaro, S., Pucci, D., and Nori, F. (2015).
\newblock In situ calibration of six-axis force-torque sensors using
  accelerometer measurements.
\newblock In {\em 2015 IEEE International Conference on Robotics and Automation
  (ICRA)}, pages 2111--2116. IEEE.

\bibitem[Traversaro et~al., 2016]{traversaro2016}
Traversaro, S., Pucci, D., and Nori, F. (2016).
\newblock On the base frame choice in free-floating mechanical systems and its
  connection to “centroidal” dynamics.
\newblock {\em Submitted to Humanoid Robots (Humanoids), 2016 IEEE-RAS
  International Conference on}.

\bibitem[Traversaro and Saccon, 2019]{a9258435eeef4c5fb26a698d5db629f9}
Traversaro, S. and Saccon, A. (2019).
\newblock {\em Multibody dynamics notation (version 2)}.
\newblock Technische Universiteit Eindhoven.
\newblock Dept. of Mechanical Engineering. Report locator DC 2019.100.

\bibitem[Trevelyan et~al., 2016]{Trevelyan2016}
Trevelyan, J., Hamel, W.~R., and Kang, S.-C. (2016).
\newblock Robotics in hazardous applications.
\newblock In {\em Springer handbook of robotics}, pages 1521--1548. Springer.

\bibitem[Venture et~al., 2019]{Venture2019}
Venture, G., Bonnet, V., and Kulic, D. (2019).
\newblock {\em Creating Personalized Dynamic Models}, pages 91--104.
\newblock Springer International Publishing, Cham.

\bibitem[W{\"a}chter and Biegler, 2006]{Wachter2006}
W{\"a}chter, A. and Biegler, L.~T. (2006).
\newblock On the implementation of an interior-point filter line-search
  algorithm for large-scale nonlinear programming.
\newblock {\em Mathematical Programming}, 106(1):25--57.

\bibitem[Wachter and Nagel, 1999]{wachter1999tracking}
Wachter, S. and Nagel, H.-H. (1999).
\newblock Tracking persons in monocular image sequences.
\newblock {\em Computer Vision and Image Understanding}, 74(3):174--192.

\bibitem[Wagh and Kanade, 2019]{Wagh2019}
Wagh, K. and Kanade, S.~S. (2019).
\newblock Robust human tracking using harmonious polling tracker.
\newblock {\em SN Applied Sciences}, 1(10):1227.

\bibitem[{Wang} et~al., 2015]{Wang2015}
{Wang}, A., {Ramos}, J., {Mayo}, J., {Ubellacker}, W., {Cheung}, J., and {Kim},
  S. (2015).
\newblock The hermes humanoid system: A platform for full-body teleoperation
  with balance feedback.
\newblock In {\em 2015 IEEE-RAS 15th International Conference on Humanoid
  Robots (Humanoids)}.

\bibitem[Wensing and Orin, 2013]{wensing2013generation}
Wensing, P.~M. and Orin, D.~E. (2013).
\newblock Generation of dynamic humanoid behaviors through task-space control
  with conic optimization.
\newblock In {\em Robotics and Automation (ICRA), 2013 IEEE International
  Conference on}, pages 3103--3109. IEEE.

\bibitem[Winter, 1990]{Winter1990}
Winter, D.~A. (1990).
\newblock {\em Biomechanics and motor control of human movement}, chapter
  Anthropometry.
\newblock Wiley.

\bibitem[Xiang et~al., 2010]{xiang2010predictive}
Xiang, Y., Chung, H.-J., Kim, J.~H., Bhatt, R., Rahmatalla, S., Yang, J.,
  Marler, T., Arora, J.~S., and Abdel-Malek, K. (2010).
\newblock Predictive dynamics: an optimization-based novel approach for human
  motion simulation.
\newblock {\em Structural and Multidisciplinary Optimization}, 41(3):465--479.

\bibitem[Yap et~al., 2016]{yap2016design}
Yap, H.~K., Mao, A., Goh, J.~C., and Yeow, C.-H. (2016).
\newblock Design of a wearable fmg sensing system for user intent detection
  during hand rehabilitation with a soft robotic glove.
\newblock In {\em Biomedical Robotics and Biomechatronics (BioRob), 2016 6th
  IEEE International Conference on}, pages 781--786. IEEE.

\bibitem[Yoshida et~al., 2017]{yoshida2017towards}
Yoshida, E., Ayusawa, K., Imamura, Y., and Tanaka, T. (2017).
\newblock Towards new humanoid applications: wearable device evaluation through
  human motion reproduction.

\bibitem[Zamalloa et~al., 2017]{zamalloa2017dissecting}
Zamalloa, I., Kojcev, R., Hern{\'a}ndez, A., Muguruza, I., Usategui, L.,
  Bilbao, A., and Mayoral, V. (2017).
\newblock Dissecting robotics-historical overview and future perspectives.
\newblock {\em arXiv preprint arXiv:1704.08617}.

\bibitem[Zhou et~al., 2018]{zhou2018multi}
Zhou, Y., Fang, Y., Zeng, J., Li, K., and Liu, H. (2018).
\newblock A multi-channel emg-driven fes solution for stroke rehabilitation.
\newblock In {\em International Conference on Intelligent Robotics and
  Applications}, pages 235--243. Springer.

\bibitem[Zhu and Zhou, 2004]{zhu2004real}
Zhu, R. and Zhou, Z. (2004).
\newblock A real-time articulated human motion tracking using tri-axis
  inertial/magnetic sensors package.
\newblock {\em IEEE Transactions on Neural systems and rehabilitation
  engineering}, 12(2):295--302.

\bibitem[Zucker et~al., 2015]{Zucker2015}
Zucker, M., Joo, S., Grey, M.~X., Rasmussen, C., Huang, E., Stilman, M., and
  Bobick, A. (2015).
\newblock A general-purpose system for teleoperation of the drc-hubo humanoid
  robot.
\newblock {\em Journal of Field Robotics}, 32(3):336--351.

\end{thebibliography}

\end{spacing}

\begin{appendices} %

\chapter{Human Dynamics Estimation \\ Maximum A-Posteriori Solution}
\label{app:appendix-HumanDynamicsEstimation}

The Maximum A-Posteriori (MAP) Solution for the human dynamics estimation is originally presented in detail in \cite{latella2018thesis}. The following mathematical details are directly borrowed from \cite{latella2018thesis} and presented in this appendix for the sake of minimal clarity to the reader. 

Since the normal distributions $\comVar{d}$ and $\comVar{y}$ are jointly Gaussian, the conditional probability distribution  $p(\comVar{d}|\comVar{y})$ is such that

\begin{equation}\label{conditional_prob_distribution}
    p(\comVar{d}|\comVar{y}) = \frac{p(\comVar{d}, \comVar{y})}{p(\comVar{y})} = \frac{ p(\comVar{d}) p(\comVar{y}|\comVar{d})} {p(\comVar{y})}
\end{equation} 

In the following computation, the term $p(\comVar{y})$ is negligible since it does not depend on $\comVar{d}$.  This is the reason of the proportionality between the conditional probability distribution and the joint distribution in \eqref{d_map}.  Hereafter each term in \eqref{conditional_prob_distribution} is computed separately to obtain the final analytical solution. For the sake of simplicity,  $(\comVar{q}, \dot {\comVar{q}})$ dependencies are omitted in the computations.

\subsection*{Computation of $p(\comVar{y}|\comVar{d})$}

\noindent
Let us first give an expression for the conditional probability density function (PDF) $p (\comVar{y}|\comVar{d})$:

\begin{eqnarray}\label{pdf:y_given_d}
    p(\comVar{y} | \comVar{d}) &\propto& \exp{{-}\frac{1}{2} \left\{ \left(\comVar{y} - \comVar{\mu}_{y|d}\right)^\top \comVar{\Sigma}_{y|d}^{-1} \left(\comVar{y} - \comVar{\mu}_{y|d}\right)\right\}} \notag \\
    & =&  \exp{{-}\frac{1}{2} \left\{\big[\comVar{y} - \left(\comVar{Y} \comVar{d} + \comVar{b}_Y \right) \big]^\top \comVar{\Sigma}_{y|d}^{-1} \big[\comVar{y} - \left(\comVar{Y} \comVar{d} + \comVar{b}_Y \right) \big]\right\}}
\end{eqnarray} 
which implicitly makes the assumption that the measurements equation \eqref{eq:measEquation} is affected by a Gaussian noise with zero mean and covariance $ \comVar{\Sigma}_{y|d}$.

\subsection*{Computation of $p(\comVar{d})$}

\noindent
Define now a PDF for the normal distribution $\comVar{d}$.  By pursuing the same methodology, we would like to write  its distribution in the following form
\begin{equation} \label{distr:dmoddyn_short}
    \comVar{d} \sim \mathcal N ({\comVar{\mu}}_D,{\comVar{\Sigma}}_D)
\end{equation}
such that the PDF
\begin{equation}\label{pdf:dmoddyn}
    p(\comVar{d}) \propto \exp  -\frac{1}{2}{\left\{\comVar{e}(\comVar{d})^\top \comVar{\Sigma}_D^{-1} \comVar{e}(\comVar{d})\right\}}
\end{equation}
taking into account constraints of equation \eqref{eq:matRNEA} with $\comVar{e} (\comVar{d})=\comVar{D} \comVar{d} + \comVar{b}_D$.

However, this intuitive choice leads to a degenerate normal distribution and a regularization term is needed. For example, if we have a Gaussian prior knowledge on $\comVar{d}$ in the form of $\comVar{d} \sim \mathcal N \left({\comVar{\mu}}_d,{\comVar{\Sigma}}_d\right)$ distribution, we can reformulate Equation
 \eqref{distr:dmoddyn_short} as follows:
\begin{equation} \label{distr:d_short}
    \comVar{d} \sim \mathcal N(\xoverline {\comVar{\mu}}_D, \xoverline {\comVar{\Sigma}}_D)
\end{equation}
such that \eqref{pdf:dmoddyn} becomes

\begin{eqnarray} \label{pdf:d}
    p(\comVar{d}) &\propto& \exp -\frac{1}{2}{\left\{\comVar{e}(\comVar{d})^\top \comVar{\Sigma}_D^{-1} \comVar{e}(\comVar{d})+(\comVar{d}{-}\comVar{\mu}_d)^\top \comVar{\Sigma}_d^{-1} (\comVar{d} -\comVar{\mu}_d) \right\}} \notag\\
    &=& \exp{-\frac{1}{2} \left\{ (\comVar{D} \comVar{d} + \comVar{b}_D)^\top \comVar{\Sigma}_D^{-1}
    (\comVar{D} \comVar{d} + \comVar{b}_D){+}(\comVar{d} -\comVar{\mu}_d)^\top \comVar{\Sigma}_d^{-1} (\comVar{d}-\comVar{\mu}_d)\right\}} \notag \\
    &=& \exp{{-}\frac{1}{2} \left\{ \big (\comVar{d} -\xoverline {\comVar{\mu}}_D \big )^\top
    \xoverline {\comVar{\Sigma}}_D^{-1} \big (\comVar{d} - \xoverline {\comVar{\mu}}_D \big )\right\}}
\end{eqnarray}
where the covariance and the mean are, respectively,
\begin{subequations} 
    \begin{eqnarray} 
        \label{eq:sigmaDbar}
        \xoverline {\comVar{\Sigma}}_D &=& \left(\comVar{D}^\top \comVar{\Sigma}_D^{-1} \comVar{D} + \comVar{\Sigma}_d^{-1}\right)^{-1}
        \\
        \label{eq:muDbar}
        \xoverline {\comVar{\mu}}_D &=& \xoverline {\comVar{\Sigma}}_D \left(\comVar{\Sigma}_d^{-1} \comVar{\mu}_d - \comVar{D}^\top \comVar{\Sigma}_D^{-1} \comVar{b}_D \right)
\end{eqnarray} 
\end{subequations}
The role of $\comVar{\Sigma}_D$ is to establish how much the dynamic model \eqref{eq:matRNEA} should be considered correct. The quantities $\comVar{\mu}_d$ and $\comVar{\Sigma}_d$, instead, define the Gaussian prior distribution on $\comVar{d}$ (namely, the regularization term).

\subsection*{Computation of $p(\comVar{d} | \comVar{y})$}

\noindent
By combining Equations \eqref{pdf:y_given_d} and \eqref{pdf:d} we are now ready to give a new formulation of the estimation problem for the conditional PDF of $\comVar{d}$ given $\comVar{y}$, i.e.,
\begin{eqnarray}
    p(\comVar{d} |\comVar{y}) &\propto& \exp {-\frac{1}{2}} \left\{ \big (\comVar{d} -\xoverline {\comVar{\mu}}_D \big)^\top \xoverline {\comVar{\Sigma}}_D^{-1} \big (\comVar{d}-\xoverline {\comVar{\mu}}_D \big)~+ \right. \notag
    \\
    &+& \left. \Big[\comVar{y} -(\comVar{Y} \comVar{d} + \comVar{b}_Y)\Big]^\top \comVar{\Sigma}_{y|d}^{-1}
    \Big [\comVar{y} - (\comVar{Y} \comVar{d} + \comVar{b}_Y)\Big] \right\}
\end{eqnarray}
with a covariance matrix and a mean as follows:

\begin{subequations}\label{MAP_solution}
    \begin{eqnarray}
        \label{eq:sigma_dgiveny}
        \comVar{\Sigma}_{d|y} &=& \left(\xoverline{ \comVar{\Sigma}}_D^{-1} + \comVar{Y}^\top \comVar{\Sigma}_{y|d}^{-1}\comVar{Y}\right)^{-1} 
        \\
        \label{eq:mu_dgiveny} 
        \comVar{\mu}_{d|y} &=& \comVar{\Sigma}_{d|y} \left[ \comVar{Y}^\top \comVar{\Sigma}_{y|d}^{-1} (\comVar{y}-\comVar{b}_Y) + \xoverline{\comVar{\Sigma}}_D^{-1} \xoverline{\comVar{\mu}}_D\right]
    \end{eqnarray} 
\end{subequations}
Moreover, in the Gaussian case the MAP solution coincides with the mean of the PDF $p(\comVar{d}|\comVar{y})$ yielding to:
\begin{eqnarray} \label{eq:d_gaussian_original}
    \comVar{d}^{\mbox{\tiny{MAP}}} = \comVar{\mu}_{d|y}
\end{eqnarray}

\chapter{Whole-Body Momentum Control}
\label{app:whole-body-momentum-control}

In the case of complex humanoids robots, state-of-the-art whole-body controllers often consider controlling the robot momentum \cite{7803266, nava2016stability}. For the sake of minimal clarity, this appendix presents whole-body momentum control for humanoid robots originally presented in~\cite{nava2016stability}.

Recall the equations of motion of a robotic system from Eq.\eqref{NERobot2},

\begin{equation}
	\robVar{M}(\robVar{q}) \dot{\robnu} + \robVar{C}(\robVar{q},\robnu) \robnu + \robVar{G}(\robVar{q}) =
	\robVar{B} {\robtau} + \sum_{i = 1}^{n_c} \robVar{J}^{\top}_{\mathcal{C}_i}  \robVar{f}_i,
\end{equation}

The coordinates of the state space are $(\robVar{q}, \robnu)$.  Consider the mass matrix partitioned as following,

\begin{equation}
	\robVar{M} = \begin{bmatrix}
	\robVar{M}_b & \robVar{M}_{bj}\\
	\robVar{M}_{bj}^\top & \robVar{M}_j
	\end{bmatrix}
\end{equation}

with $\robVar{M}_b \in \mathbb{R}^{6 \times 6}$, $\robVar{M}_{bj} \in \mathbb{R}^{6 \times n}$ and $\robVar{M}_j \in \mathbb{R}^{n \times n}$. Now, consider the following change of state variables:

\begin{align}
	\robVar{q} := \robVar{q}, ~~ \bar{\robnu} := T(\robVar{q}){\robnu}, \IEEEyessubnumber
\label{eq:centroidalTransform}
\end{align}
with 
\begin{IEEEeqnarray}{rCL}
	\label{eq:generalStructure}
	T &:=& \begin{bmatrix}
		\prescript{c}{}{\comVar{X}}_{\mathcal{F}} & \prescript{c}{}{\comVar{X}}_{\mathcal{F}} \robVar{M}^{-1}_b \robVar{M}_{bj} \\
		\comVar{0}_{n \times 6} & \comVar{1}_n
	\end{bmatrix}, \IEEEyessubnumber \\
	\prescript{c}{}{\comVar{X}}_{\mathcal{F}} &:=& \begin{bmatrix}
		\comVar{1}_3 & -\bm \skewOp(\prescript{\mathcal{I}}{}{\robVar{p}}_{c}-\prescript{\mathcal{I}}{}{\robVar{p}}_{\mathcal{F}}) \\ 
		\comVar{0}_{3 \times 3} & \comVar{1}_3
	\end{bmatrix}\IEEEyessubnumber
\end{IEEEeqnarray}
where the superscript $c$ denotes the frame with the origin located at the center of mass, and with the orientation of $\mathcal{I}$. Then, the equations of motion with new state variables $(\robVar{q},\xoverline{ \robnu})$ becomes,

\begin{align}
\label{eq:decoupled_system_dynamics}
\xoverline{\robVar{M}}(\robVar{q})\dot{\xoverline{\robnu}} + \xoverline{\robVar{C}}(\robVar{q}, \xoverline{\robnu}) \xoverline{\robnu} + \xoverline{\robVar{G}} = \robVar{B} \robtau +
\sum_{i = 1}^{n_c} \xoverline{\robVar{J}}^{\top}_{\mathcal{C}_i}  \robVar{f}_i,
\end{align}
with
\begin{IEEEeqnarray}{RCL}
	\xoverline{\robVar{M}}(\robVar{q}) &=& T^{-\top} {\robVar{M}} T^{-1} = \begin{bmatrix}
		\xoverline{\robVar{M}}_b(\robVar{q}) & \comVar{0}_{6\times n} \\ \comVar{0}_{n\times 6} & \xoverline{\robVar{M}}_j(\robVar{q}_j)
	\end{bmatrix}, \IEEEyessubnumber \label{eq:massMatrixStructure}\\
	\xoverline{\robVar{C}}(\robVar{q},\xoverline{\robnu}) &=& T^{-\top}({\robVar{M}}\dot{T}^{-1} + {\robVar{C}}T^{-1}), \IEEEyessubnumber \label{eq:coriolisMatrixStructure} \\
	\xoverline{\robVar{G}} &=& T^{-\top}{\robVar{G}} = mg\comVar{e}_3, \IEEEyessubnumber \label{eq:gravityStructure} \\
	\xoverline{\robVar{J}}_{\mathcal{C}_i}(\robVar{q}) &=& \robVar{J}_{\mathcal{C}_i}(\robVar{q}) T^{-1} = \begin{bmatrix} \xoverline{\robVar{J}}_{\mathcal{C}_i}^b(\robVar{q}) & 
		\xoverline{\robVar{J}}_{\mathcal{C}_i}^j(\robVar{q}_j)\end{bmatrix},\IEEEyessubnumber \label{eq:jacobianStructure}
\end{IEEEeqnarray}
\begin{align*}
\xoverline{\robVar{M}}_b(\robVar{q}) &= \begin{bmatrix}
m \comVar{1}_3 & \comVar{0}_{3\times3} \\ \comVar{0}_{3\times3} & \robVar{I}(\robVar{q})
\end{bmatrix},
\xoverline{\robVar{J}}_{\mathcal{C}_i}^b(\robVar{q}) {=}
\begin{bmatrix}
\comVar{1}_3 & {-}\bm \skewOp(\robVar{p}_{\mathcal{C}_i}{-}\prescript{\mathcal{I}}{}{\robVar{p}}_{c})  \\ 
\comVar{0}_{3\times3} & \comVar{1}_3 \\ 
\end{bmatrix}
\end{align*}
where $m$ is the mass of the robot and $\robVar{I}$ is the inertia matrix computed with respect to the center of mass, with the orientation of the inertial frame $\mathcal{I}$.

The mass matrix of the  transformed system~\eqref{eq:decoupled_system_dynamics} is a block matrix that decouples the transformed base acceleration and the joint acceleration~\cite{traversaro2016}. More precisely, the transformed robot's velocity $\xoverline{\robnu}$ is given by  
$\xoverline{\robnu} =
\begin{pmatrix}
^\mathcal{I}\dot{\robVar{p}}_c^\top &
^\mathcal{I}{\omega}_c^\top &
\dot{\robVar{s}}^\top
\end{pmatrix}^\top$
where $^\mathcal{I}\dot{\robVar{p}}_c$ is the velocity of the center-of-mass of the robot, and
$^\mathcal{I}\omega_c$ is the so-called \emph{average angular velocity}\footnote{The term $^\mathcal{I}{\omega}_c$ is also known as 
	the \emph{locked angular velocity}~\cite{marsden1993reduced}.}.

Assuming that the only contact constraint is between the environment and the robot's foot, we can write:

\begin{align}
	\label{eq:extForces}
	\sum_{k = 1}^{n_c} \xoverline{\robVar{J}}^\top_{\mathcal{C}_k} \robVar{ f}_i = \xoverline{\robVar{J}}_c^\top(\robVar{q}) \robVar{ f}
\end{align}

where $\xoverline{\robVar{J}}_c(\robVar{q}) \in \mathbb{R}^{6\times n+6}$ is the Jacobian of a frame attached to the foot's sole in contact with the environment, and $ \robVar{ f} \in \mathbb{R}^{6}$ the contact wrench.
Differentiating the kinematic constraint 
\begin{IEEEeqnarray}{RCL}
	\label{JNuEqualZero}
	\xoverline{\robVar{J}}_c(\robVar{q}) \xoverline{\robnu} = \begin{bmatrix}
		\xoverline{\robVar{J}}_b &  \xoverline{\robVar{J}}_j
	\end{bmatrix} \xoverline{\robnu} = 0
\end{IEEEeqnarray}
associated with the contact, yield
\begin{equation}
	\label{eq:constraints_simple}
	\xoverline{\robVar{J}}_c(\robVar{q}) \dot{\xoverline{\robnu}} + \dot{\xoverline{\robVar{J}}}_c(\robVar{q}) \xoverline{\robnu}  = 0
\end{equation}
\begin{equation}
	\label{eq:constraints_acc}
	\begin{bmatrix}
		\xoverline{\robVar{J}}_b &  \xoverline{\robVar{J}}_j
	\end{bmatrix} 
	\begin{bmatrix}
		\dot{\robVar{v}}_{\mathcal{F}}\\ 
		\ddot{\robVar{s}}
	\end{bmatrix} + 
	\begin{bmatrix}
		\dot{\xoverline{\robVar{J}}}_b & \dot{\xoverline{\robVar{J}}}_j
	\end{bmatrix} 
	\begin{bmatrix}
		\robVar{v}_{\mathcal{F}}\\
		\dot{\robVar{s}}
	\end{bmatrix} = 0
\end{equation}

The term $\dot{\xoverline{\robnu}}$ is the robot's acceleration that can be obtained from the decoupled equations of motion Eq.~\eqref{eq:decoupled_system_dynamics} as,

\begin{align}
	\dot{\xoverline{\robnu}} = {\xoverline{\robVar{M}}}^{-1} [ \robVar{B} \robtau +
	\xoverline{\robVar{J}}_c^\top \robVar{ f} - \xoverline{\robVar{h}} ]
	\label{eq:decoupled-robot-acceleration}
\end{align}

where, $\xoverline{\robVar{h}} = \xoverline{\robVar{C}}(\robVar{q}, \xoverline{\robnu}) \xoverline{\robnu} + \xoverline{\robVar{G}} \in \mathbb{R}^{n+6}$. Using the relation from Eq.~\eqref{eq:decoupled-robot-acceleration} in the constraints equation Eq.~\eqref{eq:constraints_simple} we obtain,

\begin{align}
	\xoverline{\robVar{J}}_c {\xoverline{\robVar{M}}}^{-1} [ \robVar{B} \robtau +
	\xoverline{\robVar{J}}_c^\top \robVar{ f} - \xoverline{\robVar{h}} ] + \dot{\xoverline{\robVar{J}}}_c \xoverline{\robnu}
	= 0
	\label{eq:control_torques_equation}
\end{align}

Now, the robot's momentum $\robVar{L} \in \mathbb{R}^6$ is given by $\robVar{L} = \robVar{M}_b \robVar{v}_\mathcal{F}$. The rate-of-change of the robot momentum equals the net external wrench acting on the robot which in the present case reduces to the contact wrench $\robVar{ f}$ plus the gravity wrench. The contact wrench $\robVar{ f}$ is assumed to be a virtual control input through which the robot momentum is controlled. Note that given the particular form of~\eqref{eq:decoupled_system_dynamics}, the first six rows correspond to the dynamics of the robot's momentum, i.e.
\begin{IEEEeqnarray}{RCL}
	\label{hDot}
	\frac{\dif }{\dif t}(\robVar{M}_b {\robVar{v}_\mathcal{F}}) &=& \xoverline{\robVar{J}}_b^\top \robVar{ f} - mg \robVar{e}_3 = \dot{\robVar{L}}(\robVar{ f}^*)
\end{IEEEeqnarray}
where ${\robVar{L}}:=(\robVar{L}_L,\robVar{L}_\omega)$, with 
$H_L, H_\omega \in \mathbb{R}^3$ linear and angular momentum, 
respectively.
The control objective can then be  defined as the stabilization of a desired robot momentum $\robVar{L}^d$. Let us define the momentum error as follows
$\tilde{\robVar{L}} = \robVar{L} - \robVar{L}^d$. So, the virtual control input $\robVar{ f}$ in Eq.~\eqref{hDot} is chosen such that,

\begin{IEEEeqnarray}{RCL}
	\label{hDotDes}
	\dot{\robVar{L}}(\robVar{ f}) &=& 
	\dot{\robVar{L}}^* := \dot{\robVar{L}}^d - \robVar{K}_p \tilde{\robVar{L}} - \robVar{K}_i \int \tilde{\robVar{L}}    
	\IEEEeqnarraynumspace  \IEEEyessubnumber 
\end{IEEEeqnarray}
where $\robVar{K}_p, \robVar{K}_i \in \mathbb{R}^{6\times 6}$ are two symmetric, positive definite matrices.

Given the assumption that the contact wrench $\robVar{ f}$ is a virtual control input, it is chosen as to satisfy the following relation,

\begin{equation}
	\label{eq:forces}
	\robVar{ f}^* = \xoverline{\robVar{J}}_b^{-\top} \left( 
	\dot{\robVar{L}}^*
	+ mg \robVar{e}_3\right)
\end{equation}

Now, to determine the control torques that instantaneously realize the contact force given by~\eqref{eq:forces}, we use the relation from Eq.~\eqref{eq:control_torques_equation} that is derived from the decoupled dynamic equations~\eqref{eq:decoupled_system_dynamics} along with the constraints~\eqref{eq:constraints_acc}, which yield
\begin{equation}
	\label{eq:torques}
	\robtau = \robVar{\Lambda}^\dagger (\xoverline{\robVar{J}}_c {\xoverline{\robVar{M}}}^{-1}(\xoverline{\robVar{h}} - \xoverline{\robVar{J}}_c^\top \robVar{ f}^*) - \dot{\xoverline{\robVar{J}}}_c \xoverline{\robnu} )+ N_{\robVar{\Lambda}} \tau_0
\end{equation}

where $\robVar{\Lambda} = \xoverline{\robVar{J}}_j \xoverline{\robVar{M}}_j^{-1} \in \mathbb{R}^{6\times n}$,  $N_{\robVar{\Lambda}} \in \mathbb{R}^{n\times n}$ is the nullspace projector of $\robVar{\Lambda}$, and $\tau_0 \in \mathbb{R}^n$ is a free variable that can be used for additional tasks. The free variable  ${\robtau}_0$ is exploited to achieve the postural task. A classical state-of-the-art choice for this postural task consists in:
${\robtau}_0= \xoverline{\robVar{h}}_j - \xoverline{\robVar{J}}_j^\top \robVar{ f} -  \robVar{K}^j_{p}(\robVar{s}- \robVar{s}^d) -\robVar{K}^j_{d}\dot{\robVar{s}}$,
where  
$\xoverline{\robVar{h}}_j - \xoverline{\robVar{J}}_j^\top \robVar{ f}$ compensates for the nonlinear effect and the external wrenches acting on the joint space of the system.

The reader is advised to refer \cite{nava2016stability} for a detailed discussion on whole-body momentum control of a humanoid robot for the balancing task.

\end{appendices}

\printthesisindex %

\end{document}